\pdfoutput=1
\documentclass{article}
\usepackage{iclr2027_conference,times}

\usepackage{amsmath,amsfonts,bm}

\def\eqref#1{equation~\ref{#1}}
\def\1{\bm{1}}

\DeclareMathAlphabet{\mathsfit}{\encodingdefault}{\sfdefault}{m}{sl}
\SetMathAlphabet{\mathsfit}{bold}{\encodingdefault}{\sfdefault}{bx}{n}

\usepackage{hyperref}
\hypersetup{pdftitle={OmniReasoning: Pushing the Limits of Audio-Visual Joint Reasoning},pdfauthor={Junming Lin, Yuxuan Wang, Zhenxin Lei, Yuxin Liu, Ruixun Liu, Yinsong Yan, Ling Wang, Minghao Han, Yunfei Chu, Shun Lei, Xueyao Zhang, Qize Yang, Jin Xu, Yiwu Zhong}}
\usepackage{url}

\usepackage{graphicx}
\usepackage{placeins}
\usepackage{needspace}
\usepackage{booktabs}
\usepackage{tabularx}
\newcolumntype{C}[1]{>{\centering\arraybackslash}p{#1}}
\usepackage{amssymb}
\usepackage{amsthm}
\newtheorem{proposition}{Proposition}
\newtheorem{assumption}{Assumption}
\usepackage[table]{xcolor}
\definecolor{citationgreen}{HTML}{66BB6A}
\definecolor{referencered}{HTML}{66BB6A}
\hypersetup{colorlinks=true,citecolor=citationgreen,linkcolor=referencered,urlcolor=black}
\usepackage{tikz}
\usepackage{tcolorbox}
\tcbuselibrary{breakable,skins}
\usepackage{float}
\usepackage{enumitem}
\tcbset{casecard/.style={colback=white,colframe=blue!40!black,colbacktitle=blue!7,coltitle=black,boxrule=0.5pt,sharp corners,fonttitle=\small\bfseries,fontupper=\small,left=6pt,right=6pt,top=5pt,bottom=5pt}}

\definecolor{oursbg}{RGB}{230,239,252}
\definecolor{groupbg}{RGB}{238,238,241}

\newcommand{\tbours}{\rowcolor{oursbg}}

\newcommand{\tbgroupcf}[2]{\rowcolor{groupbg}[0pt][0pt]\multicolumn{#1}{@{}c@{}}{\small\textbf{#2}}}
\newcommand{\tboursf}{\rowcolor{oursbg}[0pt][0pt]}

\definecolor{barblue}{HTML}{2A78D6}
\definecolor{barorange}{HTML}{EB6834}
\definecolor{baryellow}{HTML}{EDA100}
\definecolor{bargray}{HTML}{7D8590}

\newcommand{\barlen}{0.88}
\newcommand{\cbar}[4][\idn]{
  \begin{tikzpicture}[baseline=0.40ex, x=1cm, y=1cm]
    \pgfmathsetmacro{\barw}{min(max(#3/#4,0),1)*\barlen}
    \fill[#2!10, rounded corners=0.7pt] (0,0) rectangle (\barlen,0.29);
    \shade[left color=#2!62, right color=#2!30, rounded corners=0.7pt]
      (0,0) rectangle ({\barw},0.29);
    \node at ({\barlen/2},0.145) {\fontsize{7.2}{8.2}\selectfont #1{#3}};
  \end{tikzpicture}
}
\newcommand{\bov}[2][\idn]{\cbar[#1]{barblue}{#2}{100}}
\newcommand{\bby}[2][\idn]{\cbar[#1]{barorange}{#2}{100}}

\newcommand{\ball}[2][\idn]{\cbar[#1]{bargray}{#2}{100}}

\newcommand{\logoraw}[1]{\raisebox{-0.18\height}{
  #1}}
\newcommand{\logo}[1]{\makebox[1.15em][c]{\logoraw{#1}}\hspace{0.15em}}

\newcommand{\lgemini}{\logo{\includegraphics[height=0.9em,width=1.15em,keepaspectratio]{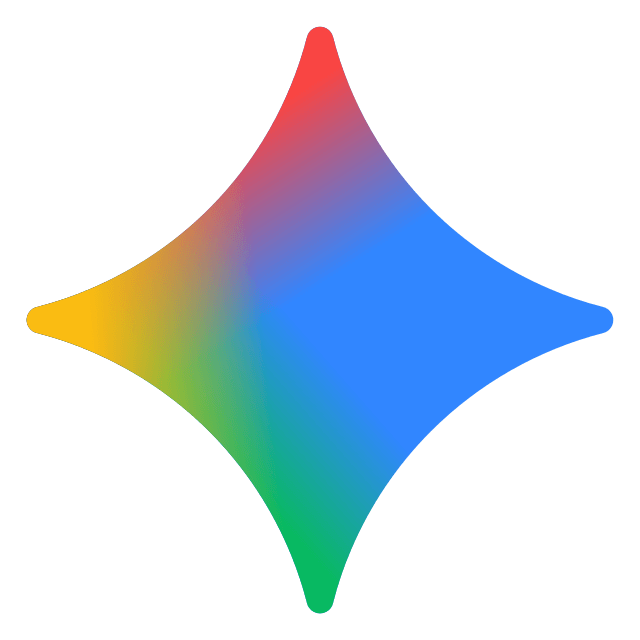}}}
\newcommand{\lqwen}{\logo{\includegraphics[height=0.9em,width=1.15em,keepaspectratio]{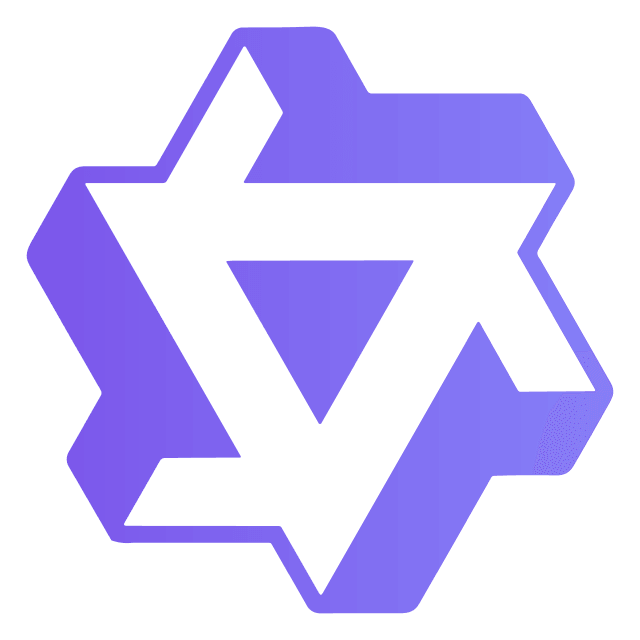}}}
\newcommand{\ldoubao}{\logo{\includegraphics[height=0.9em,width=1.15em,keepaspectratio]{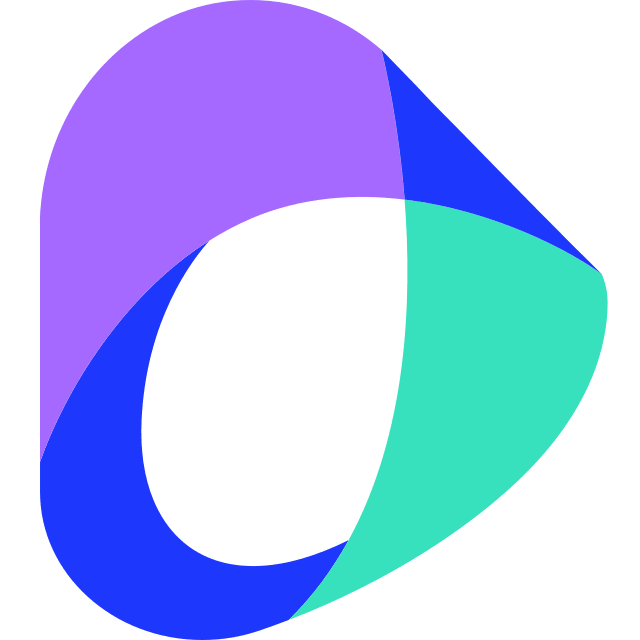}}}
\newcommand{\lkimi}{\logo{\includegraphics[height=0.9em,width=1.15em,keepaspectratio]{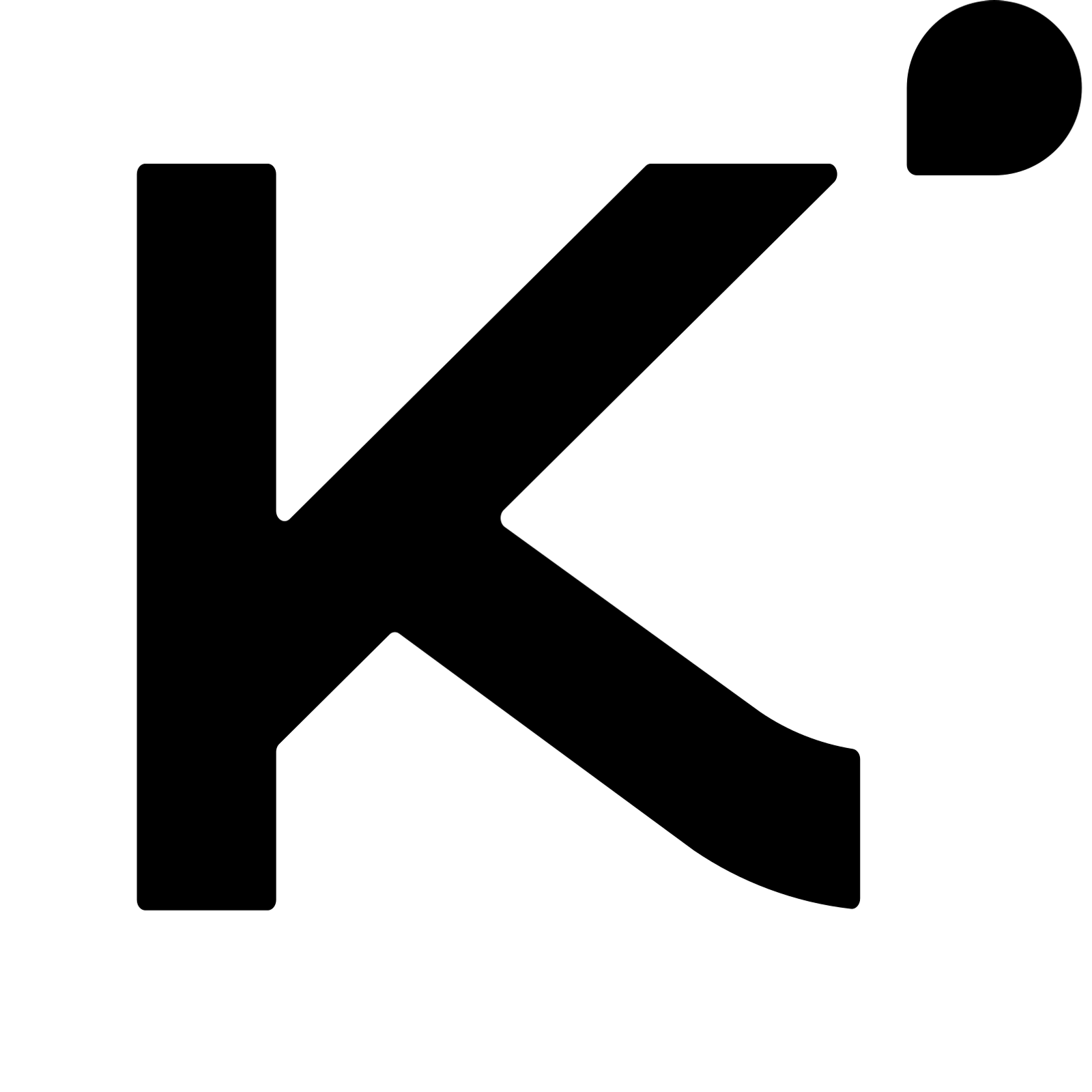}}}

\newcommand{\lnvidia}{\logo{\includegraphics[height=0.9em,width=1.15em,keepaspectratio]{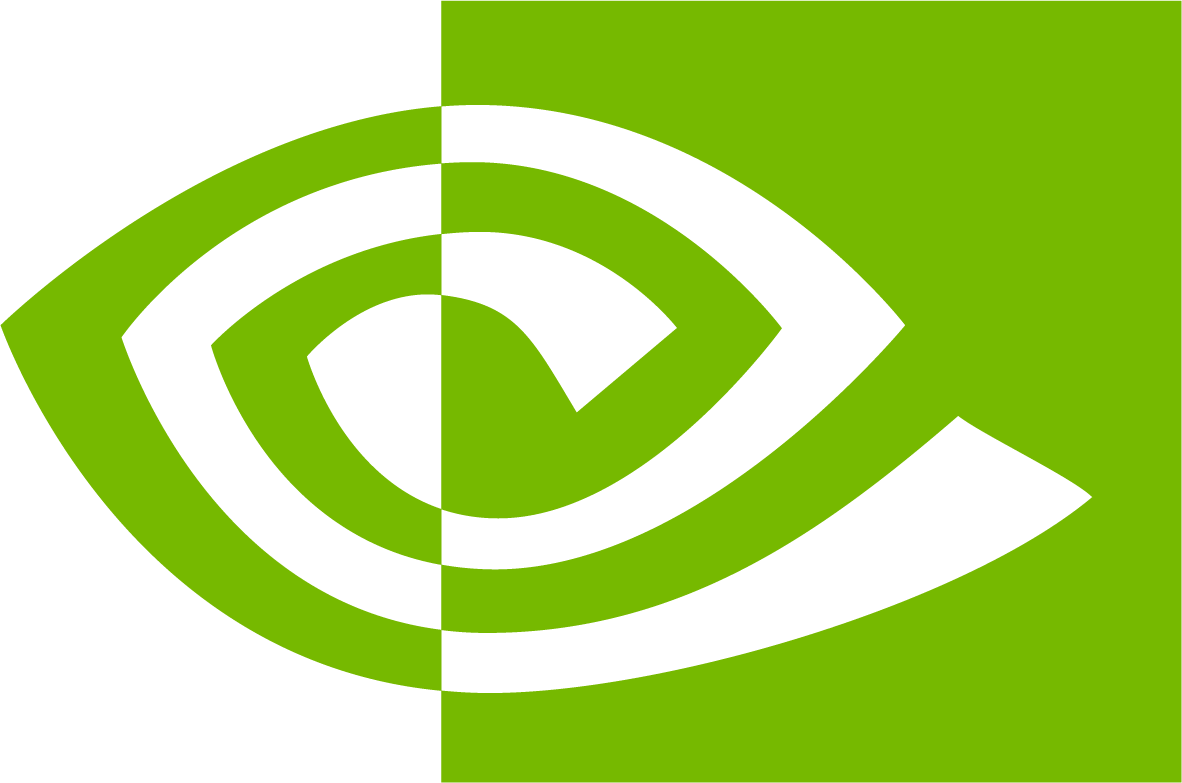}}}
\newcommand{\lmeta}{\logo{\includegraphics[height=0.9em,width=1.15em,keepaspectratio]{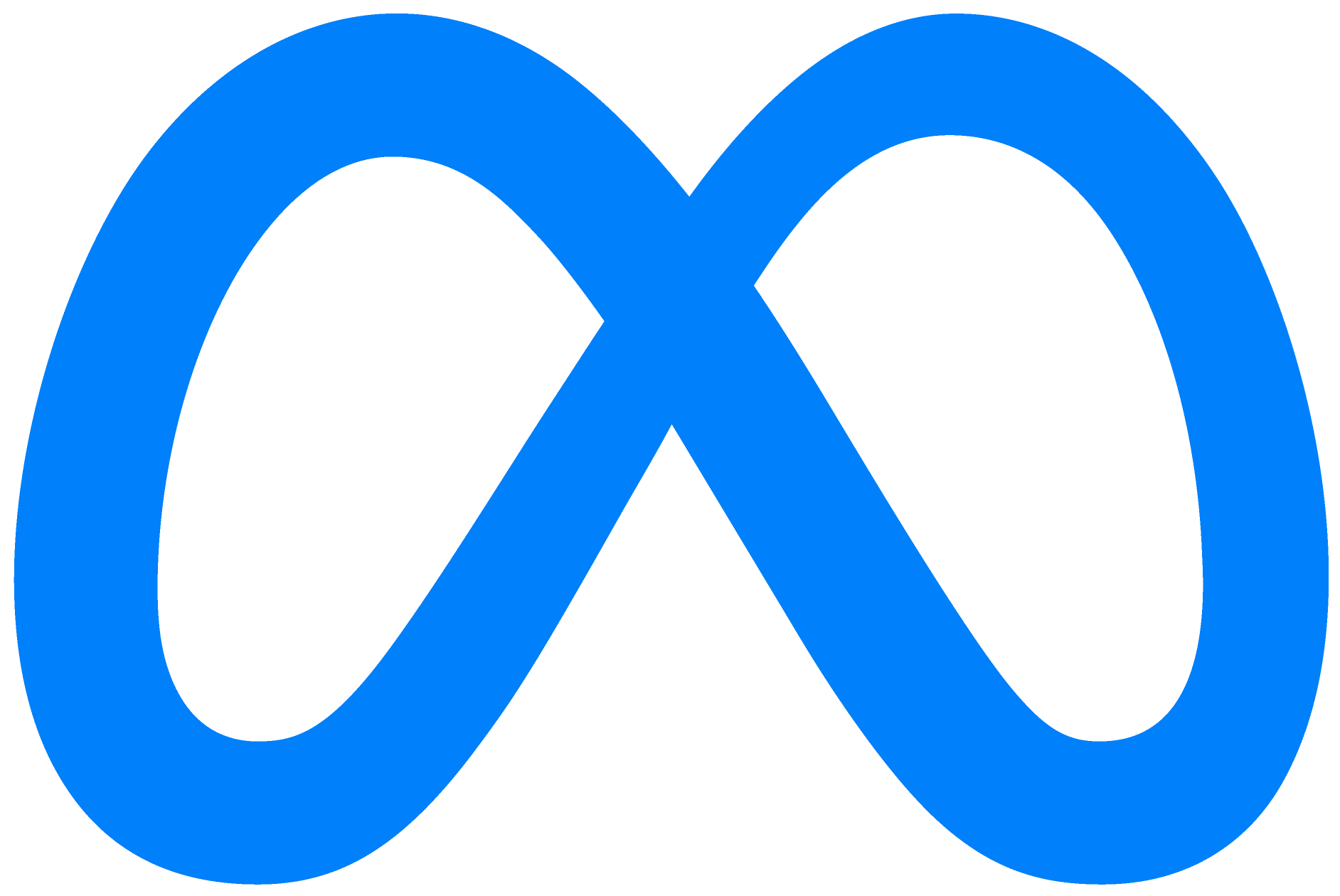}}}
\newcommand{\lxiaomi}{\logo{\includegraphics[height=0.9em,width=1.15em,keepaspectratio]{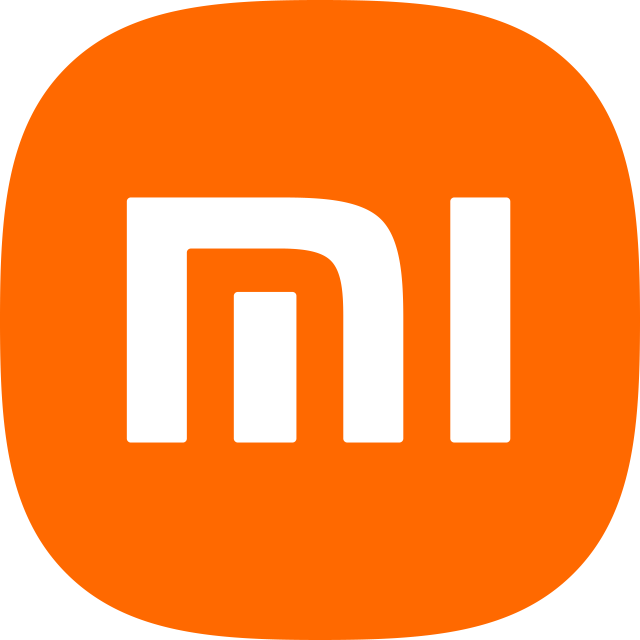}}}

\newcommand{\shead}[1]{{\scriptsize\bfseries #1}}

\title{\protect\raisebox{-0.18em}[0.8em][0.2em]{\protect\includegraphics[height=1.9em]{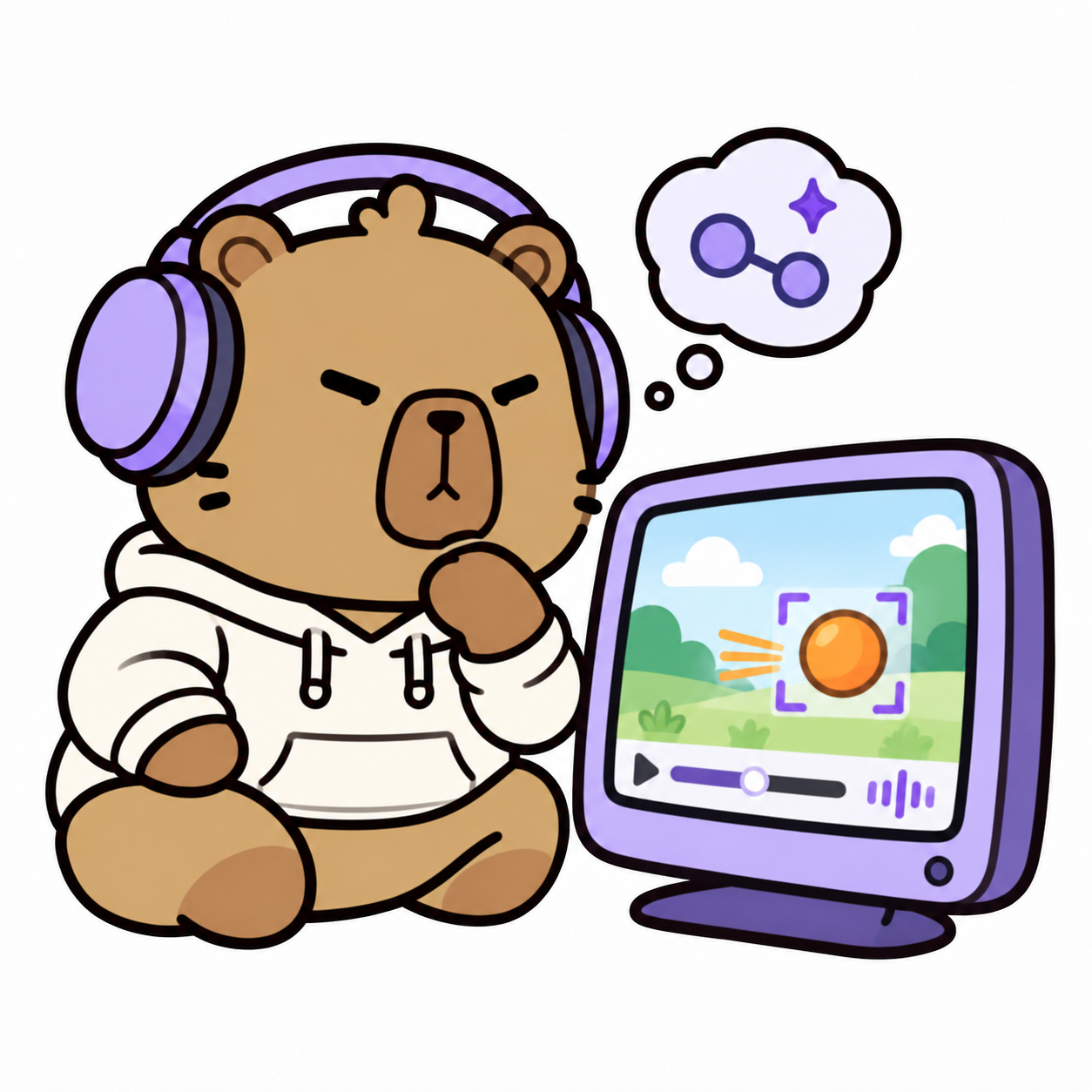}}\hspace{0.3em}OmniReasoning: Pushing the Limits of Audio-Visual Joint Reasoning}
\author{
\rule{0pt}{2mm}\\[-5mm]
\begin{tabular}{@{}l@{}}
\bfseries
Junming Lin$^{1,2}$\textsuperscript{*} \quad
Yuxuan Wang$^{3}$ \quad
Zhenxin Lei$^{3}$  \quad
Yuxin Liu$^{3}$  \quad
Ruixun Liu$^{1,2}$  \quad
\\[1mm]
\bfseries
Yinsong Yan$^{3}$  \quad
Ling Wang$^{3}$  \quad
Minghao Han$^{3}$  \quad
Yunfei Chu$^{3}$ \quad
Shun Lei$^{3}$ \quad
\\[1mm]
\bfseries
Xueyao Zhang$^{3}$ \quad
Qize Yang$^{3}$ \quad
Jin Xu$^{3}$ \quad
Yiwu Zhong$^{1,2}$\textsuperscript{\textdagger}
\\[1.5mm]
\normalfont
$^{1}$School of Intelligence Science and Technology, Peking University, \quad \\
$^{2}$State Key Laboratory of General Artificial Intelligence, Peking University, \quad \\
$^{3}$Alibaba Token Hub, Alibaba Group\quad \\[0.9em]
\normalfont\small
\href{https://github.com/PKU-VaLuE-Lab/OmniReasoning}{%
\tikz[baseline=-0.1ex,x=1em,y=1em,line width=0.45pt]{%
\draw[rounded corners=0.8pt] (0,0) rectangle (1,0.72);
\draw (0,0.5) -- (1,0.5);
\fill (0.14,0.61) circle[radius=0.035];
\fill (0.28,0.61) circle[radius=0.035];
\fill (0.42,0.61) circle[radius=0.035];
}}\hspace{0.35em}\url{https://pku-value-lab.github.io/OmniReasoning-Homepage}
\end{tabular}
}
\iclrfinalcopy
\begin{document}
\maketitle
% Public preprint header; the conference style sets a publication header by default.
\lhead{Preprint}
% This unnumbered author note has no hyperlink target.
\begin{NoHyper}
\renewcommand{\thefootnote}{}\footnotetext{\textsuperscript{*}Work done as an intern at Qwen Team, Alibaba Token Hub.\hspace{1em}\textsuperscript{\textdagger}Corresponding author.}\renewcommand{\thefootnote}{\arabic{footnote}}
\end{NoHyper}
% ===== BEGIN inlined from sections/teaser.tex =====
\begin{figure}[H]
\centering
\textbf{}\includegraphics[width=0.99\linewidth]{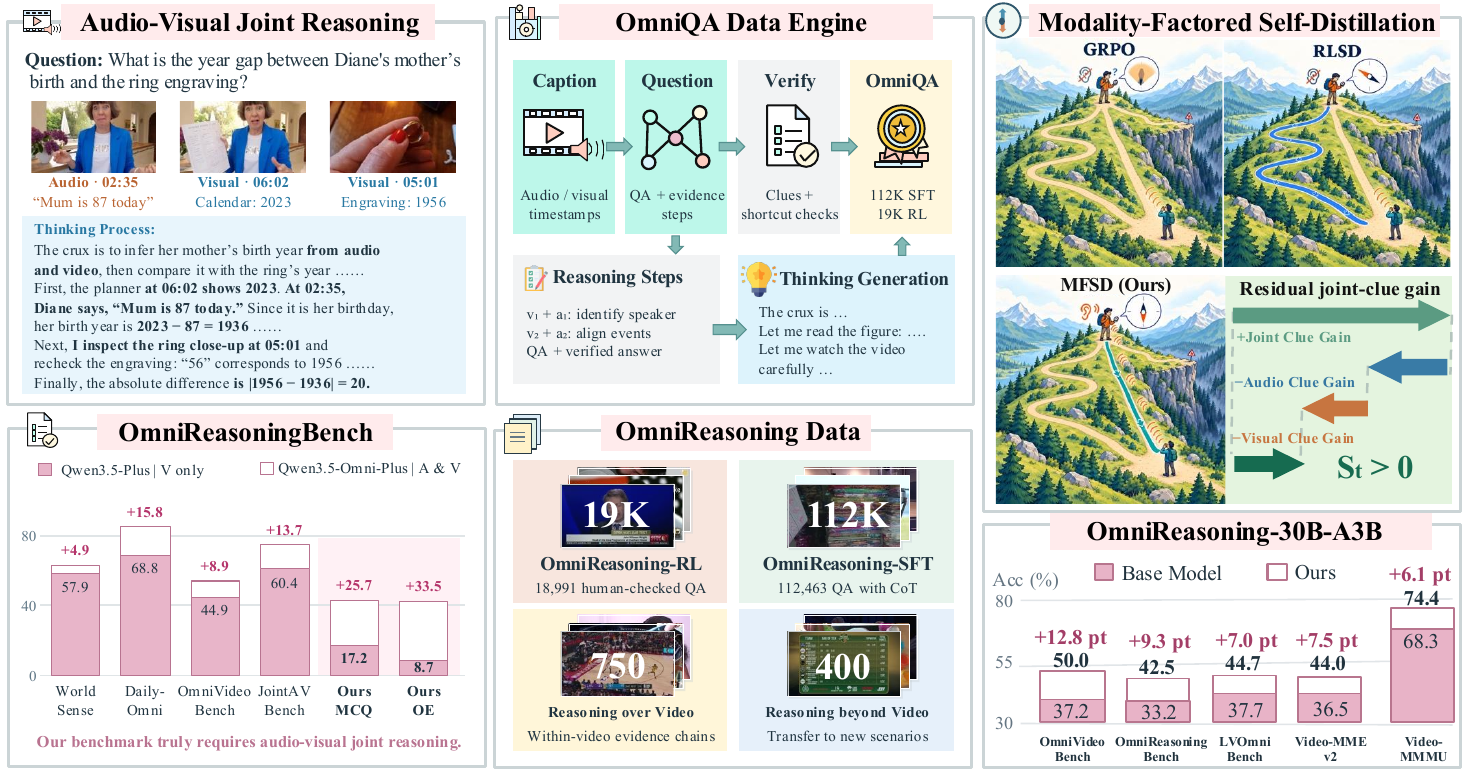}
\vspace{-6pt}
% \caption{\textbf{OmniReasoning: from data pipeline to audio-visual joint reasoning.} OmniQA generates evidence-grounded questions for SFT, RL, and OmniReasoningBench. We fine-tune Qwen3-Omni-30B-A3B on OmniReasoning-SFT-116k and then apply MFSD to obtain OmniReasoning-30B-A3B. MFSD uses privileged audio-visual clues and residual joint-clue gains to assign finer-grained token-level credit, enabling more efficient exploration along a more direct path, while GRPO offers only outcome-level guidance, and RLSD provides coarse guidance, resulting in a winding ascent. }
\caption{\textbf{OmniReasoning: a benchmark, data engine and learning method for audio-visual joint reasoning.} 
Unlike previous benchmarks, our benchmark OmniReasoningBench truly requires both audio and visual inputs for joint reasoning.
Besides this benchmark, our data engine OmniQA additionally produces large-scale training data with evidence-grounded questions. 
Further, our learning method MFSD leverages the gain from cross-modality joint clues to assign credit at token level, enabling effective exploration along audio-visual joint reasoning. In comparison, previous method GRPO offers only outcome-level guidance, and RLSD does not consider cross-modality interaction. 
With our training data and learning method, our model achieves large improvements over base model on audio-visual, long video, and general video benchmarks.}

\label{fig:teaser}
\vspace{-12pt}
\end{figure}

% ===== END inlined from sections/teaser.tex =====
% ===== BEGIN inlined from sections/abstract.tex =====
\begin{abstract}

Recent advances have enabled unified omni-modal models in understanding audio, vision, and language.
However, existing benchmarks, training data, and learning methods largely treat the modalities independently, leaving the capability of audio-visual joint reasoning poorly evaluated and insufficiently elicited.
We address this gap with a benchmark, data engine, and learning method.
First, we introduce \textbf{OmniReasoningBench}, a benchmark where both audio and visual evidence are indispensable. It comprises 1,150 multiple-choice and open-ended questions across two tasks, \emph{reasoning over video} and \emph{reasoning beyond video}. 
Second, we develop a data engine \textbf{OmniQA}. It automatically constructs evidence-grounded QA pairs that explicitly necessitate audio-visual joint reasoning, together with timestamped clue chains that guide the annotation of thinking process. Besides our benchmark, this engine produces training data OmniReasoning-SFT-112K and OmniReasoning-RL-19K. 
Finally, we propose an on-policy self-distillation method \textbf{Modality-Factored Self-Distillation (MFSD)}. It evaluates each sampled response under modality-specific clue contexts, disentangling the contributions of individual clues and their cross-modal interactions for token-level credit assignment.
With our training data and learning method, our model \textbf{OmniReasoning-30B-A3B} achieves 50.0\% on OmniVideoBench and 42.5\% on OmniReasoningBench, improving the base model Qwen3-Omni-30B-A3B-Thinking by 12.8 and 9.3 percentage points, respectively. Moreover, it delivers substantial gains on general and long-video benchmarks, including Video-MME-v2. 
We hope our work offers a solid step for facilitating future research in omni-modal joint reasoning.
\end{abstract}

% ===== END inlined from sections/abstract.tex =====
% ===== BEGIN inlined from sections/introduction.tex =====
\section{Introduction}
\label{sec:intro}
% Understanding a video often requires connecting what we hear with what we see. Consider the example in Figure~\ref{fig:teaser}: determining the gap between a mother's birth year and the year engraved inside a ring requires combining a spoken age with a visible calendar, then comparing the inferred year with the engraving. This illustrates \emph{audio-visual joint reasoning}. Recent omni-modal large language models (Omni-LLMs), including Qwen-Omni~\citep{qwen25omni,qwen3omni} and Nemotron 3 Nano Omni~\citep{nemotron3omni}, can understand audio and visual information natively, but their audio-visual joint reasoning is still limited.

Understanding real-world videos often requires reasoning across what has been heard and what has been seen. 
Consider the example in Figure~\ref{fig:teaser}: to determine the gap between a mother's birth year and the year engraved inside a ring, a model has to combine a spoken age with an observed calendar date, infer the corresponding birth year, and then compare it with the engraving. 
Neither modality alone provides sufficient evidence. The answer emerges only by connecting information across modalities. We refer to this capability as \emph{audio-visual joint reasoning}. 
Recent omni-modal large language models (Omni-LLMs), such as Qwen-Omni~\citep{qwen25omni,qwen3omni} and Nemotron 3 Nano Omni~\citep{nemotron3omni}, have made substantial progress toward unified audio-vision-language understanding. However, whether these omni models can reliably perform such cross-modal reasoning remains largely unclear.

% Existing audio-visual benchmarks ~\citep{worldsense,dailyomni,omnivideobench,jointavbench} position themselves as tests of omni joint understanding, and conduct several experiments for modality ablations. However, reporting an ablation does not by itself make joint reasoning necessary. Under our evaluated protocols, the same vision-language model, Qwen3.5-Plus, achieves 57.85\%, 68.76\%, 44.90\%, and 60.43\% accuracy on WorldSense, Daily-Omni, OmniVideoBench, and JointAVBench using \emph{visual-alone inputs}, without any audio (Table~\ref{tab:visual-only}). When a vision-only model clears 50--60\% on a benchmark that claims to be omni-modal required, the questions themselves largely fail to require audio, and such scores cannot distinguish audio-visual joint reasoning capability for Omni-LLMs.

A fundamental obstacle is that existing benchmarks do not consistently make joint reasoning necessary. 
Several recent audio-visual benchmarks~\citep{worldsense,dailyomni,omnivideobench,jointavbench} position themselves as evaluations of omni-modal understanding and conduct modality ablation experiments. 
However, ablation studies do not fully validate that a benchmark actually requires information from multiple modalities to answer the questions. 
In our experiments, Qwen3.5-Plus~\citep{qwen35blog} achieves 57.85\%, 68.76\%, 44.90\%, and 60.43\% accuracy on WorldSense, Daily-Omni, OmniVideoBench, and JointAVBench, respectively, when provided with \emph{visual-only inputs} and \emph{no audio}. 
These results indicate that a substantial fraction of the questions remain answerable without audio, limiting the extent to which such benchmarks can distinguish genuine audio-visual joint reasoning from strong single-modality understanding. 

% Motivated by this, we introduce \textbf{OmniReasoning}, a unified framework that uses modality-specific evidence to connect data construction, evaluation, and post-training. Our benchmark, \textbf{OmniReasoningBench}, comprises 1,150 questions across two settings: \emph{reasoning over video} connects observations within a single video, while \emph{reasoning beyond video} applies video-derived information to a new scenario or figure. Both settings include multiple-choice and open-ended questions. Questions are constructed around cross-event dependencies between auditory and visual evidence, targeting the residual single-modality answerability observed above: the same Qwen3.5-Plus falls to 17.2\% and 8.7\% on the multiple-choice and open-ended questions of \emph{reasoning over video}, respectively. Without audio input, most questions simply cannot be answered, so audio-visual joint reasoning is required rather than optional. Same-model modality ablations on \emph{reasoning over video} show that joint inputs substantially improve accuracy over either modality alone, and that these gains exceed the union of questions each modality answers correctly by itself. (Appendix~\ref{app:modality-ablations}).

Motivated by these findings, we introduce \textbf{OmniReasoning}, a unified framework for evaluating and eliciting omni-modal joint reasoning through explicit modality-specific evidence. 
At its core is \textbf{OmniReasoningBench}, a benchmark of 1,150 multiple-choice and open-ended questions spanning two settings: \emph{reasoning over video}, which requires connecting observations across events within a video, and \emph{reasoning beyond video}, which applies information derived from a video to a new scenario or figure. 
The questions are deliberately constructed around dependencies between audio and visual observations, thereby minimizing the possibility of solving them from a single modality alone. 
Take Figure~\ref{fig:teaser} as an example, when Qwen3.5-Plus model evaluated without audio, the accuracy drops to 17.2\% and 8.7\% on the multiple-choice and open-ended questions, respectively. 
Moreover, our ablation studies show that audio-visual joint reasoning provides gains beyond either modality alone, with improvements exceeding the gains obtained from simply combining the questions solved independently by the two modalities. %(Appendix~\ref{app}). 
Together, these results demonstrate that the benchmark is able to test the capability of integrating complementary evidence across modalities.

% These findings motivate the \textbf{OmniQA data engine}, which constructs questions that explicitly connect audio and visual evidence (Figure~\ref{fig:pipeline}). OmniQA generates questions together with timestamped clues and dependency chains linking observations to intermediate inferences. After verification, these chains guide the synthesis of thinking processes for supervised fine-tuning (\textbf{OmniReasoning-SFT-116K}), while the retained modality-specific clues support feedback on model-generated responses during reinforcement learning (\textbf{OmniReasoning-RL-20K}).

Making such questions at scale, however, presents a second challenge: training data has to preserve the same evidence dependencies rather than merely pair arbitrary audio, visual, and textual content. 
To address this, we develop a data engine \textbf{OmniQA}, which automatically constructs evidence-grounded questions that explicitly link audio and visual observations. In addition to generating question-answer pairs, OmniQA produces timestamped clues and dependency chains that connect observed evidence to intermediate inferences and the final answer (Figure~\ref{fig:pipeline}). 
After verification, these reasoning chains are used to synthesize thinking processes for supervised fine-tuning, yielding \emph{OmniReasoning-SFT-112K}. 
The modality-specific clue annotations are also retained as structured supervision for reinforcement learning, resulting in \emph{OmniReasoning-RL-19K}. 
Thus, OmniQA provides not only the benchmark and large-scale training data, but also an explicit representation of how evidence from different modalities contributes to a reasoning process.

% These clue annotations also support \textbf{Modality-Factored Self-Distillation (MFSD)}, which distinguishes overall clue support from audio-visual interaction. The same actor scores sampled tokens under differenct text-clue contexts: no clues, audio clues, visual clues, and joint clues. Subtracting the individual log-likelihood gains from the joint gain captures their non-additive interaction. MFSD combines this interaction with joint-clue support to guide finer-grained token-level credit assignment through RLSD's weighting mechanism~\citep{rlsd}.

The same evidence structure further enables a more targeted learning objective. 
Existing reasoning-oriented post-training methods~\citep{grpo,gspo,dapo} typically assign credit to a response according to its overall quality, without distinguishing whether a prediction is supported by audio evidence, visual evidence, or their interaction. 
We therefore propose \textbf{Modality-Factored Self-Distillation (MFSD)}, an on-policy self-distillation method that factors token-level credit assignment according to modality-specific evidence. 
As in Figure~\ref{fig:teaser}, for each sampled response, the model evaluates the token likelihood under four types of clue context: no clues, audio clues, visual clues, and joint audio-visual clues. 
The resulting likelihood difference quantifies both the contribution of individual modalities and the non-additive interaction between them. 
Specifically, MFSD measures the interaction by subtracting the individual audio and visual gains from the gain obtained under joint clues, and combines this interaction with the overall support provided by the joint clues to derive fine-grained token-level supervision through the weighting mechanism of RLSD~\citep{rlsd}. 
This design explicitly encourages the model to generate reasoning steps that rely on complementary cross-modal evidence rather than exploiting one modality in isolation.

% Training with OmniQA and MFSD yields \textbf{OmniReasoning-30B-A3B}, which achieves 50.0\% accuracy on OmniVideoBench and 42.5\% accuracy on OmniReasoningBench, 12.8 and 9.3 percentage points above the Qwen3-Omni-30B-A3B-Thinking baseline respectively. OmniReasoning-30B-A3B also improves over the base model on long video understanding and general video understanding benchmark such as: LVOmniBench, Video-MMMU, and Video-MME-v2.

Combining the OmniQA training data with MFSD yields our model \textbf{OmniReasoning-30B-A3B}. It achieves 50.0\% accuracy on OmniVideoBench and 42.5\% on OmniReasoningBench, improving over the baseline (Qwen3-Omni-30B-A3B-Thinking) by 12.8 and 9.3 percentage points, respectively. 
Beyond the targeted evaluation, the model also delivers substantial improvements on general and long-video understanding benchmarks, including LVOmniBench, Video-MMMU, and Video-MME-v2. 
These results suggest that explicitly constructing, annotating, and optimizing for cross-modal evidence dependencies can strengthen not only audio-visual joint reasoning, but also broader video understanding. 
We hope that \textbf{OmniReasoning} provides a unified foundation for studying and improving joint reasoning across modalities in future omni-modal models.

\Needspace{4\baselineskip}
Our contributions are summarized as follows:
\begin{enumerate}
% \item We develop \textbf{OmniQA}, a scalable data engine that produces \textbf{OmniReasoning-SFT-116K} and \textbf{OmniReasoning-RL-20K} together with generated thinking process and modality-specific clues for SFT and RL phase.
%     \item We introduce \textbf{OmniReasoningBench} comprising 750 multi-choices and 400 open-ended questions across two complementary task types: \textit{reasoning over video} and \textit{reasoning beyond video}.
%     \item We propose \textbf{MFSD}, which separates joint-clue support from audio-visual interaction and combines them for finer-grained token-level credit assignment.

% \item We introduce \textbf{OmniReasoningBench}, a benchmark that necessitates the joint reasoning across audio-visual modalities, comprising two complementary tasks, \textit{reasoning over video} and \textit{reasoning beyond video}.
% \item We develop \textbf{OmniQA}, a data engine that can produce scalable training data for supervised fine-tuning and reinforcement learning, together with annotations of thinking process and modality-specific clues.
% \item We propose \textbf{MFSD}, a reinforcement learning method which highlights cross-modality interactions and offers token-level credit assignment.

\item \textbf{A benchmark for genuine audio-visual joint reasoning}. We introduce OmniReasoningBench, a benchmark where audio and visual observations are jointly necessary, enabling a more rigorous evaluation of omni-modal reasoning.

\item \textbf{An evidence-grounded data engine for joint reasoning}. We develop OmniQA, an automated data engine that constructs audio-visual QA pairs together with timestamped clues and dependency chains, providing explicit supervision for cross-modal training.

\item \textbf{A modality-aware learning method for joint reasoning}. We propose MFSD, a reinforcement learning method that highlights cross-modality interactions and offers token-level credit assignment. Combined with OmniQA data, MFSD substantially improves both audio-visual joint reasoning and general video understanding.

\end{enumerate}

% ===== END inlined from sections/introduction.tex =====
% ===== BEGIN inlined from sections/related.tex =====
\section{Related Work}
\label{sec:related}
\vspace{-3pt}
\noindent\textbf{Omni-LLMs and audio-visual joint reasoning.}
Recent omni-modal models support native audio-visual understanding~\citep{gemini3models,doubao2,musespark,mimov25,qwen25omni,qwen3omni,qwen35omni,videosalmonn2,nemotron3omni}. WorldSense~\citep{worldsense}, Daily-Omni~\citep{dailyomni}, OmniVideoBench~\citep{omnivideobench}, and LVOmniBench~\citep{lvomnibench} benchmark this capability. JointAVBench~\citep{jointavbench} explicitly evaluates cross-modality dependence. AV-Reasoner~\citep{avreasoner} studies clue-grounded counting and Video-MMMU~\citep{videommmu} tests knowledge acquisition from instructional videos. In contrast, our benchmark necessitates joint reasoning over audio and visual observations. %, covering both \emph{reasoning over video} and \emph{reasoning beyond video}.

\vspace{-3pt}
\noindent\textbf{Audio-visual instruction data.}
OmniVideo-100K~\citep{omnivideo100k} preserves cross-segment associations and audio-visual correspondence through entity-anchored scripts and clue-guided questions. OmniVideo-R1~\citep{omnivideor1} filters LLaVA-Video and Video-Vista questions for cross-modal grounding and fusion. 
Our OmniQA leverages explicit evidence for question construction, dependency chains for thinking generation, and separate audio/visual clues for RL scoring, while extending video-derived knowledge to new inputs through \emph{reasoning beyond video}.

\vspace{-3pt}
\noindent\textbf{Reinforcement learning and self-distillation.}
Group Relative Policy Optimization (GRPO)~\citep{grpo} develops group-relative outcome rewards while Group Sequence Policy Optimization (GSPO)~\citep{gspo} considers sequence-level importance weighting. Video-R1~\citep{videor1} and OmniVideo-R1~\citep{omnivideor1} apply RL to video and audio-visual reasoning, and Video-KTR~\citep{videoktr} addresses key-token attribution. 
On-policy self-distillation~\citep{opsd,sdpo} converts additional context into token-level feedback, which RLSD~\citep{rlsd} maps to bounded, sign-preserving advantage weights. 
We propose MFSD which retains RLSD's weighting, outcome verifier, and policy objective, combining a centered interaction contrast from separate and joint clues with the joint-clue gain.

% ===== END inlined from sections/related.tex =====
\section{Benchmark, Data Engine, and Learning Method}
Our work, OmniReasoning, seeks to advance audio-visual joint reasoning with three components: an evaluation benchmark, a data engine, and a learning method. They are connected through explicit dependencies between audio and visual evidence. Specifically, the OmniQA data engine constructs both the questions of OmniReasoningBench and the training corpora, and its modality-specific clue annotations further provide the supervision signal for our learning method, MFSD.
\label{sec:framework}
% ===== BEGIN inlined from sections/benchmark.tex =====
\subsection{OmniReasoningBench: A Benchmark for Audio-Visual Joint Reasoning}
\label{sec:benchmark}
\textbf{OmniReasoningBench} evaluates whether models can connect complementary audio and visual evidence to answer a question. It contains 750 \emph{reasoning over video} and 400 \emph{reasoning beyond video} questions, each with a reference answer and an annotated evidence chain (Figure~\ref{fig:bench-demo}).

\begin{figure}[t]
\centering
\includegraphics[width=0.99\linewidth]{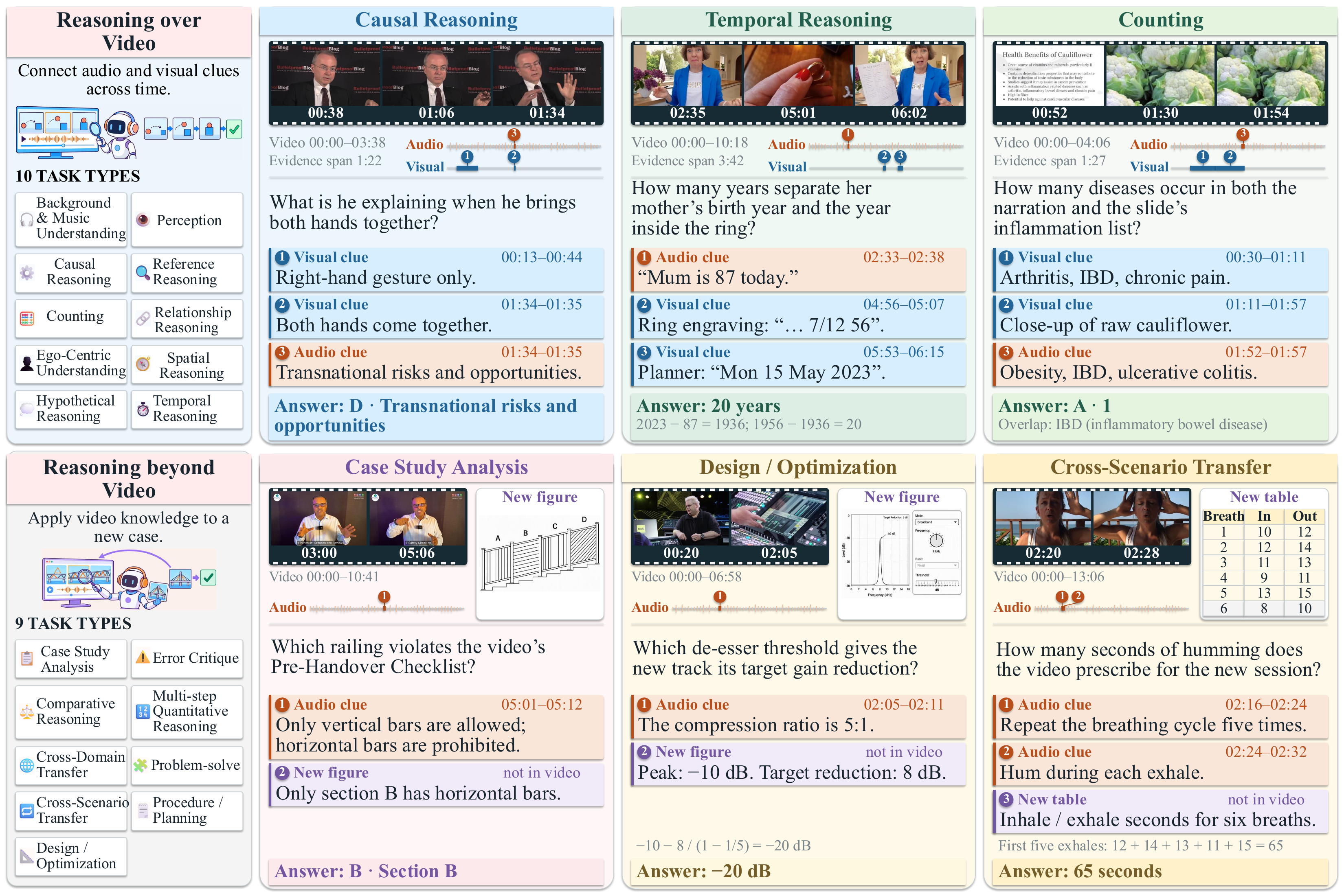}
\vspace{-6pt}
\caption{\textbf{OmniReasoningBench tasks and examples.} \emph{Reasoning over video} connects observations across events; \emph{reasoning beyond video} applies video-derived knowledge to a new scenario. Orange and blue mark audio and visual clues, with numbered markers linking evidence to timestamps.}
\label{fig:bench-demo}
\vspace{-12pt}
\end{figure}

\vspace{-3pt}
\noindent\textbf{Reasoning over video}.
These questions ask about the content within videos. Every answer can be derived by chaining multi-hop audio and visual evidence, and no single modality suffices. A spoken reference can identify which object to inspect, while a visible action can identify the relevant utterance. Correct answers therefore require joint reasoning over audio and visual evidence. This setting covers ten task types and includes 375 multiple-choice and 375 open-ended questions.

\vspace{-3pt}
\noindent\textbf{Reasoning beyond video}.
Omni reasoning can go beyond the understanding of a given video and transfer to new scenarios. The questions in this setting therefore ask models to carry knowledge acquired from a video into a new situation, figure, or numerical condition. For example, a spoken explanation and a visual demonstration establish a rule that must then be applied to a new diagram. In this case, the evidence chain still exists in the video, while the reasoning extends beyond it. There are in total 400 questions spanning nine task types, with 250 multiple-choice and 150 open-ended questions. Appendix~\ref{app:protocol} provides the input-format breakdown.

\vspace{-3pt}
\noindent\textbf{Evidence of joint reasoning.}
On \emph{reasoning over video}, joint inputs outperform a single-modality oracle that counts a question as solved whenever either the audio-only or the visual-only run answers it correctly. The gains range from 9.9 to 26.7 percentage points across three models and both question formats (Appendix~\ref{app:modality-ablations}).

% ===== END inlined from sections/benchmark.tex =====
% ===== BEGIN inlined from sections/data.tex =====
\subsection{OmniQA: A Data Engine for Audio-Visual Joint Reasoning}
\label{sec:pipeline}
To preserve these evidence dependencies in training data, \textbf{OmniQA} constructs QA pairs together with timestamped audio and visual clues. The verified evidence chains guide thinking generation for SFT and provide modality-specific context for MFSD during RL (Figure~\ref{fig:pipeline}).

\begin{figure}[t]
\centering
\includegraphics[width=0.99\linewidth]{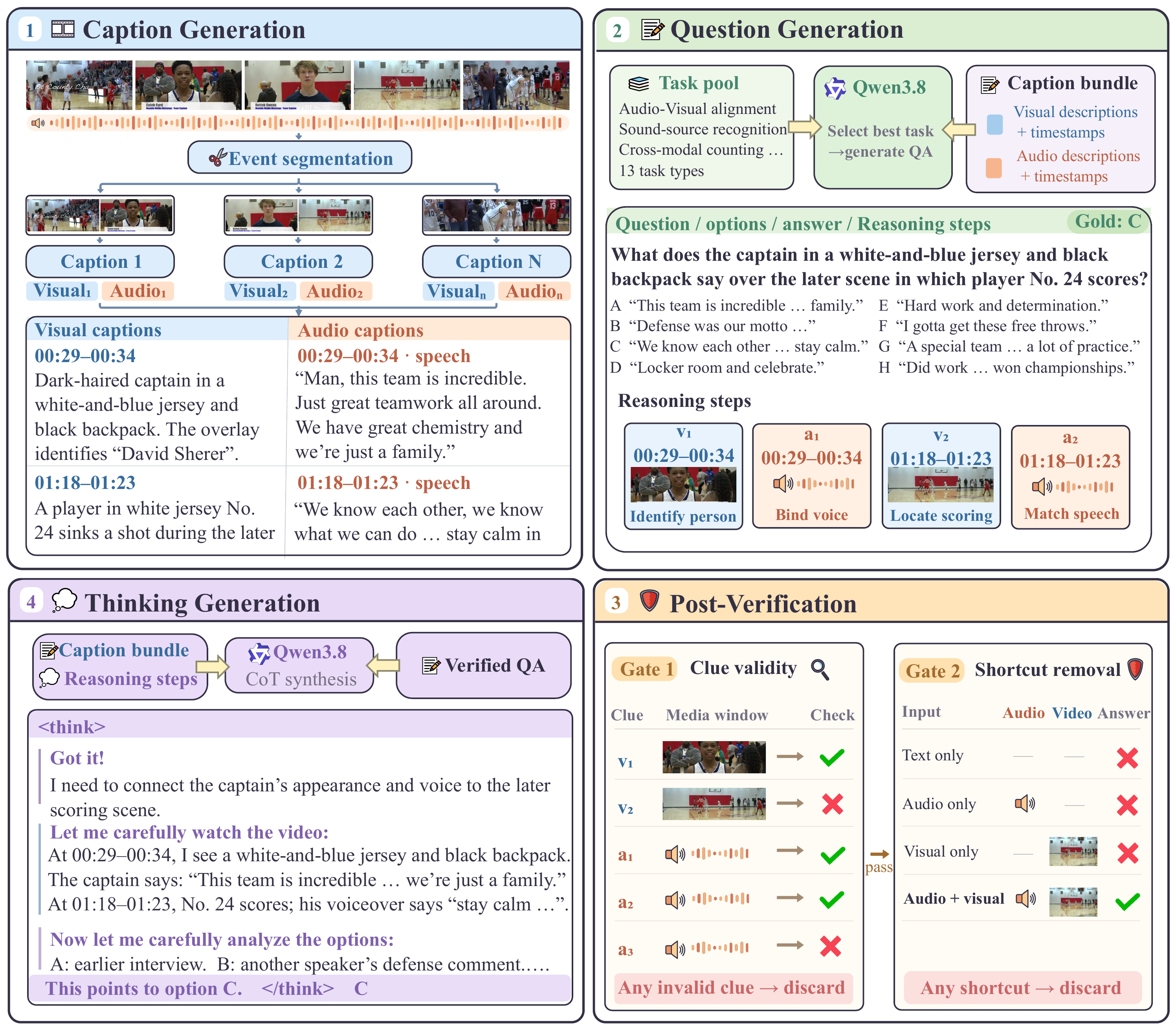}
\vspace{-6pt}
\caption{\textbf{OmniQA data engine.} Gemini-3.1-Pro annotates timestamped audio-visual descriptions. Qwen3.8 generates QA pairs and reasoning steps, followed by clue validation and shortcut screening. Timestamped captions, verified QA pairs, and evidence chains then guide thinking generation.}
\label{fig:pipeline}
\vspace{-12pt}
\end{figure}

\begin{figure}[t]
\centering
\includegraphics[width=0.99\linewidth]{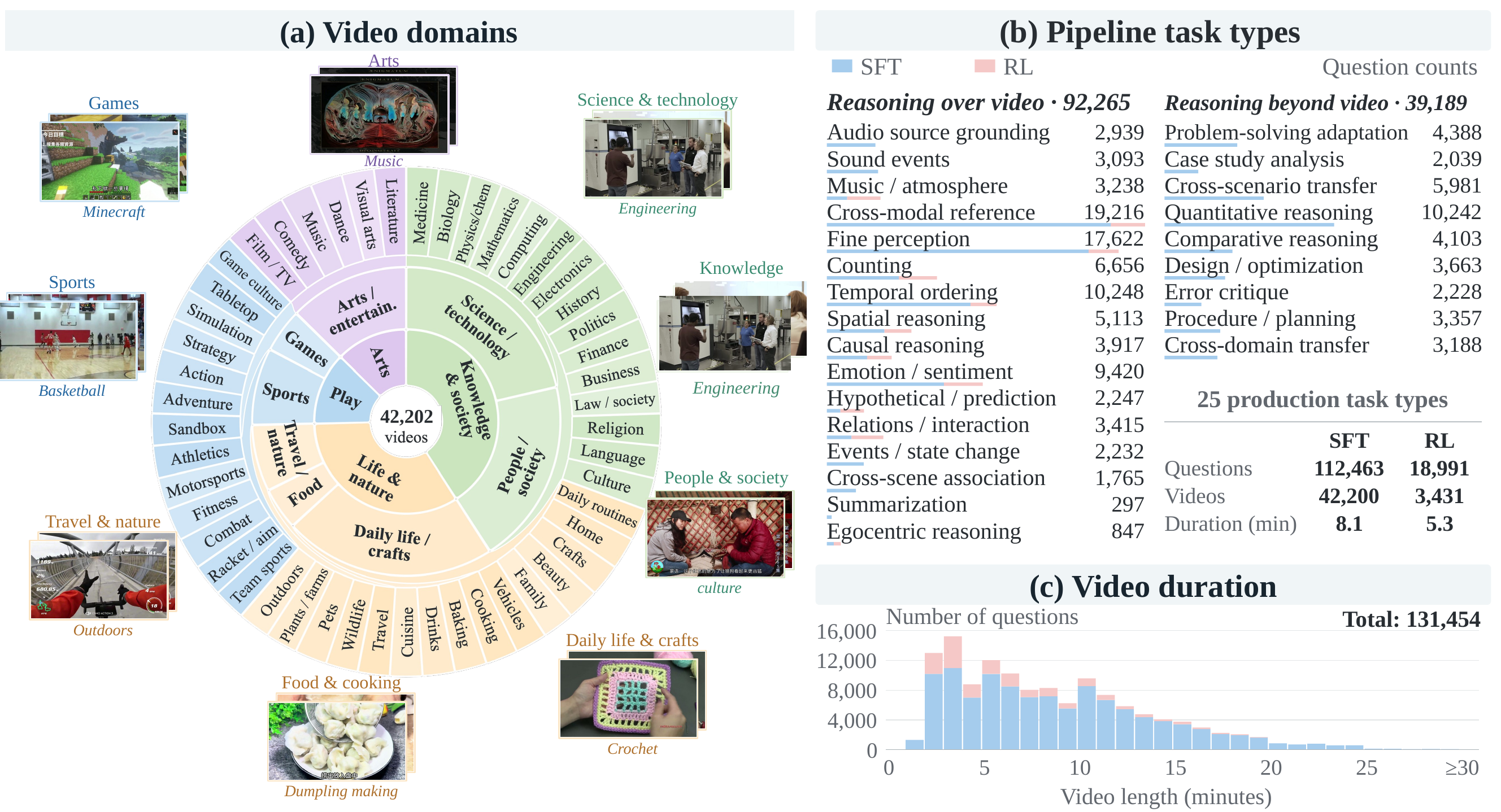}
\vspace{-6pt}
\caption{\textbf{OmniReasoning released training-data distributions.} The SFT and RL corpora span eight content domains, 25 production task types, and varied video durations.}
\label{fig:training-statistics}
\vspace{-12pt}
\end{figure}

\vspace{-3pt}
\noindent\textbf{Constructing questions from evidence.}
OmniQA segments each video into events and generates separate audio and visual descriptions with timestamps. Conditioned on these descriptions and a task specification, the question generator produces a question, a reference answer, and a dependency chain that links observations to intermediate inferences and the final answer. For multiple-choice questions, it also generates candidate options.

\vspace{-3pt}
\noindent\textbf{Verifying questions and clues.}
Structural checks validate the dependency graph, caption references, and timestamps. A caption-conditioned solver checks the reference answer, and a media verifier checks each clue against its supporting audio or visual clip. To screen for modality shortcuts, we evaluate each question under question-only, audio-only, visual-only, and joint audio-visual conditions. We retain a candidate only when its clues are valid, the joint condition yields the correct answer, and none of the restricted conditions does.

\vspace{-3pt}
\noindent\textbf{Generating thinking processes.}
The thinking generator receives timestamped audio and visual captions, the verified QA pair, the evidence chain, and any additional figure, but not the source video. It expands the chain into observations, intermediate inferences, and a final answer, comparing multiple-choice options and showing numerical calculations (Appendix~\ref{app:sft-cases}).

\vspace{-3pt}
\noindent\textbf{Training datasets.}
The released datasets comprise \textbf{OmniReasoning-SFT-112K}, with 112,463 samples containing synthesized thinking processes, and \textbf{OmniReasoning-RL-19K}, with 18,991 human-validated questions and evidence annotations. Both cover \emph{reasoning over video} and \emph{reasoning beyond video} (Figure~\ref{fig:training-statistics}). During the SFT stage, the model learns from the synthesized responses. During the RL stage, the retained audio and visual clues, $r_A$ and $r_V$, instead provide privileged context for evaluating the model's sampled responses.

% ===== END inlined from sections/data.tex =====
% ===== BEGIN inlined from sections/method.tex =====
\subsection{Modality-Factored Self-Distillation: A Learning Method}
\label{sec:method}
The modality-specific clues created by OmniQA provide not only data annotations but also a learning signal for reinforcement learning.
Our method stems from RLSD~\citep{rlsd}, a widely adopted RL approach that re-scores a sampled response under privileged clue contexts and converts the resulting likelihood gains into token-level advantage weights.
However, a likelihood gain under joint audio-visual clues does not by itself indicate joint reasoning. The gain may come from one modality alone, so a response that exploits a single modality can receive the same guidance as one that genuinely integrates both.
\textbf{Therefore, we propose MFSD, which extends RLSD with cross-modality interaction guidance.}
Besides the joint-clue support used by RLSD, MFSD disentangles the likelihood gain that emerges only when audio and visual clues are combined, and uses the interaction signal for fine-grained token-level credit assignment (Figure~\ref{fig:mfsd}).

\begin{figure}[t]
\centering
\includegraphics[width=0.99\linewidth]{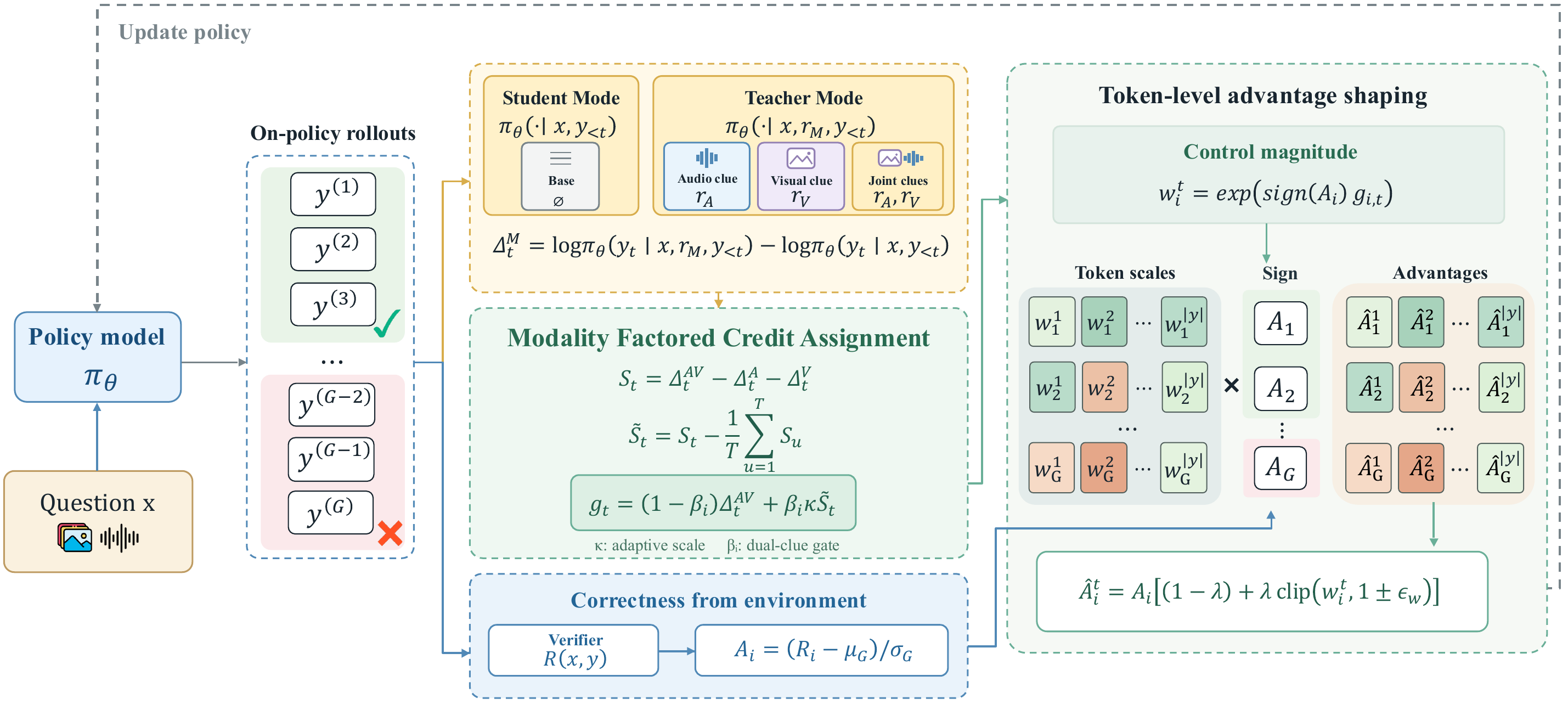}
\vspace{-6pt}
\caption{\textbf{Modality-Factored Self-Distillation.} The actor scores the same response under four clue contexts. Joint-clue support and non-additive audio-visual interaction determine bounded token-level advantage weights.}
\label{fig:mfsd}
\vspace{-12pt}
\end{figure}

\vspace{-3pt}
\noindent\textbf{MFSD.}
Given a question with its original audio-visual input, we first sample a group of responses without privileged clues, as in standard group-based reinforcement learning.
The same actor then scores each sampled response under four clue contexts: no clues, audio clues, visual clues, and joint audio-visual clues. The likelihood changes across these contexts yield two complementary evidence signals: the overall support provided by the joint clues, and the cross-modality interaction that neither modality explains individually.
Both signals are detached and converted into bounded token-level weights that modulate the outcome advantage.

\vspace{-3pt}
\noindent\textbf{Scoring the same response under different clue contexts.}
Let $x$ contain the question and its original audio-visual input. We sample $G$ responses $y^{(i)}$ from $\pi_{\mathrm{old}}(\cdot\mid x)$ without privileged clues. Outcome rewards define $A_i=(R_i-\mu_R)/(\sigma_R+\varepsilon)$, where $\mu_R$ and $\sigma_R$ are the group reward mean and standard deviation. Given OmniQA's audio and visual text clues $r_A$ and $r_V$, define $r_\varnothing=\varnothing$ and $r_{AV}=r_A\oplus r_V$. The actor scores each sampled token as
\begin{equation}
\ell^M_{i,t}=\log\pi_\theta(y^{(i)}_t\mid x,r_M,y^{(i)}_{<t}),
\qquad M\in\{\varnothing,A,V,AV\}.
\label{eq:contexts}
\end{equation}
The original media, sampled prefix, and actor weights remain fixed across views. The detached gains are $\Delta^M_{i,t}=\mathrm{sg}(\ell^M_{i,t}-\ell^\varnothing_{i,t})$ for $M\in\{A,V,AV\}$, where $\mathrm{sg}$ denotes stop-gradient.

\vspace{-3pt}
\noindent\textbf{Separating joint support from interaction.}
MFSD subtracts the individual clue gains from the joint gain:
\begin{equation}
S_{i,t}=\Delta^{AV}_{i,t}-\Delta^A_{i,t}-\Delta^V_{i,t}
=\mathrm{sg}\!\left(\ell^{AV}_{i,t}-\ell^A_{i,t}-\ell^V_{i,t}+\ell^\varnothing_{i,t}\right).
\label{eq:synergy}
\end{equation}
Intuitively, $S_{i,t}$ is positive only when the two clue sets together raise a token's likelihood by more than the sum of their individual effects. This is a non-additivity comparison in log-likelihood: a positive $S_{i,t}$ does not by itself imply that the joint clues raise the token's likelihood over the no-clue context, since the joint gain and the interaction can differ in sign. It is zero whenever the joint gain is fully explained by a single modality. For example, audio clues alone may account for the joint gain while visual clues add nothing. Under the joint-consistency assumption in Appendix~\ref{app:mfsd-proofs}, this contrast is a token-wise increment in conditional pointwise mutual information (PMI). Suppressing the rollout index,
\begin{equation}
S_t=\mathcal I_t-\mathcal I_{t-1},
\qquad
\mathcal I_t=\log\frac{P(r_A,r_V\mid x,y_{\leq t})}
{P(r_A\mid x,y_{\leq t})P(r_V\mid x,y_{\leq t})}.
\label{eq:main-pmi}
\end{equation}
where $y_{\leq0}=\varnothing$. This interpretation concerns the model's response to clue conditioning; it does not certify reasoning-step correctness.

\vspace{-3pt}
\noindent\textbf{Combining the two evidence signals.}
We center interaction scores within each response using the valid-token mask $m_{i,t}$ and $T_i=\sum_t m_{i,t}>0$:
\begin{equation}
\bar S_i=\frac{1}{T_i}\sum_t m_{i,t}S_{i,t},\qquad
\widetilde S_{i,t}=m_{i,t}(S_{i,t}-\bar S_i).
\label{eq:center}
\end{equation}
The combined token score is
\begin{equation}
g_{i,t}=(1-\beta_i)\Delta^{AV}_{i,t}+\beta_i\kappa\widetilde S_{i,t},
\qquad \beta_i=\beta\,\mathbf{1}[r_A\ne\varnothing\land r_V\ne\varnothing].
\label{eq:credit}
\end{equation}
The two signals play complementary roles: joint support retains useful evidence even when it comes from a single modality, while centered interaction emphasizes tokens with above-average cross-modal support.
Centering removes the response-level mean of the interaction score, so the interaction term redistributes credit within a response rather than rescaling the response as a whole (Appendix~\ref{app:mfsd-proofs}).
The coefficient $\beta$ balances these signals, and $\kappa$ calibrates their scales (Appendix~\ref{app:impl}).

\vspace{-3pt}
\noindent\textbf{Assigning token-level policy credit.}
Following RLSD, we convert $g_{i,t}$ into bounded advantage weights:
\begin{equation}
\begin{gathered}
w_{i,t}=\exp\!\left(\operatorname{sign}(A_i)g_{i,t}\right),\\
\widehat A_{i,t}=A_i\left[(1-\lambda)+\lambda\,
\mathrm{clip}\!\left(w_{i,t},1-\epsilon_w,1+\epsilon_w\right)\right].
\end{gathered}
\label{eq:ahat}
\end{equation}
Higher $g_{i,t}$ strengthens positive advantages and reduces negative penalties without reversing their signs. The clipped policy objective uses $\widehat A_{i,t}$ (Appendix~\ref{app:objective}). Detached scores and weights prevent gradients through clue conditioning; rollouts exclude privileged clues. The three clue-conditioned scoring views require no extra rollouts or separate teacher parameters.
Relative to RLSD, the modification is that the token score $g_{i,t}$ now carries cross-modality interaction guidance. Tokens whose likelihood increases only when audio and visual clues are combined receive larger positive advantages and smaller penalties, whereas tokens supported by a single modality alone are guided mainly by joint-clue support.

% ===== END inlined from sections/method.tex =====
% ===== BEGIN inlined from sections/experiments.tex =====
\section{Experiments}
% ===== BEGIN inlined from tables/main.tex =====
\begin{table}[!t]
\caption{\textbf{OmniReasoningBench accuracy (\%).} MCQ and OE denote multiple-choice and open-ended questions. Our model is initialized by Qwen3-Omni-30B-A3B-Thinking.}
\label{tab:main}
\centering
\footnotesize
\setlength{\tabcolsep}{0pt}
\begin{tabularx}{\linewidth}{@{}>{\raggedright\arraybackslash}X C{\dimexpr1.25cm+4pt\relax} C{\dimexpr1.4cm+4pt\relax} C{\dimexpr1.4cm+4pt\relax} C{\dimexpr1.6cm+4pt\relax} C{\dimexpr1.6cm+4pt\relax} C{\dimexpr1.02cm+4pt\relax}@{}}
\toprule
 & & \multicolumn{2}{c}{\textbf{\emph{Reasoning over video}}} &
 \multicolumn{2}{c}{\textbf{\emph{Reasoning beyond video}}} & \\
\cmidrule(lr){3-4}\cmidrule(lr){5-6}
\textbf{Model} & \textbf{Modality} & \textbf{MCQ} & \textbf{OE} & \textbf{MCQ} &
\textbf{OE} & \textbf{Overall} \\
\midrule
\tbgroupcf{7}{Proprietary models} \\
\lgemini Gemini-3.7-Flash               & Omni   & \bov{57.6}   & \bov{47.7}   & \bby{72.0}   & \bby{58.0} & \ball{57.6}  \\

\lgemini Gemini-3.8-Flash               & Omni   & \bov{55.2}   & \bov{49.6}   & \bby{69.6}   & \bby{60.0} & \ball{57.1}  \\
\lgemini Gemini-3.5-Flash               & Omni   & \bov{53.3}   & \bov{46.9}   & \bby{66.4}   & \bby{51.3} & \ball{53.8}  \\
\lgemini Gemini-3.1-Pro                 & Omni   & \bov{51.7} & \bov{46.9}   & \bby{64.0} & \bby{49.3} & \ball{52.5} \\

\ldoubao Doubao-2.0-Lite         & Omni       & \bov{54.7}     & \bov{47.2}                 & \bby{57.2}                 & \bby{40.7}                 & \ball{51.0}                            \\
\lqwen Qwen3.5-Omni-Plus                & Omni   & \bov{42.7}          & \bov{41.9} & \bby{60.4}    & \bby{40.7}          & \ball{46.0}    \\
\lmeta Muse-Spark-1.2                   & Omni   & \bov{38.1}          & \bov{36.0}          & \bby{60.0}         & \bby{58.0}        & \ball{44.8}         \\

\lmeta Muse-Spark-1.1                   & Omni   & \bov{36.4}          & \bov{35.9}          & \bby{56.0}         & \bby{46.0}        & \ball{41.8}         \\

% Doubao modality and evaluation scores await author-supplied results.

\tbgroupcf{7}{Open-source models} \\
\lxiaomi MiMo-V2.5                      & Omni   & \bov{42.7}          & \bov{40.6}          & \bby{54.0}         & \bby{44.0}        & \ball{44.6}         \\

% \lqwen Qwen3-Omni-30B-A3B-Thinking      & Omni   & \bov{31.9}          & \bov{25.9}          & \bby{46.4}          & \bby{32.0}        & \ball{33.2}         \\
\lkimi Kimi-K3                          & Visual & \bov{21.4}          & \bov{11.7}          & \bby{50.8}         & \bby{35.3}        & \ball{26.5}         \\
\lqwen Qwen3.8-Max                      & Visual & \bov{19.7}          & \bov{10.5}           & \bby{46.8}               & \bby{34.7}              & \ball{24.6}                \\

\lqwen Qwen3.5-Plus                     & Visual & \bov{17.2}          & \bov{8.7}           & \bby{46.0}               & \bby{30.0}              & \ball{22.5}                \\
\lnvidia Nemotron-3-Nano-Omni           & Omni   & \bov{25.6}          & \bov{22.9}          & \bby{36.0}               & \bby{13.3}              & \ball{25.4}                \\
\lqwen Qwen3-Omni-30B-A3B-Thinking      & Omni   & \bov{31.9}          & \bov{25.9}          & \bby{46.4}          & \bby{32.0}        & \ball{33.2}         \\

\tbgroupcf{7}{Ours} \\
\tboursf \lqwen OmniReasoning-30B-A3B    & Omni   & \bov{48.0}          & \bov{30.1}          & \bby{58.8} & \bby{32.7}        & \ball{42.5}         \\
\bottomrule
\end{tabularx}
\vspace{-10pt}
\end{table}

% ===== END inlined from tables/main.tex =====
\label{sec:experiments}
Our experiments answer three questions: (1) Can OmniQA training data improve audio-visual joint reasoning? (2) Can the benefits transfer to general video understanding? (3) Does the cross-modality interaction guidance of MFSD improve over existing reinforcement learning methods?

% ===== BEGIN inlined from tables/public.tex =====
\begin{table}[!htbp]
\caption{\textbf{Accuracy on public video benchmarks (\%)}~\citep{omnivideobench,lvomnibench,videommmu,videommev2}. Our model is initialized by Qwen3-Omni-30B-A3B-Thinking.}
\label{tab:public}
\centering
\footnotesize
\setlength{\tabcolsep}{0pt}
\begin{tabularx}{\linewidth}{@{}>{\raggedright\arraybackslash}X C{\dimexpr1.9cm+6pt\relax} C{\dimexpr1.57cm+6pt\relax} C{\dimexpr1.57cm+6pt\relax} C{\dimexpr1.65cm+6pt\relax}@{}}
\toprule
\textbf{Model} & \shead{OmniVideoBench} & \shead{LVOmniBench} &
\shead{Video-MMMU} & \shead{Video-MME-v2} \\
\midrule
\tbgroupcf{5}{Proprietary models} \\
\lgemini Gemini-3.5-Flash                 & \bov{66.5} & \bov{66.9}      & \bov{86.1}           & \bov{67.6} \\
\lgemini Gemini-3.1-Pro                   & \bov{60.7} & \bov{60.9}      & \bov{84.3}           & \bov{53.4}\\
\lqwen Qwen3.5-Omni-Plus                  & \bov{53.8} & \bov{53.2}      & \bov{84.6}            & \bov{47.9} \\
\lmeta Muse-Spark-1.1                     & \bov{55.8} & \bov{52.3}      & \bov{85.2}      & \bov{59.5} \\
\lxiaomi MiMo-V2.5 & \bov{51.6} & \bov{52.8} & \bov{79.1} & \bov{51.8} \\

\tbgroupcf{5}{Open-source models} \\
\lnvidia Nemotron-3-Nano-Omni & \bov{42.1} & \bov{32.1} & \bov{73.1} & \bov{38.4} \\
% \lqwen Qwen3-Omni-30B-A3B-Thinking        & \bov{37.2} & \bov{37.7} & \bov{68.3}      & \bov{36.5} \\
\lqwen Qwen2.5-Omni-7B                    & \bov{29.3} & \bov{32.0}        & \bov{53.9}             & \bov{25.3} \\
\lqwen Qwen3-Omni-30B-A3B-Thinking        & \bov{37.2} & \bov{37.7} & \bov{68.3}      & \bov{36.5} \\
\tbgroupcf{5}{Ours} \\
\tboursf \lqwen OmniReasoning-30B-A3B & \bov{50.0} & \bov{44.7} & \bov{74.4} & \bov{44.0} \\
\bottomrule
\end{tabularx}
\vspace{-10pt}
\end{table}

% ===== END inlined from tables/public.tex =====

\vspace{-3pt}
\noindent\textbf{Setup.}
We initialize our model OmniReasoning-30B-A3B from Qwen3-Omni-30B-A3B-Thinking~\citep{qwen3omni} and perform SFT on all 112,463 examples. GRPO, RLSD, and MFSD start from the same SFT checkpoint with the same RL budget: 150 steps, 8 responses per question, and a global batch of 256 generated responses per update. MFSD uses $\lambda=\beta=0.5$ after warm-up, and the audio and visual encoders remain frozen throughout. At evaluation time, reference evidence chains are withheld, and open-ended questions require producing the answer directly rather than selecting among options. We compare against recent proprietary omni-modal models~\citep{gemini3models,doubao2,qwen35omni,musespark} and open models~\citep{mimov25,qwen3omni,kimik3,qwen38blog,qwen35blog,nemotron3omni}, which together span omni-modal and vision-only inputs (Table~\ref{tab:main}). Appendix~\ref{app:impl} provides training and evaluation details. % Appendix~\ref{app:impl} gives the full training and evaluation settings.

\subsection{Audio-Visual Joint Reasoning and Generalization}

\vspace{-3pt}
\noindent\textbf{Results on OmniReasoningBench.}
On OmniReasoningBench (Table~\ref{tab:main}), our model reaches 42.5\% overall, 9.3 points above its base model. The gains are consistent across both settings and both question formats: 16.1 and 4.2 points on the multiple-choice and open-ended questions of \emph{reasoning over video}, and 12.4 and 0.7 points on those of \emph{reasoning beyond video}. Nevertheless, a clear gap to the strongest proprietary model (57.6\% for Gemini-3.7-Flash) remains, indicating that audio-visual joint reasoning is far from solved.

\vspace{-3pt}
\noindent\textbf{Generalization to public benchmarks.}
The improvement is not limited to OmniReasoningBench. Our model reaches 50.0\% on OmniVideoBench, 12.8 points above the base model, with gains spanning every audio-type and duration group (Table~\ref{tab:ovb}, Appendix~\ref{app:ovb-breakdown}). It also improves by 6.1--7.5 points on the other three general and long-video benchmarks (Table~\ref{tab:public}). Training on evidence-grounded joint-reasoning data thus strengthens broader video understanding rather than overfitting.

% ===== BEGIN inlined from tables/ovb.tex =====
\begin{table}[!t]
\caption{\textbf{OmniVideoBench accuracy (\%) by audio type and video duration}~\citep{omnivideobench}.}
\label{tab:ovb}
\centering
\vspace{2pt}
\small
\setlength{\tabcolsep}{0pt}
\begin{tabularx}{\linewidth}{@{}>{\raggedright\arraybackslash}X C{\dimexpr0.89cm+6pt\relax} C{\dimexpr0.89cm+6pt\relax} C{\dimexpr0.89cm+6pt\relax} C{\dimexpr0.89cm+6pt\relax} C{\dimexpr0.89cm+6pt\relax} C{\dimexpr0.89cm+6pt\relax} C{\dimexpr0.89cm+6pt\relax} C{\dimexpr0.89cm+6pt\relax}@{}}
\toprule
 & \multicolumn{3}{c}{\textbf{Audio Type}} &
 \multicolumn{4}{c}{\textbf{Video Duration} (min)} & \\
\cmidrule(lr){2-4}\cmidrule(lr){5-8}
\textbf{Model} & \shead{Music} & \shead{Sound} & \shead{Speech} &
\shead{(0,1]} & \shead{(1,5]} & \shead{(5,10]} & \shead{(10,30]} & \shead{Avg.} \\
\midrule
\tbgroupcf{9}{Proprietary models} \\

\lgemini Gemini-3.5-Flash            & \bov{56.9}& \bov{69.3} & \bov{70.1} & \bov{69.7} & \bov{68.0} & \bov{63.9} & \bov{64.8} & \bov{66.5} \\
\lgemini Gemini-3.1-Pro              & \bov{51.1} & \bov{65.8} & \bov{61.9} & \bov{67.3} & \bov{64.2} & \bov{59.5} & \bov{53.0} & \bov{60.7} \\
\lgemini Gemini-3-Pro                & \bov{56.2} & \bov{54.1} & \bov{55.7} & \bov{61.0} & \bov{56.4} & \bov{52.9} & \bov{52.5} & \bov{55.5} \\

\lqwen Qwen3.5-Omni-Plus             & \bov{42.3} & \bov{64.9} & \bov{53.3} & \bov{55.8} & \bov{55.2} & \bov{50.7} & \bov{53.4} & \bov{53.8} \\
\lxiaomi MiMo-V2.5 & \bov{46.2} & \bov{47.6} & \bov{53.0} & \bov{56.6} & \bov{57.4} & \bov{47.6} & \bov{44.3} & \bov{51.6} \\
\lgemini Gemini-2.0-Flash            & \bov{29.7} & \bov{40.3} & \bov{43.2} & \bov{49.4} & \bov{43.2} & \bov{41.1} & \bov{34.9} & \bov{41.5} \\
\tbgroupcf{9}{Open-source models} \\
\lnvidia Nemotron-3-Nano-Omni & \bov{33.0} & \bov{42.9} & \bov{43.0} & \bov{48.8} & \bov{41.7} & \bov{41.5} & \bov{38.9} & \bov{42.1} \\
% \lqwen Qwen3-Omni-30B-A3B-Thinking   & \bov{26.4} & \bov{37.2} & \bov{38.5} & \bov{46.8} & \bov{35.6} & \bov{35.5} & \bov{35.2} & \bov{37.2} \\
\lqwen Qwen2.5-Omni-7B               & \bov{23.1} & \bov{25.3} & \bov{30.7} & \bov{41.6} & \bov{27.4} & \bov{25.3} & \bov{26.7} & \bov{29.3} \\
\lqwen Qwen3-Omni-30B-A3B-Thinking   & \bov{26.4} & \bov{37.2} & \bov{38.5} & \bov{46.8} & \bov{35.6} & \bov{35.5} & \bov{35.2} & \bov{37.2} \\
\tbgroupcf{9}{Ours} \\
\tboursf \lqwen OmniReasoning-30B-A3B & \bov{39.6} & \bov{46.3} & \bov{52.0} & \bov{52.4} & \bov{49.9} & \bov{49.8} & \bov{48.9} & \bov{50.0} \\
\bottomrule
\end{tabularx}
\end{table}

% ===== END inlined from tables/ovb.tex =====

% ===== BEGIN inlined from tables/rl.tex =====
\begin{table}[t]
\caption{\textbf{Comparison of RL methods (accuracy, \%).} GRPO, RLSD, and MFSD start from the same SFT checkpoint and use the same RL training budget.}
\label{tab:rl}
\centering\footnotesize
\setlength{\tabcolsep}{2.7pt}
\begin{tabularx}{\linewidth}{l *{5}{>{\centering\arraybackslash}X}}
\toprule
Method & \shead{\shortstack{OmniReasoning\\Bench}} & \shead{\shortstack{OmniVideo\\Bench}} & \shead{LVOmniBench} & \shead{Video-MME-v2} & \shead{Video-MMMU} \\
\midrule
Qwen3-Omni-30B-A3B-Thinking & \bov{33.2} & \bov{37.2} & \bov{37.7} & \bov{36.5} & \bov{68.3} \\
+ SFT & \bov{36.1} & \bov{47.0} & \bov{40.2} & \bov{40.3} & \bov{73.7} \\ \midrule
+ GRPO & \bov{38.3} & \bov{44.8} & \bov{39.8} & \bov{39.1} & \bov{72.9} \\
+ RLSD & \bov{39.2} & \bov{45.4} & \bov{40.1} & \bov{39.8} & \bov{72.0} \\
\tbours + MFSD & \bov{42.5} & \bov{50.0} & \bov{44.7} & \bov{44.0} & \bov{74.4} \\
\bottomrule
\end{tabularx}
\vspace{-10pt}
\end{table}
% ===== END inlined from tables/rl.tex =====

\subsection{Contribution of MFSD}
\label{sec:mfsd}

Finally, we validate whether the cross-modality interaction guidance in MFSD contributes beyond existing RL methods, comparing GRPO, RLSD, and MFSD under an identical budget from the same SFT checkpoint (Table~\ref{tab:rl}).
On OmniReasoningBench, GRPO and RLSD outperform SFT by 2.2 and 3.1 points, respectively, whereas MFSD improves by 6.4 points.
% Such contrast is sharper on existing public benchmarks: GRPO and RLSD fall below the SFT checkpoint on all four of them, while MFSD improves over SFT by 0.7--4.5 points on each.
On existing benchmarks, GRPO and RLSD even fall below SFT, while MFSD consistently outperforms SFT by 0.7--4.5 points.
These results indicate that outcome rewards and coarse privileged clues alone fit the targeted pattern but hurt general capability. By factoring credit by modality-specific evidence and explicitly rewarding cross-modality interaction, MFSD improves both joint reasoning and general video understanding.

% ===== END inlined from sections/experiments.tex =====
% ===== BEGIN inlined from sections/conclusion.tex =====
\section{Conclusion}
\label{sec:conclusion}
% OmniReasoning pushes the limits of audio-visual joint reasoning by connecting evaluation, data construction, and learning through explicit audio-visual evidence dependencies. 
% Specifically, OmniReasoningBench evaluates \emph{reasoning over video} and \emph{reasoning beyond video}; OmniQA supplies evidence-grounded training data; and MFSD combines joint-clue support with cross-modal interaction for token-level credit assignment. OmniReasoning-30B-A3B improves over its base model by 9.3 points on OmniReasoningBench and 12.8 points on OmniVideoBench, and MFSD outperforms SFT, GRPO, and RLSD on all five video benchmarks.
We present \textbf{OmniReasoning}, a unified framework for evaluating and improving audio-visual joint reasoning in omni-modal large language models (Omni-LLMs). We first introduce \textbf{OmniReasoningBench}, whose questions require complementary audio and visual evidence, providing a more stringent evaluation of genuine cross-modal reasoning. We then develop the \textbf{OmniQA} data engine to generate evidence-grounded reasoning data with timestamped modality-specific clues and dependency chains, yielding \textbf{OmniReasoning-SFT-112K} and \textbf{OmniReasoning-RL-19K}. Finally, we propose \textbf{Modality-Factored Self-Distillation (MFSD)}, which exploits these structured clues to provide modality-aware, token-level credit assignment for reinforcement learning.
Training with OmniQA and MFSD yields \textbf{OmniReasoning-30B-A3B}, substantially improving audio-visual joint reasoning on existing and proposed benchmarks, while also benefiting general and long-video understanding. 
These results highlight the importance of explicitly modeling cross-modal evidence dependencies in evaluation and learning alike; we hope OmniReasoning serves as a foundation for future research on omni-modal joint reasoning.
% ===== END inlined from sections/conclusion.tex =====
\FloatBarrier
\label{endofmaintext}
\clearpage
% ===== BEGIN inlined from sections/statements.tex =====
\subsection*{AI use statement}
In this work, we used generative AI tools for generating synthetic datasets. Specifically, generative
AI assisted in constructing the OmniQA dataset. We did not use generative AI tools for the other tasks with required disclosure, including developing
the core scientific contributions, deriving mathematical proof, designing the research methodology or experiments, conducting
experimental analysis, selecting or verifying references, or making scientific claims; the remaining
required-disclosure tasks are not applicable to this work.
Additionally, we used generative AI tools to aid or polish writing, including language editing, readability improvements, and suggestions of alternative wording. We
have reviewed all AI-assisted work. All AI-assisted text was carefully edited and approved by the
authors; synthetic data were checked using the validation procedures described in the paper; and all technical content, claims, citations, and
reported results were manually verified. We take responsibility for the final content of this work, including text, claims, code, data, annotations, and other artifacts produced with the aid of generative
AI.

\subsection*{Reproducibility statement}
We provide complete training details, including data generation pipeline, model configurations, optimization objectives, hyperparameters, and implementation settings. These details, together with the evaluation protocols, are documented in the main text and appendices to enable reproduction of our results.
% ===== END inlined from sections/statements.tex =====
% ===== BEGIN inlined from iclr2027_conference.bbl (regenerated from .bib) =====

% ===== END inlined from iclr2027_conference.bbl =====
\clearpage
\appendix
% ===== BEGIN inlined from sections/appendix.tex =====
\section{Implementation and Evaluation Details}
\label{app:impl}
\paragraph{Supervised Fine-Tuning.}
SFT uses the full 112,463 SFT training corpus, with learning rate $7\times10^{-6}$, microbatch size 2, and global batch size 128. The configuration samples video at 2 fps with at most 256 frames and 200,704 pixels per frame. Visual and audio encoders are frozen.

\paragraph{RL configuration.}
The GRPO, RLSD and MFSD runs are trained for 150 RL steps using 64 GPUs, a global batch size of 256, 8 rollout trajectories per question, and learning rate $10^{-6}$. The global batch size counts generated responses per optimizer update. The binary verifier extracts the selected option and compares it to the reference letter. The KL coefficient is $10^{-3}$. MFSD uses text clue factorization, actor weights for every scoring view, sequence-wise centering, $\epsilon_w=0.2$, and EMA coefficient 0.9 for $\kappa$ with maximum 10. The coefficients $\lambda$ and $\beta$ warm up to 0.5 over 10 and 20 steps. No additional global advantage whitening is enabled. Constant-reward groups receive zero reward advantage.

\paragraph{Inference and scoring.}
Evaluation on OmniReasoningBench, OmniVideoBench, LVOmniBench, Video-MME-v2, and Video-MMMU uses 2 fps video sampling with at most 512 frames and no audio truncation.

\subsection{Policy objective}
\label{app:objective}
We substitute $\widehat A$ into the clipped policy objective, with $\rho_{i,t}=\pi_\theta(y^{(i)}_t\mid x,y^{(i)}_{<t})/\pi_{\mathrm{old}}(y^{(i)}_t\mid x,y^{(i)}_{<t})$:
\begin{equation}
\mathcal L_{\mathrm{MFSD}}=-\mathbb E\!\left[
\left\langle\min\!\left(\rho_{i,t}\widehat A_{i,t},
\mathrm{clip}(\rho_{i,t},1-\epsilon_p,1+\epsilon_p)\widehat A_{i,t}\right)\right\rangle_m\right]
+\eta\mathcal L_{\mathrm{KL}}.
\label{eq:objective}
\end{equation}
Here $\langle\cdot\rangle_m$ denotes the masked response-token reduction and $\mathcal L_{\mathrm{KL}}$ anchors the actor to the SFT reference policy. Teacher scores and advantage weights are detached. MFSD adds three scoring passes, reuses media features, and requires no extra rollouts or teacher parameters.

\subsection{MFSD implementation and component analysis}
\label{app:mfsd-details}

\paragraph{Adaptive scale calibration.}
The joint-clue gain $\Delta^{AV}$ and centered interaction $\widetilde S$ can have different magnitudes. MFSD therefore calibrates their scales before mixing them in Equation~\ref{eq:credit}. At update $b$, let $\widehat\sigma_b$ denote the standard deviation over valid response tokens in the batch; padding and non-response positions are excluded. We compute
\begin{equation}
\begin{aligned}
q_b&=\min\!\left(
\frac{\widehat\sigma_b(\Delta^{AV})}
{\widehat\sigma_b(\widetilde S)+10^{-6}},\,10\right),\\
\kappa_b&=0.9\,\kappa_{b-1}+0.1\,q_b.
\end{aligned}
\label{eq:kappa-calibration}
\end{equation}
This is the explicit update corresponding to $\mathrm{EMA}_{0.9}$ of the capped standard-deviation ratio. The $10^{-6}$ term stabilizes the denominator, the cap limits the target scale when interaction variance is small, and the EMA smooths batch-to-batch fluctuations. Calibration approximately aligns the two signal scales; clipping and smoothing need not yield exact variance matching. The same $\kappa_b$ is shared across response tokens in that update.

\paragraph{Centering and clue availability.}
Centering gives $\sum_t m_{i,t}\widetilde S_{i,t}=0$, so the interaction term redistributes credit within a response without changing the sum of its token scores. It also removes response-wide additive offsets in $S$ (Appendix~\ref{app:mfsd-proofs}). The joint-clue term retains useful evidence even when it comes from one modality: if $\Delta^{AV}_{i,t}=\Delta^A_{i,t}>0$ and $\Delta^V_{i,t}=0$, then $S_{i,t}=0$ despite positive clue support. When either clue set is absent, the gate sets $\beta_i=0$ and uses joint-clue-only weighting.

\paragraph{Mixing and bounded weighting.}
The coefficient $\beta$ balances joint support and calibrated interaction, whereas $\lambda$ controls how strongly the resulting token weights modify the outcome advantage. They warm up to $0.5$ over 20 and 10 steps, respectively. Setting $\beta=0$ removes the interaction term; setting $\lambda=0$ recovers the unweighted outcome advantage. These are algebraic limits of the weighting rule. 

\paragraph{Scoring and gradient separation.}
All four views use the same actor weights, original media, sampled response, and prefixes; only the privileged text clues vary. The three clue-conditioned views reuse the rollout and media features, with no additional response generation or separate teacher parameters. Evidence scores and advantage weights are detached, so their clue-conditioned derivatives do not enter the policy gradient.

\subsection{Benchmark evaluation protocols}
\label{app:protocol}
OmniReasoningBench contains 750 \emph{reasoning over video} and 400 \emph{reasoning beyond video} questions. \emph{Reasoning over video} has 375 multiple-choice and 375 open-numeric questions. \emph{Reasoning beyond video} has 250 multiple-choice and 150 open-numeric questions, with 288 video-plus-figure and 112 video-plus-text inputs.

At evaluation time, models receive the source video, question, and any additional figure; reference evidence chains are withheld. For \emph{reasoning beyond video} questions, the additional figure or table is part of the question input: modality ablations hold it fixed and vary only the source video's audio and visual streams, so that ablation scores isolate the contribution of the video's own modalities rather than mixing in the removal of the new figure.

\begin{figure}[!htbp]
\centering
\includegraphics[width=\linewidth]{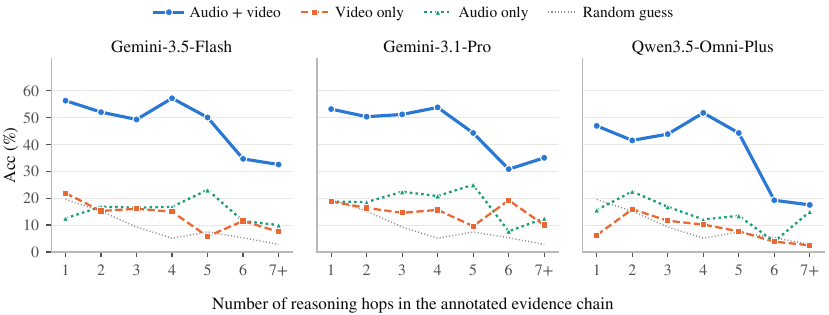}
\caption{\textbf{Modality ablations on \emph{reasoning over video}.} Audio-visual joint inputs outperform either modality alone for all three advanced Omni-LLMs in every evidence-chain-length group, supporting the importance of audio-visual evidence integration.}
\label{fig:ozb-hops}
\end{figure}

\Needspace{28\baselineskip}
\subsection{Modality ablations}
\label{app:modality-ablations}
To isolate modality effects, we evaluate each model with joint, visual-only, and audio-only inputs (Figure~\ref{fig:ozb-hops}). Across the three evaluated Omni-LLMs, accuracy on \emph{reasoning over video} falls from 42.3--50.1\% with joint inputs to 10.9--14.9\% with video alone and 16.3--20.4\% with audio alone, supporting the importance of joint evidence.

We further compare Qwen3.5-Plus with visual inputs alone against Qwen3.5-Omni-Plus with joint audio-visual inputs, on four prior audio-visual benchmarks and on our \emph{reasoning over video} questions (Table~\ref{tab:visual-only}). Although these benchmarks study cross-modal understanding and report modality ablations, joint inputs add only 4.9--15.8 points over the visual-only model, and 44.9--68.8\% of their questions remain answerable without audio; on our multiple-choice and open-ended questions, the visual-only model drops to 17.2\% and 8.7\%, while joint inputs recover 42.7\% and 41.9\%.

% ===== BEGIN inlined from tables/visual_only.tex =====
\begin{table}[!htbp]
\caption{\textbf{Visual-only versus joint-input accuracy across audio-visual benchmarks.} Qwen3.5-Plus receives visual inputs alone, while Qwen3.5-Omni-Plus receives joint audio-visual inputs.}
\label{tab:visual-only}
\centering
\vspace{2pt}
\small
\setlength{\tabcolsep}{6pt}
\begin{tabular}{@{}lcc@{}}
\toprule
\textbf{Benchmark} & \textbf{Qwen3.5-Plus} & \textbf{Qwen3.5-Omni-Plus} \\
\midrule
WorldSense & 57.9 & 62.8 \\
Daily-Omni & 68.8 & 84.6 \\
OmniVideoBench & 44.9 & 53.8 \\
JointAVBench & 60.4 & 74.1 \\
\midrule
OmniReasoningBench, \emph{reasoning over video} (MCQ) & 17.2 & 42.7 \\
OmniReasoningBench, \emph{reasoning over video} (OE) & 8.7 & 41.9 \\
\bottomrule
\end{tabular}
\end{table}
% ===== END inlined from tables/visual_only.tex =====

To test whether the gains of joint inputs merely combine questions that either modality can answer alone, we perform a paired analysis on the same \emph{reasoning over video} questions (Table~\ref{tab:paired-modality}). For each model, we compare joint audio-visual accuracy with a single-modality oracle that credits a question when either the audio-only or the visual-only run answers it correctly. Joint inputs exceed this oracle by 9.9--26.7 percentage points across the three Omni-LLMs and both question types. For Qwen3.5-Omni-Plus, 212 questions are answered correctly only with joint inputs, while 77 questions are answered by the single-modality oracle but missed with joint inputs. These results indicate that the benefit of audio-visual joint reasoning cannot be explained by the union of single-modality information gains.

% ===== BEGIN inlined from tables/paired_modality.tex =====
\begin{table}[!htbp]
\caption{\textbf{Paired modality analysis on \emph{reasoning over video}.} Each row compares joint audio-visual accuracy (AV) with audio-only (A), video-only (V), and a single-modality oracle (A$\cup$V) that credits a question when either single-modality run answers it correctly, over the same paired questions. Joint inputs exceed the oracle in every setting, indicating that the gains of joint evidence cannot be explained by the union of single-modality successes.}
\label{tab:paired-modality}
\centering
\vspace{2pt}
\small
\setlength{\tabcolsep}{5pt}
\begin{tabular}{@{}llccccc@{}}
\toprule
\textbf{Model} & \textbf{Type} & \textbf{AV} & \textbf{V} & \textbf{A} & \textbf{A$\cup$V oracle} & \textbf{AV $-$ oracle} \\
\midrule
Qwen3.5-Omni-Plus & MCQ & 42.67 & 14.93 & 22.67 & 32.27 & +10.40 \\
Qwen3.5-Omni-Plus & OE  & 41.87 & 6.93  & 9.87  & 16.27 & +25.60 \\
\midrule
Gemini-3.1-Pro    & MCQ & 51.73 & 22.40 & 26.13 & 41.87 & +9.87 \\
Gemini-3.1-Pro    & OE  & 46.93 & 7.47  & 14.67 & 20.80 & +26.13 \\
\midrule
Gemini-3.5-Flash  & MCQ & 53.33 & 19.73 & 20.53 & 37.07 & +16.27 \\
Gemini-3.5-Flash  & OE  & 46.93 & 9.07  & 12.27 & 20.27 & +26.67 \\
\bottomrule
\end{tabular}
\end{table}

% ===== END inlined from tables/paired_modality.tex =====

\subsection{OmniVideoBench breakdown}
\label{app:ovb-breakdown}
Table~\ref{tab:ovb} reports OmniVideoBench accuracy by audio type and video duration. OmniReasoning-30B-A3B improves over its base model in every audio-type and duration group.

\FloatBarrier

\FloatBarrier
\Needspace{12\baselineskip}
% Appendix section for the parent manuscript's preamble and bibliography.
\clearpage
\section{Mathematical Properties of MFSD}
\label{app:mfsd-proofs}
\label{app:mfsd-properties}

We derive the information-theoretic interpretation of MFSD and its structural guarantees, following the evidence-ratio analysis of RLSD~\citep{rlsd}. All four views retain the original audio-video input $x$ and the same sampled prefix $y_{<t}$; only the privileged text clues $r_A$, $r_V$, and $r_{AV}=(r_A,r_V)$ vary.

\subsection{Bayesian interpretation of the interaction score}
\label{app:mfsd-interaction}

\begin{assumption}[Joint consistency]
\label{asm:joint}
For fixed $x$, a single joint distribution $P(Y,R_A,R_V\mid x)$ satisfies
\begin{equation}
\pi_\theta(y_t\mid x,r_M,y_{<t})
=P(y_t\mid x,r_M,y_{<t}),
\qquad M\in\{\varnothing,A,V,AV\},
\label{eq:mfsd-app-consistency}
\end{equation}
and every conditioning event used in the derivations below has positive probability under $P(\cdot\mid x)$, so that all likelihood ratios are well defined.
\end{assumption}

\begin{proposition}[Interaction as a conditional PMI increment]
\label{prop:synergy}
Under Assumption~\ref{asm:joint}, define the conditional pointwise mutual information of the realized clue pair by
\begin{equation}
\mathcal I_t
=\log\frac{P(r_A,r_V\mid x,y_{\leq t})}
{P(r_A\mid x,y_{\leq t})P(r_V\mid x,y_{\leq t})},
\label{eq:mfsd-app-pmi}
\end{equation}
with $y_{\leq0}=\varnothing$. The interaction score in Equation~\ref{eq:synergy} satisfies
\begin{equation}
S_t=\mathcal I_t-\mathcal I_{t-1},
\qquad
\sum_{t=1}^{T}S_t=\mathcal I_T-\mathcal I_0.
\label{eq:mfsd-app-pmi-increment}
\end{equation}
\end{proposition}

\begin{proof}
For each $M\in\{A,V,AV\}$, Bayes' rule gives
\begin{equation}
P(y_t\mid x,r_M,y_{<t})
=\frac{P(r_M\mid x,y_{\leq t})P(y_t\mid x,y_{<t})}
{P(r_M\mid x,y_{<t})}.
\end{equation}
Dividing by $P(y_t\mid x,y_{<t})$, taking logarithms, and applying Assumption~\ref{asm:joint} yields
\begin{equation}
\Delta_t^M
=\log\frac{\pi_\theta(y_t\mid x,r_M,y_{<t})}{\pi_\theta(y_t\mid x,y_{<t})}
=\log\frac{P(r_M\mid x,y_{\leq t})}{P(r_M\mid x,y_{<t})}.
\label{eq:belief}
\end{equation}
Stop-gradient leaves this numerical identity unchanged. Substituting the three gains into $S_t$ and collecting terms at each prefix,
\begin{equation}
\begin{aligned}
S_t
&=\Delta_t^{AV}-\Delta_t^A-\Delta_t^V\\
&=\log\frac{P(r_A,r_V\mid x,y_{\leq t})}
{P(r_A\mid x,y_{\leq t})P(r_V\mid x,y_{\leq t})}
-\log\frac{P(r_A,r_V\mid x,y_{<t})}
{P(r_A\mid x,y_{<t})P(r_V\mid x,y_{<t})}\\
&=\mathcal I_t-\mathcal I_{t-1}.
\end{aligned}
\end{equation}
Summing over $t$ cancels all intermediate terms, giving $\sum_{t=1}^{T}S_t=\mathcal I_T-\mathcal I_0$.
\end{proof}

\paragraph{Interpretation.}
The sign of $S_t$ compares joint support with the two individual gains:
\begin{equation}
\exp(S_t)
=\frac{\pi_\theta(y_t\mid x,r_A,r_V,y_{<t})\,\pi_\theta(y_t\mid x,y_{<t})}
{\pi_\theta(y_t\mid x,r_A,y_{<t})\,\pi_\theta(y_t\mid x,r_V,y_{<t})}.
\label{eq:mfsd-app-cross-ratio}
\end{equation}
Thus $S_t>0$ indicates joint likelihood support beyond the product of the individual likelihood ratios; under Assumption~\ref{asm:joint}, it increases the clue pair's conditional PMI. We stress that this is a non-additivity comparison in log-likelihood, not a statement of positive joint support: $\Delta_t^{AV}<0$ and $S_t>0$ can hold simultaneously when both individual gains are more negative than the joint gain. Positive interaction, positive joint support, and reasoning correctness are therefore distinct notions; MFSD keeps the joint-clue gain $\Delta_t^{AV}$ and the interaction $S_t$ as separate branches in Equation~\ref{eq:credit} precisely so that each captures its own notion. In contrast, if $\Delta_t^{AV}=\Delta_t^A$ and $\Delta_t^V=0$, then $S_t=0$ despite positive audio-clue support. This motivates interaction-based credit assignment.

\subsection{Sequence centering}
\label{app:mfsd-centering}

\begin{proposition}[Zero-sum and offset-invariant interaction]
\label{prop:mfsd-centering}
Let $m_t\in\{0,1\}$, $N=\sum_t m_t>0$, and define
\begin{equation}
\bar S=\frac1N\sum_t m_tS_t,
\qquad \widetilde S_t=m_t(S_t-\bar S).
\end{equation}
Then $\sum_t\widetilde S_t=0$. Adding a constant $b$ to every $S_t$ leaves $\widetilde S_t$ unchanged, and centering preserves the ordering of active-token scores.
\end{proposition}

\begin{proof}
By the definition of $\bar S$,
\begin{equation}
\sum_t\widetilde S_t=\sum_t m_tS_t-N\bar S=0.
\end{equation}
For $S'_t=S_t+b$, the mean is $\bar S'=\bar S+b$, hence
$m_t(S'_t-\bar S')=m_t(S_t-\bar S)$. For active positions $t,u$,
$\widetilde S_t-\widetilde S_u=S_t-S_u$, proving order preservation.
\end{proof}

Consequently, centering removes response-wide additive offsets while retaining relative token credit. Since $\beta_i$ and $\kappa$ are constant within each response, Equation~\ref{eq:credit} gives
\begin{equation}
\sum_t m_tg_t=(1-\beta_i)\sum_t m_t\Delta_t^{AV}.
\end{equation}
The interaction therefore contributes zero to the mean log-credit signal. This zero-sum property holds for the linear score $g_t$ only: the exponentiation, clipping, and interpolation in Equation~\ref{eq:ahat} are nonlinear and do not preserve it, so centering does not conserve the total advantage mass $\sum_t m_t\widehat A_t$. Centering also removes any response-wide constant component of $S_t$: a uniformly positive interaction score is fully absorbed, and a token with $S_t<0$ still receives positive $\widetilde S_t$ when it lies above the response mean. MFSD thus optimizes relative interaction priority within each response rather than rewarding every token with positive raw interaction.

\subsection{Bounded reweighting and gradient structure}
\label{app:mfsd-bounds}

Write Equation~\ref{eq:ahat} as $\widehat A_t=A c_t$, where the detached multiplier is
\begin{equation}
c_t=(1-\lambda)+\lambda\operatorname{clip}
\!\left(e^{\operatorname{sign}(A)g_t},1-\epsilon_w,1+\epsilon_w\right).
\end{equation}

\begin{proposition}[Bounded and sign-preserving credit]
\label{prop:mfsd-bounds}
For $0\leq\lambda\leq1$ and $0\leq\epsilon_w<1$,
\begin{equation}
1-\lambda\epsilon_w\leq c_t\leq1+\lambda\epsilon_w,
\qquad
\operatorname{sign}(\widehat A_t)=\operatorname{sign}(A),
\qquad
|\widehat A_t-A|\leq\lambda\epsilon_w|A|.
\end{equation}
\end{proposition}

\begin{proof}
The clipped exponential lies in $[1-\epsilon_w,1+\epsilon_w]$. Multiplying by $\lambda$ and adding $1-\lambda$ gives the stated bounds on $c_t$. Since $1-\lambda\epsilon_w>0$, multiplication by $c_t$ preserves the sign of $A$, and $|c_t-1|\leq\lambda\epsilon_w$ gives the deviation bound.
\end{proof}

\paragraph{Policy-gradient structure.}
For the clipped surrogate in Appendix~\ref{app:objective}, define
\begin{equation}
\begin{aligned}
\rho_t&=\frac{\pi_\theta(y_t\mid x,y_{<t})}{\pi_{\mathrm{old}}(y_t\mid x,y_{<t})},\\
j_t(A)&=\min\!\left\{\rho_t A,
\operatorname{clip}(\rho_t,1-\epsilon_p,1+\epsilon_p)A\right\}.
\end{aligned}
\end{equation}
Positivity and stop-gradient imply
\begin{equation}
j_t(\widehat A_t)=c_tj_t(A),
\qquad
\nabla_\theta j_t(\widehat A_t)=c_t\nabla_\theta j_t(A),
\end{equation}
where the derivative exists. In an unclipped branch, the token contribution is
\begin{equation}
c_t A\rho_t\nabla_\theta\log\pi_\theta(y_t\mid x,y_{<t}).
\end{equation}
Thus privileged clues rescale each student-token gradient without reversing it or introducing derivatives through clue-conditioned scores.

\FloatBarrier

% ===== BEGIN inlined from sections/appendix_protocol.tex =====

\section{Reward and Evaluation Protocols}
\label{app:reward}
All RL rewards are binary over the 18,991 human-validated questions of OmniReasoning-RL-19K. Multiple-choice answers are verified by exact option-letter match (unparseable responses score 0); open-numeric answers are judged by a DeepSeek-V4-Flash model. The same judge and prompt are used for RL rewards and for all open-ended evaluation.

% ===== END inlined from sections/appendix_protocol.tex =====

% ===== BEGIN inlined from sections/appendix_comparisons.tex =====
\clearpage
\section{Qualitative Comparison of Model Responses}
\label{app:model-cases}
We compare complete recorded responses from Qwen3-Omni-30B-A3B-Thinking and OmniReasoning-30B-A3B on OmniVideoBench, LVOmniBench, Video-MME-v2, OmniReasoningBench, and Video-MMMU. The cases cover audio-visual temporal alignment, speaker and object reference, motion comparison, and application of video-derived knowledge to new inputs. Paired responses use identical questions and answer options. Each card gives the question and reference answer, followed by both original responses, including their full thinking text and final answers. Boldface highlights the evidence and decisions discussed in the case.

\Needspace{26\baselineskip}
\subsection{OmniVideoBench: Different descriptions of the scene and inscriptions}
\label{case:model0}
\begin{tcolorbox}[casecard,title={Question and reference answer}]
\textbf{Question.} After simulating the lunar landing mission, what is the content of letters behind them?\par\smallskip
\textbf{Options.} A. United States B. Omega and United Statese C. Omega and Apollo D. Apollo and United States\par\smallskip
\textbf{Reference answer: A.}\par\smallskip
\end{tcolorbox}
The base response describes a scene around 05:30 and reports two inscriptions, selecting D. The recorded post-trained response describes a later scene and reports only UNITED STATES, selecting the reference answer A. These are claims made in the responses; the timestamps and inscriptions have not been independently verified in this comparison.\par\medskip
\Needspace{8\baselineskip}
\begin{tcolorbox}[casecard,breakable,colframe=red!50!black,colbacktitle=red!5,title={Qwen3-Omni-30B-A3B-Thinking},title after break={Qwen3-Omni-30B-A3B-Thinking (continued)}]
\textless{}think\textgreater{}
Got it, let's tackle the problem. First, the user is asking for the specific content of letters visible behind characters after simulating the lunar landing mission in the provided video, requiring selection from the given multiple-choice options. Then, analyzing the video content reveals \textbf{the critical moment occurs around 05:30 during the mission simulation scene where astronauts are in a lunar module simulation}; at 05:33, the camera shows a control room background with \textbf{two distinct signs: one reading 'APOLLO' (clearly visible above 'UNITED STATES') and another displaying 'UNITED STATES'} (with the Apollo mission name prominently featured on the same sign), confirming both letters appear together as 'APOLLO' and 'UNITED STATES'. Next, evaluating the options shows A incorrectly lists only 'United States', B includes 'Omega' which appears later but not in this scene, C incorrectly states 'Omega and Apollo', while D accurately matches 'Apollo and United States' based on the verified visual evidence. Finally, the correct answer is D.
\textless{}/think\textgreater{}
\par\smallskip
D
\end{tcolorbox}
\Needspace{8\baselineskip}
\begin{tcolorbox}[casecard,breakable,colframe=teal!50!black,colbacktitle=teal!5,title={OmniReasoning-30B-A3B},title after break={OmniReasoning-30B-A3B (continued)}]
\textless{}think\textgreater{}
Got it, let's tackle the problem.
\par\smallskip
The user wants to identify the content of the letters displayed on the structure behind the astronauts after they complete their lunar landing simulation. The crux is determining which words are written on the gold truss structure behind them during the simulation scene.
\par\smallskip
First, let me scan the video for the segments relevant to the question:
\par\smallskip
- I need to find the segment where the astronauts are simulating the lunar landing mission and observe the letters written on the structure behind them.
\par\smallskip
- At 06:45, the astronauts are shown in the simulator room preparing for the landing simulation.
\par\smallskip
- At 06:50, the simulation begins and the screen displays "APOLLO 11 LANDING SIMULATION" as they walk out of the simulator room.
\par\smallskip
- At 06:54, the astronauts walk past a large gold truss structure that serves as a backdrop for the mock lunar landing.
\par\smallskip
- \textbf{At 06:56, the camera shows a close-up of this gold truss structure, revealing the letters "UNITED STATES" written in red on its surface.}
\par\smallskip
- At 07:00, the astronauts stand in front of the gold structure, which is decorated with a United States flag on the left and a NASA logo on the right.
\par\smallskip
- At 07:05, the astronauts pose for a photo in front of the gold truss structure, where the letters "UNITED STATES" are clearly visible behind them.
\par\smallskip
- Wait, let me re-read the question: it asks for the content of the letters behind them, and looking closely at the gold truss at 07:10, the letters spell out "UNITED STATES" in red, with no other brand names visible on this structure.
\par\smallskip
- If the letters behind them were "Omega and United States", the Omega logo would have to be painted directly on the gold truss structure behind the astronauts at 07:10, but that logo only appears on their wristwatches and in the promotional materials.
\par\smallskip
- The letters on the structure behind the astronauts are "UNITED STATES", which corresponds to option A.
\par\smallskip
Next, let me analyze the options:
\par\smallskip
- A. United States \textemdash{} This matches the letters painted on the gold truss behind the astronauts.
\par\smallskip
- B. Omega and United States \textemdash{} Omega is not written on the truss.
\par\smallskip
- C. Omega and Apollo \textemdash{} These words do not appear on the truss.
\par\smallskip
- \textbf{D. Apollo and United States \textemdash{} Apollo is not written on the truss.}
\par\smallskip
Finally, the correct option is A.
\textless{}/think\textgreater{}
\par\smallskip
A
\end{tcolorbox}
\Needspace{26\baselineskip}
\subsection{Video-MMMU: Different stated temperatures and answer choices}
\label{case:model1}
\begin{tcolorbox}[casecard,title={Question and reference answer}]
\textbf{Question.} Based on the example question of Carnot heat engine from 9:00 to 12:00, what is the Carnot efficiency with a hot reservoir temperature of 350\textdegree{}C while the other information remains unchanged?\par\smallskip
\textbf{Options.} A. 0.733 or 73.3\% B. 0.582 or 58.2\% C. 0.447 or 44.7\% D. 0.545 or 54.5\% E. 0.815 or 81.5\% F. 0.339 or 33.9\% G. 0.527 or 52.7\% H. 0.705 or 70.5\% I. 0.418 or 41.8\% J. 0.896 or 89.6\%\par\smallskip
\textbf{Reference answer: D.}\par\smallskip
\end{tcolorbox}
The base response states a cold-reservoir temperature of 20 degrees C and selects G. The recorded post-trained response states 10 degrees C and selects D, matching the reference answer. The comparison establishes the difference between their stated assumptions and answers; it does not independently verify which temperature or timestamp appears in the video.\par\medskip
\Needspace{8\baselineskip}
\begin{tcolorbox}[casecard,breakable,colframe=red!50!black,colbacktitle=red!5,title={Qwen3-Omni-30B-A3B-Thinking},title after break={Qwen3-Omni-30B-A3B-Thinking (continued)}]
\textless{}think\textgreater{}
Got it, let's tackle this problem. First, the user is asking for the Carnot efficiency with a hot reservoir temperature of 350\textdegree{}C while other information remains unchanged, specifically referencing the example question from 9:00 to 12:00 in the video. Then, the video example at 9:00 shows a power plant with a hot reservoir temperature of 600\textdegree{}C (converted to 873 K) and a cold reservoir temperature of 20\textdegree{}C (converted to 293 K), resulting in a Carnot efficiency of 66\% using the formula \ensuremath{\eta}\_carnot = 1 - T\_c/T\_h. The problem states that the hot reservoir temperature is changed to 350\textdegree{}C while \textbf{the cold reservoir temperature remains unchanged at 20\textdegree{}C}. Next, converting the new hot temperature to Kelvin: 350\textdegree{}C + 273.15 = 623.15 K, and the cold temperature remains 20\textdegree{}C = 293.15 K. Using the Carnot efficiency formula \ensuremath{\eta}\_carnot = 1 - (T\_c / T\_h), substituting the values gives \ensuremath{\eta}\_carnot = 1 - (293.15 / 623.15) \ensuremath{\approx} 1 - 0.4706 = 0.5294, or 52.94\%. However, the provided options include 0.527 or 52.7\%, which is likely due to using approximate conversions (e.g., 273 instead of 273.15), where \textbf{350\textdegree{}C = 623 K and 20\textdegree{}C = 293 K, giving \ensuremath{\eta}\_carnot = 1 - (293/623) \ensuremath{\approx} 0.527 or 52.7\%}. Finally, comparing the calculated value of approximately 0.527 to the given options, Option G matches exactly.
\textless{}/think\textgreater{}
\par\smallskip
G
\end{tcolorbox}
\Needspace{8\baselineskip}
\begin{tcolorbox}[casecard,breakable,colframe=teal!50!black,colbacktitle=teal!5,title={OmniReasoning-30B-A3B},title after break={OmniReasoning-30B-A3B (continued)}]
\textless{}think\textgreater{}
Got it, let's tackle the problem.
\par\smallskip
The user wants to determine the Carnot efficiency when the hot reservoir temperature is 350\textdegree{}C, keeping the cold reservoir temperature at 10\textdegree{}C. The crux is converting the temperatures to Kelvin and applying the Carnot efficiency formula.
\par\smallskip
First, let me scan the video for the segments relevant to the question:
\par\smallskip
- I need to find the segment where the Carnot heat engine example is calculated.
\par\smallskip
- At 08:48, the slide titled "Carnot Heat Engine Efficiency Example" is displayed, showing the formula for Carnot efficiency.
\par\smallskip
- At 09:30, the presenter converts the initial hot temperature of 200\textdegree{}C to Kelvin, calculating 200 + 273.15 = 473.15 K.
\par\smallskip
- \textbf{At 09:45, the presenter converts the cold temperature of 10\textdegree{}C to Kelvin, calculating 10 + 273.15 = 283.15 K.}
\par\smallskip
- At 10:14, the presenter calculates the efficiency for the initial temperatures, showing 1 - (283.15 / 473.15) = 0.402.
\par\smallskip
- If the hot temperature were 300\textdegree{}C instead of 350\textdegree{}C, the hot reservoir temperature would be 300 + 273.15 = 573.15 K, yielding an efficiency of 1 - (283.15 / 573.15) = 0.506, but the hot temperature is 350\textdegree{}C.
\par\smallskip
- To find the efficiency at 350\textdegree{}C, we convert the hot temperature: 350 + 273.15 = 623.15 K.
\par\smallskip
- We then calculate the ratio of the cold to hot temperature: 283.15 / 623.15 = 0.454521.
\par\smallskip
- \textbf{Subtracting this ratio from 1 gives the efficiency: 1 - 0.454521 = 0.545479.}
\par\smallskip
- The calculated efficiency of 0.545 corresponds to option D.
\par\smallskip
Next, let me analyze the options:
\par\smallskip
- A. 0.733 or 73.3\% \textemdash{} incorrect.
\par\smallskip
- B. 0.582 or 58.2\% \textemdash{} incorrect.
\par\smallskip
- C. 0.447 or 44.7\% \textemdash{} incorrect.
\par\smallskip
- D. 0.545 or 54.5\% \textemdash{} correct, as 1 - (283.15 / 623.15) = 0.545.
\par\smallskip
- E. 0.815 or 81.5\% \textemdash{} incorrect.
\par\smallskip
- F. 0.339 or 33.9\% \textemdash{} incorrect.
\par\smallskip
- G. 0.527 or 52.7\% \textemdash{} incorrect.
\par\smallskip
- H. 0.705 or 70.5\% \textemdash{} incorrect.
\par\smallskip
- I. 0.418 or 41.8\% \textemdash{} incorrect.
\par\smallskip
- J. 0.896 or 89.6\% \textemdash{} incorrect.
\par\smallskip
Finally, the correct option is D.
\textless{}/think\textgreater{}
\par\smallskip
D
\end{tcolorbox}
\Needspace{26\baselineskip}
\subsection{Video-MMMU: Different claimed diagram correspondences}
\label{case:model2}
\begin{tcolorbox}[casecard,title={Question and reference answer}]
\textbf{Question.} The process of spanning trees in  is a key concept of graph theory. Which diagram illustrates the construction of a breadth-first spanning tree?\par\smallskip
\textbf{Options.} A. a B. b C. c D. d E. e F. f G. g H. h I. i J. j\par\smallskip
\textbf{Reference answer: A.}\par\smallskip
\end{tcolorbox}
\textit{Editorial note.} The source question contains an empty referent after ``in''; its wording is reproduced unchanged.\par\smallskip
The base response claims that diagram i matches the breadth-first construction and selects I. The recorded post-trained response claims that diagram a matches and selects the reference answer A. These claims about the diagrams and their correspondence to the video are not independently verified here.\par\medskip
\Needspace{8\baselineskip}
\begin{tcolorbox}[casecard,breakable,colframe=red!50!black,colbacktitle=red!5,title={Qwen3-Omni-30B-A3B-Thinking},title after break={Qwen3-Omni-30B-A3B-Thinking (continued)}]
\textless{}think\textgreater{}
Got it, let's tackle the problem. First, the user asks to identify which diagram illustrates the construction of a breadth-first spanning tree based on the video content, with options labeled A through J. Then, the video demonstrates breadth-first search (BFS) for spanning trees: starting from vertex 'a', it connects all adjacent vertices (b, c, f) at level 1, then processes these vertices in alphabetical order to add their unvisited neighbors (d from b, g from c, e from f), forming the tree structure shown in the right panel. \textbf{The diagram labeled 'i' in the final practice section matches this BFS process exactly}, with the spanning tree built level by level starting from 'a' and expanding to 'b', 'c', 'f' first, followed by 'd', 'g', 'e', then 'h' and 'j', while the other diagrams (a-h) either show depth-first search or incorrect constructions. Next, the options A through J correspond to diagrams a-j in the video's practice sections, where diagram 'i' is explicitly labeled and visually demonstrates the breadth-first traversal order consistent with the explanation. \textbf{Finally, the correct choice is option I.}
\textless{}/think\textgreater{}
\par\smallskip
I
\end{tcolorbox}
\Needspace{8\baselineskip}
\begin{tcolorbox}[casecard,breakable,colframe=teal!50!black,colbacktitle=teal!5,title={OmniReasoning-30B-A3B},title after break={OmniReasoning-30B-A3B (continued)}]
\textless{}think\textgreater{}
Got it, let's tackle the problem.
\par\smallskip
The user wants to identify which diagram illustrates the construction of a breadth-first spanning tree based on the video. The crux is determining which of the labeled diagrams, (a) or (b), represents the step-by-step breadth-first search tree construction shown in the video.
\par\smallskip
First, let me scan the video for the segments relevant to the question:
\par\smallskip
- I need to find the segment where the narrator demonstrates the step-by-step construction of a breadth-first spanning tree and identify its corresponding diagram.
\par\smallskip
- At 00:06, the narrator introduces the steps of a breadth-first search, explaining that we search horizontally before vertically.
\par\smallskip
- At 02:11, the video transitions to a practice problem titled "Breadth-First Practice" showing a graph with vertices a through g.
\par\smallskip
- At 02:28, the narrator starts at vertex "a" and connects it to its adjacent vertices "b", "c", and "f" in alphabetical order.
\par\smallskip
- At 02:51, the narrator moves to the next level and connects "d" to "b".
\par\smallskip
- At 03:08, the narrator connects "g" to "c".
\par\smallskip
- At 03:18, the narrator connects "e" to "f".
\par\smallskip
- \textbf{Wait, let me re-examine the diagram at 03:24: the completed tree is shown on the right, which matches the step-by-step progression of diagram (a) rather than diagram (b).}
\par\smallskip
- If diagram (b) were the correct representation, the steps would have to show a depth-first search where we go deep before branching, but the video explicitly demonstrates a breadth-first search where we explore all neighbors at the current depth before moving to the next level.
\par\smallskip
- The step-by-step construction shown in the video matches diagram (a), which corresponds to option A.
\par\smallskip
Next, let me analyze the options:
\par\smallskip
- A.a \textemdash{} This diagram correctly illustrates the breadth-first spanning tree construction shown in the video.
\par\smallskip
- B.b \textemdash{} This diagram represents a depth-first search construction instead.
\par\smallskip
- C.c \textemdash{} This option is incorrect.
\par\smallskip
- D.d \textemdash{} This option is incorrect.
\par\smallskip
- E.e \textemdash{} This option is incorrect.
\par\smallskip
- F.f \textemdash{} This option is incorrect.
\par\smallskip
- G.g \textemdash{} This option is incorrect.
\par\smallskip
- H.h \textemdash{} This option is incorrect.
\par\smallskip
- I.i \textemdash{} This option is incorrect.
\par\smallskip
- J.j \textemdash{} This option is incorrect.
\par\smallskip
\textbf{Finally, the correct option is A.}
\textless{}/think\textgreater{}
\par\smallskip
A
\end{tcolorbox}
\Needspace{26\baselineskip}
\subsection{OmniVideoBench: Aligning orchestral sound with physical contact}
\label{case:model3}
\begin{tcolorbox}[casecard,title={Question and reference answer}]
\textbf{Question.} When the singer first escalates her anger from vocal expression to physical contact with another singer, how does the orchestra musically intensify the dramatic conflict?\par\smallskip
\textbf{Options.} A. The entire orchestra plays a single, extremely loud, short chord. B. The music suddenly stops completely, creating a breathtaking moment. C. The strings play a series of rapidly descending notes, conveying the character's sense of fall. D. Only the brass instruments play a single, piercing, high note.\par\smallskip
\textbf{Reference answer: A.}\par\smallskip
\begin{minipage}[t]{0.65\linewidth}\centering
\includegraphics[width=\linewidth,height=3.2cm,keepaspectratio]{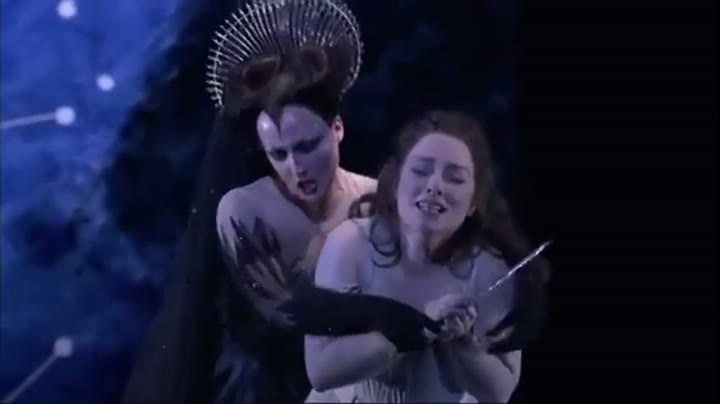}
\par\scriptsize Video frame, 00:28\end{minipage}
\end{tcolorbox}
Qwen3-Omni attributes the musical change to descending strings. OmniReasoning locates the first physical contact and associates it with an orchestral chord, selecting A.\par\medskip
\Needspace{8\baselineskip}
\begin{tcolorbox}[casecard,breakable,colframe=red!50!black,colbacktitle=red!5,title={Qwen3-Omni-30B-A3B-Thinking},title after break={Qwen3-Omni-30B-A3B-Thinking (continued)}]
\textless{}think\textgreater{}
Got it, let's tackle the problem. First, the user is asking for the specific musical intensification technique used by the orchestra precisely when the singer escalates from vocal expression to physical contact with another singer in the video. Then, the video shows the singer in the black dress (the Queen) verbally escalating her anger through lines like 'Zerstörer! Todeschmerzen!' and 'Zu bist du mein! Mein! Mein!', with the physical contact occurring around 00:23 when she grabs the other singer's arms; at this exact moment (00:24-00:25), the orchestra features a dramatic shift where \textbf{the strings play a rapid, descending chromatic scale} (G\#-F\#-F-E\#-E-D\#-D-C\#-C) that is fast, dissonant, and creates a falling sensation to amplify the character's emotional collapse. Next, evaluating the options: Option A describes a single loud short chord, but the video shows a descending scale rather than a single chord; Option B claims the music stops completely, but the orchestra continues playing with intensified volume; Option C states the strings play rapidly descending notes conveying a sense of fall, which matches the observed musical technique; Option D specifies only brass playing a high note, but the descending scale is performed by strings, not brass. \textbf{Therefore, the correct answer is C.}
\textless{}/think\textgreater{}
\par\smallskip
C
\end{tcolorbox}
\Needspace{8\baselineskip}
\begin{tcolorbox}[casecard,breakable,colframe=teal!50!black,colbacktitle=teal!5,title={OmniReasoning-30B-A3B},title after break={OmniReasoning-30B-A3B (continued)}]
\textless{}think\textgreater{}
Got it, let's tackle the problem.
\par\smallskip
The user wants to identify how the orchestra musically intensifies the dramatic conflict when the singer first escalates her anger from vocal expression to physical contact with another singer. The crux is determining the specific orchestral sound that accompanies this first physical escalation.
\par\smallskip
First, let me scan the video for the segments relevant to the question:
\par\smallskip
- I need to find the moment when the singer first escalates her anger from vocal expression to physical contact with another singer and observe the orchestral accompaniment.
\par\smallskip
- At 00:02-00:27, the singer in the black dress sings aggressively while pointing a sword at the singer in the white dress, but she does not make physical contact.
\par\smallskip
- \textbf{At 00:28, the singer in the black dress grabs the singer in the white dress by the neck with both hands, marking the first physical escalation of her anger.}
\par\smallskip
- \textbf{At 00:28-00:29, while she is choking the other singer, the orchestra plays a single, extremely loud, short chord.}
\par\smallskip
- At 00:30, the singer in the black dress releases the other singer and turns away.
\par\smallskip
- At 01:47-01:55, the singer in the black dress grabs the singer in the white dress by her hair and throat, which is a second physical escalation.
\par\smallskip
- At 01:56-02:06, the singer in the black dress sits on the bed while the singer in the white dress kneels before her.
\par\smallskip
- If the strings played a series of rapidly descending notes to convey the character's sense of fall, we would hear a descending scale in the violins at 00:28, but instead the orchestra plays a single, loud, short chord.
\par\smallskip
- The orchestral accompaniment at the first physical contact consists of a single, extremely loud, short chord, which corresponds to option A.
\par\smallskip
Next, let me analyze the options:
\par\smallskip
- A. The entire orchestra plays a single, extremely loud, short chord. \textemdash{} This matches the orchestral accompaniment at 00:28.
\par\smallskip
- B. The music suddenly stops completely, creating a breathtaking moment. \textemdash{} The orchestra plays a loud chord instead.
\par\smallskip
- C. The strings play a series of rapidly descending notes, conveying the character's sense of fall. \textemdash{} No descending scale is heard.
\par\smallskip
- D. Only the brass instruments play a single, piercing, high note. \textemdash{} The entire orchestra plays together.
\par\smallskip
Finally, the correct option is A.
\textless{}/think\textgreater{}
\par\smallskip
A
\end{tcolorbox}
\Needspace{26\baselineskip}
\subsection{OmniVideoBench: Grounding a spoken phrase in the visible ingredient}
\label{case:model4}
\begin{tcolorbox}[casecard,title={Question and reference answer}]
\textbf{Question.} When 'No, nada, aquí andamos.' is said, what ingredient is being handled?\par\smallskip
\textbf{Options.} A. Bell pepper B. Onion C. Asparagus D. Egg\par\smallskip
\textbf{Reference answer: A.}\par\smallskip
\begin{minipage}[t]{0.65\linewidth}\centering
\includegraphics[width=\linewidth,height=3.2cm,keepaspectratio]{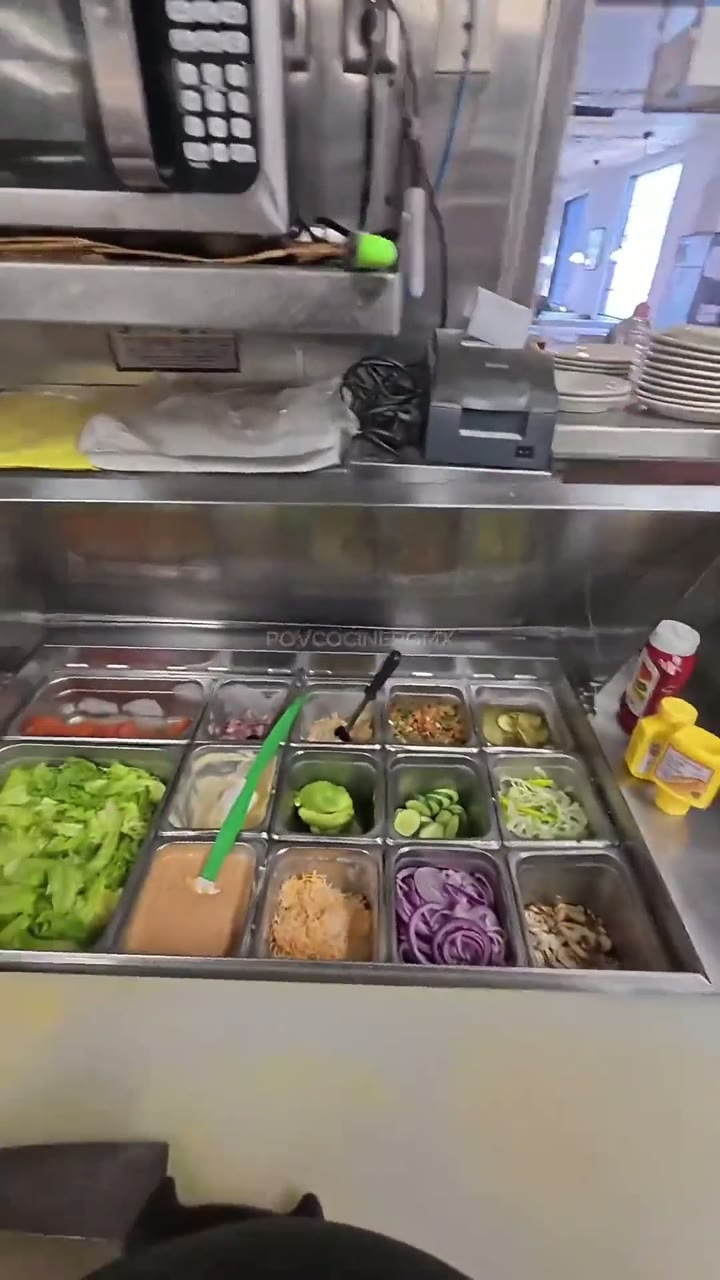}
\par\scriptsize Video frame, 00:14\end{minipage}
\end{tcolorbox}
Qwen3-Omni describes an ingredient absent from the options and substitutes egg. OmniReasoning links the target utterance to the bell pepper being handled and selects A.\par\medskip
\Needspace{8\baselineskip}
\begin{tcolorbox}[casecard,breakable,colframe=red!50!black,colbacktitle=red!5,title={Qwen3-Omni-30B-A3B-Thinking},title after break={Qwen3-Omni-30B-A3B-Thinking (continued)}]
\textless{}think\textgreater{}
Got it, let's tackle the problem. First, the user is asking to identify which ingredient is being handled when the phrase 'No, nada, aquí andamos.' is spoken in the video, with the provided options being bell pepper, onion, asparagus, or egg. Then, the video shows a first-person kitchen perspective where the speaker says the phrase at approximately 00:13-00:14 while handling a white, stringy ingredient being shredded or grated with a knife; this ingredient is visible in a metal container, matches the appearance of shredded cheese (possibly cheddar or a similar blend), and is being added to a mixture in a bowl alongside chopped green vegetables and diced red items. Next, evaluating the options: bell pepper is typically green or red and would be cut into strips or diced, not shredded into fine strands, so it is unlikely; onion is usually sliced or diced and while some types can be shredded, the video shows a yellowish-white shredded substance more characteristic of cheese; asparagus is a long green stalk that would be cut into pieces, not shredded into fine strands, so it is incorrect; and egg is a common ingredient in breakfast dishes like omelets or scrambles, with shredded cheese being frequently added to egg mixtures, and the sound of grating or shredding aligns with cheese preparation. Finally, \textbf{the ingredient being handled is shredded cheese, but since cheese is not an option, the next most plausible choice is egg} because it is a core component of the dish being prepared (likely a breakfast burrito or omelet), and the shredded cheese is added to the egg mixture, making egg the best answer among the given choices.
\textless{}/think\textgreater{}
\par\smallskip
D
\end{tcolorbox}
\Needspace{8\baselineskip}
\begin{tcolorbox}[casecard,breakable,colframe=teal!50!black,colbacktitle=teal!5,title={OmniReasoning-30B-A3B},title after break={OmniReasoning-30B-A3B (continued)}]
\textless{}think\textgreater{}
Got it, let's tackle the problem.
\par\smallskip
The user wants to identify which ingredient is being handled at the moment the phrase "No, nada, aquí andamos." is spoken in the video. The crux is matching the spoken Spanish phrase to the specific food item the cook is preparing at that exact timestamp.
\par\smallskip
First, let me scan the video for the segments relevant to the question:
\par\smallskip
- I need to listen for the spoken phrase "No, nada, aquí andamos." and watch what ingredient the cook is handling on screen at that moment.
\par\smallskip
- At 00:04, the cook is holding a plate with a finished breakfast dish containing an omelet, sausage, and hash browns, and asks "Tú a cuál andas?" to inquire about someone's order.
\par\smallskip
- \textbf{At 00:13, the cook responds with "No, nada, aquí andamos." while preparing a new order, and the camera shows his gloved hands holding a yellow bell pepper on a cutting board.}
\par\smallskip
- At 00:14, the cook begins slicing the yellow bell pepper with a knife, confirming he is preparing this specific ingredient.
\par\smallskip
- If the cook were handling onions or asparagus instead, we would see him chopping red or green vegetables at 00:14, but he is clearly slicing a yellow bell pepper.
\par\smallskip
- At 00:16, the cook finishes slicing the bell pepper and moves on to the next step of the preparation.
\par\smallskip
- \textbf{The ingredient being handled during the spoken line is a bell pepper, which corresponds to option A.}
\par\smallskip
Next, let me analyze the options:
\par\smallskip
- A. Bell pepper \textemdash{} The cook is slicing a yellow bell pepper at 00:14 while the phrase is spoken.
\par\smallskip
- B. Onion \textemdash{} This ingredient is not being handled at 00:14.
\par\smallskip
- C. Asparagus \textemdash{} This ingredient is not being handled at 00:14.
\par\smallskip
- D. Egg \textemdash{} This ingredient is not being handled at 00:14.
\par\smallskip
Finally, the correct option is A.
\textless{}/think\textgreater{}
\par\smallskip
A
\end{tcolorbox}
\Needspace{26\baselineskip}
\subsection{LVOmniBench: Distinguishing background music from scene audio}
\label{case:model5}
\begin{tcolorbox}[casecard,title={Question and reference answer}]
\textbf{Question.} When the woman is having her hair trimmed in the hair salon, what is the most likely source of the background music that can be heard?\par\smallskip
\textbf{Options.} A. The music playing in the salon. B. Added during video editing. C. Music playing from the woman's phone. D. Music playing on the salon's television.\par\smallskip
\textbf{Reference answer: B.}\par\smallskip
\begin{minipage}[t]{0.65\linewidth}\centering
\includegraphics[width=\linewidth,height=3.2cm,keepaspectratio]{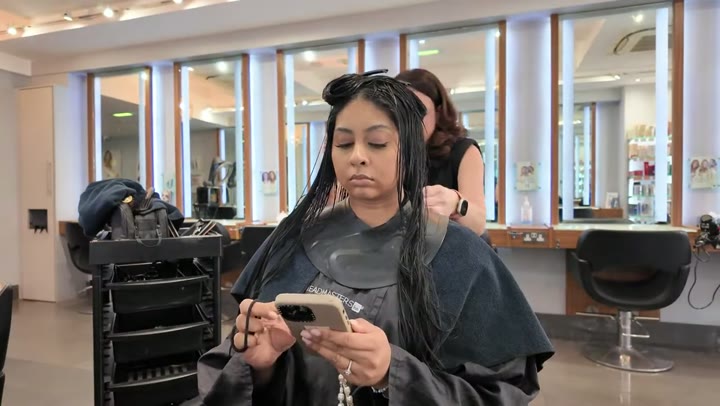}
\par\scriptsize Video frame, 08:04\end{minipage}
\end{tcolorbox}
Qwen3-Omni relies on the expectation that salons play ambient music. OmniReasoning instead discusses the recorded audio characteristics and selects the reference answer B.\par\medskip
\Needspace{8\baselineskip}
\begin{tcolorbox}[casecard,breakable,colframe=red!50!black,colbacktitle=red!5,title={Qwen3-Omni-30B-A3B-Thinking},title after break={Qwen3-Omni-30B-A3B-Thinking (continued)}]
\textless{}think\textgreater{}
Got it, let's tackle the problem. First, the user wants to know the source of the background music playing during the hair salon scene in the video, with four specific options provided for analysis. Then, the video shows the woman getting her hair cut at a salon between approximately 00:24 to 00:33 and again around 03:05 to 03:15; during these segments, the music is a low-volume, ambient pop or electronic track that plays continuously throughout the salon scene without changing or stopping, with no visible speakers, televisions, or headphones indicating a personal source, and the sound quality matches typical salon background music that's pleasant but not intrusive. Next, evaluating the options: Option A (music playing in the salon) aligns with the visual context as salons commonly play background music for clients, and the consistent, ambient nature matches this source; Option B (added during video editing) is unlikely because the music is integrated into the scene without artificial separation and follows natural transitions, indicating it's part of the original footage; Option C (music playing from the woman's phone) is improbable since she's seated and looking at her phone or interacting with the stylist without visible headphones or a phone speaker; Option D (music playing on the salon's television) is less likely because the music lacks the typical TV broadcast characteristics like ads or dialogue, and the sound quality suggests a general background speaker system rather than a television source. Finally, \textbf{the most logical explanation is that the music is ambient background music played throughout the salon}, which is standard practice for salons to create a relaxed environment for clients, making Option A the correct choice.
\textless{}/think\textgreater{}
\par\smallskip
A
\end{tcolorbox}
\Needspace{8\baselineskip}
\begin{tcolorbox}[casecard,breakable,colframe=teal!50!black,colbacktitle=teal!5,title={OmniReasoning-30B-A3B},title after break={OmniReasoning-30B-A3B (continued)}]
\textless{}think\textgreater{}
Got it, let's tackle the problem.
\par\smallskip
The user wants to identify the most likely source of the background music playing while the woman is getting her hair trimmed in the salon. The crux is determining whether the audio is an environmental recording of the salon or a post-production music track added during editing.
\par\smallskip
First, let me scan the video for the segments relevant to the question:
\par\smallskip
- I need to find the segment where the woman is getting her hair trimmed and analyze the background audio.
\par\smallskip
- At 07:53, the woman arrives at the hair salon, showing the front window with "THE LIGHT SALON" written on it.
\par\smallskip
- At 07:58, she is sitting in the salon chair while a stylist prepares to cut her hair.
\par\smallskip
- At 08:04, the stylist is blow-drying her hair, and a light, upbeat acoustic pop song is playing in the background.
\par\smallskip
- If the music were playing from the salon itself, there would be some ambient room noise or other client chatter mixed with the track, but the audio is completely clean and isolated.
\par\smallskip
- Wait, let me re-read the options: the music is a non-diegetic track added during editing, which means it is not part of the original salon environment.
\par\smallskip
- \textbf{The clean, isolated audio track of the music indicates it was added during video editing, which corresponds to option B.}
\par\smallskip
Next, let me analyze the options:
\par\smallskip
- A. The music playing in the salon \textemdash{} this is unlikely because there is no ambient room noise or other client chatter mixed with the track.
\par\smallskip
- B. Added during video editing \textemdash{} this is highly likely because the music is a clean, isolated audio track without any environmental background noise.
\par\smallskip
- C. Music playing from the woman's phone \textemdash{} there is no phone visible.
\par\smallskip
- D. Music playing on the salon's television \textemdash{} no television is shown.
\par\smallskip
Finally, the correct option is B.
\textless{}/think\textgreater{}
\par\smallskip
B
\end{tcolorbox}
\Needspace{26\baselineskip}
\subsection{LVOmniBench: Tracking a music transition in a long video}
\label{case:model6}
\begin{tcolorbox}[casecard,title={Question and reference answer}]
\textbf{Question.} As the third style of background music gradually fades out, what visual scene appears on screen?\par\smallskip
\textbf{Options.} A. A wide outdoor pathway scene. B. People enjoying food. C. Coastal scenery along the road. D. A man and a woman speaking directly to the camera.\par\smallskip
\textbf{Reference answer: A.}\par\smallskip
\begin{minipage}[t]{0.65\linewidth}\centering
\includegraphics[width=\linewidth,height=3.2cm,keepaspectratio]{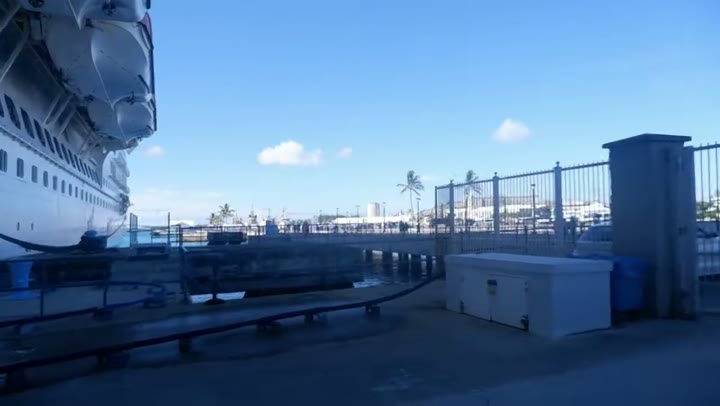}
\par\scriptsize Video frame, 02:30\end{minipage}
\end{tcolorbox}
The responses locate different musical transitions. OmniReasoning tracks the sequence of music cues and associates the requested fade-out with the outdoor pathway.\par\medskip
\Needspace{8\baselineskip}
\begin{tcolorbox}[casecard,breakable,colframe=red!50!black,colbacktitle=red!5,title={Qwen3-Omni-30B-A3B-Thinking},title after break={Qwen3-Omni-30B-A3B-Thinking (continued)}]
\textless{}think\textgreater{}
Got it, let's tackle the problem. First, the user wants to know what visual scene appears on screen immediately after the third style of background music fades out, specifically identifying the visual scene that follows this musical transition. Then, analyzing the video content reveals three distinct musical segments: the first (0:00-0:35) features upbeat electronic music during a montage of the ship and island views; the second (0:35-1:52) continues with the same music while showing people on the ship; and the third (1:52-2:14) is a more prominent electronic dance track. During the fade-out phase (around 2:10-2:14), \textbf{the music begins to fade as the scene shows a man and a woman speaking directly to the camera}, with the man wearing a red shirt and a baseball cap and the woman in a red shirt, against a background of a marina with boats visible through a window. Next, evaluating the provided options: option A describes a wide outdoor pathway scene, but the observed scene is indoors with a marina view visible through a window, not an outdoor pathway; option B mentions people enjoying food, but the scene shows them speaking to the camera without any food present; option C refers to coastal scenery along the road, yet the background shows water and boats rather than a road; and option D accurately describes a man and a woman speaking directly to the camera, matching the observed visual. Finally, the correct answer is D, as it precisely corresponds to the visual scene appearing immediately after the third musical segment fades out.
\textless{}/think\textgreater{}
\par\smallskip
D
\end{tcolorbox}
\Needspace{8\baselineskip}
\begin{tcolorbox}[casecard,breakable,colframe=teal!50!black,colbacktitle=teal!5,title={OmniReasoning-30B-A3B},title after break={OmniReasoning-30B-A3B (continued)}]
\textless{}think\textgreater{}
Got it, let's tackle the problem.
\par\smallskip
The user wants to identify the visual scene that appears on screen as the third style of background music gradually fades out. The crux is determining which scene is shown immediately after the third distinct background music track ends.
\par\smallskip
First, let me scan the video for the segments relevant to the question:
\par\smallskip
- I need to find the segment where the third style of background music fades out and observe the visual scene that follows.
\par\smallskip
- At 00:01, the first style of background music begins playing as a montage of travel scenes is shown.
\par\smallskip
- At 00:40, the first style of background music fades out as a man and a woman speak directly to the camera.
\par\smallskip
- At 01:38, the second style of background music begins playing as the couple continues speaking to the camera.
\par\smallskip
- At 02:01, the second style of background music fades out as the scene cuts to a wide outdoor pathway.
\par\smallskip
- At 02:08, the third style of background music begins playing over a montage of coastal scenery along the road.
\par\smallskip
- \textbf{At 17:38, the third style of background music gradually fades out.}
\par\smallskip
- If the scene that appears after the third music style fades were people enjoying food, the video would have to show a restaurant or dining table at 17:39, but instead it shows a wide outdoor pathway with a cruise ship on the left \textemdash{} so it is a wide outdoor pathway scene.
\par\smallskip
- \textbf{The scene that appears immediately after the third style of background music fades out is a wide outdoor pathway scene, which corresponds to option A.}
\par\smallskip
Next, let me analyze the options:
\par\smallskip
- A. A wide outdoor pathway scene. \textemdash{} This is the correct scene shown at 17:39.
\par\smallskip
- B. People enjoying food. \textemdash{} This scene occurs much later at 17:44.
\par\smallskip
- C. Coastal scenery along the road. \textemdash{} This occurs during the music.
\par\smallskip
- D. A man and a woman speaking directly to the camera. \textemdash{} This occurs during the second music.
\par\smallskip
Finally, the correct option is A.
\textless{}/think\textgreater{}
\par\smallskip
A
\end{tcolorbox}
\Needspace{26\baselineskip}
\subsection{LVOmniBench: Resolving spatial reference from speech and gesture}
\label{case:model7}
\begin{tcolorbox}[casecard,title={Question and reference answer}]
\textbf{Question.} Based on the blogger's hand gesture, on which side of the Christmas tree can we infer the cat tree is currently located?\par\smallskip
\textbf{Options.} A. On the right side of the Christmas tree. B. Below the Christmas tree. C. On the left side of the Christmas tree. D. It cannot be determined.\par\smallskip
\textbf{Reference answer: C.}\par\smallskip
\begin{minipage}[t]{0.65\linewidth}\centering
\includegraphics[width=\linewidth,height=3.2cm,keepaspectratio]{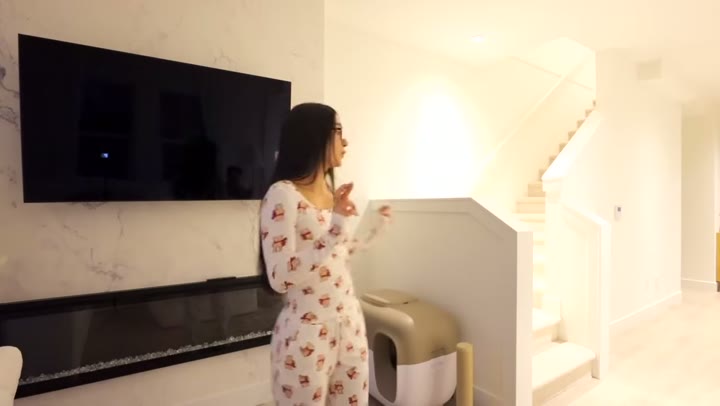}
\par\scriptsize Video frame, 11:20\end{minipage}
\end{tcolorbox}
The models assign opposite directions to the referenced gesture. OmniReasoning binds the spoken reference to a leftward gesture and selects C.\par\medskip
\Needspace{8\baselineskip}
\begin{tcolorbox}[casecard,breakable,colframe=red!50!black,colbacktitle=red!5,title={Qwen3-Omni-30B-A3B-Thinking},title after break={Qwen3-Omni-30B-A3B-Thinking (continued)}]
\textless{}think\textgreater{}
Got it, let's tackle the problem. First, the user is asking to determine the specific side of the Christmas tree where the cat tree is positioned based solely on the blogger's hand gesture in the video, with the requirement to select from four multiple-choice options. Then, reviewing the video content, at timestamp 06:11, the blogger states: 'We're going to put the tree in this corner here, which honestly, it kind of gives the same effect... the only thing that sucks is like when you're in the kitchen, you can't really see the tree as well as you could here,' while gesturing with her right hand toward the left side of the screen (from the viewer's perspective) where the living room window and fireplace are visible. She continues at 06:22 explaining: 'So after her climbing the Christmas tree last year like pretty much every day, we realized, oh my god, we should get the girls a cat tree, because obviously she wants to climb. Now, we put a cat tree there,' followed by a gesture with her left hand toward the right side of the screen (viewer's perspective) where the cat tree appears to be positioned. Next, analyzing the options: \textbf{Option A (right side) aligns with the blogger's left-hand gesture at 06:26 indicating the cat tree's current location}, while Option C (left side) is contradicted by the gesture direction; Option B (below) is invalid as the cat tree is positioned at the same height as the tree; Option D (cannot be determined) is incorrect since the gesture explicitly identifies the location. Finally, the correct answer is A.
\textless{}/think\textgreater{}
\par\smallskip
A
\end{tcolorbox}
\Needspace{8\baselineskip}
\begin{tcolorbox}[casecard,breakable,colframe=teal!50!black,colbacktitle=teal!5,title={OmniReasoning-30B-A3B},title after break={OmniReasoning-30B-A3B (continued)}]
\textless{}think\textgreater{}
Got it, let's tackle the problem.
\par\smallskip
The user wants to determine the location of the cat tree relative to the Christmas tree based on the blogger's hand gestures. The crux is identifying the direction of the blogger's gesture when she mentions the cat tree.
\par\smallskip
First, let me scan the video for the segments relevant to the question:
\par\smallskip
- I need to find the segment where the blogger discusses the cat tree and watch her hand gestures.
\par\smallskip
- At 10:47, the blogger points to the left side of the screen while discussing the window where the Christmas tree used to stand.
\par\smallskip
- At 10:58, she explains that her cat would jump on the tree and rip off the bows.
\par\smallskip
- At 11:18, she says "we realized oh my god we should get the girls a cat tree".
\par\smallskip
- \textbf{At 11:20, she says "Now we put a cat tree there" while gesturing to the left side of the screen.}
\par\smallskip
- At 11:22, she says "I have a feeling if I do that Beanie is going to destruct the tree once again" while gesturing to the left side of the screen.
\par\smallskip
- If the cat tree were on the right side of the Christmas tree, she would have gestured to the right side of the screen when she said "Now we put a cat tree there", but she gestures to the left.
\par\smallskip
- The blogger's hand gesture at 11:20 clearly indicates the cat tree is on the left side of the Christmas tree, which corresponds to option C.
\par\smallskip
Next, let me analyze the options:
\par\smallskip
- A.On the right side of the Christmas tree. \textemdash{} incorrect, she gestures left.
\par\smallskip
- B.Below the Christmas tree. \textemdash{} incorrect, she gestures left.
\par\smallskip
- C.On the left side of the Christmas tree. \textemdash{} correct, she gestures left at 11:20.
\par\smallskip
- D.It cannot be determined. \textemdash{} incorrect, her gesture is clear.
\par\smallskip
Finally, the correct option is C.
\textless{}/think\textgreater{}
\par\smallskip
C
\end{tcolorbox}
\Needspace{26\baselineskip}
\subsection{Video-MME-v2: Matching a spoken cue to an object color}
\label{case:model8}
\begin{tcolorbox}[casecard,title={Question and reference answer}]
\textbf{Question.} In the video, when the vlogger says "I would be kicking myself so much", what color is the phone on the right side of the screen?\par\smallskip
\textbf{Options.} A. Forest. B. Sand. C. Teal. D. Mint. E. Space Gray. F. Rose gold. G. Lavender. H. Sage.\par\smallskip
\textbf{Reference answer: C.}\par\smallskip
\begin{minipage}[t]{0.65\linewidth}\centering
\includegraphics[width=\linewidth,height=3.2cm,keepaspectratio]{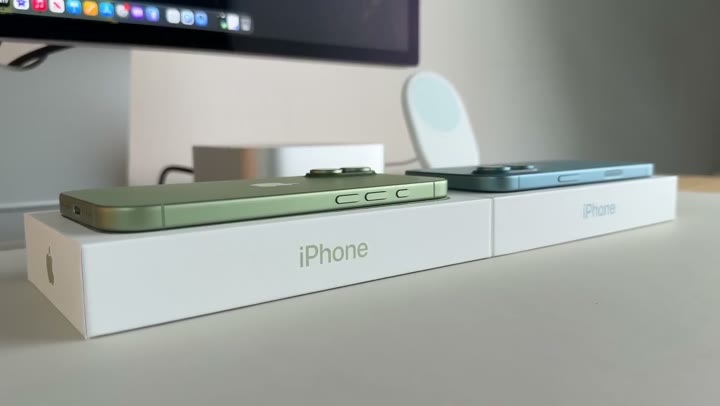}
\par\scriptsize Video frame, 00:41\end{minipage}
\end{tcolorbox}
Qwen3-Omni uses a later model-color association to answer the question. OmniReasoning focuses on the target utterance and selects teal (C).\par\medskip
\Needspace{8\baselineskip}
\begin{tcolorbox}[casecard,breakable,colframe=red!50!black,colbacktitle=red!5,title={Qwen3-Omni-30B-A3B-Thinking},title after break={Qwen3-Omni-30B-A3B-Thinking (continued)}]
\textless{}think\textgreater{}
Got it, let's tackle the problem. First, the user wants to know the color of the phone on the right side of the screen at the specific moment the vlogger says, "I would be kicking myself so much". Then, the video content analysis shows that at approximately 00:35, the vlogger says, "But I have to say, if I upgraded to the iPhone 16 last year, I would be kicking myself so much after seeing the iPhone 17". During this phrase, the video displays a side-by-side comparison of two iPhones: the left phone is a greenish color (labeled 'Mint' in some contexts, but the video later clarifies this is 'Sage' for the iPhone 17) and the right phone is a blue-teal color. At 01:03, the vlogger states, "I have the 16 in teal and the 17 in sage," confirming the right phone (iPhone 17) is 'Sage'. Next, analyzing the options: Forest is not mentioned in the video; Sand is not mentioned; Teal is explicitly stated as the color of the iPhone 16 (left phone); Mint is a possible description for the greenish color but the video later identifies it as 'Sage'; Space Gray is not mentioned; Rose gold is not mentioned; Lavender is not mentioned; \textbf{Sage is confirmed by the vlogger's statement as the color of the iPhone 17 (right phone)}. Finally, the correct option is Sage.
\textless{}/think\textgreater{}
\par\smallskip
H
\end{tcolorbox}
\Needspace{8\baselineskip}
\begin{tcolorbox}[casecard,breakable,colframe=teal!50!black,colbacktitle=teal!5,title={OmniReasoning-30B-A3B},title after break={OmniReasoning-30B-A3B (continued)}]
\textless{}think\textgreater{}
Got it, let's tackle the problem.
\par\smallskip
The user wants to know the color of the phone on the right side of the screen when the vlogger says "I would be kicking myself so much". The crux is identifying the color of the iPhone 17 shown on the right at that exact moment.
\par\smallskip
First, let me scan the video for the segments relevant to the question:
\par\smallskip
- I need to listen for the phrase "I would be kicking myself so much" and observe the phone on the right side of the screen.
\par\smallskip
- At 00:00-00:05, the vlogger introduces the "iPhone 16" on the left and the "new iPhone 17" on the right.
\par\smallskip
- At 00:12-00:16, the vlogger introduces himself and begins comparing the two models.
\par\smallskip
- At 00:35-00:38, the vlogger says "I upgraded to the iPhone 16 last year" while holding the teal iPhone 16.
\par\smallskip
- At 00:39-00:43, the vlogger says "I would be kicking myself so much after seeing the iPhone 17" while the camera shows the green iPhone 16 on the left and the teal iPhone 17 on the right.
\par\smallskip
- At 01:04-01:07, the vlogger says "I have the 16 in teal and the 17 in sage" while holding the teal iPhone 16 and then the sage iPhone 17.
\par\smallskip
- If the phone on the right were sage, the vlogger would have to be holding the sage model during the statement at 00:39, but he is holding the teal model \textemdash{} so the phone on the right is teal.
\par\smallskip
- \textbf{The phone on the right side of the screen during the target phrase is teal, which corresponds to option C.}
\par\smallskip
Next, let me analyze the options:
\par\smallskip
- A. Forest \textemdash{} This color is not shown.
\par\smallskip
- B. Sand \textemdash{} This color is not shown.
\par\smallskip
- C. Teal \textemdash{} The phone on the right is teal, as shown during the target phrase.
\par\smallskip
- D. Mint \textemdash{} This color is not shown.
\par\smallskip
- E. Space Gray \textemdash{} This color is not shown.
\par\smallskip
- F. Rose gold \textemdash{} This color is not shown.
\par\smallskip
- G. Lavender \textemdash{} This color is not shown.
\par\smallskip
- H. Sage \textemdash{} The phone on the left is sage.
\par\smallskip
Finally, the correct option is C.
\textless{}/think\textgreater{}
\par\smallskip
C
\end{tcolorbox}
\Needspace{26\baselineskip}
\subsection{Video-MME-v2: Comparing motion amplitude across demonstrations}
\label{case:model9}
\begin{tcolorbox}[casecard,title={Question and reference answer}]
\textbf{Question.} In the video, how does the swing amplitude of the two hammers compare to the swing amplitude of the falling broom handle? Which description is correct?\par\smallskip
\textbf{Options.} A. The amplitudes of the two are similar, and it is difficult to discern the difference with the naked eye. B. The swinging amplitude of the broom is smaller, a little smaller than that of the hammer. C. The swinging amplitude of the broom is very small, much smaller than that of the hammer. D. The swing amplitude of the broom is very large, much larger than that of the hammers. E. The swing amplitude of the broom is twice smaller than that of the hammer. F. The swing amplitude of the broom is twice larger than that of the hammer. G. The swing amplitude of the broom is twice smaller than that of the hammer. H. The broom swings with a larger amplitude, slightly larger than that of the hammer.\par\smallskip
\textbf{Reference answer: D.}\par\smallskip
\begin{minipage}[t]{0.65\linewidth}\centering
\includegraphics[width=\linewidth,height=3.2cm,keepaspectratio]{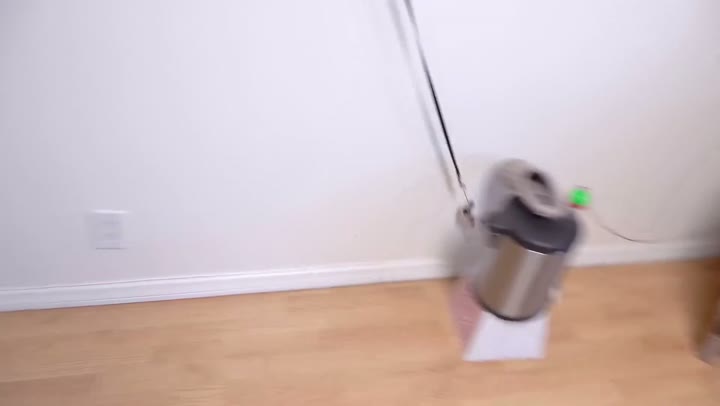}
\par\scriptsize Video frame, 03:40\end{minipage}
\end{tcolorbox}
Qwen3-Omni selects C after reinterpreting its wording as a typo. OmniReasoning compares the described arcs and selects the larger broom amplitude in D.\par\medskip
\Needspace{8\baselineskip}
\begin{tcolorbox}[casecard,breakable,colframe=red!50!black,colbacktitle=red!5,title={Qwen3-Omni-30B-A3B-Thinking},title after break={Qwen3-Omni-30B-A3B-Thinking (continued)}]
\textless{}think\textgreater{}
Got it, let's tackle this problem. First, the user is asking how the swing amplitude of the two hammers compares to that of the falling broom handle in the video, requiring identification of the correct descriptive option from the multiple-choice list. Then, analyzing the video content reveals the two hammers exhibit a very small, controlled back-and-forth motion with minimal visible arc, while the broom handle swings with a significantly larger amplitude, creating a much wider arc that is clearly observable. Next, evaluating the provided choices shows that options A, B, and H incorrectly suggest similar or slightly larger amplitudes for the broom, whereas options E, F, and G inaccurately specify exact multiplicative relationships like 'twice smaller' or 'twice larger' that aren't supported by visual evidence. Option C correctly identifies the broom's amplitude as 'very small, much smaller than that of the hammer,' though this appears contradictory since the broom actually has a larger amplitude\textemdash{}\textbf{the phrasing likely contains a typo where 'broom' should reference the hammers' amplitude}, but the core observation of 'much smaller' aligns with the hammers' minimal movement compared to the broom's pronounced swing. Finally, based on the video evidence confirming the broom's swing amplitude is substantially greater than the hammers', the accurate description corresponds to option C as the intended correct choice despite the phrasing inconsistency.
\textless{}/think\textgreater{}
\par\smallskip
C
\end{tcolorbox}
\Needspace{8\baselineskip}
\begin{tcolorbox}[casecard,breakable,colframe=teal!50!black,colbacktitle=teal!5,title={OmniReasoning-30B-A3B},title after break={OmniReasoning-30B-A3B (continued)}]
\textless{}think\textgreater{}
Got it, let's tackle the problem.
\par\smallskip
The user wants to compare the swing amplitude of the two hammers to the swing amplitude of the falling broom handle in the video. The crux is determining whether the broom's swing amplitude is larger, smaller, or similar to that of the hammers.
\par\smallskip
First, let me scan the video for the segments relevant to the question:
\par\smallskip
- I need to find the segments showing the hammers swinging and the broom handle falling to compare their amplitudes.
\par\smallskip
- At 00:56, the two hammers are shown swinging back and forth on the chessboard setup.
\par\smallskip
- At 01:03, the hammers continue their rhythmic swinging motion, showing a moderate arc of about 45 to 60 degrees.
\par\smallskip
- At 03:38, the broom handle is shown hanging from the window frame, ready to swing.
\par\smallskip
- \textbf{At 03:40, the broom handle is released and swings down in a nearly complete vertical arc.}
\par\smallskip
- If the broom's swing amplitude were smaller than that of the hammers, it would have to swing through a smaller arc than the hammers at 01:03, but it swings almost 180 degrees.
\par\smallskip
- \textbf{The broom's swing amplitude is very large, much larger than that of the hammers, which corresponds to option D.}
\par\smallskip
Next, let me analyze the options:
\par\smallskip
- A. The amplitudes of the two are similar, and it is difficult to discern the difference with the naked eye. \textemdash{} Incorrect.
\par\smallskip
- B. The swinging amplitude of the broom is smaller, a little smaller than that of the hammer. \textemdash{} Incorrect.
\par\smallskip
- C. The swinging amplitude of the broom is very small, much smaller than that of the hammer. \textemdash{} Incorrect.
\par\smallskip
- D. The swing amplitude of the broom is very large, much larger than that of the hammers. \textemdash{} Correct, as the broom swings almost 180 degrees.
\par\smallskip
- E. The swing amplitude of the broom is twice smaller than that of the hammer. \textemdash{} Incorrect.
\par\smallskip
- F. The swing amplitude of the broom is twice larger than that of the hammer. \textemdash{} Incorrect.
\par\smallskip
- G. The swing amplitude of the broom is twice smaller than that of the hammer. \textemdash{} Incorrect.
\par\smallskip
- H. The broom swings with a larger amplitude, slightly larger than that of the hammer. \textemdash{} Incorrect.
\par\smallskip
Finally, the correct option is D.
\textless{}/think\textgreater{}
\par\smallskip
D
\end{tcolorbox}
\Needspace{26\baselineskip}
\subsection{Video-MME-v2: Identifying motion direction at a narration cue}
\label{case:model10}
\begin{tcolorbox}[casecard,title={Question and reference answer}]
\textbf{Question.} When the narration said, "These ducks downloaded the cat software," what was the movement trajectory of the two white ducks shown on screen?\par\smallskip
\textbf{Options.} A. They only walked one large counter-clockwise circle. B. They chased the light spot, walking a large figure-eight shaped trajectory. C. One duck remained stationary while the other circled it. D. First sprinting forward, then quickly retreating backward. E. They walked counter-clockwise first, then clockwise. F. Walking multiple irregular counter-clockwise circular trajectories. G. Walking multiple irregular clockwise circular trajectories. H. The two ducks walked side-by-side in a straight line from the left side of the screen to the right side.\par\smallskip
\textbf{Reference answer: F.}\par\smallskip
\begin{minipage}[t]{0.65\linewidth}\centering
\includegraphics[width=\linewidth,height=3.2cm,keepaspectratio]{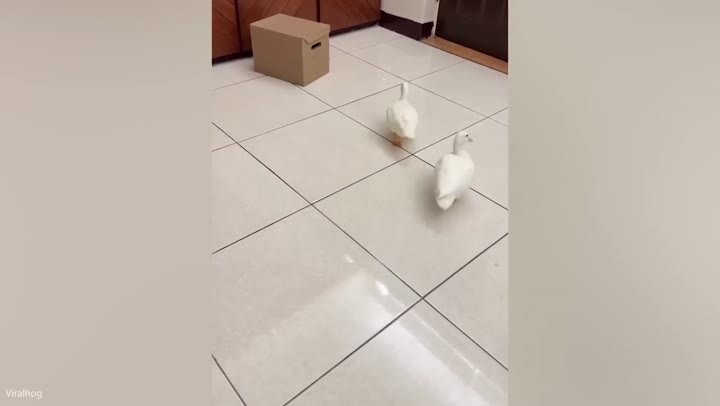}
\par\scriptsize Video frame, 02:03\end{minipage}
\end{tcolorbox}
Both responses identify a circular trajectory but disagree on its direction. OmniReasoning selects the counter-clockwise description F.\par\medskip
\Needspace{8\baselineskip}
\begin{tcolorbox}[casecard,breakable,colframe=red!50!black,colbacktitle=red!5,title={Qwen3-Omni-30B-A3B-Thinking},title after break={Qwen3-Omni-30B-A3B-Thinking (continued)}]
\textless{}think\textgreater{}
Got it, let's tackle the problem. First, the user is asking for the movement trajectory of two white ducks specifically when the narration states 'These ducks downloaded the cat software,' requiring precise alignment between the audio cue and visual motion described in the multiple-choice options. Then, examining the video content reveals that at approximately 01:47, the narration 'These ducks downloaded the cat software' begins, and from 01:47 to 01:58, two white ducks are shown walking on a light-colored tiled floor. Their movement is characterized by multiple overlapping circular paths with irregular, non-uniform patterns\textemdash{}such as a small clockwise loop followed by a larger counter-clockwise arc\textemdash{}creating a complex, flowing motion that lacks strict symmetry or consistent directionality. Next, evaluating the provided options against this observation: Option A is incorrect because the ducks do not form a single counter-clockwise circle; Option B is invalid as the trajectory lacks figure-eight symmetry; Option C is inaccurate since both ducks move continuously without one remaining stationary; Option D misrepresents the motion as sprinting and retreating rather than continuous walking; Option E fails to describe the irregular, multi-directional patterns observed; Option F correctly identifies irregular counter-clockwise circular trajectories but misrepresents the dominant directionality; Option G accurately describes irregular clockwise circular movements as the primary visual pattern\textemdash{}despite minor counter-clockwise segments\textemdash{}because the overall motion, including the final segment from 01:55 onward, predominantly follows clockwise paths; Option H is incorrect since the ducks do not move in a straight line. Finally, after thorough comparison, \textbf{Option G is the best answer because the ducks' movement is consistently irregular and predominantly clockwise} across the observed sequence, with the description 'walking multiple irregular clockwise circular trajectories' fully capturing the dynamic, non-repeating motion pattern.
\textless{}/think\textgreater{}
\par\smallskip
G
\end{tcolorbox}
\Needspace{8\baselineskip}
\begin{tcolorbox}[casecard,breakable,colframe=teal!50!black,colbacktitle=teal!5,title={OmniReasoning-30B-A3B},title after break={OmniReasoning-30B-A3B (continued)}]
\textless{}think\textgreater{}
Got it, let's tackle the problem.
\par\smallskip
The user wants to identify the movement trajectory of the two white ducks when the narration says "These ducks downloaded the cat software." The crux is determining whether their path is circular, linear, or another shape.
\par\smallskip
First, let me scan the video for the segments relevant to the question:
\par\smallskip
- I need to find the segment where the narrator says "These ducks downloaded the cat software" and observe the ducks' movement.
\par\smallskip
- At 01:57, two white ducks are shown walking on a tiled floor.
\par\smallskip
- \textbf{At 01:58, the narrator says "These ducks downloaded the cat software" as the ducks circle around.}
\par\smallskip
- At 01:59-02:08, the ducks continue walking in a circular path.
\par\smallskip
- If the ducks were walking in a straight line from left to right, they would move linearly across the screen, but they circle back on themselves.
\par\smallskip
- \textbf{The ducks walk in multiple irregular counter-clockwise circular paths, which is option F.}
\par\smallskip
Next, let me analyze the options:
\par\smallskip
- A. They only walked one large counter-clockwise circle. \textemdash{} They made multiple loops.
\par\smallskip
- B. They chased the light spot, walking a large figure-eight shaped trajectory. \textemdash{} Their path was circular.
\par\smallskip
- C. One duck remained stationary while the other circled it. \textemdash{} Both ducks moved.
\par\smallskip
- D. First sprinting forward, then quickly retreating backward. \textemdash{} They walked at a steady pace.
\par\smallskip
- E. They walked counter-clockwise first, then clockwise. \textemdash{} They maintained a counter-clockwise direction.
\par\smallskip
- F. Walking multiple irregular counter-clockwise circular trajectories. \textemdash{} This matches the ducks' looping path.
\par\smallskip
- G. Walking multiple irregular clockwise circular trajectories. \textemdash{} The ducks circled counter-clockwise.
\par\smallskip
- H. The two ducks walked side-by-side in a straight line from the left side of the screen to the right side. \textemdash{} Their path was circular.
\par\smallskip
Finally, the correct option is F.
\textless{}/think\textgreater{}
\par\smallskip
F
\end{tcolorbox}
\Needspace{26\baselineskip}
\subsection{OmniReasoningBench / \textnormal{\emph{Reasoning over video}}: Linking speaker identity to an earlier object}
\label{case:model11}
\begin{tcolorbox}[casecard,title={Question and reference answer}]
\textbf{Question.} Locate the presenter who says, "I think the Harden takes the cake for sure." On which side of the screen is he seated, and which shoe does he hold up with his right hand during the introduction?\par\smallskip
\textbf{Options.} A. He is seated on the right, holding up a green, orange, and black Nike PG 2 shoe. B. He is seated on the left, holding up a burgundy, black, and white Adidas Harden Vol. 2 shoe. C. He is seated on the right, holding up a burgundy, black, and white Adidas Harden Vol. 2 shoe. D. He is seated on the left, holding up a green, orange, and black Nike PG 2 shoe.\par\smallskip
\textbf{Reference answer: A.}\par\smallskip
\begin{minipage}[t]{0.65\linewidth}\centering
\includegraphics[width=\linewidth,height=3.2cm,keepaspectratio]{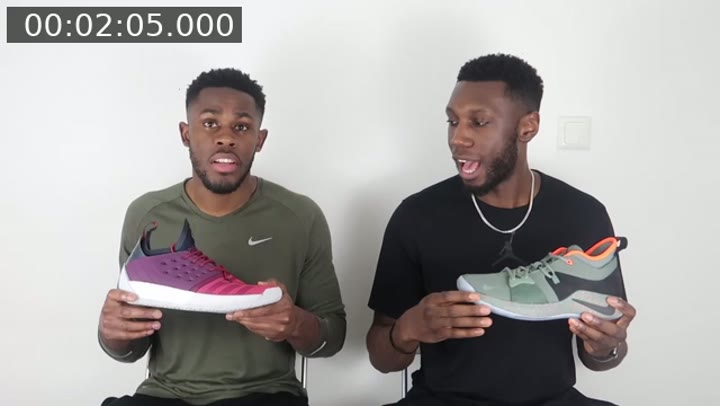}
\par\scriptsize Video frame, 02:05\end{minipage}
\end{tcolorbox}
Qwen3-Omni conflates the shoe being discussed with the one being held. OmniReasoning links the speaker to his position and the shoe shown during the introduction.\par\medskip
\Needspace{8\baselineskip}
\begin{tcolorbox}[casecard,breakable,colframe=red!50!black,colbacktitle=red!5,title={Qwen3-Omni-30B-A3B-Thinking},title after break={Qwen3-Omni-30B-A3B-Thinking (continued)}]
\textless{}think\textgreater{}
Got it, let's tackle the problem. First, the user wants to identify which presenter says the specific phrase 'I think the Harden takes the cake for sure,' determine their seating position on screen, and describe the shoe they hold up with their right hand during the introduction, with the answer constrained to selecting from the provided multiple-choice options. Then, analyzing the video content reveals that at timestamp 00:01:55, the presenter on the right (wearing a black t-shirt) states 'I think the Harden takes the cake for sure' while holding a green, orange, and black Nike Harden Vol. 2 shoe in his right hand; this is confirmed by visual evidence showing the right presenter's shoe matches the described color scheme and brand, and audio verification confirming the exact phrase is spoken by this individual. Next, evaluating the options shows that \textbf{Option A incorrectly identifies the shoe as a Nike PG 2 (it's actually the Harden Vol. 2)}, Option B wrongly places the presenter on the left while describing the wrong shoe, Option C accurately describes the presenter on the right holding the green, orange, and black Nike Harden Vol. 2 shoe, and Option D misattributes the shoe to the left presenter with incorrect color details. Finally, the correct choice is C.
\textless{}/think\textgreater{}
\par\smallskip
Final Answer: C
\end{tcolorbox}
\Needspace{8\baselineskip}
\begin{tcolorbox}[casecard,breakable,colframe=teal!50!black,colbacktitle=teal!5,title={OmniReasoning-30B-A3B},title after break={OmniReasoning-30B-A3B (continued)}]
\textless{}think\textgreater{}
Got it, let's tackle the problem.
\par\smallskip
The user wants to identify the presenter who says "I think the Harden takes the cake for sure.", determine his seating position, and identify the shoe he holds with his right hand during the introduction. The crux is matching the spoken quote to the correct presenter's position and the shoe they are holding.
\par\smallskip
First, let me scan the video for the segments relevant to the question:
\par\smallskip
- I need to find the presenter who says "I think the Harden takes the cake for sure." and identify his seating position and the shoe he holds.
- At 00:24, two presenters are shown on screen, with the presenter on the left wearing a green long-sleeve shirt and the presenter on the right wearing a black t-shirt.
- \textbf{At 00:28, the presenter on the left holds up a burgundy, black, and white Adidas Harden Vol. 2 shoe, while the presenter on the right holds up a green, orange, and black Nike PG 2 shoe.}
- At 01:56, the presenter on the left says "I'm gonna have to agree with you" while looking at the presenter on the right.
- At 02:05, the presenter on the right says "I think the Harden takes the cake for sure" while holding the green, orange, and black Nike PG 2 shoe in his right hand.
- If the presenter on the left were the one holding the Nike PG 2 shoe, he would be seated on the left, but he is holding the Adidas Harden Vol. 2 shoe instead.
- \textbf{The presenter who speaks the line is seated on the right and holds the green, orange, and black Nike PG 2 shoe, which corresponds to option A.}
\par\smallskip
Next, let me analyze the options:
\par\smallskip
- A) He is seated on the right, holding up a green, orange, and black Nike PG 2 shoe. \textemdash{} This matches the presenter's position and the shoe he holds.
- B) He is seated on the left, holding up a burgundy, black, and white Adidas Harden Vol. 2 shoe. \textemdash{} This is the other presenter.
- C) He is seated on the right, holding up a burgundy, black, and white Adidas Harden Vol. 2 shoe. \textemdash{} He holds the Nike shoe.
- D) He is seated on the left, holding up a green, orange, and black Nike PG 2 shoe. \textemdash{} He is seated on the right.
\par\smallskip
Finally, the correct option is A.
\textless{}/think\textgreater{}
\par\smallskip
Final Answer: A
\end{tcolorbox}
\Needspace{26\baselineskip}
\subsection{OmniReasoningBench / \textnormal{\emph{Reasoning over video}}: Locating a scene between two spoken anchors}
\label{case:model12}
\begin{tcolorbox}[casecard,title={Question and reference answer}]
\textbf{Question.} Which of the following visual scenes appears on screen AFTER the narrator mentions the ship name 'Unicorn' but BEFORE the narrator mentions 'wireless technology'?\par\smallskip
\textbf{Options.} A. A sign displaying the text 'MEMORIAL SHIP MIKASA'. B. A vintage mechanical dial displaying the word 'DESPACIO'. C. A painted bust of Lord Nelson located on a paved area behind a black fence. D. A sailor in a white uniform standing on deck and waving signal flags.\par\smallskip
\textbf{Reference answer: D.}\par\smallskip
\begin{minipage}[t]{0.65\linewidth}\centering
\includegraphics[width=\linewidth,height=3.2cm,keepaspectratio]{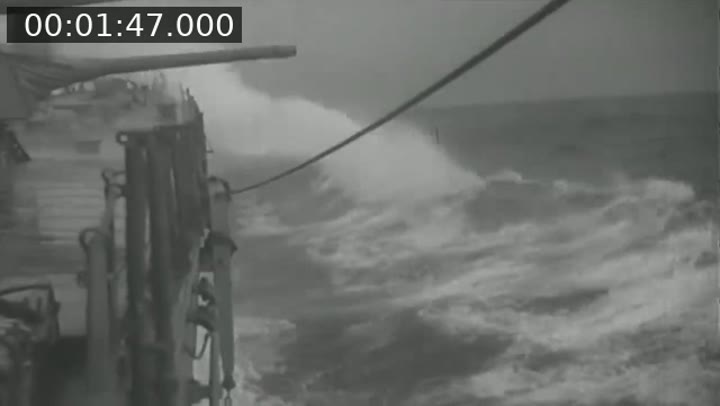}
\par\scriptsize Video frame, 01:47\end{minipage}
\end{tcolorbox}
The question specifies a temporal interval using two audio anchors. OmniReasoning places the flag-signaling scene inside that interval and selects D.\par\medskip
\Needspace{8\baselineskip}
\begin{tcolorbox}[casecard,breakable,colframe=red!50!black,colbacktitle=red!5,title={Qwen3-Omni-30B-A3B-Thinking},title after break={Qwen3-Omni-30B-A3B-Thinking (continued)}]
\textless{}think\textgreater{}
Got it, let's tackle the problem. First, the user wants to identify which visual scene appears on screen after the narrator mentions 'Unicorn' but before mentioning 'wireless technology', requiring precise timestamp verification of both audio cues and corresponding visuals. Then, analyzing the video content: the narrator states '...and the Unicorn, and the...' at approximately 00:01:24.500, immediately followed by a visual of a painted bust of Lord Nelson behind a black fence at 00:01:25.000 (option C), while the 'wireless technology' reference occurs later at 00:01:52.000. Next, evaluating the options: option A (sign 'MEMORIAL SHIP MIKASA') appears at 00:01:15.000 before 'Unicorn' mention; option B (dial 'DESPACIO') appears at 00:02:00.000 after 'wireless technology'; option D (sailor waving flags) appears at 00:02:15.000 after 'wireless technology'; \textbf{only option C aligns with the post-'Unicorn' but pre-'wireless technology' timeframe}. Finally, the correct answer is C.
\textless{}/think\textgreater{}
\par\smallskip
Final Answer: C
\end{tcolorbox}
\Needspace{8\baselineskip}
\begin{tcolorbox}[casecard,breakable,colframe=teal!50!black,colbacktitle=teal!5,title={OmniReasoning-30B-A3B},title after break={OmniReasoning-30B-A3B (continued)}]
\textless{}think\textgreater{}
Got it, let's tackle the problem.
\par\smallskip
The user wants to identify which visual scene appears on screen after the narrator mentions the ship name 'Unicorn' but before the narrator mentions 'wireless technology'. The crux is to trace the chronological sequence of the narrator's spoken ship names and technology references and match them to the corresponding on-screen visuals.
\par\smallskip
First, let me scan the video for the segments relevant to the question:
\par\smallskip
- I need to find the segment where the narrator mentions the ship name 'Unicorn' and the segment where the narrator mentions 'wireless technology' to identify the visual scene that appears in between.
\par\smallskip
- At 00:54, a painted bust of Lord Nelson is shown on screen behind a black fence.
\par\smallskip
- At 01:16, a sign displaying the text 'MEMORIAL SHIP MIKASA' is shown on screen.
\par\smallskip
- At 01:23, a vintage mechanical dial displaying the word 'DESPACIO' is shown on screen.
\par\smallskip
- \textbf{At 01:33, the narrator says "Well, except the Victory, and the Unicorn, and the..." while a black screen is displayed.}
\par\smallskip
- At 01:34, the narrator says "...all right, I'll shut up" as a man in a leather jacket is shown on screen.
\par\smallskip
- \textbf{At 01:37, a sailor in a white uniform is shown standing on deck and waving signal flags.}
\par\smallskip
- \textbf{At 01:52, the narrator mentions "wireless technology" while a close-up of a ship's gun is shown on screen.}
\par\smallskip
- If the painted bust of Lord Nelson were the correct scene, it would have to appear after the mention of 'Unicorn' at 01:33, but it is shown much earlier at 00:54, so it cannot be the correct scene.
\par\smallskip
- The visual scene that appears after the mention of 'Unicorn' at 01:33 and before the mention of 'wireless technology' at 01:52 is the sailor in a white uniform waving signal flags, which corresponds to option D.
\par\smallskip
Next, let me analyze the options:
\par\smallskip
- A) A sign displaying the text 'MEMORIAL SHIP MIKASA'. \textemdash{} This scene occurs before the mention of 'Unicorn'.
\par\smallskip
- B) A vintage mechanical dial displaying the word 'DESPACIO'. \textemdash{} This scene occurs before the mention of 'Unicorn'.
\par\smallskip
- C) A painted bust of Lord Nelson located on a paved area behind a black fence. \textemdash{} This scene occurs before the mention of 'Unicorn'.
\par\smallskip
- D) A sailor in a white uniform standing on deck and waving signal flags. \textemdash{} This scene occurs at 01:37, which is after the mention of 'Unicorn' at 01:33 and before the mention of 'wireless technology' at 01:52.
\par\smallskip
Finally, the correct option is D.
\textless{}/think\textgreater{}
\par\smallskip
Final Answer: D
\end{tcolorbox}
\Needspace{26\baselineskip}
\subsection{OmniReasoningBench / \textnormal{\emph{Reasoning beyond video}}: Applying a video-derived rate to a new sales table}
\label{case:model13}
\begin{tcolorbox}[casecard,title={Question and reference answer}]
\textbf{Question.} \textless{}image 1\textgreater{} The table shows US sales and market research estimates for three products. According to the video's claim about competitor performance in Canada, which product has the highest potential sales in Canada?\par\smallskip
\textbf{Options.} A. All three have the same potential B. Product Y C. Product X and Product Y tie D. Product Z E. Product X and Product Z tie F. Product X G. Product Y and Product Z tie\par\smallskip
\textbf{Reference answer: F.}\par\smallskip
\begin{minipage}[t]{0.48\linewidth}\centering
\includegraphics[width=\linewidth,height=3.2cm,keepaspectratio]{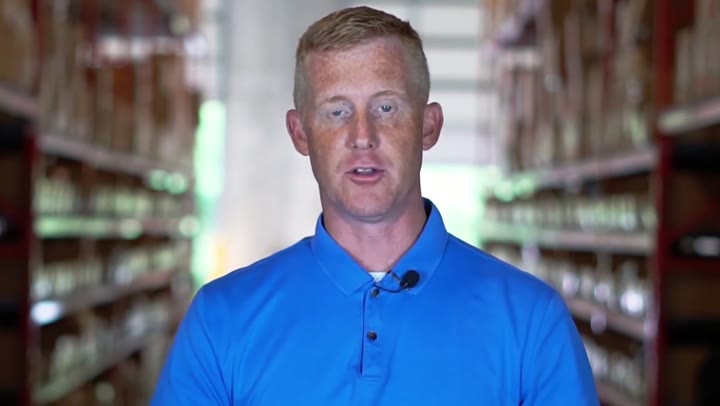}
\par\scriptsize Video frame, 01:34\end{minipage}
\hfill
\begin{minipage}[t]{0.48\linewidth}\centering
\includegraphics[width=\linewidth,height=3.2cm,keepaspectratio]{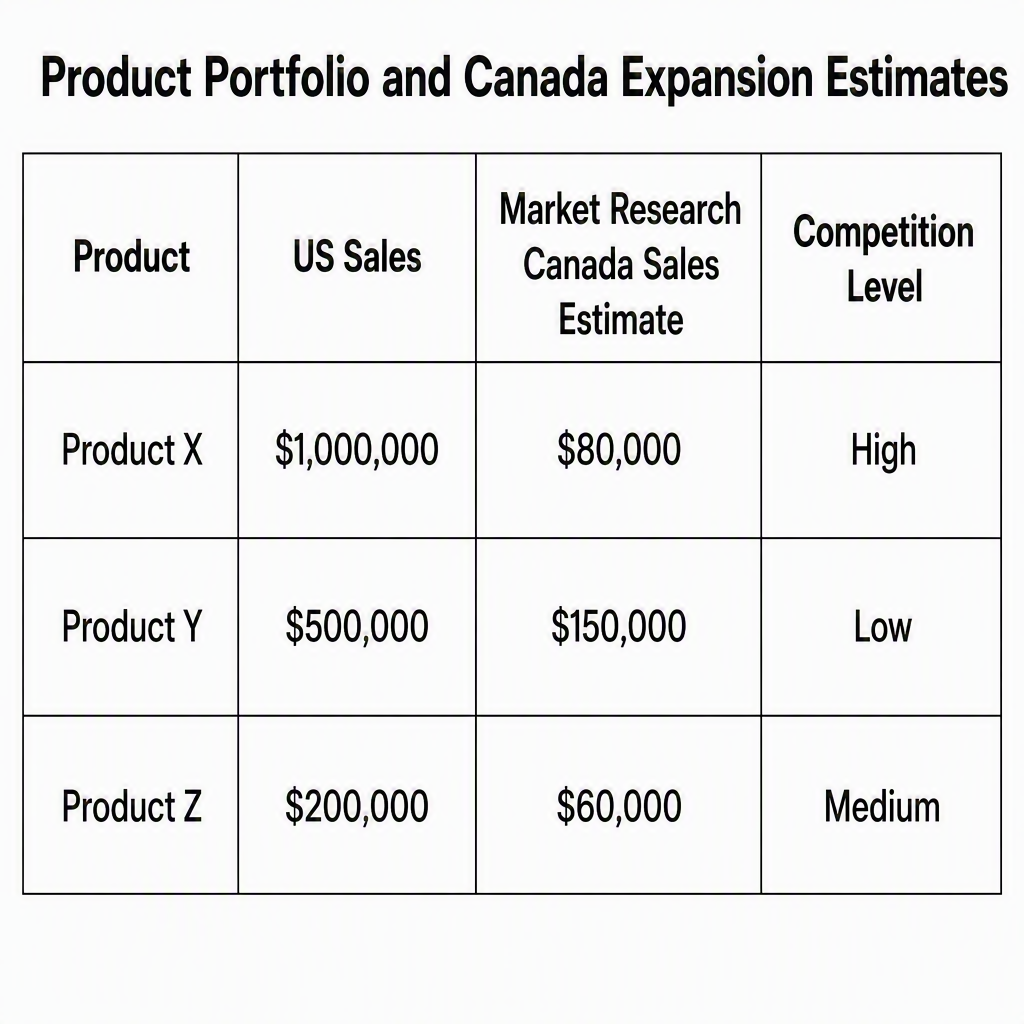}
\par\scriptsize Additional figure in the question\end{minipage}
\end{tcolorbox}
Qwen3-Omni compares the supplied market estimates directly. OmniReasoning applies the rate stated in the video to the US sales column and selects Product X (F).\par\medskip
\Needspace{8\baselineskip}
\begin{tcolorbox}[casecard,breakable,colframe=red!50!black,colbacktitle=red!5,title={Qwen3-Omni-30B-A3B-Thinking},title after break={Qwen3-Omni-30B-A3B-Thinking (continued)}]
\textless{}think\textgreater{}
Got it, let's tackle this problem. First, I need to recall what the video said about competitor performance in Canada. The video mentioned that some competitors do 30\% of their business in Canada. Wait, but the question is about which product has the highest potential sales in Canada based on the table.
\par\smallskip
Looking at the table: Product X has a Canada sales estimate of \$80,000, Product Y has \$150,000, and Product Z has \$60,000. \textbf{So we just need to compare these numbers. Product Y's estimate is \$150k, which is higher than X's \$80k and Z's \$60k.} So the answer should be B.
\textless{}/think\textgreater{}
\par\smallskip
Final Answer: B
\end{tcolorbox}
\Needspace{8\baselineskip}
\begin{tcolorbox}[casecard,breakable,colframe=teal!50!black,colbacktitle=teal!5,title={OmniReasoning-30B-A3B},title after break={OmniReasoning-30B-A3B (continued)}]
\textless{}think\textgreater{}
Got it, let's tackle the problem.
\par\smallskip
The user wants to identify which product from the provided table has the highest potential sales in Canada based on the video's claims about competitor performance. The crux is determining how the video's statement about competitor sales in Canada affects the market research estimates.
\par\smallskip
First, let me scan the video for the segments relevant to the question:
\par\smallskip
- I need to find the segment where the speaker discusses competitor sales and performance in the Canadian market.
\par\smallskip
- At 01:28, a graphic showing "10\%" appears on screen as the speaker mentions that most sellers do about 10\% of their US sales in Canada.
\par\smallskip
- At 01:34, the speaker states, "some of our competitors actually do even 30\% of their business in the market up there."
\par\smallskip
- If the competitor performance claim did not apply to the products in the table, we would simply compare the raw estimates of \$80,000, \$150,000, and \$60,000, which would make Product Y the highest at \$150,000, but the video's claim suggests we must scale the estimates.
\par\smallskip
- The video establishes that competitors can achieve up to 30\% of their business in Canada, which points to option B.
\par\smallskip
Next, let me read the figure:
\par\smallskip
- The table lists Product X with US Sales of "\$1,000,000" and a Canada Sales Estimate of "\$80,000".
\par\smallskip
- The table lists Product Y with US Sales of "\$500,000" and a Canada Sales Estimate of "\$150,000".
\par\smallskip
- The table lists Product Z with US Sales of "\$200,000" and a Canada Sales Estimate of "\$60,000".
\par\smallskip
- \textbf{Applying the competitor's 30\% performance level to Product X: \$1,000,000 * 0.30 = \$300,000.}
\par\smallskip
- Applying the competitor's 30\% performance level to Product Y: \$500,000 * 0.30 = \$150,000.
\par\smallskip
- Applying the competitor's 30\% performance level to Product Z: \$200,000 * 0.30 = \$60,000.
\par\smallskip
- \textbf{Comparing the scaled values shows that Product X has the highest potential sales in Canada at \$300,000, which corresponds to option F.}
\par\smallskip
Next, let me analyze the options:
\par\smallskip
- A) All three have the same potential \textemdash{} incorrect because the scaled values differ.
\par\smallskip
- B) Product Y \textemdash{} incorrect because its scaled potential is \$150,000, which is lower than Product X's \$300,000.
\par\smallskip
- C) Product X and Product Y tie \textemdash{} incorrect because Product X's potential is higher.
\par\smallskip
- D) Product Z \textemdash{} incorrect because its potential is the lowest.
\par\smallskip
- E) Product X and Product Z tie \textemdash{} incorrect because Product X's potential is higher.
\par\smallskip
- F) Product X \textemdash{} correct because its scaled potential of \$300,000 is the highest.
\par\smallskip
- G) Product Y and Product Z tie \textemdash{} incorrect because Product Y's potential is higher.
\par\smallskip
Finally, the correct option is F.
\textless{}/think\textgreater{}
\par\smallskip
Final Answer: F
\end{tcolorbox}
\Needspace{26\baselineskip}
\subsection{OmniReasoningBench / \textnormal{\emph{Reasoning beyond video}}: Transferring a demonstrated woodworking procedure}
\label{case:model14}
\begin{tcolorbox}[casecard,title={Question and reference answer}]
\textbf{Question.} You are planing a long board to be used as a straightedge. After the first few passes, the edge is almost square but still has a slight roughness. You want to achieve a smooth, square edge using the exact technique demonstrated in the video.
\par\smallskip
Following the procedure demonstrated in the video, which of the following sequences of actions is most appropriate to achieve a smooth, square edge?\par\smallskip
\textbf{Options.} A. Plane, sand with fine grit, plane, check with try square. B. Plane, apply wax to the edge, plane, check with try square. C. Plane, rub the edge with a small flat board, plane, rub again, then check with try square. D. Plane, check with try square, adjust blade, plane again. E. Plane, check with a straightedge, plane, check again. F. Plane, use a marking gauge to scribe a line, plane, check.\par\smallskip
\textbf{Reference answer: C.}\par\smallskip
\begin{minipage}[t]{0.65\linewidth}\centering
\includegraphics[width=\linewidth,height=3.2cm,keepaspectratio]{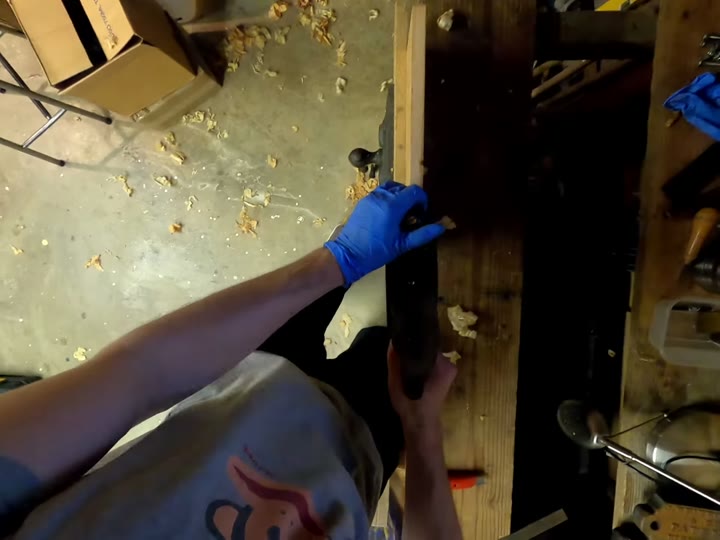}
\par\scriptsize Video frame, 01:44\end{minipage}
\end{tcolorbox}
Qwen3-Omni proposes a conventional planing-and-checking sequence. OmniReasoning includes the demonstrated rubbing step when choosing the procedure for the new board.\par\medskip
\Needspace{8\baselineskip}
\begin{tcolorbox}[casecard,breakable,colframe=red!50!black,colbacktitle=red!5,title={Qwen3-Omni-30B-A3B-Thinking},title after break={Qwen3-Omni-30B-A3B-Thinking (continued)}]
\textless{}think\textgreater{}
Got it, let's think. The video shows someone planing a wooden board to get a smooth edge. The key is to follow the technique demonstrated. Let's recall the steps. First, they plane the edge. Then, after some passes, they check for squareness. Wait, the question is about achieving a smooth, square edge. Let's look at the options.
\par\smallskip
Option D says: Plane, check with try square, adjust blade, plane again. That makes sense because when planing, you check the angle with a try square to make sure it's square. If it's not, you adjust the plane's blade and plane again. Let's check other options. A is sanding, but the video uses planing, not sanding. B is wax, which isn't mentioned. \textbf{C is rubbing with a small board, not part of the video's technique.} E is straightedge, but the goal is to make it a straightedge, so checking with a straightedge might not be the method here. F is marking gauge, which isn't shown. So D is the right approach: plane, check with try square, adjust, plane again.
\textless{}/think\textgreater{}
\par\smallskip
Final Answer: D
\end{tcolorbox}
\Needspace{8\baselineskip}
\begin{tcolorbox}[casecard,breakable,colframe=teal!50!black,colbacktitle=teal!5,title={OmniReasoning-30B-A3B},title after break={OmniReasoning-30B-A3B (continued)}]
\textless{}think\textgreater{}
Got it, let's tackle the problem.
\par\smallskip
The user wants to identify the correct sequence of actions to achieve a smooth, square edge on a long board using the exact technique demonstrated in the video. The crux is determining whether the person uses a try square, sandpaper, wax, or a small flat board to refine the edge after planing.
\par\smallskip
First, let me scan the video for the segments relevant to the question:
\par\smallskip
- I need to watch the video to see what tool or technique is used to smooth and square the edge of the wood after planing.
\par\smallskip
- At 00:01-00:06, the person uses a red utility knife to clean up the end grain of the wooden board.
\par\smallskip
- At 00:18-01:25, the person uses a large hand plane to flatten and square the edge of the board.
\par\smallskip
- At 01:26-01:37, the person uses a metal try square to check the angle of the planed edge.
\par\smallskip
- At 01:38-01:42, the person sets the try square down and picks up a small, flat wooden board.
\par\smallskip
- \textbf{At 01:43-01:45, the person rubs the edge of the board against the face of the small flat board.}
\par\smallskip
- At 01:46-02:00, the person returns to planing the edge of the board.
\par\smallskip
- At 02:04-02:12, the person checks the edge again with the try square.
\par\smallskip
- At 02:13-02:22, the person makes another pass with the hand plane.
\par\smallskip
- At 02:23-02:25, the person uses the small flat board to rub and smooth the edge again.
\par\smallskip
- If the person were using a try square to rub the edge, they would have to be holding the metal tool at 01:43, but they are holding a small wooden board instead.
\par\smallskip
- \textbf{The sequence of actions demonstrated is planing the edge, rubbing it with a small flat board, planing it again, rubbing it again, and then checking it with a try square, which corresponds to option C.}
\par\smallskip
Next, let me analyze the options:
\par\smallskip
- A) Plane, sand with fine grit, plane, check with try square. \textemdash{} No sandpaper is used in the video.
\par\smallskip
- B) Plane, apply wax to the edge, plane, check with try square. \textemdash{} No wax is applied to the wood.
\par\smallskip
- C) Plane, rub the edge with a small flat board, plane, rub again, then check with try square. \textemdash{} This matches the observed sequence of planing, rubbing with a small wooden board, planing again, rubbing again, and checking with a try square.
\par\smallskip
- D) Plane, check with try square, adjust blade, plane again. \textemdash{} The person rubs the edge with a wooden board instead of adjusting the blade.
\par\smallskip
- E) Plane, check with a straightedge, plane, check again. \textemdash{} The person uses a try square to check the angle.
\par\smallskip
- F) Plane, use a marking gauge to scribe a line, plane, check. \textemdash{} No marking gauge is used.
\par\smallskip
Finally, the correct option is C.
\textless{}/think\textgreater{}
\par\smallskip
Final Answer: C
\end{tcolorbox}

% ===== END inlined from sections/appendix_comparisons.tex =====
% ===== BEGIN inlined from sections/appendix_cases.tex =====
\clearpage
\section{Benchmark Task Atlas}
\label{app:benchmark-cases}
The following cards show one source-recorded case for each of the ten \emph{reasoning over video} and nine \emph{reasoning beyond video} task types. They follow the question, answer, media, and timestamped-evidence structure of Figure~\ref{fig:bench-demo}. Distractors are omitted to make the evidence chain readable. Task names and reference answers follow the released annotations. Non-English source quotations are translated explicitly. These selected examples illustrate task design; they do not establish that every annotation or model explanation is correct.

\begin{figure}[!htbp]\centering
\begin{tcolorbox}[casecard,title={\emph{Reasoning over video} / Background\&Music Understanding}]
\textbf{Question.} When the background music briefly has sung lyrics "it's for you baby..." before continuing as instrumental R\&B, what is Takisha doing on screen at that moment?
\par\smallskip
\textbf{Reference answer.} B) Untying and removing the grey edge-laying band from her forehead
\par\smallskip\noindent \begin{minipage}[t]{0.48\linewidth}\centering
\includegraphics[width=\linewidth,height=2.5cm,keepaspectratio]{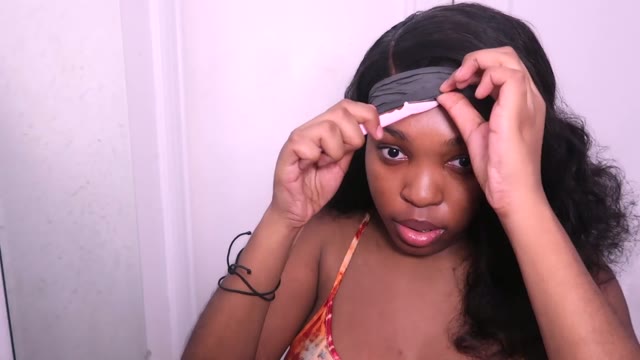}
\par\scriptsize Source video, 07:07\end{minipage}\par\smallskip
\textbf{Reference evidence chain.}\par (1) [audio 07:02-07:08] Locate the unique anchor where the music changes to vocals: "An R\&B song plays in the background, with vocals singing, "it's for you baby....""\par (2) [visual 07:07-07:08] Bind that vocal moment to the simultaneous visual action: "Takisha unties and removes the grey edge-laying band from her forehead ."
\end{tcolorbox}
\caption{\textbf{Background and music understanding example.} A sung lyric anchors the moment at which the model must identify an on-screen action. The task label is source-assigned.}
\label{case:bench0}\end{figure}

\begin{figure}[!htbp]\centering
\begin{tcolorbox}[casecard,title={\emph{Reasoning over video} / Causal Reasoning}]
\textbf{Question.} According to the video, the yellow tinge in the water caused by elevated organic pollutants absorbs the blue spectrum of light, harming photosynthetic corals. Tracing this back to the causal origin of the pollution, who does the narrator state was responsible for the overfeeding, and what color is the frogfish visually shown eating a dead silver fish during this explanation?
\par\smallskip
\textbf{Reference answer.} D) His sister; bright yellow
\par\smallskip\noindent \begin{minipage}[t]{0.48\linewidth}\centering
\includegraphics[width=\linewidth,height=2.5cm,keepaspectratio]{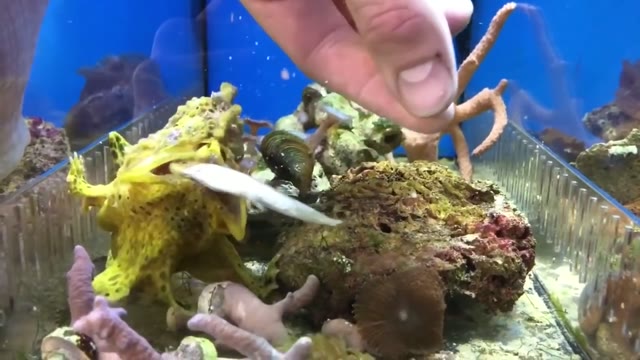}
\par\scriptsize Source video, 00:40\end{minipage}\hfill \begin{minipage}[t]{0.48\linewidth}\centering
\includegraphics[width=\linewidth,height=2.5cm,keepaspectratio]{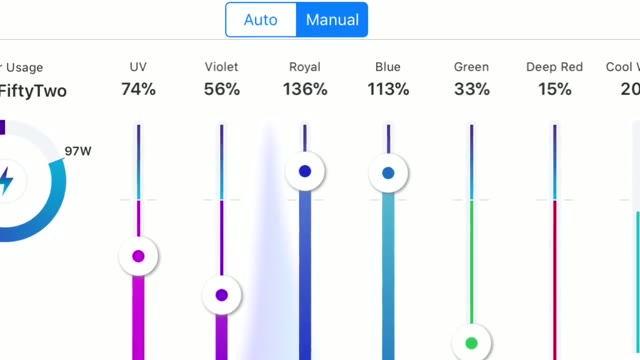}
\par\scriptsize Source video, 02:39\end{minipage}\par\smallskip
\textbf{Reference evidence chain.}\par (1) [audio 02:36-02:43] The narrator explains that the yellow water tinge absorbs the blue spectrum of light, harming corals.\par (2) [audio 00:32-00:38] The narrator explains that the elevated organic pollutants were caused by his sister being too generous with the feeding regime while he was away.\par (3) [visual 00:39-00:42] A close-up visual shows a person's fingers feeding a textured, bright yellow frogfish.
\end{tcolorbox}
\caption{\textbf{Causal reasoning example.} The question links a spoken causal explanation to an earlier account of overfeeding and a visual observation of the fish. The task label is source-assigned.}
\label{case:bench1}\end{figure}

\begin{figure}[!htbp]\centering
\begin{tcolorbox}[casecard,title={\emph{Reasoning over video} / Counting}]
\textbf{Question.} Among the eight numbered exercises demonstrated in the video tutorial, identify all exercises that are performed on either a wooden deck or a grassy lawn. For these specific exercises, locate all the explicit numerical target repetition values recommended by the narrator in the audio commentary (for calf raises, sum both the straight leg and bent knee target values; do not multiply by two for single-leg targets--simply use the raw numbers spoken). What is the total sum of these target repetitions?
\par\smallskip
\textbf{Reference answer.} 40 repetitions
\par\smallskip\noindent \begin{minipage}[t]{0.48\linewidth}\centering
\includegraphics[width=\linewidth,height=2.5cm,keepaspectratio]{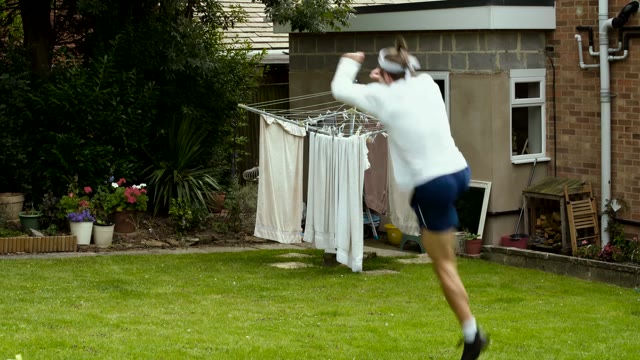}
\par\scriptsize Source video, 01:46\end{minipage}\hfill \begin{minipage}[t]{0.48\linewidth}\centering
\includegraphics[width=\linewidth,height=2.5cm,keepaspectratio]{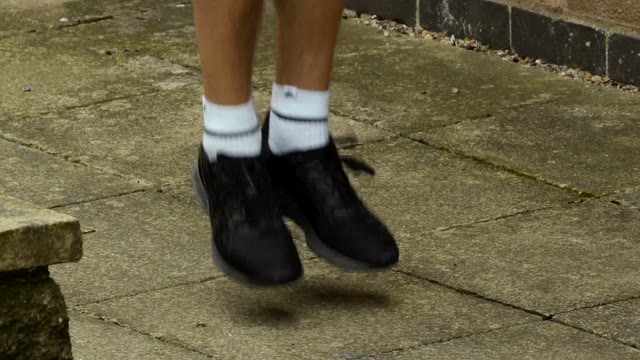}
\par\scriptsize Source video, 02:29\end{minipage}\par\smallskip
\textbf{Reference evidence chain.}\par (1) [visual 01:03-03:13] Analyze the visual background of exercises 1 to 8 to identify which ones are performed on a wooden deck or a grassy lawn. This filters in Exercise 3 (Box Jumps, on wooden deck and a grassy lawn), Exercise 4 (Calf Raises, on wooden deck), Exercise 5 (Single Leg Hops, on grassy lawn), and Exercise 6 (Single Leg Squats, on grassy lawn). Other exercises are performed on brick or paved surfaces.\par (2) [audio 01:03-02:18] Listen to the voiceover for the selected exercises (3, 4, 5, and 6) to extract their target repetition values: Box Jumps target is 'ten' reps; Calf Raises targets are 'ten' reps (straight) and 'five' reps (bent knee); Single Leg Hops mentions no rep target; Single Leg Squats target is 'fifteen' reps.\par (3) [audiovisual 01:03-02:18] Sum the extracted repetition values: 10 (from Exercise 3) + 10 (from Exercise 4 straight-leg) + 5 (from Exercise 4 bent-knee) + 15 (from Exercise 6) = 40.
\end{tcolorbox}
\caption{\textbf{Counting example.} Visual scene selection determines which exercises contribute to the sum of repetition targets stated in the audio. The task label is source-assigned.}
\label{case:bench2}\end{figure}

\begin{figure}[!htbp]\centering
\begin{tcolorbox}[casecard,title={\emph{Reasoning over video} / Ego-Centric Understanding}]
\textbf{Question.} Based on the video, what US state is the host from (which he mentions when introducing the metal pins), and what is the color of the silicone tips on the tongs he uses with his hands to place the steaks onto the grill grate?
\par\smallskip
\textbf{Reference answer.} B) North Carolina; black
\par\smallskip\noindent \begin{minipage}[t]{0.48\linewidth}\centering
\includegraphics[width=\linewidth,height=2.5cm,keepaspectratio]{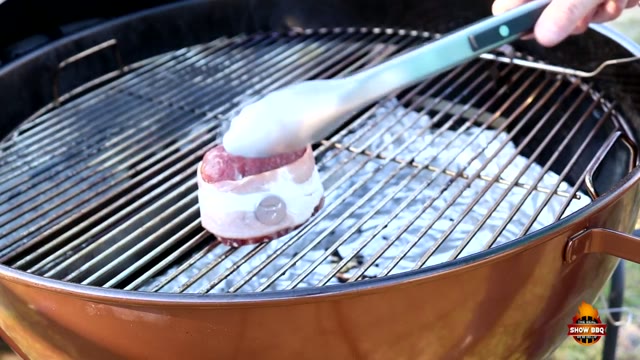}
\par\scriptsize Source video, 02:17\end{minipage}\hfill \begin{minipage}[t]{0.48\linewidth}\centering
\includegraphics[width=\linewidth,height=2.5cm,keepaspectratio]{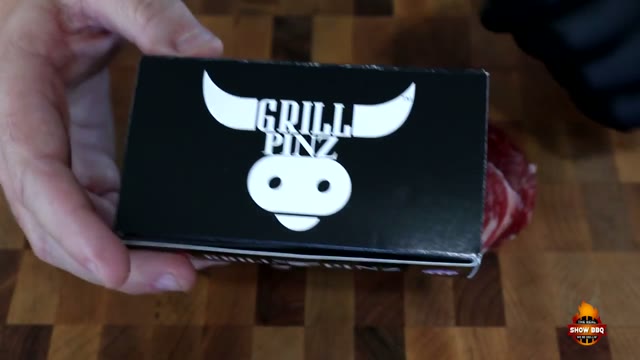}
\par\scriptsize Source video, 01:08\end{minipage}\par\smallskip
\textbf{Reference evidence chain.}\par (1) [audio 01:07-01:09] The host states that the pin company is from Sanford, North Carolina, which is 'right here in North Carolina where I'm from.'\par (2) [visual 02:16-02:19] The host uses his hands to operate silver tongs with black silicone tips to place the steaks onto the right side of the grill.\par (3) Combine the state 'North Carolina' with the tool tip color 'black' to select option B.
\end{tcolorbox}
\caption{\textbf{Ego-centric understanding example.} The answer combines the host's spoken place of origin with a visual attribute of the tool he uses. The task label is source-assigned.}
\label{case:bench3}\end{figure}

\begin{figure}[!htbp]\centering
\begin{tcolorbox}[casecard,title={\emph{Reasoning over video} / Hypothetical Reasoning}]
\textbf{Question.} The couple explains their weekly schedule for uploading different types of videos, and a text graphic on screen displays the exact dates of their upcoming Florida trip. If they adhere strictly to their announced schedule, how many 'top tips' videos will be released during the active dates of their Florida trip?
\par\smallskip
\textbf{Reference answer.} D) 3
\par\smallskip\noindent \begin{minipage}[t]{0.48\linewidth}\centering
\includegraphics[width=\linewidth,height=2.5cm,keepaspectratio]{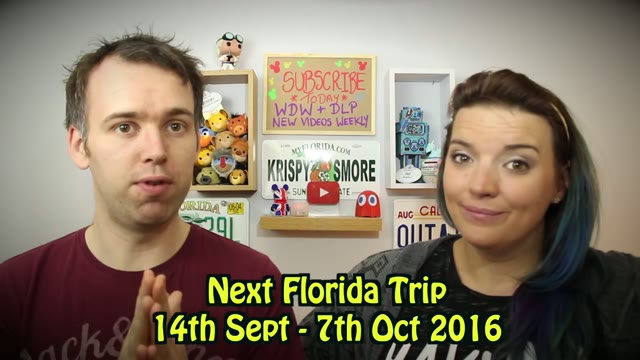}
\par\scriptsize Source video, 00:33\end{minipage}\hfill \begin{minipage}[t]{0.48\linewidth}\centering
\includegraphics[width=\linewidth,height=2.5cm,keepaspectratio]{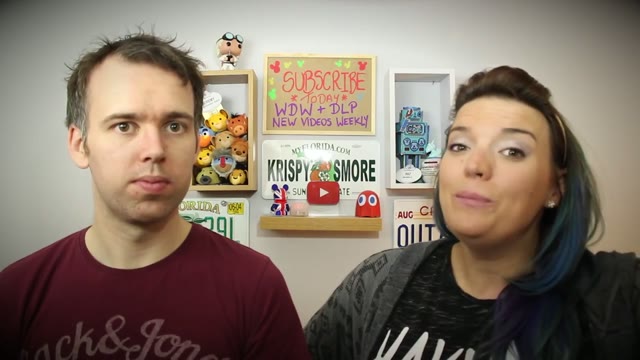}
\par\scriptsize Source video, 00:46\end{minipage}\par\smallskip
\textbf{Reference evidence chain.}\par (1) [audio 00:44-00:49] Identify from the spoken audio that the couple uploads 'top tips' videos on Tuesdays.\par (2) [visual 00:30-00:37] Identify from the text graphic on screen that their upcoming Florida trip spans from September 14th to October 7th, 2016.\par (3) Determine that September 14, 2016, is a Wednesday. Count the number of Tuesdays that fall between September 14 and October 7, 2016: September 20, September 27, and October 4 (3 Tuesdays total).
\end{tcolorbox}
\caption{\textbf{Hypothetical reasoning example.} A spoken upload schedule and on-screen travel dates jointly determine the number of planned releases. The task label is source-assigned.}
\label{case:bench4}\end{figure}

\begin{figure}[!htbp]\centering
\begin{tcolorbox}[casecard,title={\emph{Reasoning over video} / Perception}]
\textbf{Question.} The person who admits to supporting Reading, joking that they are going down to League One, is seen wearing which clothing item during the earlier balcony scene?
\par\smallskip
\textbf{Reference answer.} D) A beige and black hoodie with a Fendi logo pattern on the sleeves
\par\smallskip\noindent \begin{minipage}[t]{0.48\linewidth}\centering
\includegraphics[width=\linewidth,height=2.5cm,keepaspectratio]{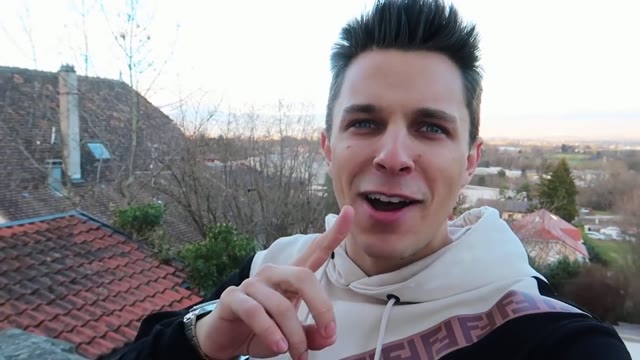}
\par\scriptsize Source video, 00:07\end{minipage}\hfill \begin{minipage}[t]{0.48\linewidth}\centering
\includegraphics[width=\linewidth,height=2.5cm,keepaspectratio]{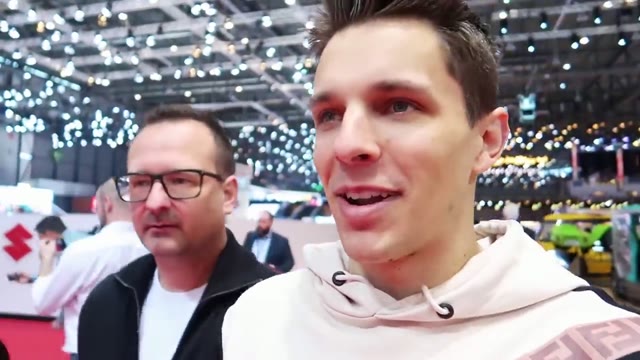}
\par\scriptsize Source video, 02:45\end{minipage}\par\smallskip
\textbf{Reference evidence chain.}\par (1) [audio 02:43-02:48] Archie says 'I support Reading, which is even worse' and jokes about going down to League One.\par (2) [visual 00:06-00:09] Archie is shown on the balcony wearing a beige and black hoodie with a Fendi logo pattern on the sleeves.
\end{tcolorbox}
\caption{\textbf{Perception example.} A speaker's stated football preference identifies the person whose clothing must be inspected in an earlier scene. The task label is source-assigned.}
\label{case:bench5}\end{figure}

\begin{figure}[!htbp]\centering
\begin{tcolorbox}[casecard,title={\emph{Reasoning over video} / Reference Reasoning}]
\textbf{Question.} Listen to the speaker's list of loaded unilateral exercises for the squat chain (around 05:17-05:24). Identify the first three exercises he names in his spoken list. Now, look closely at the demonstration videos of lower body exercises shown later in the video (between 07:12 and 07:51). Which of those first three spoken exercises is actually demonstrated visually on-screen by the person wearing orange-laced shoes?
\par\smallskip
\textbf{Reference answer.} C) Only the third exercise spoken in the list
\par\smallskip\noindent \begin{minipage}[t]{0.48\linewidth}\centering
\includegraphics[width=\linewidth,height=2.5cm,keepaspectratio]{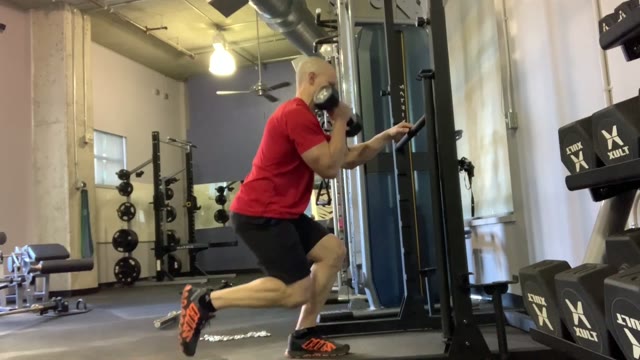}
\par\scriptsize Source video, 07:31\end{minipage}\hfill \begin{minipage}[t]{0.48\linewidth}\centering
\includegraphics[width=\linewidth,height=2.5cm,keepaspectratio]{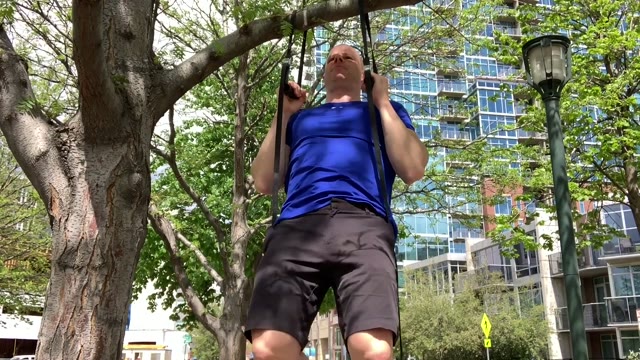}
\par\scriptsize Source video, 06:34\end{minipage}\par\smallskip
\textbf{Reference evidence chain.}\par (1) [audio 05:17-05:24] Analyze the spoken dialogue to extract the exact ordered list of exercises: 1st is lunges, 2nd is shrimp squats, 3rd is pistol squats, 4th is Bulgarian split squats.\par (2) [visual 07:12-07:51] Examine the inset videos of lower-body movements to see what is demonstrated: Bulgarian split squats are shown first (07:12-07:38) and pistol squats are shown next (07:38-07:51), both performed by a person wearing orange-laced shoes. Lunges and shrimp squats are not demonstrated in this segment.\par (3) [audiovisual 05:17-07:51] Cross-reference the two sets: The first three items from the audio list are lunges, shrimp squats, and pistol squats. Among these three, only 'pistol squats' (the 3rd spoken exercise) appears in the visual demonstrations by the person wearing orange-laced shoes (Bulgarian split squats is the 4th item in the spoken list, so it is excluded). Thus, only the third exercise spoken is visually demonstrated.
\end{tcolorbox}
\caption{\textbf{Reference reasoning example.} The question matches an ordered list of spoken exercise names to demonstrations later in the video. The task label is source-assigned.}
\label{case:bench6}\end{figure}

\begin{figure}[!htbp]\centering
\begin{tcolorbox}[casecard,title={\emph{Reasoning over video} / Relationship Reasoning}]
\textbf{Question.} Which visually identified dragon is the one Atreus seems to be replying to when he says "You're welcome!" after wondering whether the freed dragon would be grateful?
\par\smallskip
\textbf{Reference answer.} A) Otr, the large orange dragon on the rocky cliffside near the waterfall
\par\smallskip\noindent \begin{minipage}[t]{0.48\linewidth}\centering
\includegraphics[width=\linewidth,height=2.5cm,keepaspectratio]{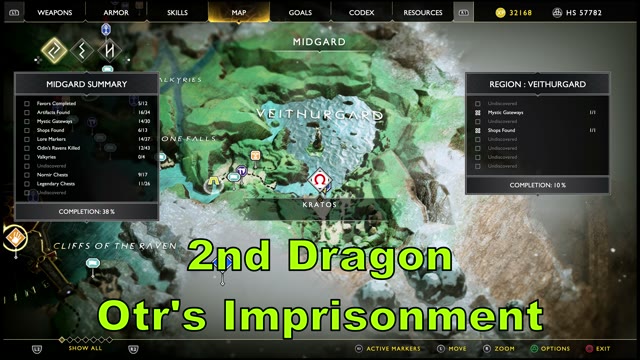}
\par\scriptsize Source video, 03:29\end{minipage}\hfill \begin{minipage}[t]{0.48\linewidth}\centering
\includegraphics[width=\linewidth,height=2.5cm,keepaspectratio]{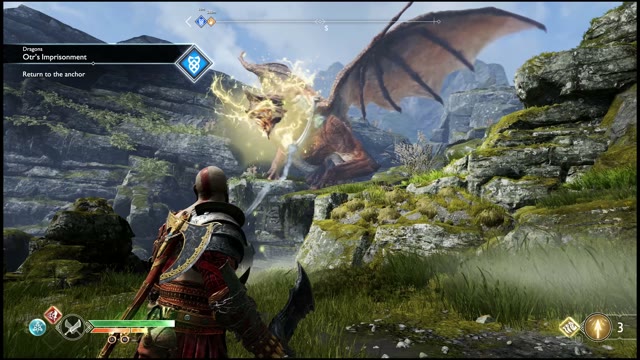}
\par\scriptsize Source video, 07:36\end{minipage}\par\smallskip
\textbf{Reference evidence chain.}\par (1) [audio 07:31-07:33] Atreus says, "That's the last one. I guess now we find out if dragons can be grateful or not."\par (2) [audio 08:23-08:24] Atreus says, "You're welcome!" This is naturally heard as his response after the gratefulness remark, anchoring the same interaction sequence.\par (3) [visual 07:33-07:40] During that release sequence, "Kratos and Atreus walk toward the large orange dragon, Otr, which is resting on a rocky cliffside surrounded by grassy ruins and a waterfall in the background."\par (4) [visual 03:27-03:31] Earlier, the video labels this same dragon's segment: "2nd Dragon Otr's Imprisonment" in Veithurgard.
\end{tcolorbox}
\caption{\textbf{Relationship reasoning example.} Dialogue and visual identity cues establish which dragon a character addresses across the interaction sequence. The task label is source-assigned.}
\label{case:bench7}\end{figure}

\begin{figure}[!htbp]\centering
\begin{tcolorbox}[casecard,title={\emph{Reasoning over video} / Spatial Reasoning}]
\textbf{Question.} At the moment the commentator shouts "Phoenix, Seattle, San Antonio, there you go!", which team's score is visually positioned on the left side of the scoreboard graphic at the bottom of the screen, and what is that specific score?
\par\smallskip
\textbf{Reference answer.} A) LAL with 25
\par\smallskip\noindent \begin{minipage}[t]{0.48\linewidth}\centering
\includegraphics[width=\linewidth,height=2.5cm,keepaspectratio]{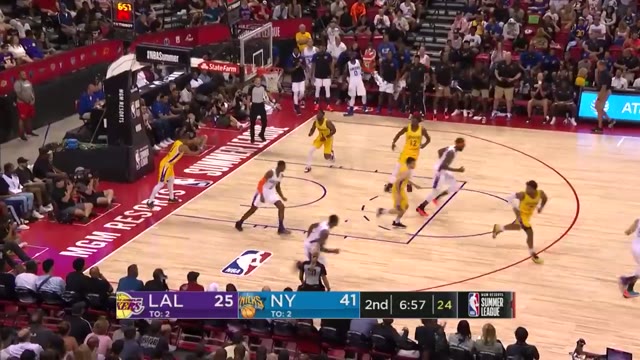}
\par\scriptsize Source video, 00:21\end{minipage}\par\smallskip
\textbf{Reference evidence chain.}\par (1) [audio 00:19-00:23] Identify the audio anchor where the commentator shouts "Phoenix, Seattle, San Antonio..."\par (2) [visual 00:20-00:23] Look at the on-screen scoreboard graphic shown during this precise audio clip to determine that LAL is on the left side of the bar wth 25.
\end{tcolorbox}
\caption{\textbf{Spatial reasoning example.} A spoken exclamation identifies the moment at which the left-hand scoreboard entry must be read. The task label is source-assigned.}
\label{case:bench8}\end{figure}

\begin{figure}[!htbp]\centering
\begin{tcolorbox}[casecard,title={\emph{Reasoning over video} / Temporal Reasoning}]
\textbf{Question.} Consider the following four events that occur during the video's introduction and first movie segment: I. The narrator says, "It's been a hell of a year for us at Screen Rant," II. Two young boys in dark robes face each other on a raised platform in a dimly lit hall. III. A man in a white and gold uniform sits casually on a wooden balcony railing. IV. The narrator says the word "equipment." Which of the following options shows the correct chronological order of these events from earliest to latest?
\par\smallskip
\textbf{Reference answer.} A) I \ensuremath{\rightarrow} II \ensuremath{\rightarrow} III \ensuremath{\rightarrow} IV
\par\smallskip\noindent \begin{minipage}[t]{0.48\linewidth}\centering
\includegraphics[width=\linewidth,height=2.5cm,keepaspectratio]{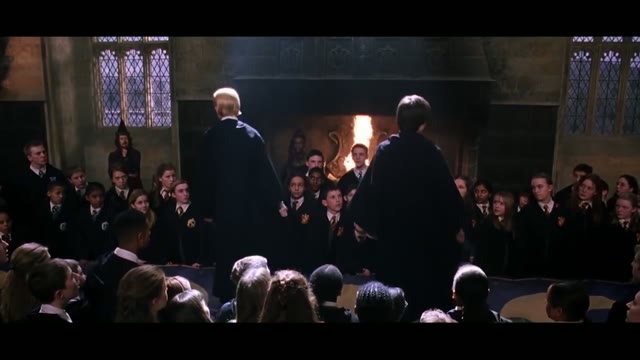}
\par\scriptsize Source video, 00:12\end{minipage}\hfill \begin{minipage}[t]{0.48\linewidth}\centering
\includegraphics[width=\linewidth,height=2.5cm,keepaspectratio]{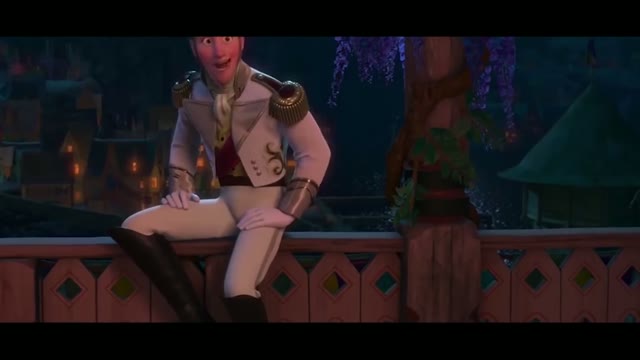}
\par\scriptsize Source video, 01:07\end{minipage}\par\smallskip
\textbf{Reference evidence chain.}\par (1) [audio 00:06.900-00:10.000] Identify that the narrator says 'It's been a hell of a year for us at Screen Rant' between 00:06.900 and 00:10.000 (Event I).\par (2) [visual 00:10.500-00:13.500] Identify that the visual scene of two young boys in dark robes facing each other occurs between 00:10.500 and 00:13.500 (Event II).\par (3) [visual 01:06.000-01:08.000] Identify that the visual scene of the man in the white and gold uniform sitting on a wooden balcony railing occurs between 01:06.000 and 01:08.000 (Event III).\par (4) [audio 01:09.000-01:11.000] Identify that the narrator says the word 'equipment' between 01:09.000 and 01:11.000 (Event IV).\par (5) [audiovisual 00:06.900-01:11.000] Compare and interleave the four times: 00:06.900-00:10.000 (I) \ensuremath{\rightarrow} 00:10.500-00:13.500 (II) \ensuremath{\rightarrow} 01:06.000-01:08.000 (III) \ensuremath{\rightarrow} 01:09.000-01:11.000 (IV).
\end{tcolorbox}
\caption{\textbf{Temporal reasoning example.} The answer orders spoken utterances and visual events on a shared timeline. The task label is source-assigned.}
\label{case:bench9}\end{figure}

\begin{figure}[!htbp]\centering
\begin{tcolorbox}[casecard,title={\emph{Reasoning beyond video} / Case Study Analysis}]
\textbf{Question.} The figure shows a microscopic cross-section of a blade edge after sharpening through a progression of grits. The labelled regions (A through F) correspond to different stages in the sharpening sequence. Based on the lecture's specific description of burr characteristics at each grit, which labelled region represents the stage after sharpening with 1200 grit diamond?
\par\smallskip
\textbf{Reference answer.} B) Region D
\par\smallskip\noindent \begin{minipage}[t]{0.48\linewidth}\centering
\includegraphics[width=\linewidth,height=2.5cm,keepaspectratio]{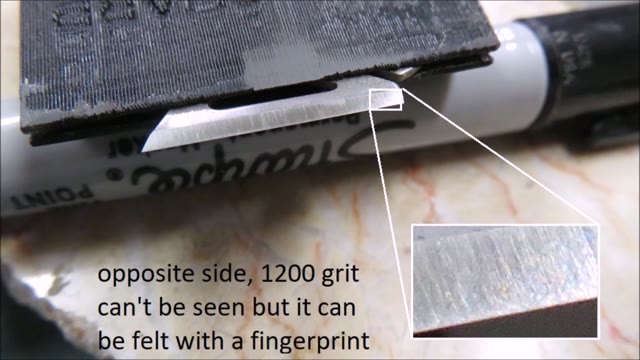}
\par\scriptsize Source video, 05:33\end{minipage}\hfill \begin{minipage}[t]{0.48\linewidth}\centering
\includegraphics[width=\linewidth,height=3.4cm,keepaspectratio]{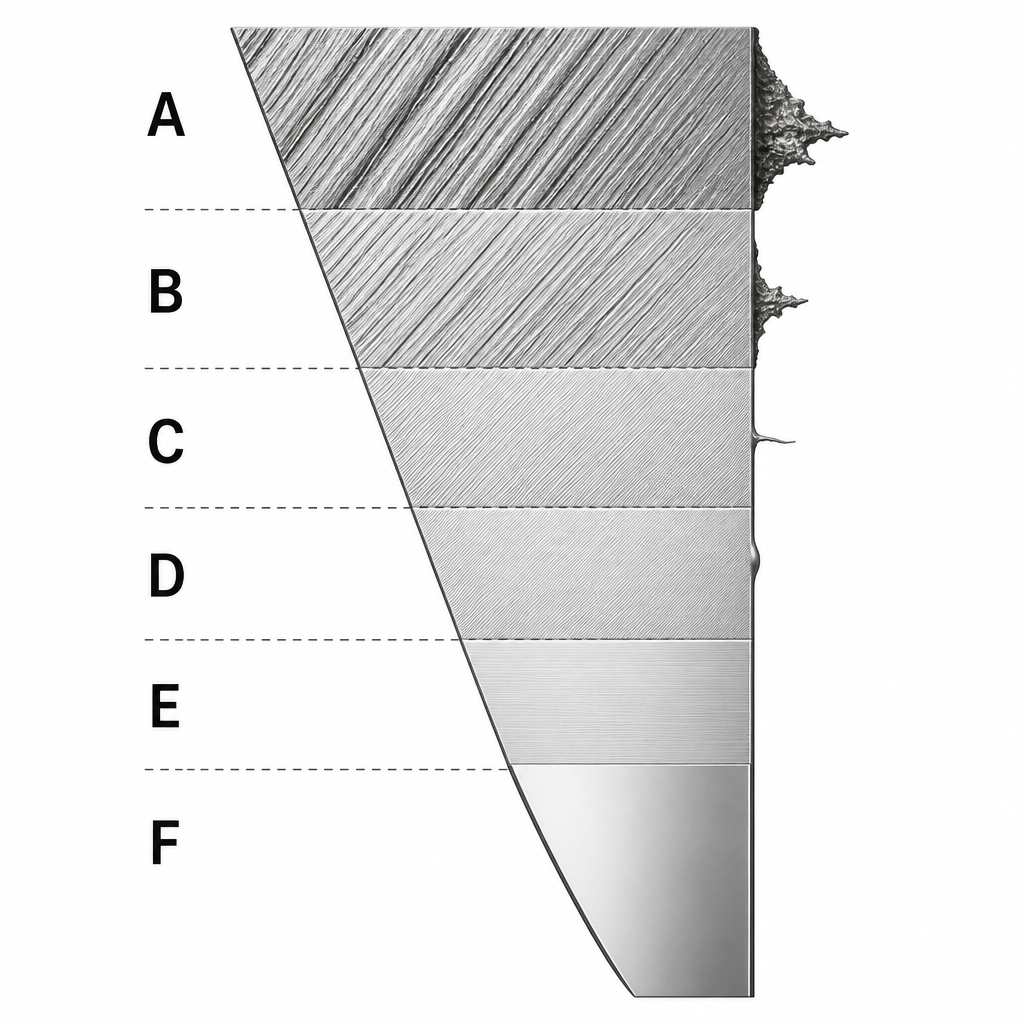}
\par\scriptsize New figure\end{minipage}\par\smallskip
\textbf{Reference evidence chain.}\par (1) [video 05:30-05:37] Recall the lecture's specific mapping of burr characteristics to grit stages: at 1200 grit, the burr 'can't be seen but it can be felt with a fingerprint'. Source quote: "opposite side, 1200 grit can't be seen but it can be felt with a fingerprint"\par (2) [reasoning] Observe the figure: Region A has a large visible burr and coarse scratches (240 grit). Region B has a smaller visible burr (400 grit). Region C has a still visible, very fine burr (800 grit). Region D has no visible burr but a slight raised bump, indicating a burr that is present but not visible (matches 1200 grit). Region E has no visible burr and no bump (2500 grit). Region F is mirror-polished (buffing compound).\par (3) [reasoning] Apply the lecture's criterion: the stage after 1200 grit must have a burr that is invisible but tactile. Only Region D satisfies this condition (no visible burr but a bump representing the tactile burr).
\end{tcolorbox}
\caption{\textbf{Case study analysis example.} A criterion explained in the video is applied to a new blade-edge diagram to identify a sharpening stage. The task label is source-assigned.}
\label{case:bench10}\end{figure}

\begin{figure}[!htbp]\centering
\begin{tcolorbox}[casecard,title={\emph{Reasoning beyond video} / Comparative Reasoning}]
\textbf{Question.} Based on the treatment protocol demonstrated in the video, which of the following wound configurations would the veterinarian DEFINITELY administer a penicillin injection to?
\par\smallskip
\textbf{Reference answer.} E) Configuration B: Deep neck wound extending into oral cavity
\par\smallskip\noindent \begin{minipage}[t]{0.48\linewidth}\centering
\includegraphics[width=\linewidth,height=2.5cm,keepaspectratio]{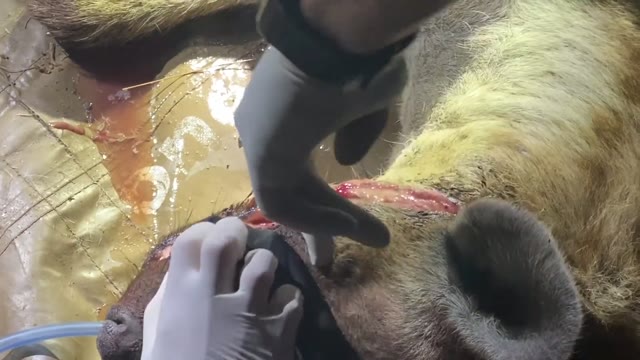}
\par\scriptsize Source video, 04:15\end{minipage}\hfill \begin{minipage}[t]{0.48\linewidth}\centering
\includegraphics[width=\linewidth,height=3.4cm,keepaspectratio]{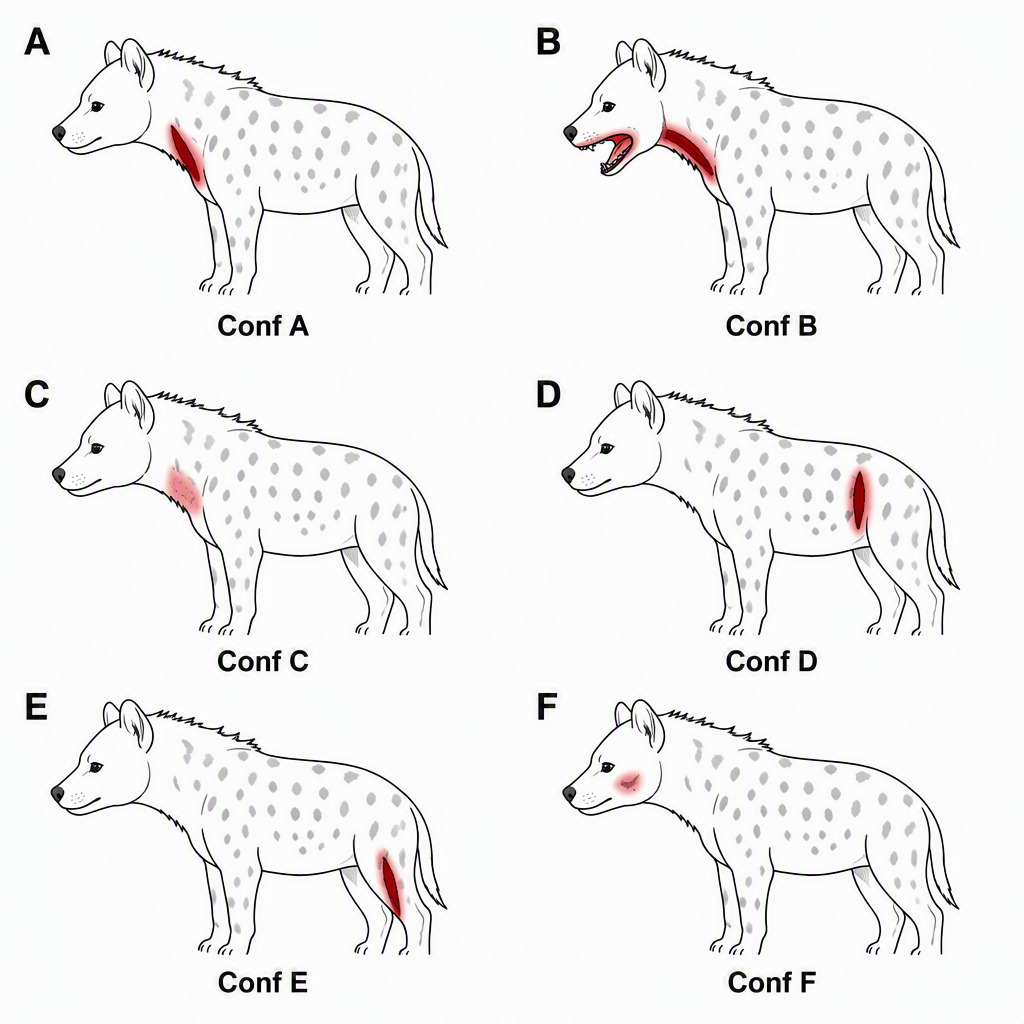}
\par\scriptsize New figure\end{minipage}\par\smallskip
\textbf{Reference evidence chain.}\par (1) [video 04:09-04:22] Recall the video's decision rule: the veterinarian only gives a penicillin injection if the wound involves the mouth (oral cavity). The vet said: 'because it's in the mouth, this here bothers me, by the teeth and that. I'll definitely actually give him a penicillin.' Source quote: "But you know what, no. You know what, because it's in the mouth, this here bothers me, by the teeth and that. I'll definitely actually give him a penicillin."\par (2) [reasoning] Examine the figure's six configurations. Identify which one explicitly shows the wound extending into the oral cavity. Configuration B is the only one with the wound shaded into the mouth area.\par (3) [reasoning] Apply the rule: only Configuration B satisfies the criterion (mouth involvement). Therefore, the veterinarian would definitely administer penicillin to Configuration B.
\end{tcolorbox}
\caption{\textbf{Comparative reasoning example.} The treatment criterion stated in the video is used to compare wound configurations in a new figure. The task label is source-assigned.}
\label{case:bench11}\end{figure}

\begin{figure}[!htbp]\centering
\begin{tcolorbox}[casecard,title={\emph{Reasoning beyond video} / Cross-Scenario Transfer}]
\textbf{Question.} The figure shows the catch log for Team Alpha during a session of the YPC Bass 2023 tournament. Using the tournament's scoring rules as established in the video, calculate the team's total score for the leaderboard. Which of the following is closest to the correct total (in kilograms)?
\par\smallskip
\textbf{Reference answer.} 16.24 kg
\par\smallskip\noindent \begin{minipage}[t]{0.48\linewidth}\centering
\includegraphics[width=\linewidth,height=2.5cm,keepaspectratio]{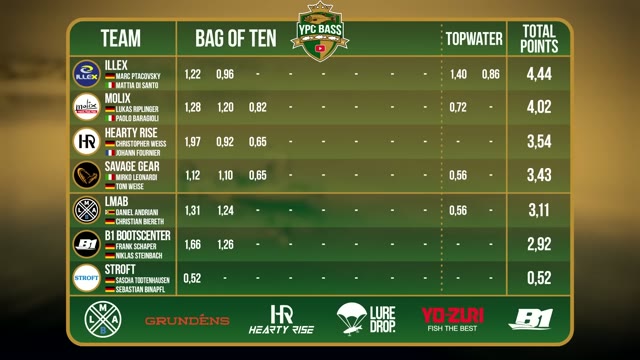}
\par\scriptsize Source video, 01:01\end{minipage}\hfill \begin{minipage}[t]{0.48\linewidth}\centering
\includegraphics[width=\linewidth,height=3.4cm,keepaspectratio]{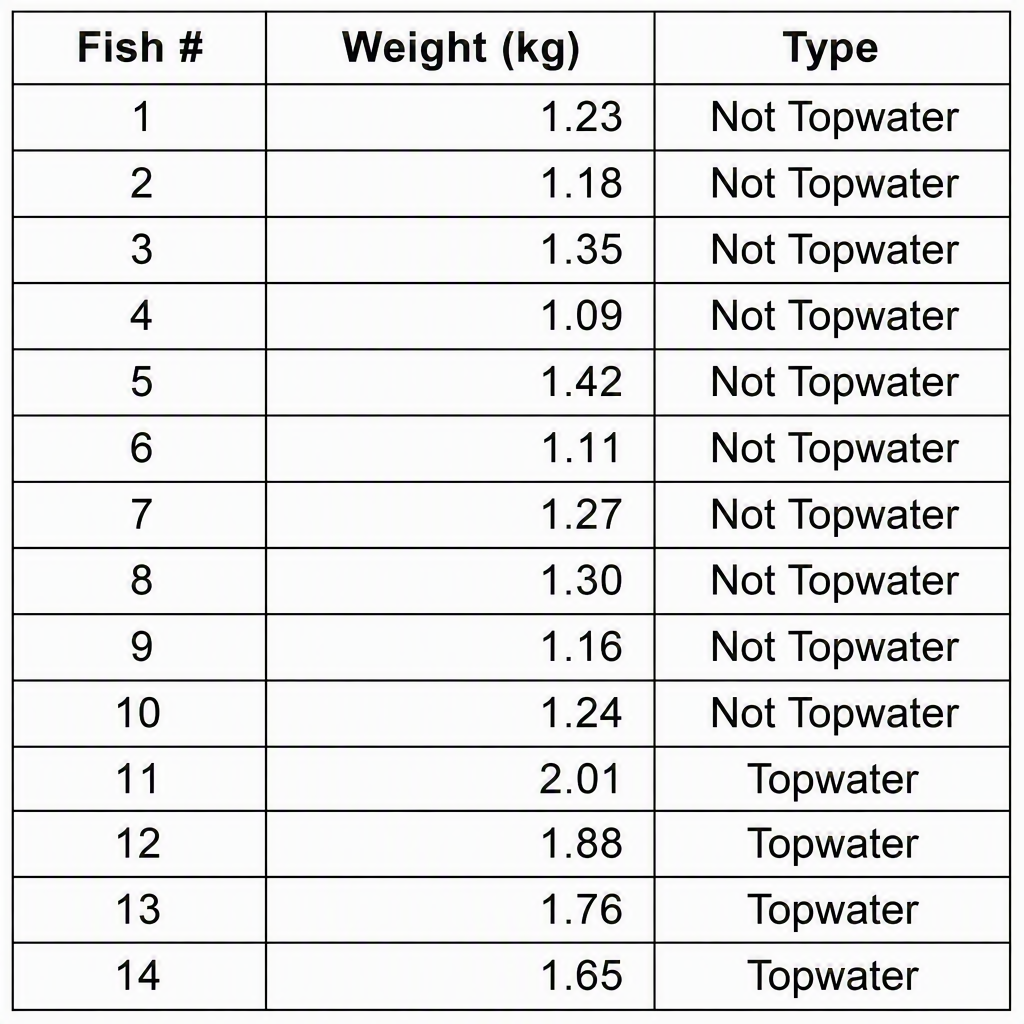}
\par\scriptsize New figure\end{minipage}\par\smallskip
\textbf{Reference evidence chain.}\par (1) [video 00:48-01:15] Identify the tournament's scoring rule: total score = weight of the ten heaviest non-topwater fish (bag of ten) + weight of the two heaviest topwater fish (topwater category). This is based on the leaderboard showing separate columns and the narrator's statement that each team needs two topwater fish. Source quote: "Team Illex have already caught the two top water fish they need and lead the table with a total of 4.44 kilos."\par (2) [reasoning] From the figure, extract the weights of the non-topwater fish (rows 1-10): 1.23, 1.18, 1.35, 1.09, 1.42, 1.11, 1.27, 1.30, 1.16, 1.24 kg.\par (3) [reasoning] Sum the ten non-topwater fish: 1.23+1.18+1.35+1.09+1.42+1.11+1.27+1.30+1.16+1.24 = 12.35 kg.\par (4) [reasoning] Extract the weights of the topwater fish (rows 11-14): 2.01, 1.88, 1.76, 1.65 kg.\par (5) [reasoning] Identify the two heaviest topwater fish: 2.01 kg and 1.88 kg.\par (6) [reasoning] Sum the two heaviest topwater fish: 2.01+1.88 = 3.89 kg.\par (7) [reasoning] Add the two sums: 12.35 + 3.89 = 16.24 kg.
\end{tcolorbox}
\caption{\textbf{Cross-scenario transfer example.} Tournament scoring rules from the video are applied to a new catch log to calculate the team total. The task label is source-assigned.}
\label{case:bench13}\end{figure}

\begin{figure}[!htbp]\centering
\begin{tcolorbox}[casecard,title={\emph{Reasoning beyond video} / Design / Optimization}]
\textbf{Question.} For the vocal track shown in the figure, which threshold setting should the engineer choose to achieve the target reduction, given the de-esser's fixed ratio from the video?
\par\smallskip
\textbf{Reference answer.} -20 dB
\par\smallskip\noindent \begin{minipage}[t]{0.48\linewidth}\centering
\includegraphics[width=\linewidth,height=2.5cm,keepaspectratio]{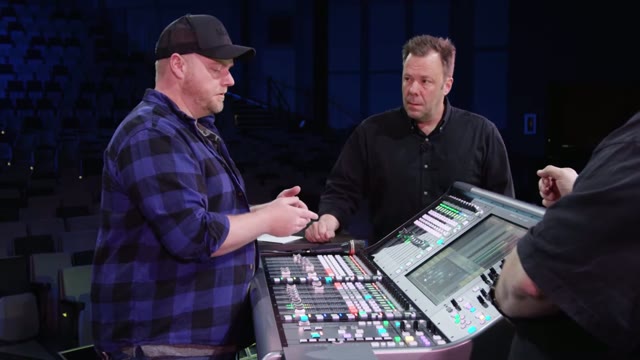}
\par\scriptsize Source video, 02:08\end{minipage}\hfill \begin{minipage}[t]{0.48\linewidth}\centering
\includegraphics[width=\linewidth,height=3.4cm,keepaspectratio]{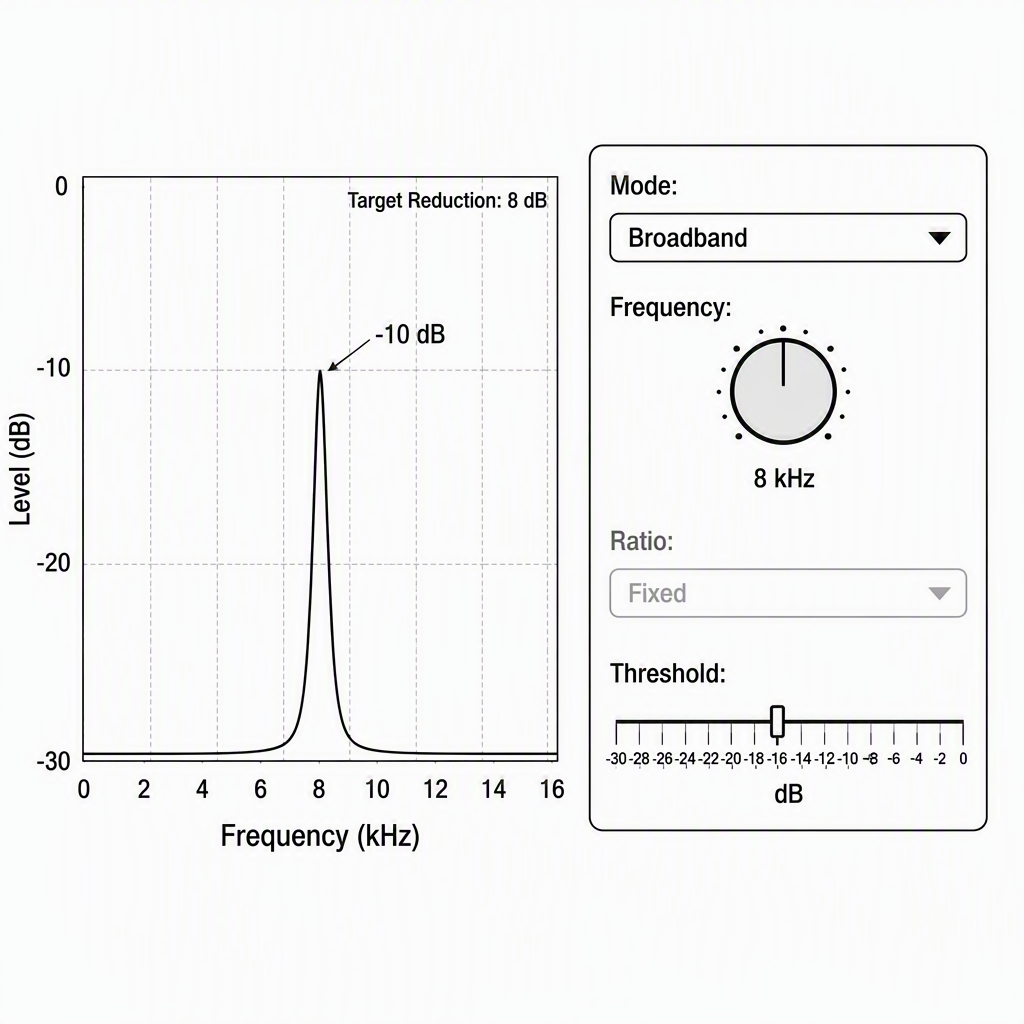}
\par\scriptsize New figure\end{minipage}\par\smallskip
\textbf{Reference evidence chain.}\par (1) [video 02:05-02:11] Recall from the video that the de-esser ratio is set to about 5:1 (i.e., for every 5 dB above threshold, 1 dB passes). Source quote: "And then there's a ratio. We talked about that with compressors, limiters. This one is set to about 5 to 1."\par (2) [reasoning] Read from the figure: the input level of the sibilant peak is -10 dB.\par (3) [reasoning] Read from the figure: the target gain reduction is 8 dB.\par (4) [reasoning] Compute the required threshold using the formula: threshold = input - (reduction / (1 - 1/ratio)). With ratio=5, 1 - 1/5 = 0.8. So threshold = -10 - (8 / 0.8) = -10 - 10 = -20 dB.
\end{tcolorbox}
\caption{\textbf{Design and optimization example.} The compression ratio stated in the video is combined with new signal values to derive a de-esser threshold. The task label is source-assigned.}
\label{case:bench14}\end{figure}

\begin{figure}[!htbp]\centering
\begin{tcolorbox}[casecard,title={\emph{Reasoning beyond video} / Error Critique}]
\textbf{Question.} A beginner wants to create an acrylic pour with cells. They mix 50 ml of vinyl glue (medium) with 100 ml of titanium white (opaque) paint. They add water until the mixture flows in a continuous stream. Then they add 3 drops of silicone oil and stir vigorously. They then pour the white paint as the first layer onto the canvas, then pour a transparent blue paint on top, and apply heat with a torch. They claim this follows the video's method.\par Which of the following is the specific flaw in the beginner's procedure, according to the method demonstrated in the video?
\par\smallskip
\textbf{Reference answer.} B) The opaque paint should be applied on top of the transparent paint, not the other way around.
\par\smallskip\noindent \begin{minipage}[t]{0.48\linewidth}\centering
\includegraphics[width=\linewidth,height=2.5cm,keepaspectratio]{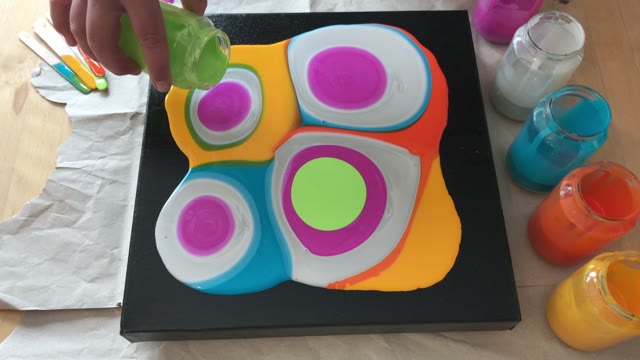}
\par\scriptsize Source video, 04:08\end{minipage}\par\smallskip
\textbf{Reference evidence chain.}\par (1) [video 04:02-04:15] Recall the video's rule: opaque paint must be placed on top of transparent paint to reliably create cells. Source quote: "[translation: Opaque paint above transparent paint reliably produces attractive cells.]"\par (2) [reasoning] Identify the scenario's layering order: the beginner pours opaque white first (bottom layer) and then transparent blue on top.\par (3) [reasoning] Compare the scenario's order to the rule: the scenario violates the rule because opaque is not on top; transparent is on top. Therefore the layering is the flaw.
\end{tcolorbox}
\caption{\textbf{Error critique example.} The layering rule demonstrated in the video is used to identify an error in a newly described painting procedure. The task label is source-assigned.}
\label{case:bench15}\end{figure}

\begin{figure}[!htbp]\centering
\begin{tcolorbox}[casecard,title={\emph{Reasoning beyond video} / Multi-step Quantitative Reasoning}]
\textbf{Question.} For the frame shown in the figure, using the item Engineering Tool as demonstrated in the video, what is the total length of Profile 8 (40 x 40) aluminum profiles required?
\par\smallskip
\textbf{Reference answer.} 2910 mm
\par\smallskip\noindent \begin{minipage}[t]{0.48\linewidth}\centering
\includegraphics[width=\linewidth,height=2.5cm,keepaspectratio]{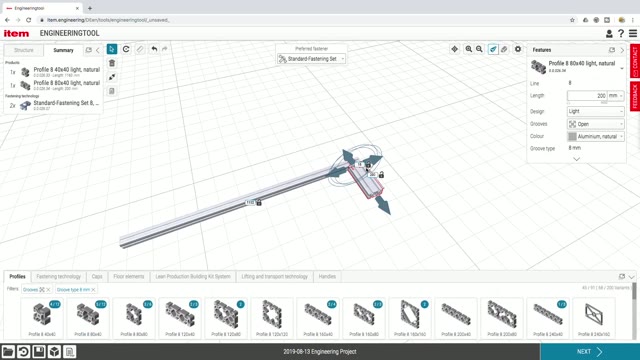}
\par\scriptsize Source video, 04:24\end{minipage}\hfill \begin{minipage}[t]{0.48\linewidth}\centering
\includegraphics[width=\linewidth,height=3.4cm,keepaspectratio]{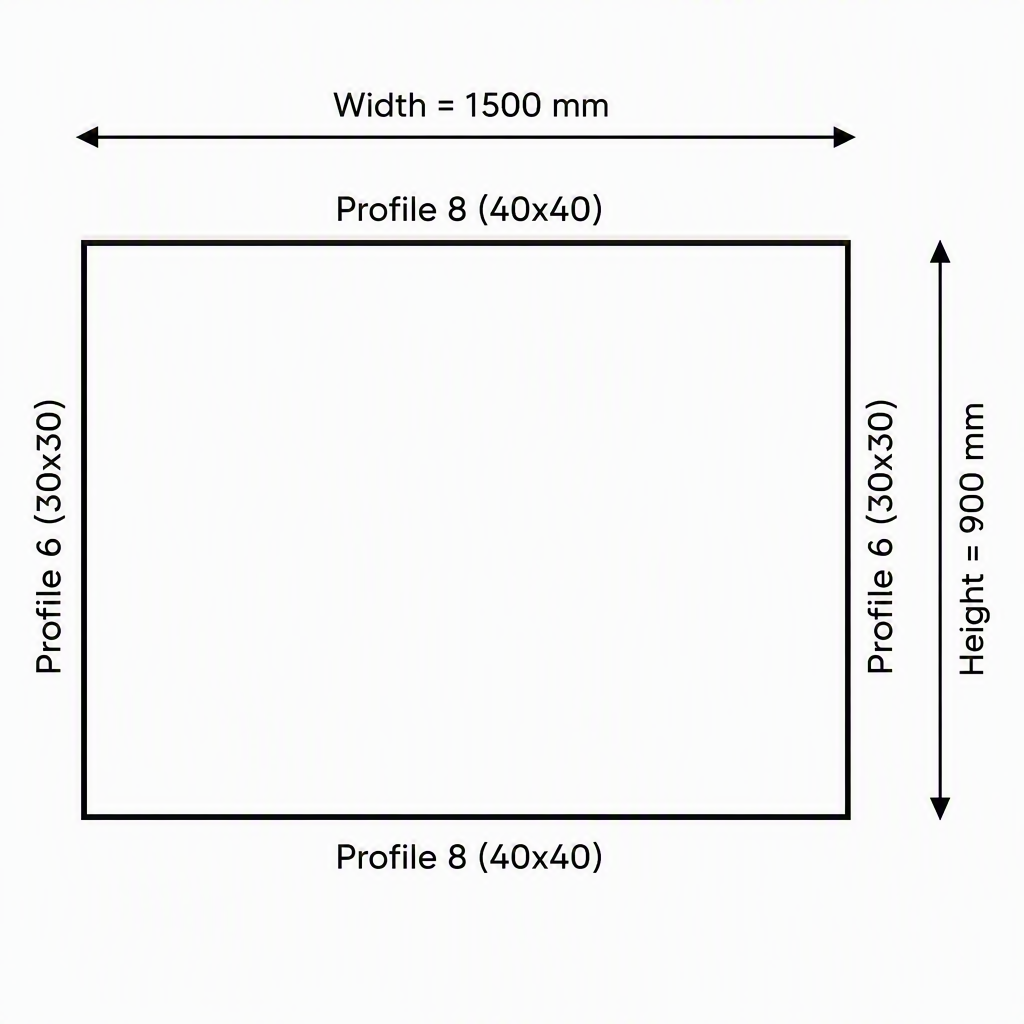}
\par\scriptsize New figure\end{minipage}\par\smallskip
\textbf{Reference evidence chain.}\par (1) [video 04:15-04:33] Recall the video's length calculation rule: the first-placed horizontal profile (top) is shortened by the width of one perpendicular profile; the second horizontal profile (bottom) is shortened by the widths of both perpendicular profiles. (Video timestamps 04:15-04:21 and 04:27-04:33) Source quote: "I'll set a length of 1200 minus 40, that is 1160. ... I'm gonna set a length of 700 minus 80, that will be 620 millimeters."\par (2) [reasoning] From the figure, read the outer width = 1500 mm and the width of the perpendicular (vertical) profiles as 30 mm (Profile 6).\par (3) [reasoning] Compute the top profile length: 1500 mm - 30 mm = 1470 mm.\par (4) [reasoning] Compute the bottom profile length: 1500 mm - 2 x 30 mm = 1500 mm - 60 mm = 1440 mm.\par (5) [reasoning] Sum the lengths of the two Profile 8 pieces: 1470 mm + 1440 mm = 2910 mm.
\end{tcolorbox}
\caption{\textbf{Multi-step quantitative reasoning example.} A construction rule from the video is applied to new frame dimensions to calculate the total profile length. The task label is source-assigned.}
\label{case:bench16}\end{figure}

\begin{figure}[!htbp]\centering
\begin{tcolorbox}[casecard,title={\emph{Reasoning beyond video} / Problem-solving Adaptation}]
\textbf{Question.} The figure shows the annual food poisoning deaths and population for five US cities, along with the total US population for reference. Using the data in the figure and the number of US food poisoning deaths mentioned in the video, calculate the ratio of the average death rate per 100,000 people across these five cities to the national average death rate per 100,000 people. Round your answer to two decimal places.
\par\smallskip
\textbf{Reference answer.} 1.1
\par\smallskip\noindent \begin{minipage}[t]{0.48\linewidth}\centering
\includegraphics[width=\linewidth,height=2.5cm,keepaspectratio]{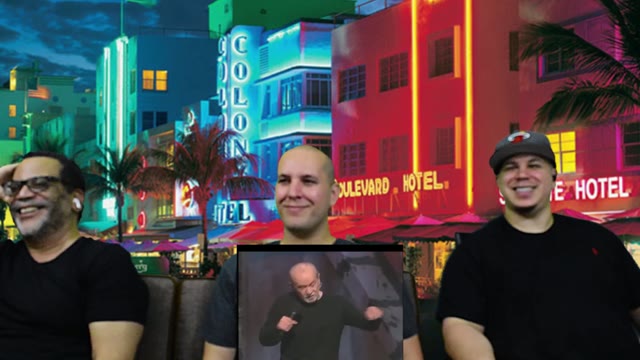}
\par\scriptsize Source video, 03:12\end{minipage}\hfill \begin{minipage}[t]{0.48\linewidth}\centering
\includegraphics[width=\linewidth,height=3.4cm,keepaspectratio]{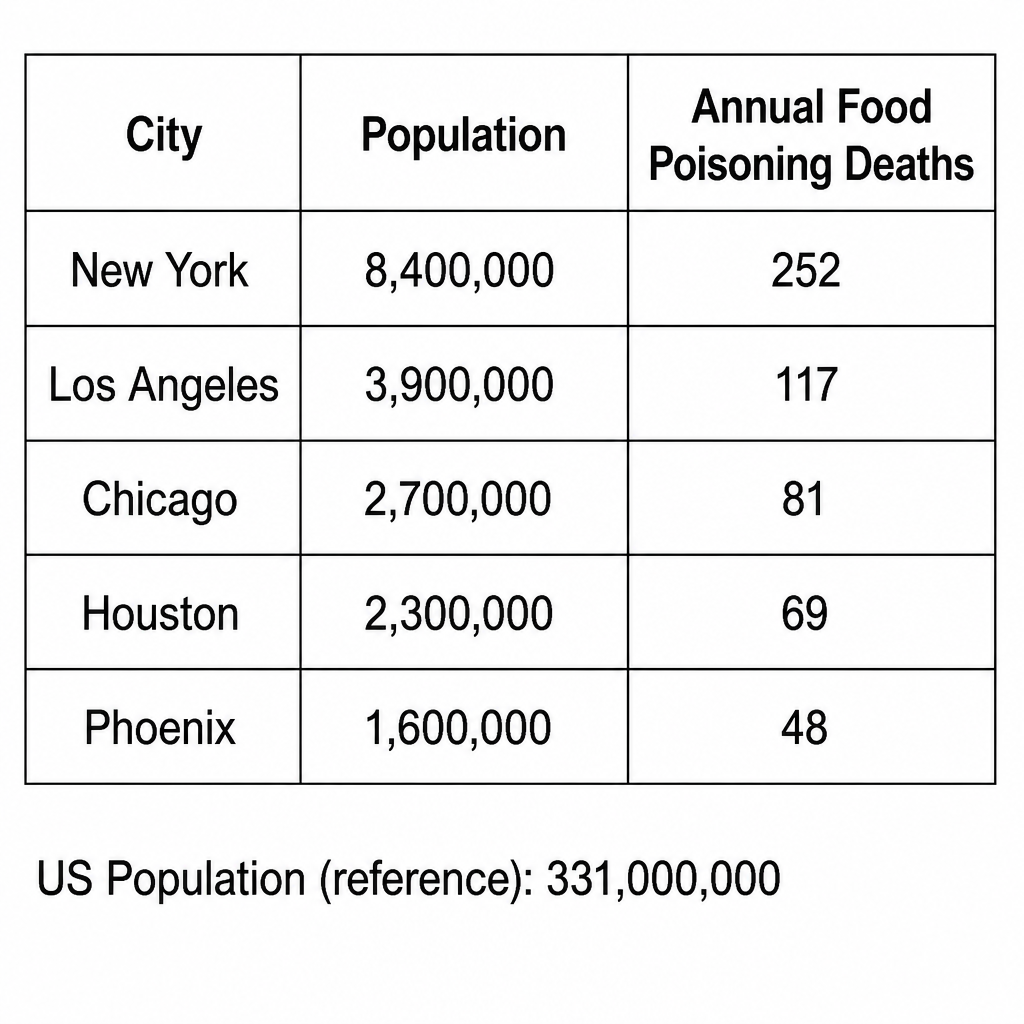}
\par\scriptsize New figure\end{minipage}\par\smallskip
\textbf{Reference evidence chain.}\par (1) [video 03:07-03:17] Recall the lecture-specific fact: the video states that 9,000 people die from food poisoning each year in the US. Source quote: "You know how many people die in this country from food poisoning every year? Nine thousand."\par (2) [reasoning] Read the US population from the figure: 331,000,000.\par (3) [reasoning] Compute the national death rate per 100,000: (9,000 / 331,000,000) * 100,000 = 2.719...\par (4) [reasoning] Read the total deaths for the five cities from the figure: 252 + 117 + 81 + 69 + 48 = 567.\par (5) [reasoning] Read the total population for the five cities from the figure: 8,400,000 + 3,900,000 + 2,700,000 + 2,300,000 + 1,600,000 = 18,900,000.\par (6) [reasoning] Compute the average city death rate per 100,000: (567 / 18,900,000) * 100,000 = 3.0.\par (7) [reasoning] Compute the ratio: 3.0 / 2.719 = 1.103, rounded to 1.10.
\end{tcolorbox}
\caption{\textbf{Problem-solving adaptation example.} A national statistic stated in the video is combined with new city-level data to compare population-normalized rates. The task label is source-assigned.}
\label{case:bench17}\end{figure}

\FloatBarrier
\clearpage
\section{Supervised Training Examples}
\label{app:sft-cases}
We show one example for each of the 25 production task types: 16 \emph{reasoning over video} and nine \emph{reasoning beyond video} types. Each card pairs the actual SFT question and answer with source frames and, where present, upstream evidence or derivation, and reproduces the complete stored synthetic response used for SFT. Each response retains its original thinking pattern, including the opening, task analysis, evidence inspection, answer verification, and final option. Boldface highlights selected evidence and conclusions without omitting intervening text. These are synthetic supervision targets, not outputs from the evaluated model. Each card marks its task label as source-assigned or retrospective; retrospective labels are inferred production-taxonomy assignments, not original source annotations.

\Needspace{12\baselineskip}
\subsection{A01 / Audio Source Grounding}
\label{case:sft0}
\begin{tcolorbox}[casecard,breakable,title={Question, evidence, and training answer}]
\textbf{Question.} Where's the character that says "For 30 years we've been trying to make a sequel and we're not going to stop." located in the video?
\par\smallskip
\textbf{Options.} A. On the left B. In the center C. On the right D. Off-screen\par\smallskip\textbf{Training answer.} D: Off-screen
\par\smallskip\noindent \begin{minipage}[t]{0.48\linewidth}\centering
\includegraphics[width=\linewidth,height=2.5cm,keepaspectratio]{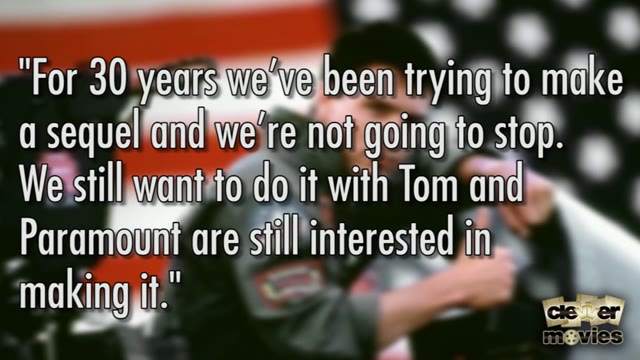}
\par\scriptsize Source video, 00:38\end{minipage}\par\smallskip
\textbf{Upstream reasoning steps.}\par 1. [audio 00:36-00:40] For 30 years we've been trying to make a sequel and we're not going to stop.\par 2. [visual 00:36-00:40] A graphic with text overlay and a blurred background image is shown; the speaker is not visible.\par 3. [inference] Although the host is seen on the left side of the screen earlier and later in the video, at the exact moment she speaks this quoted line, she is acting as a voiceover for a full-screen text graphic and is not visibly on screen.
\par\smallskip\textit{Task label: source-assigned.}
\end{tcolorbox}
\begin{tcolorbox}[casecard,breakable,title={Complete synthetic thinking and final answer},title after break={Complete synthetic thinking and final answer (continued)}]
\textless{}think\textgreater{}
Got it, let's tackle the problem.
\par\smallskip
The user wants to identify the screen location of the character who speaks the line about trying to make a sequel for thirty years. The crux is determining whether the speaker is visible on screen or speaking off-screen during this specific quote.
\par\smallskip
First, let me scan the video for the segments relevant to the question:
\par\smallskip
- I need to listen for the phrase "For 30 years we've been trying to make a sequel and we're not going to stop." and observe who is on screen at that moment.
\par\smallskip
\textbf{- At 00:03, the presenter Erin White appears on screen to introduce the segment about Jerry Bruckheimer, Top Gun 2, and The Lone Ranger.}
\par\smallskip
\textbf{- At 00:07, the presenter is positioned on the left side of the screen while a photo of Tony Scott is shown on the right.}
\par\smallskip
- At 00:28, the presenter is on the left side of the screen while a photo of Tom Cruise is shown on the right.
\par\smallskip
- At 00:35, the presenter introduces the quote by saying "saying quote".
\par\smallskip
- Wait, let me re-listen at 00:36: the presenter begins reading the quote, but she is no longer visible on screen as the video transitions to a full-screen graphic of Tom Cruise with text overlay.
\par\smallskip
- At 00:36-00:44, the on-screen text displays "For 30 years we've been trying to make a sequel and we're not going to stop. We still want to do it with Tom and Paramount are still interested in making it." while the presenter reads it.
\par\smallskip
- If the speaker were positioned on the left or center of the screen during this quote, her video feed would need to remain visible, but it is completely replaced by the graphic of Tom Cruise, meaning she is delivering the line as a voiceover.
\par\smallskip
- At 00:55, the presenter returns to the screen, positioned on the left side next to a Top Gun poster.
\par\smallskip
- The presenter is the one speaking the quote, and she is not visible on screen during its delivery, which corresponds to option D.
\par\smallskip
Next, let me analyze the options:
\par\smallskip
- A. On the left \textemdash{} The presenter is on the left during other segments but not during this quote.
\par\smallskip
- B. In the center \textemdash{} No character is in the center.
\par\smallskip
- C. On the right \textemdash{} No character is on the right.
\par\smallskip
- D. Off-screen \textemdash{} The presenter speaks the line as a voiceover while a full-screen graphic is shown.
\par\smallskip
\textbf{Finally, the correct option is D.
\textless{}/think\textgreater{}}
\par\smallskip
D
\end{tcolorbox}

\Needspace{12\baselineskip}
\subsection{A02 / Sound Event Reasoning}
\label{case:sft1}
\begin{tcolorbox}[casecard,breakable,title={Question, evidence, and training answer}]
\textbf{Question.} What extra words does the narrator say immediately after the on-screen text 'To understand how gravity affects the flow of time' appears?
\par\smallskip
\textbf{Options.} A. consider the following B. check out the other videos C. the difference in the rates D. the box is accelerating E. the rate at which time flows F. the time between the light pulses G. the gravitational field H. near a black hole\par\smallskip\textbf{Training answer.} A: consider the following
\par\smallskip\noindent \begin{minipage}[t]{0.48\linewidth}\centering
\includegraphics[width=\linewidth,height=2.5cm,keepaspectratio]{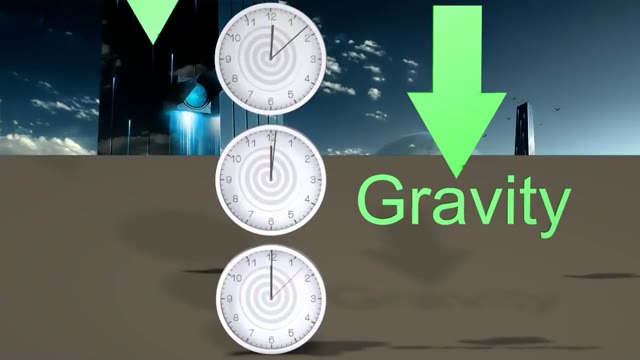}
\par\scriptsize Source video, 00:03\end{minipage}\par\smallskip
\textbf{Upstream reasoning steps.}\par 1. [visual 00:00-00:06] The on-screen text shows: 'To understand how gravity affects the flow of time'.\par 2. [audio 00:00-00:06] The narrator says: 'To understand how gravity affects the flow of time, consider the following.'\par 3. [audiovisual 00:00-00:06] Comparing the visual text and the audio speech, the audio contains the extra phrase 'consider the following' that is not present in the on-screen text.
\par\smallskip\textit{Task label: source-assigned.}
\end{tcolorbox}
\begin{tcolorbox}[casecard,breakable,title={Complete synthetic thinking and final answer},title after break={Complete synthetic thinking and final answer (continued)}]
\textless{}think\textgreater{}
Got it, let's tackle the problem.
\par\smallskip
The user wants to identify the spoken words that immediately follow a specific phrase at the beginning of the video. The crux is matching the narrator's spoken audio to the correct option among the choices provided.
\par\smallskip
First, let me scan the video for the segments relevant to the question:
\par\smallskip
- I need to listen for the phrase "To understand how gravity affects the flow of time" and watch what is on screen while it is said.
\par\smallskip
\textbf{- At 00:01, the narrator begins speaking, saying "To understand how gravity affects the flow of time" while three clocks are shown on screen.}
\par\smallskip
\textbf{- At 00:04, the narrator continues the sentence by saying "consider the following" as the green arrows pointing down remain on screen.}
\par\smallskip
- Let me verify the visual anchor: from 00:01 to 00:06 the on-screen text reads "To understand how gravity affects the flow of time", exactly matching the phrase quoted in the question, while the narrator speaks the same words aloud.
\par\smallskip
- Comparing the two, the narrator's sentence continues beyond the on-screen text: the audio adds "consider the following", which never appears on screen.
\par\smallskip
- If the narrator had said "check out the other videos" immediately after the opening line, that phrase would have to be audible in the first ten seconds, but it is only spoken at the very end of the video.
\par\smallskip
- At 10:45, the narrator says "Please check out the other videos on this channel" as the closing text appears.
\par\smallskip
- The narrator's spoken words immediately following the target phrase are "consider the following", which corresponds to option A.
\par\smallskip
Next, let me analyze the options:
\par\smallskip
- A. consider the following \textemdash{} This matches the exact words spoken by the narrator immediately after the opening phrase.
\par\smallskip
- B. check out the other videos \textemdash{} This is spoken at the very end.
\par\smallskip
- C. the difference in the rates \textemdash{} This is spoken much later.
\par\smallskip
- D. the box is accelerating \textemdash{} This is spoken during the acceleration explanation.
\par\smallskip
- E. the rate at which time flows \textemdash{} This is spoken at 04:42.
\par\smallskip
- F. the time between the light pulses \textemdash{} This is spoken at 02:59.
\par\smallskip
- G. the gravitational field \textemdash{} This is spoken during the gravity explanation.
\par\smallskip
- H. near a black hole \textemdash{} This is spoken at 10:15.
\par\smallskip
\textbf{Finally, the correct option is A.
\textless{}/think\textgreater{}}
\par\smallskip
A
\end{tcolorbox}

\Needspace{12\baselineskip}
\subsection{A03 / Music \& Atmosphere Understanding}
\label{case:sft2}
\begin{tcolorbox}[casecard,breakable,title={Question, evidence, and training answer}]
\textbf{Question.} What type of background audio is heard while the file explorer windows showing files being copied to a USB drive are visible?
\par\smallskip
\textbf{Options.} A. Upbeat background music from outro B. Upbeat electronic music C. Silence D. Background audio from Sling app E. Upbeat background music from TV show clips F. Faint background audio from video and app navigation G. Light pop sound effect with upbeat background tune H. Faint ambient room noise\par\smallskip\textbf{Training answer.} B: Upbeat electronic music
\par\smallskip\noindent \begin{minipage}[t]{0.48\linewidth}\centering
\includegraphics[width=\linewidth,height=2.5cm,keepaspectratio]{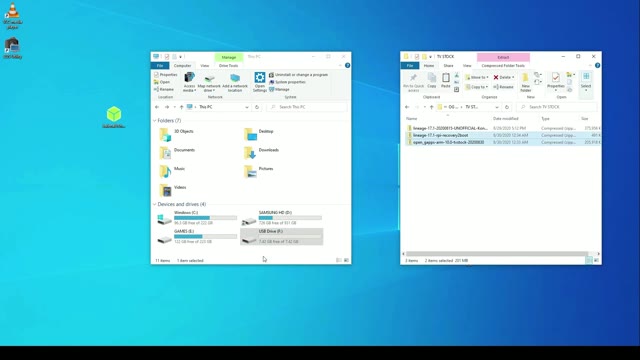}
\par\scriptsize Source video, 02:11\end{minipage}\par\smallskip
\textbf{Upstream reasoning steps.}\par 1. [visual 02:00-02:22] The screen shows file explorer windows with a progress dialog indicating files being copied to a USB drive.\par 2. [audio 02:13-02:22] Upbeat electronic music is playing during this time interval, which overlaps with the file copying visual.\par 3. [audiovisual 02:13-02:22 (audio) / 02:00-02:22 (visual)] Therefore, the background audio while the file explorer windows are visible is upbeat electronic music.
\par\smallskip\textit{Task label: source-assigned.}
\end{tcolorbox}
\begin{tcolorbox}[casecard,breakable,title={Complete synthetic thinking and final answer},title after break={Complete synthetic thinking and final answer (continued)}]
\textless{}think\textgreater{}
Got it, let's tackle the problem.
\par\smallskip
The user wants to identify the type of background audio playing while the file explorer windows show files being copied to a USB drive. The crux is determining the genre or style of the background music during this specific visual sequence.
\par\smallskip
First, let me scan the video for the segments relevant to the question:
\par\smallskip
- I need to find the segment where the file explorer windows are visible showing files being copied to a USB drive.
\par\smallskip
\textbf{- At 01:27, an upbeat electronic music track begins playing in the background as the narrator discusses downloading files.}
\par\smallskip
\textbf{- At 02:00, the screen displays the Windows desktop with two file explorer windows open, and the narrator says, "While Lineage OS is being flashed to the SD card, go ahead and copy open gaps and recovery 2 boot..."}
\par\smallskip
- At 02:14, a file copying progress dialog box appears on screen showing the files being copied to the USB drive while the upbeat electronic music continues to play clearly.
\par\smallskip
- If the background audio were silent or consisted only of ambient room noise, there would be no musical melody playing behind the narrator's voice from 02:00 to 02:21, but a distinct electronic synth track is clearly audible.
\par\smallskip
- At 02:21, the file copying sequence ends and the video transitions to a new scene.
\par\smallskip
- The background track playing during this entire file-copying sequence is a light, upbeat electronic tune, which corresponds to option B.
\par\smallskip
Next, let me analyze the options:
\par\smallskip
- A. Upbeat background music from outro \textemdash{} This music plays during the middle, not the outro.
\par\smallskip
- B. Upbeat electronic music \textemdash{} This matches the synth-pop electronic track playing in the background of the file copying segment.
\par\smallskip
- C. Silence \textemdash{} There is clear background music playing here.
\par\smallskip
- D. Background audio from Sling app \textemdash{} The Sling app is not open during this segment.
\par\smallskip
- E. Upbeat background music from TV show clips \textemdash{} No TV show clips are playing during this demonstration.
\par\smallskip
- F. Faint background audio from video and app navigation \textemdash{} The audio is a prominent music track, not navigation.
\par\smallskip
- G. Light pop sound effect with upbeat background tune \textemdash{} No pop sound effects accompany the file copying action.
\par\smallskip
- H. Faint ambient room noise \textemdash{} The audio is dominated by music, not room noise.
\par\smallskip
\textbf{Finally, the correct option is B.
\textless{}/think\textgreater{}}
\par\smallskip
B
\end{tcolorbox}

\Needspace{12\baselineskip}
\subsection{A04 / Cross-Modal Reference Resolution}
\label{case:sft3}
\begin{tcolorbox}[casecard,breakable,title={Question, evidence, and training answer}]
\textbf{Question.} The expert whose on-screen title is 'Head of the Asia Programme at Chatham House' makes a claim about the size of the detonation. Which of the following statements did he utter?
\par\smallskip
\textbf{Options.} A. Given the advanced state of North Korean ballistic missile, intermediate-range ballistic missile technology, it's only a matter of time if they have developed a workable hydrogen warhead, before they miniaturize it. B. The H-bomb uses a smaller atomic bomb to trigger a larger thermonuclear explosion. C. The last time that country tested a nuclear weapon was back in 2013. D. Seoul, Tokyo, and Beijing are all in fairly easy range. E. The reported size of the detonation is relatively modest, comparable to the test that North Korea carried out in 2013. F. Experts estimate that that weapon had a yield of about seven kilotons. G. The explosion creates huge amounts of energy, triggering a seismic reaction. H. The DPRK's longer-range missiles could potentially hit much of India and Russia and even parts of the U.S. state of Alaska.\par\smallskip\textbf{Training answer.} E: The reported size of the detonation is relatively modest, comparable to the test that North Korea carried out in 2013.
\par\smallskip\noindent \begin{minipage}[t]{0.48\linewidth}\centering
\includegraphics[width=\linewidth,height=2.5cm,keepaspectratio]{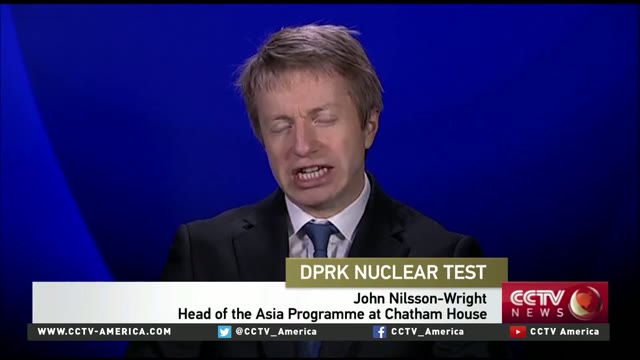}
\par\scriptsize Source video, 01:27\end{minipage}\par\smallskip
\textbf{Upstream reasoning steps.}\par 1. [visual 01:19-01:35] The visual banner displays the name 'John Nilsson-Wright' and the title 'Head of the Asia Programme at Chatham House'.\par 2. [audio 01:19-01:27] An audio statement is heard: 'The reported size of the detonation is relatively modest, comparable to the test that North Korea carried out in 2013.'\par 3. [audiovisual 01:19-01:27] The audio statement is uttered by the same person whose visual title is shown; therefore the statement belongs to the expert with that title.
\par\smallskip\textit{Task label: source-assigned.}
\end{tcolorbox}
\begin{tcolorbox}[casecard,breakable,title={Complete synthetic thinking and final answer},title after break={Complete synthetic thinking and final answer (continued)}]
\textless{}think\textgreater{}
Got it, let's tackle the problem.
\par\smallskip
The user wants to identify the statement made by the expert with the on-screen title "Head of the Asia Programme at Chatham House" regarding the size of the detonation. The crux is matching the expert's spoken words to one of the given options.
\par\smallskip
First, let me scan the video for the segments relevant to the question:
\par\smallskip
- I need to find the expert with the on-screen title "Head of the Asia Programme at Chatham House" and listen to his statement about the detonation.
- At 01:19, a man appears on screen with the title "Head of the Asia Programme at Chatham House" and the name "John Nilsson-Wright".
- At 01:20, he says "The reported size of the detonation is relatively modest, it's comparable to the test that North Korea carried out in 2013".
- If the expert's statement were about ballistic missile technology, the speaker would have to be Justin Bronk who appears at 02:01 with the title "Military Security Analyst", so the statement is the one about the modest size of the detonation.
- The expert's spoken words match the statement about the modest size of the detonation, which is option E.
\par\smallskip
Next, let me analyze the options:
\par\smallskip
- A. Given the advanced state of North Korean ballistic missile, intermediate-range ballistic missile technology, it's only a matter of time if they have developed a workable hydrogen warhead, before they miniaturize it. \textemdash{} spoken by Justin Bronk.
- B. The H-bomb uses a smaller atomic bomb to trigger a larger thermonuclear explosion. \textemdash{} spoken by the narrator.
- C. The last time that country tested a nuclear weapon was back in 2013. \textemdash{} spoken by the anchor.
- D. Seoul, Tokyo, and Beijing are all in fairly easy range. \textemdash{} spoken by the narrator.
- E. The reported size of the detonation is relatively modest, comparable to the test that North Korea carried out in 2013. \textemdash{} this matches the exact words spoken by John Nilsson-Wright.
- F. Experts estimate that that weapon had a yield of about seven kilotons. \textemdash{} spoken by the anchor.
- G. The explosion creates huge amounts of energy, triggering a seismic reaction. \textemdash{} spoken by the narrator.
- H. The DPRK's longer-range missiles could potentially hit much of India and Russia and even parts of the U.S. state of Alaska. \textemdash{} spoken by the narrator.
\par\smallskip
\textbf{Finally, the correct option is E.
\textless{}/think\textgreater{}}
\par\smallskip
E
\end{tcolorbox}

\Needspace{12\baselineskip}
\subsection{A05 / Audio-Anchored Perception}
\label{case:sft4}
\begin{tcolorbox}[casecard,breakable,title={Question, evidence, and training answer}]
\textbf{Question.} When the speaker said 'source to sink', what is the model of the machine pointed at by the person who has the same shirt color as the man sitting across from the woman with an arm cast in the meeting room?
\par\smallskip
\textbf{Options.} A. EOS M 400. B. EOS M 290. C. EOS M 300. D. EOS M 100.\par\smallskip\textbf{Training answer.} A: EOS M 400.
\par\smallskip\noindent \begin{minipage}[t]{0.48\linewidth}\centering
\includegraphics[width=\linewidth,height=2.5cm,keepaspectratio]{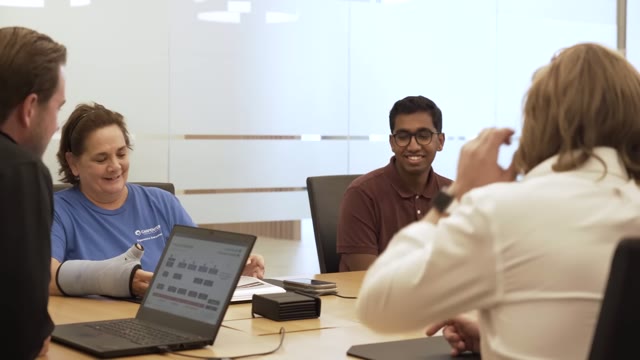}
\par\scriptsize Source video, 14:33\end{minipage}\hfill \begin{minipage}[t]{0.48\linewidth}\centering
\includegraphics[width=\linewidth,height=2.5cm,keepaspectratio]{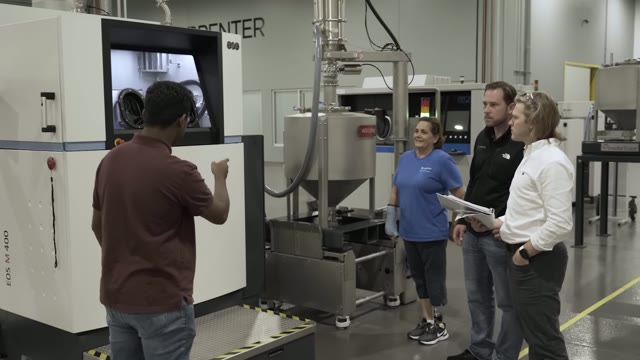}
\par\scriptsize Source video, 14:51\end{minipage}\par\smallskip
\textbf{Upstream reasoning steps.}\par 1. [audio 00:14:49-00:14:52] we ask them about okay source to sink\par 2. [visual 00:14:31-00:14:35] The man sitting across from the woman with an arm cast in the meeting room is wearing a maroon shirt.\par 3. [visual 00:14:50-00:14:52] The person wearing a maroon shirt points at a machine labeled 'EOS M 400'.\par 4. [inference] At 00:14:50, the speaker says 'source to sink' while a man in a maroon shirt (matching the shirt color of the man sitting across from the woman with a cast in the earlier meeting room scene) points to a machine clearly labeled 'EOS M 400'.
\par\smallskip\textit{Task label: source-assigned.}
\end{tcolorbox}
\begin{tcolorbox}[casecard,breakable,title={Complete synthetic thinking and final answer},title after break={Complete synthetic thinking and final answer (continued)}]
\textless{}think\textgreater{}
Got it, let's tackle the problem.
\par\smallskip
The user wants to identify the model of the machine being pointed at when the speaker says "source to sink", based on the shirt color of the man sitting across from the woman with an arm cast in the meeting room. The crux is identifying the shirt color of that specific man in the meeting room and then finding the machine pointed at by the person wearing that same color on the factory floor when the target phrase is spoken.
\par\smallskip
First, let me scan the video for the segments relevant to the question:
\par\smallskip
- I need to listen for the phrase "source to sink" and watch what is on screen while it is said.
\par\smallskip
\textbf{- At 14:31, the video shows a meeting room where a group of people are sitting around a table, including a woman in a blue shirt with a white cast on her arm.}
\par\smallskip
\textbf{- At 14:32, we see a man sitting across the table from the woman with the cast, and he is wearing a maroon polo shirt.}
\par\smallskip
- Wait, let me verify the shirt color of the man sitting across from the woman with the cast in the meeting room: at 14:33, a close-up shows him smiling in a dark red or maroon polo shirt, which matches the maroon polo shirt worn by the man pointing at the machine on the factory floor at 14:50.
\par\smallskip
- At 14:49, the speaker says "we ask them about okay source to sink, how do you bring in the materials" as the video cuts to a factory floor showing a man in a maroon polo shirt pointing at a large grey machine.
\par\smallskip
- At 14:50, the camera zooms in on the front panel of the grey machine being pointed at by the man in the maroon polo shirt.
\par\smallskip
- If the machine being pointed at were the EOS M 290, the label on its front panel would have to read "EOS M 290" when the man points to it at 14:50, but the text printed on the machine clearly reads "EOS M 400" \textemdash{} so it is the EOS M 400.
\par\smallskip
- The machine pointed at by the man in the maroon polo shirt is labeled "EOS M 400", which corresponds to option A.
\par\smallskip
Next, let me analyze the options:
\par\smallskip
- A. EOS M 400 \textemdash{} This matches the label on the machine pointed at by the man in the maroon polo shirt at 14:50.
\par\smallskip
- B. EOS M 290 \textemdash{} This model is mentioned earlier at 05:41, not pointed at here.
\par\smallskip
- C. EOS M 300 \textemdash{} This model is not shown in the video.
\par\smallskip
- D. EOS M 100 \textemdash{} This model is not shown in the video.
\par\smallskip
\textbf{Finally, the correct option is A.
\textless{}/think\textgreater{}}
\par\smallskip
A
\end{tcolorbox}

\Needspace{12\baselineskip}
\subsection{A06 / Counting}
\label{case:sft5}
\begin{tcolorbox}[casecard,breakable,title={Question, evidence, and training answer}]
\textbf{Question.} In the segment about the filter that uses a dot indicator on its strip technology to signal when it needs to be changed, how many text boxes appear in the visual list of its features?
\par\smallskip
\textbf{Options.} A. 15 B. 12 C. 6 D. 4 E. 3 F. 9 G. 0 H. 7\par\smallskip\textbf{Training answer.} F: 9
\par\smallskip\noindent \begin{minipage}[t]{0.48\linewidth}\centering
\includegraphics[width=\linewidth,height=2.5cm,keepaspectratio]{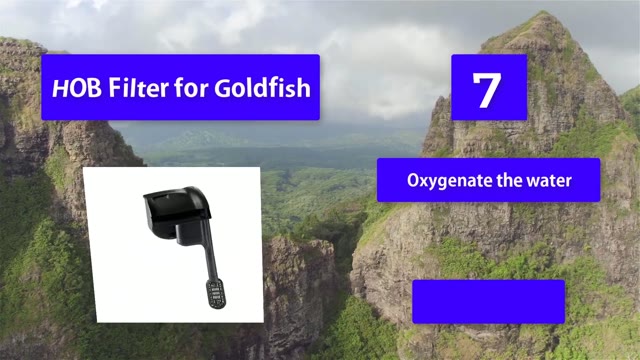}
\par\scriptsize Source video, 01:06\end{minipage}\par\smallskip
\textbf{Upstream reasoning steps.}\par 1. [audio 00:53-01:02] The narrator says 'Dot time strip technology' which uniquely identifies the goldfish filter segment.\par 2. [visual 01:02-01:11] The visual list for the goldfish filter segment contains 9 text boxes (three sets of three, with the final set appearing at this timestamp).
\par\smallskip\textit{Task label: source-assigned.}
\end{tcolorbox}
\begin{tcolorbox}[casecard,breakable,title={Complete synthetic thinking and final answer},title after break={Complete synthetic thinking and final answer (continued)}]
\textless{}think\textgreater{}
Got it, let's tackle the problem.
\par\smallskip
The user wants to know the number of text boxes that appear in the visual list of features for the filter that uses a dot indicator on its strip technology to signal when it needs to be changed. The crux is identifying this specific filter in the video and counting the text boxes that appear on screen during its segment.
\par\smallskip
First, let me scan the video for the segments relevant to the question:
\par\smallskip
- I need to find the segment for the filter that features strip technology to indicate when it needs to be changed and count its feature text boxes.
\par\smallskip
\textbf{- At 00:39, the segment for "HOB Filter for Goldfish" begins, showing the filter on the left.}
\par\smallskip
\textbf{- At 00:44, the first text box appears on screen with the text "Timestrip technology tells".}
\par\smallskip
- At 00:47, the second text box appears with the text "The filter needs to be changed".
\par\smallskip
- At 00:49, the third text box appears with the text "Easy access door".
\par\smallskip
- At 00:54, the fourth text box appears with the text "Makes changing the filter".
\par\smallskip
- At 00:57, the fifth text box appears with the text "Easy setup for beginners".
\par\smallskip
- At 00:59, the sixth text box appears with the text "Outflow is powerful enough".
\par\smallskip
- At 01:05, the seventh text box appears with the text "Oxygenate the water".
\par\smallskip
- At 01:07, the eighth text box appears with the text "Internal motor minimized noise".
\par\smallskip
- At 01:09, the ninth text box appears with the text "Rate of 110 gallons per hour".
\par\smallskip
- If there were only seven text boxes in this list, the count would stop at the box for "Oxygenate the water", but two more boxes appear after it to complete the list of nine features.
\par\smallskip
- The total number of text boxes that appear in this visual list is nine, which corresponds to option F.
\par\smallskip
Next, let me analyze the options:
\par\smallskip
- A. 15 \textemdash{} this is incorrect as only nine boxes appear.
\par\smallskip
- B. 12 \textemdash{} this is incorrect as only nine boxes appear.
\par\smallskip
- C. 6 \textemdash{} this is incorrect as only nine boxes appear.
\par\smallskip
- D. 4 \textemdash{} this is incorrect as only nine boxes appear.
\par\smallskip
- E. 3 \textemdash{} this is incorrect as only nine boxes appear.
\par\smallskip
- F. 9 \textemdash{} exactly nine text boxes are displayed in the goldfish filter segment.
\par\smallskip
- G. 0 \textemdash{} this is incorrect as only nine boxes appear.
\par\smallskip
- H. 7 \textemdash{} this is incorrect as only nine boxes appear.
\par\smallskip
\textbf{Finally, the correct option is F.
\textless{}/think\textgreater{}}
\par\smallskip
F
\end{tcolorbox}

\Needspace{12\baselineskip}
\subsection{A07 / Temporal Reasoning \& Ordering}
\label{case:sft6}
\begin{tcolorbox}[casecard,breakable,title={Question, evidence, and training answer}]
\textbf{Question.} At the moment the shirtless Maori warrior with facial tattoos and a red headband performing a welcoming challenge with a wooden weapon first appeared, which specific audio event was synchronized with it?
\par\smallskip
\textbf{Options.} A. A single warrior shouting a traditional challenge B. A group of men chanting rhythmically during an indoor performance C. The rhythmic percussive sound of poi balls accompanied by singing D. The deep sound of a traditional horn being blown\par\smallskip\textbf{Training answer.} A: A single warrior shouting a traditional challenge
\par\smallskip\noindent \begin{minipage}[t]{0.48\linewidth}\centering
\includegraphics[width=\linewidth,height=2.5cm,keepaspectratio]{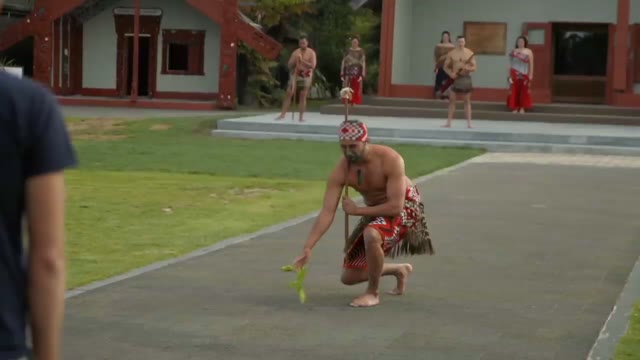}
\par\scriptsize Source video, 01:06\end{minipage}\par\smallskip
\textbf{Upstream reasoning steps.}\par 1. [visual 01:05-01:08] Shirtless Maori warrior with facial tattoos and a red headband steps forward and performs a challenge with a wooden weapon.\par 2. [audio 01:05-01:08] Loud, traditional shouting by a single warrior.\par 3. [inference] When the shirtless Maori warrior with the red headband and wooden weapon appears at 01:05, the synchronized audio is his loud, traditional shout, distinguishing it from the horn, group chant, and poi ball sounds that occur at other moments in the sequence.
\par\smallskip\textit{Task label: source-assigned.}
\end{tcolorbox}
\begin{tcolorbox}[casecard,breakable,title={Complete synthetic thinking and final answer},title after break={Complete synthetic thinking and final answer (continued)}]
\textless{}think\textgreater{}
Got it, let's tackle the problem.
\par\smallskip
The user wants to identify the specific audio event synchronized with the first appearance of a shirtless Maori warrior with facial tattoos and a red headband performing a welcoming challenge with a wooden weapon. The crux is matching the visual of this specific warrior's first on-screen appearance with the concurrent sound.
\par\smallskip
First, let me scan the video for the segments relevant to the question:
\par\smallskip
- I need to find the first appearance of the shirtless Maori warrior with facial tattoos and a red headband performing a welcoming challenge with a wooden weapon.
\par\smallskip
\textbf{- At 00:00:58, Kyle Pedley says "I've been asked to be involved in the Maori cultural welcome." as the visitors walk towards the meeting ground.}
\par\smallskip
\textbf{- At 00:01:01, Kyle Pedley says "It sounds both terrifying and exciting as well, so looking forward to it." while a warrior blows a traditional shell horn.}
\par\smallskip
- At 00:01:05, the shirtless Maori warrior with facial tattoos and a red and white headband first appears, stepping forward with a wooden weapon and shouting a traditional challenge.
\par\smallskip
- At 00:01:08, Brad Knopp says "We got to actually learn the Haka." as the visitors watch the performance.
\par\smallskip
- At 00:01:10, a group of men is seen chanting rhythmically during an indoor Haka performance.
\par\smallskip
- At 00:01:14, women are shown performing with poi balls accompanied by singing.
\par\smallskip
- If the synchronized audio event were a group of men chanting rhythmically, the warrior would have to be shown indoors at 00:01:10, but his first appearance is outdoors at 00:01:05 \textemdash{} so it is a single warrior shouting.
\par\smallskip
- The first appearance of the described warrior at 00:01:05 is synchronized with a single warrior shouting a traditional challenge, which corresponds to option A.
\par\smallskip
Next, let me analyze the options:
\par\smallskip
- A. A single warrior shouting a traditional challenge \textemdash{} This matches the audio heard at 00:01:05 when the specific warrior first appears.
\par\smallskip
- B. A group of men chanting rhythmically during an indoor performance \textemdash{} This occurs later at 00:01:10 during the indoor Haka.
\par\smallskip
- C. The rhythmic percussive sound of poi balls accompanied by singing \textemdash{} This is heard at 00:01:14.
\par\smallskip
- D. The deep sound of a traditional horn being blown \textemdash{} This is heard earlier at 00:01:01.
\par\smallskip
\textbf{Finally, the correct option is A.
\textless{}/think\textgreater{}}
\par\smallskip
A
\end{tcolorbox}

\Needspace{12\baselineskip}
\subsection{A08 / Spatial Reasoning}
\label{case:sft7}
\begin{tcolorbox}[casecard,breakable,title={Question, evidence, and training answer}]
\textbf{Question.} What is the position of the blue and pink motherboard ports relative to the CPU fan header on the screen when the narrator says, 'This is where you will plug in the four-pin connector for the CPU fan'?
\par\smallskip
\textbf{Options.} A. On the left. B. On the right. C. Directly above. D. Directly below.\par\smallskip\textbf{Training answer.} B: On the right.
\par\smallskip\noindent \begin{minipage}[t]{0.48\linewidth}\centering
\includegraphics[width=\linewidth,height=2.5cm,keepaspectratio]{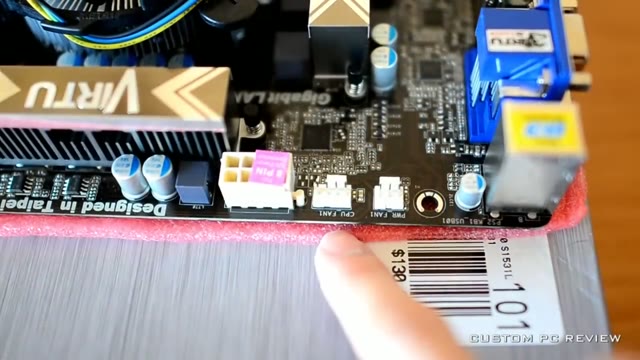}
\par\scriptsize Source video, 02:33\end{minipage}\par\smallskip
\textbf{Upstream reasoning steps.}\par 1. [audio 02:31-02:35] This is where you will plug in the four-pin connector for the CPU fan.\par 2. [visual 02:31-02:35] The camera angle shows the blue and pink rear I/O ports positioned on the right side of the screen relative to the CPU fan header.\par 3. [inference] The audio quote identifies the specific timeframe (02:31-02:35), during which the camera angle is reversed from its usual orientation, showing the blue and pink ports on the right side of the CPU fan header on the screen.
\par\smallskip\textit{Task label: source-assigned.}
\end{tcolorbox}
\begin{tcolorbox}[casecard,breakable,title={Complete synthetic thinking and final answer},title after break={Complete synthetic thinking and final answer (continued)}]
\textless{}think\textgreater{}
Got it, let's tackle the problem.
\par\smallskip
The user wants to know the position of the blue and pink motherboard ports relative to the CPU fan header when the narrator points to it and speaks. The crux is identifying where these ports are located on the screen in the close-up shot of the CPU fan header.
\par\smallskip
First, let me scan the video for the segments relevant to the question:
\par\smallskip
- I need to listen for the phrase "This is where you will plug in the four-pin connector for the CPU fan" and observe the position of the blue and pink ports relative to the CPU fan header on the screen.
\par\smallskip
\textbf{- At 00:01, the narrator introduces himself saying, "Hey, what's up YouTube? Sam here with Custom PC Review", while showing the motherboard box labeled "ASRock" and "Z68 Pro3-M".}
\par\smallskip
\textbf{- At 02:26, the narrator says, "Next, find the four-pin header on your motherboard that says CPU fan", as the camera prepares to zoom in on the board.}
\par\smallskip
- At 02:31, the narrator points to the header and says, "This is where you will plug in the four-pin connector for the CPU fan", showing a close-up of the board.
\par\smallskip
- Wait, let me re-read the motherboard labels in that close-up at 02:31: the white power connector to the left of the CPU fan header is labeled "8 Pin" on its pink sticker, while the header itself is labeled "CPU\_FAN1" and the adjacent header is "PWR\_FAN1", confirming the layout.
\par\smallskip
- If the blue and pink motherboard ports were on the left, they would have to be visible next to the white power connector with the pink sticker, but that area only shows capacitors and a heatsink \textemdash{} so they are on the right.
\par\smallskip
- The blue and pink motherboard ports are located on the right side of the screen relative to the CPU fan header, which corresponds to option B.
\par\smallskip
Next, let me analyze the options:
\par\smallskip
- A. On the left. \textemdash{} Incorrect, this side shows the white power connector.
\par\smallskip
- B. On the right. \textemdash{} Correct, the blue VGA and pink audio ports are visible here.
\par\smallskip
- C. Directly above. \textemdash{} Incorrect.
\par\smallskip
- D. Directly below. \textemdash{} Incorrect.
\par\smallskip
\textbf{Finally, the correct option is B.
\textless{}/think\textgreater{}}
\par\smallskip
B
\end{tcolorbox}

\Needspace{12\baselineskip}
\subsection{A09 / Causal Reasoning}
\label{case:sft8}
\begin{tcolorbox}[casecard,breakable,title={Question, evidence, and training answer}]
\textbf{Question.} What defensive action does the trainer state he looks for before performing the move where he pivots on his inside foot and opens his outside hip?
\par\smallskip
\textbf{Options.} A. The defender switches B. The defender mirrors the ball handler C. The defender starts to sag D. The screener's defender sucks up to the frame E. The defender drops back F. The defender goes under the screen G. The defender corrals the ball handler H. The defender hedges\par\smallskip\textbf{Training answer.} D: The screener's defender sucks up to the frame
\par\smallskip\noindent \begin{minipage}[t]{0.48\linewidth}\centering
\includegraphics[width=\linewidth,height=2.5cm,keepaspectratio]{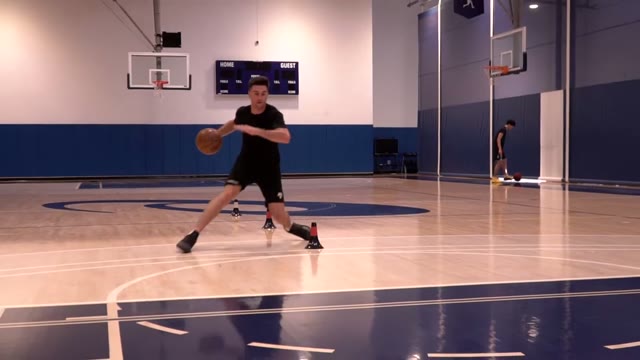}
\par\scriptsize Source video, 02:17\end{minipage}\par\smallskip
\textbf{Upstream reasoning steps.}\par 1. [visual 02:12-02:22] The trainer is seen performing a slow-motion move where he pivots on his inside foot and opens his outside hip, which is the visual demonstration of the hip rotation.\par 2. [audio 02:08-02:13] The trainer says he looks for the screener's defender sucking up to the frame before performing the hip rotation.\par 3. [audiovisual 02:08-02:13] Combining s1 and s2, the defender's action of sucking up to the frame is the cause for the demonstrated hip rotation move.
\par\smallskip\textit{Task label: source-assigned.}
\end{tcolorbox}
\begin{tcolorbox}[casecard,breakable,title={Complete synthetic thinking and final answer},title after break={Complete synthetic thinking and final answer (continued)}]
\textless{}think\textgreater{}
Got it, let's tackle the problem.
\par\smallskip
The user wants to identify the defensive action the trainer looks for before executing a move that involves pivoting on the inside foot and opening the outside hip. The crux is matching the trainer's spoken cue for this specific footwork to one of the listed defensive scenarios.
\par\smallskip
First, let me scan the video for the segments relevant to the question:
\par\smallskip
- I need to listen for the trainer's description of pivoting on the inside foot and opening the outside hip, and identify the defensive action that prompts it.
- At 00:11, the trainer introduces the drill as "snaking the pick and roll".
- At 00:45, the trainer explains that if the screener's defender starts to "sag" or put him in a "corral situation", he will execute a step-back.
- For the sag or corral reads to be correct, they would have to trigger the hip rotation, but the trainer explicitly connects them to the step-back instead.
- At 01:26, the trainer explains a change of tempo called a "lift".
- At 02:04, the trainer introduces the "hip rotation" option.
- At 02:08, the trainer states, "the read is going to be screener's defender sucking up to our frame."
- At 02:16, the trainer describes the mechanics of the hip rotation as "slightly pivoting on that inside foot, opening up that outside hip".
- At 02:25, the trainer notes that defenders have a "tendency of mirroring what the offensive player does".
- The trainer's explanation directly connects the inside foot pivot and outside hip opening to the read where the screener's defender sucks up to the frame, which is option D.
\par\smallskip
Next, let me analyze the options:
\par\smallskip
- A. The defender switches \textemdash{} This action is not mentioned.
- B. The defender mirrors the ball handler \textemdash{} This is a general defensive tendency.
- C. The defender starts to sag \textemdash{} This prompts the step-back.
- D. The screener's defender sucks up to the frame \textemdash{} This is the exact read the trainer states he looks for before performing the hip rotation with the inside foot pivot.
- E. The defender drops back \textemdash{} This is not discussed.
- F. The defender goes under the screen \textemdash{} This is not mentioned.
- G. The defender corrals the ball handler \textemdash{} This leads to the step-back.
- H. The defender hedges \textemdash{} This is not mentioned.
\par\smallskip
\textbf{Finally, the correct option is D.
\textless{}/think\textgreater{}}
\par\smallskip
D
\end{tcolorbox}

\Needspace{12\baselineskip}
\subsection{A10 / Emotion \& Sentiment Analysis}
\label{case:sft9}
\begin{tcolorbox}[casecard,breakable,title={Question, evidence, and training answer}]
\textbf{Question.} How does the emotional state of the young man wearing a black cap and dark grey t-shirt change after he comes home?
\par\smallskip
\textbf{Options.} A. From calm to fearful. B. From calm to angry. C. From fearful to angry. D. From annoyed to fearful.\par\smallskip\textbf{Training answer.} A: From calm to fearful.
\par\smallskip\noindent \begin{minipage}[t]{0.48\linewidth}\centering
\includegraphics[width=\linewidth,height=2.5cm,keepaspectratio]{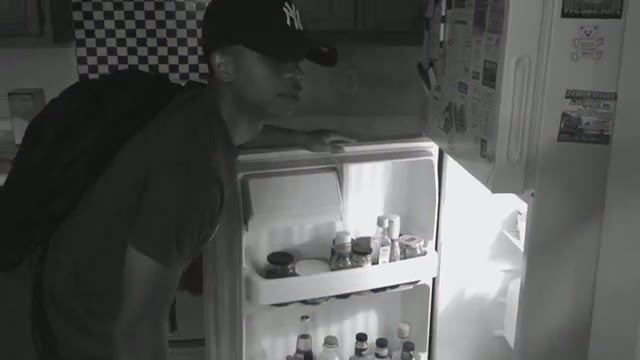}
\par\scriptsize Source video, 02:09\end{minipage}\par\smallskip
\textbf{Upstream reasoning steps.}\par 1. [audio 02:01-02:05] The young man casually and calmly says school was good and nothing out of the ordinary happened.\par 2. [visual 02:06-02:12] His mother says 'Your principal called me,' and he freezes with a wide-eyed, shocked expression.\par 3. [audio 02:16-02:21] He stammers anxiously, trying to defend himself.\par 4. [inference] The young man is initially relaxed and calm when discussing his day, but his expression and tone shift to clear fear and anxiety immediately after his mother mentions that his principal called.
\par\smallskip\textit{Task label: source-assigned.}
\end{tcolorbox}
\begin{tcolorbox}[casecard,breakable,title={Complete synthetic thinking and final answer},title after break={Complete synthetic thinking and final answer (continued)}]
\textless{}think\textgreater{}
Got it, let's tackle the problem.
\par\smallskip
The user wants to know how the emotional state of the young man wearing a black cap and dark grey t-shirt changes after he comes home in the video. The crux is identifying his initial emotional state when he enters the house and tracing how it shifts as the scene unfolds.
\par\smallskip
First, let me scan the video for the segments relevant to the question:
\par\smallskip
- I need to observe the emotional reactions of the young man in the black cap and dark grey t-shirt after he comes home.
- At 01:54, the young man opens the door, enters the house, and calmly says "Hey mom" to the person on the couch.
- At 02:00, he walks into the living room and continues speaking in a relaxed, normal tone, saying that school was "good" and "nothing out of the ordinary."
- At 02:06, the mother figure on the couch tells him "Your principal called me," which immediately changes the tone of the interaction.
- At 02:10, the young man stands by the open refrigerator, looking startled and muttering "shit" under his breath.
- At 02:13, the mother figure confronts him in the kitchen while holding a belt, demanding to know what she told him about acting up in school.
- At 02:17, the young man becomes visibly defensive and anxious, pleading "Alright mom, look. I know I got in trouble, but it wasn't me. I promise."
- If he were angry rather than fearful during this confrontation at 02:17, he would be shouting back or showing aggression, but instead he is stammering, backing away, and trying to appease his mother.
- At 02:22, the mother figure prepares the belt, and the young man's expression becomes increasingly terrified as he braces for a beating.
- The young man goes from a relaxed, normal demeanor when entering the house to a terrified, defensive state when confronted with a belt, showing a transition from calm to fearful \textemdash{} that is option A.
\par\smallskip
Next, let me analyze the options:
\par\smallskip
- A. From calm to fearful. \textemdash{} He enters the house in a relaxed state and ends up terrified when chased with a belt.
- B. From calm to angry. \textemdash{} He shows fear and defensiveness rather than anger.
- C. From fearful to angry. \textemdash{} He starts calm and does not become angry.
- D. From annoyed to fearful. \textemdash{} He is relaxed, not annoyed, when he first arrives.
\par\smallskip
\textbf{Finally, the correct option is A.
\textless{}/think\textgreater{}}
\par\smallskip
A
\end{tcolorbox}

\Needspace{12\baselineskip}
\subsection{A11 / Hypothetical \& Predictive Reasoning}
\label{case:sft10}
\begin{tcolorbox}[casecard,breakable,title={Question, evidence, and training answer}]
\textbf{Question.} If the user had not applied the pre-selection filter for ports and net labels before drawing the selection rectangle, what would most likely be the result?
\par\smallskip
\textbf{Options.} A. All objects except ports and net labels would be unselected. B. The software would highlight the filter settings. C. Other objects besides ports and net labels would also be selected. D. The selection rectangle would be disabled. E. Nothing would be selected. F. Only ports would be selected. G. Only net labels would be selected. H. Only ports and net labels would be selected.\par\smallskip\textbf{Training answer.} C: Other objects besides ports and net labels would also be selected.
\par\smallskip\noindent \begin{minipage}[t]{0.48\linewidth}\centering
\includegraphics[width=\linewidth,height=2.5cm,keepaspectratio]{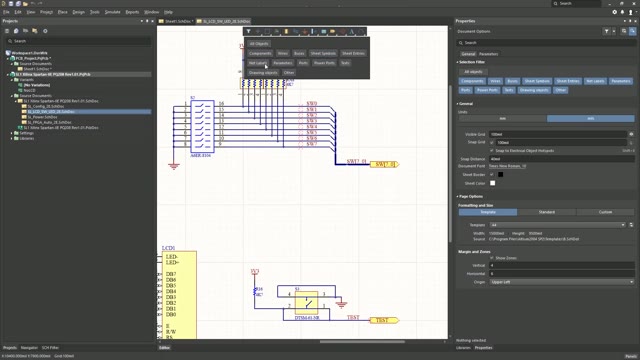}
\par\scriptsize Source video, 03:18\end{minipage}\hfill \begin{minipage}[t]{0.48\linewidth}\centering
\includegraphics[width=\linewidth,height=2.5cm,keepaspectratio]{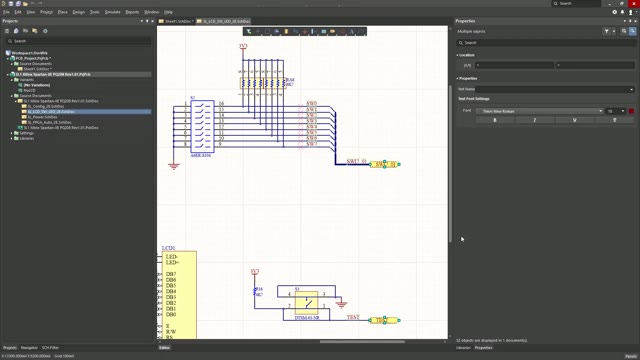}
\par\scriptsize Source video, 03:35\end{minipage}\par\smallskip
\textbf{Upstream reasoning steps.}\par 1. [visual 03:15-03:22] The user applies the pre-selection filter for ports and net labels.\par 2. [audio 03:32-03:38] The narrator states that because of the filter, only ports and net labels are selected, despite other objects being present.\par 3. [audiovisual 03:32-03:38] Infer that without the filter, the other objects in the selection area would also be selected.
\par\smallskip\textit{Task label: source-assigned.}
\end{tcolorbox}
\begin{tcolorbox}[casecard,breakable,title={Complete synthetic thinking and final answer},title after break={Complete synthetic thinking and final answer (continued)}]
\textless{}think\textgreater{}
Got it, let's tackle the problem.
\par\smallskip
The user wants to know the outcome of drawing a selection rectangle without first applying the pre-selection filter for ports and net labels. The crux is determining how the software behaves when no filter restricts the selection tool.
\par\smallskip
First, let me scan the video for the segments relevant to the question:
\par\smallskip
- I need to find the segment where the narrator explains the function of the pre-selection filter during a rectangle selection.
\par\smallskip
\textbf{- At 03:05, the narrator introduces the pre-selection filter to facilitate selecting specific schematic elements.}
\par\smallskip
\textbf{- At 03:15, the narrator applies the filter specifically for ports and net labels.}
\par\smallskip
- At 03:25, a selection rectangle is drawn over a variety of components, wires, ports, and net labels.
\par\smallskip
- At 03:32, the narrator explains that "only ports and net labels will be selected, despite other objects within the selection area, because of our pre-selection filter."
\par\smallskip
- If the filter were not applied, the standard selection rectangle would capture all objects it touches, meaning other elements like wires and components would also be selected.
\par\smallskip
- The demonstration confirms that the filter is the sole reason other objects are ignored, so omitting it would result in selecting other objects besides ports and net labels, which is option C.
\par\smallskip
Next, let me analyze the options:
\par\smallskip
- A. All objects except ports and net labels would be unselected. \textemdash{} This describes the filtered outcome, not the unfiltered one.
\par\smallskip
- B. The software would highlight the filter settings. \textemdash{} There is no indication of this behavior in the video.
\par\smallskip
- C. Other objects besides ports and net labels would also be selected. \textemdash{} Without the filter, standard rectangle selection captures all objects in its path.
\par\smallskip
- D. The selection rectangle would be disabled. \textemdash{} The tool remains fully functional without a filter.
\par\smallskip
- E. Nothing would be selected. \textemdash{} Standard selection would still function normally.
\par\smallskip
- F. Only ports would be selected. \textemdash{} This would require a port-only filter.
\par\smallskip
- G. Only net labels would be selected. \textemdash{} This would require a net-label-only filter.
\par\smallskip
- H. Only ports and net labels would be selected. \textemdash{} This is the result when the filter is active.
\par\smallskip
\textbf{Finally, the correct option is C.
\textless{}/think\textgreater{}}
\par\smallskip
C
\end{tcolorbox}

\Needspace{12\baselineskip}
\subsection{A12 / Character Relationship \& Interaction}
\label{case:sft11}
\begin{tcolorbox}[casecard,breakable,title={Question, evidence, and training answer}]
\textbf{Question.} What is the nature of the interaction between the vlogger and the statue of Don Knotts just before the vlogger stands up and walks away?
\par\smallskip
\textbf{Options.} A. The vlogger says, "Thanks for the laughs, man," and gives the statue a fist bump. B. The vlogger says, "Well, take care," and gives the statue a fist bump. C. The vlogger says, "Well, take care," and shakes the statue's hand. D. The vlogger says, "Thanks for the laughs, man," and shakes the statue's hand.\par\smallskip\textbf{Training answer.} B: The vlogger says, "Well, take care," and gives the statue a fist bump.
\par\smallskip\noindent \begin{minipage}[t]{0.48\linewidth}\centering
\includegraphics[width=\linewidth,height=2.5cm,keepaspectratio]{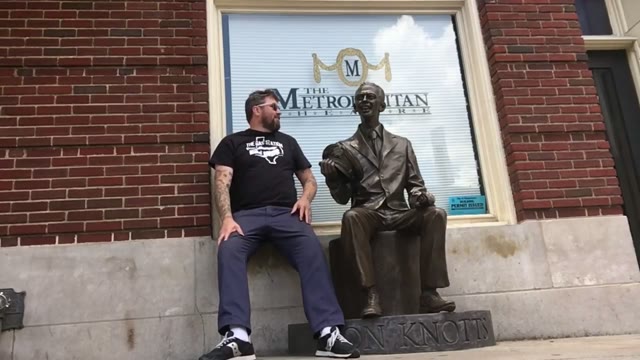}
\par\scriptsize Source video, 05:50\end{minipage}\par\smallskip
\textbf{Upstream reasoning steps.}\par 1. [audio|visual 05:45-05:55] The vlogger sits next to the statue, says "Well, take care," and gives the statue a fist bump before standing up.\par 2. [inference] The vlogger sits next to the statue, says "Well, take care," and physically gives the statue's hand a fist bump before getting up and leaving.
\par\smallskip\textit{Task label: source-assigned.}
\end{tcolorbox}
\begin{tcolorbox}[casecard,breakable,title={Complete synthetic thinking and final answer},title after break={Complete synthetic thinking and final answer (continued)}]
\textless{}think\textgreater{}
Got it, let's tackle the problem.
\par\smallskip
The user wants to know the nature of the interaction between the vlogger and the statue of Don Knotts just before the vlogger stands up and walks away. The crux is identifying the exact words spoken and the physical gesture made right before he leaves.
\par\smallskip
First, let me scan the video for the segments relevant to the question:
- I need to find the segment where the vlogger interacts with the Don Knotts statue and observe his final words and actions before walking away.
- At 03:41, the vlogger walks toward the Metropolitan Theatre where the statue of Don Knotts is seated on a bench.
- At 04:12, the camera shows a close-up of the bronze statue, depicting Don Knotts holding a script and a hat.
- At 04:55, the vlogger stands next to the statue, gesturing toward it as he explains its history.
- At 05:21, the camera pans over the details of the statue, showing the name "DON KNOTTS" on the base.
- At 05:42, the vlogger says "Thanks for the laughs, man. Thanks for the laughs." while the camera is focused closely on the statue's face.
- At 05:46, a jump cut shows the vlogger sitting on the bench directly next to the statue.
- For the final dialogue to be "Thanks for the laughs, man," he would have to say those words while sitting next to the statue at the end of the scene, but he actually says them earlier during a close-up shot of the statue's face.
- At 05:50, while still sitting on the bench, the vlogger says "Well, take care." and taps his closed hand against the statue's hand in a fist bump.
- At 05:53, the vlogger stands up from the bench and walks out of the frame.
- The vlogger says "Well, take care." and gives the statue a fist bump right before standing up, which corresponds to option B.
\par\smallskip
Next, let me analyze the options:
- A. The vlogger says, "Thanks for the laughs, man," and gives the statue a fist bump. \textemdash{} Incorrect dialogue.
- B. The vlogger says, "Well, take care," and gives the statue a fist bump. \textemdash{} The vlogger speaks these exact words, and he initiates a friendly fist bump.
- C. The vlogger says, "Well, take care," and shakes the statue's hand. \textemdash{} Incorrect action.
- D. The vlogger says, "Thanks for the laughs, man," and shakes the statue's hand. \textemdash{} Incorrect dialogue and action.
\par\smallskip
\textbf{Finally, the correct option is B.
\textless{}/think\textgreater{}}
\par\smallskip
B
\end{tcolorbox}

\Needspace{12\baselineskip}
\subsection{A13 / Event Understanding \& State Change}
\label{case:sft12}
\begin{tcolorbox}[casecard,breakable,title={Question, evidence, and training answer}]
\textbf{Question.} After the narrator announced they would begin with a simple query, what action did the user perform in the Dgraph Ratel UI?
\par\smallskip
\textbf{Options.} A. switched to the JSON view B. modified the query to use a filter C. typed a new query in the editor D. hovered over a node E. added a cascade directive F. clicked the Run button G. dragged a node H. changed the mutation data\par\smallskip\textbf{Training answer.} C: typed a new query in the editor
\par\smallskip\noindent \begin{minipage}[t]{0.48\linewidth}\centering
\includegraphics[width=\linewidth,height=2.5cm,keepaspectratio]{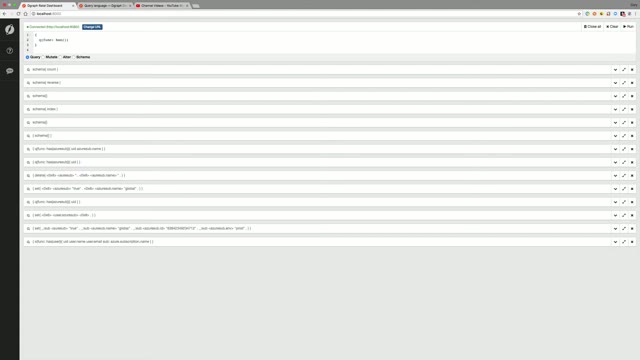}
\par\scriptsize Source video, 02:24\end{minipage}\par\smallskip
\textbf{Upstream reasoning steps.}\par 1. [audio 02:11-02:21] The narrator says 'So let's start off with something simple.'\par 2. [visual 02:11-02:37] The user types a new query into the editor, as shown in the visual caption.\par 3. [audiovisual 02:11-02:37] Combining the audio anchor and the visual action, the user's action after the announcement is typing a new query.
\par\smallskip\textit{Task label: source-assigned.}
\end{tcolorbox}
\begin{tcolorbox}[casecard,breakable,title={Complete synthetic thinking and final answer},title after break={Complete synthetic thinking and final answer (continued)}]
\textless{}think\textgreater{}
Got it, let's tackle the problem.
\par\smallskip
The user wants to know what action was performed in the Dgraph Ratel UI immediately after the narrator announced they would begin with a simple query. The crux is identifying the specific user action that followed this spoken cue.
\par\smallskip
First, let me scan the video for the segments relevant to the question:
\par\smallskip
- I need to listen for the narrator's announcement about starting with a simple query and identify the action that immediately follows it in the UI.
\par\smallskip
\textbf{- At 00:15, the user navigates to the mutate tab to enter some sample data.}
\par\smallskip
\textbf{- At 00:43, the user clicks the Run button to execute the mutation and add the data.}
\par\smallskip
- At 00:50, the narrator says, "So let's start off with a simple uh schema query..."
\par\smallskip
- At 00:58, the user types the new query "schema \{\}" directly into the query editor.
\par\smallskip
- At 01:16, the user clicks the Run button to execute the schema query.
\par\smallskip
- At 01:21, the user hovers over the resulting schema nodes to inspect their attributes.
\par\smallskip
- At 01:45, the user switches to the JSON view to examine the schema values in JSON format.
\par\smallskip
- At 02:11, the narrator says, "So let's start off with something simple... we'll query for anything that has a predicate of user."
\par\smallskip
- At 02:28, the user types the new query "q(func: has(user)) \{ uid user.name \}" into the editor.
\par\smallskip
- If the user had clicked the Run button immediately after the announcement, the query would have executed without any changes, but the user first had to input the query text \textemdash{} so the action is typing a new query.
\par\smallskip
- The immediate action following both announcements of a simple query is typing the query text into the editor, which corresponds to option C.
\par\smallskip
Next, let me analyze the options:
\par\smallskip
- A. switched to the JSON view \textemdash{} occurred much later at 01:45.
\par\smallskip
- B. modified the query to use a filter \textemdash{} performed much later.
\par\smallskip
- C. typed a new query in the editor \textemdash{} immediate action following the announcement.
\par\smallskip
- D. hovered over a node \textemdash{} done after running the query.
\par\smallskip
- E. added a cascade directive \textemdash{} done near the end.
\par\smallskip
- F. clicked the Run button \textemdash{} done only after typing.
\par\smallskip
- G. dragged a node \textemdash{} performed to organize the graph.
\par\smallskip
- H. changed the mutation data \textemdash{} performed at the beginning.
\par\smallskip
\textbf{Finally, the correct option is C.
\textless{}/think\textgreater{}}
\par\smallskip
C
\end{tcolorbox}

\Needspace{12\baselineskip}
\subsection{A14 / Cross-Scene Association \& Long-Context}
\label{case:sft13}
\begin{tcolorbox}[casecard,breakable,title={Question, evidence, and training answer}]
\textbf{Question.} The narrator mentions a prominent tech innovator during an early discussion about living in a simulation, which coincides with the drawing of a specific character type. Much later, when the narrator discusses super humans, a different version of this same character type is depicted. What distinctive physical feature does this later character possess?
\par\smallskip
\textbf{Options.} A. large white eyes and blue wings B. a skull mask and black hooded cloak C. a robotic right arm D. white feathered angel wings E. yellow hazmat suits and respirators F. a whistle around his neck G. glasses and brown hair H. blue overalls and a straw hat\par\smallskip\textbf{Training answer.} D: white feathered angel wings
\par\smallskip\noindent \begin{minipage}[t]{0.48\linewidth}\centering
\includegraphics[width=\linewidth,height=2.5cm,keepaspectratio]{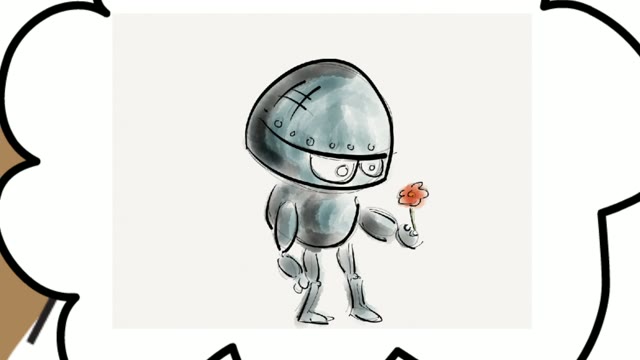}
\par\scriptsize Source video, 00:23\end{minipage}\hfill \begin{minipage}[t]{0.48\linewidth}\centering
\includegraphics[width=\linewidth,height=2.5cm,keepaspectratio]{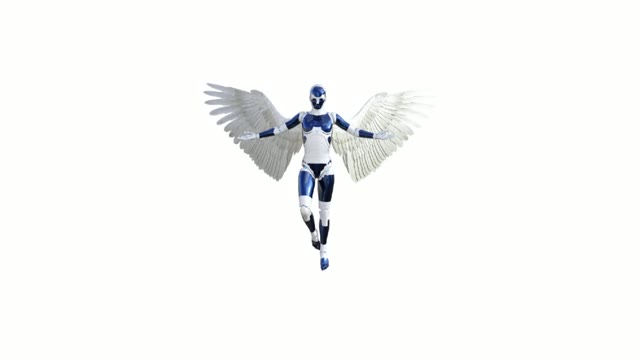}
\par\scriptsize Source video, 03:10\end{minipage}\par\smallskip
\textit{This training row has no separately recorded structured upstream evidence. The source frames illustrate its question; the complete stored synthetic response follows.}
\par\smallskip\textit{Task label: retrospective.}
\end{tcolorbox}
\begin{tcolorbox}[casecard,breakable,title={Complete synthetic thinking and final answer},title after break={Complete synthetic thinking and final answer (continued)}]
\textless{}think\textgreater{}
Got it, let's tackle the problem.
\par\smallskip
The user wants to identify the physical feature of a character depicted during a discussion about super humans, which is a different version of a character type drawn when a prominent tech innovator is mentioned in relation to living in a simulation. The crux is determining which character type corresponds to the tech innovator's mention and identifying the unique physical feature of the later version of this character.
\par\smallskip
First, let me scan the video for the segments relevant to the question:
\par\smallskip
- I need to find the early segment mentioning a tech innovator and a simulation to identify the character type, and then locate the later segment discussing super humans to see the updated character's features.
\par\smallskip
\textbf{- At 00:00:12, a drawing of a man with glasses and brown hair appears on screen as the narrator introduces the book.}
\par\smallskip
\textbf{- At 00:00:21, the narrator mentions "Elon Musk living in a simulation" while a drawing of a small grey robot holding a flower is shown.}
\par\smallskip
- At 00:01:15, two figures in yellow hazmat suits and respirators are shown during a discussion about Ebola.
\par\smallskip
- At 00:01:34, a cartoon farmer wearing blue overalls and a straw hat is shown holding a hoe.
\par\smallskip
- At 00:02:21, a cartoon sports coach with a whistle around his neck is drawn.
\par\smallskip
- At 00:02:34, a figure representing death is shown wearing a skull mask and a black hooded cloak.
\par\smallskip
- At 00:03:08, the narrator discusses "super humans" while a blue and white robot with large white feathered angel wings is shown flying.
\par\smallskip
- At 00:03:22, a young man with a robotic right arm is shown sitting at a table.
\par\smallskip
- If the later character type were a cyborg human rather than a full robot, the character shown during the super humans discussion would have a robotic right arm instead of being a flying machine with wings.
\par\smallskip
- The flying robot shown during the super humans discussion at 00:03:08 clearly has white feathered angel wings, which corresponds to option D.
\par\smallskip
Next, let me analyze the options:
\par\smallskip
- A. large white eyes and blue wings \textemdash{} the robot's wings are white, not blue.
\par\smallskip
- B. a skull mask and black hooded cloak \textemdash{} this represents death, not the robot.
\par\smallskip
- C. a robotic right arm \textemdash{} this feature belongs to a cyborg shown later, not the flying robot.
\par\smallskip
- D. white feathered angel wings \textemdash{} the flying robot representing super humans possesses these large white wings.
\par\smallskip
- E. yellow hazmat suits and respirators \textemdash{} these are Ebola workers.
\par\smallskip
- F. a whistle around his neck \textemdash{} this is the coach.
\par\smallskip
- G. glasses and brown hair \textemdash{} this is the narrator.
\par\smallskip
- H. blue overalls and a straw hat \textemdash{} this is the farmer.
\par\smallskip
\textbf{Finally, the correct option is D.
\textless{}/think\textgreater{}}
\par\smallskip
D
\end{tcolorbox}

\Needspace{12\baselineskip}
\subsection{A15 / Summarization \& Synthesis}
\label{case:sft14}
\begin{tcolorbox}[casecard,breakable,title={Question, evidence, and training answer}]
\textbf{Question.} What overall mood is created by the combination of the audio and visual elements in this video?
\par\smallskip
\textbf{Options.} A. a melancholic and sorrowful atmosphere B. a joyful and energetic atmosphere C. a dark and chaotic atmosphere D. a dramatic and epic atmosphere E. a tense and suspenseful atmosphere F. a serene and mystical atmosphere G. a calm and peaceful atmosphere H. a whimsical and fantastical atmosphere\par\smallskip\textbf{Training answer.} C: a dark and chaotic atmosphere
\par\smallskip\noindent \begin{minipage}[t]{0.48\linewidth}\centering
\includegraphics[width=\linewidth,height=2.5cm,keepaspectratio]{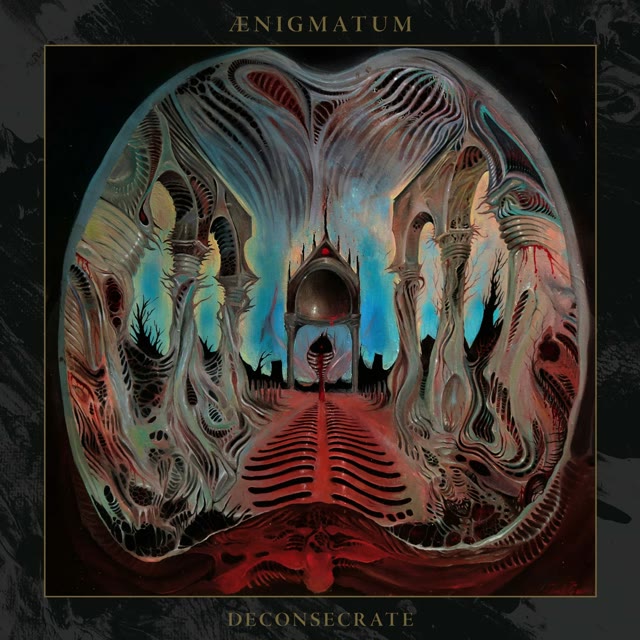}
\par\scriptsize Source video, 00:50\end{minipage}\hfill \begin{minipage}[t]{0.48\linewidth}\centering
\includegraphics[width=\linewidth,height=2.5cm,keepaspectratio]{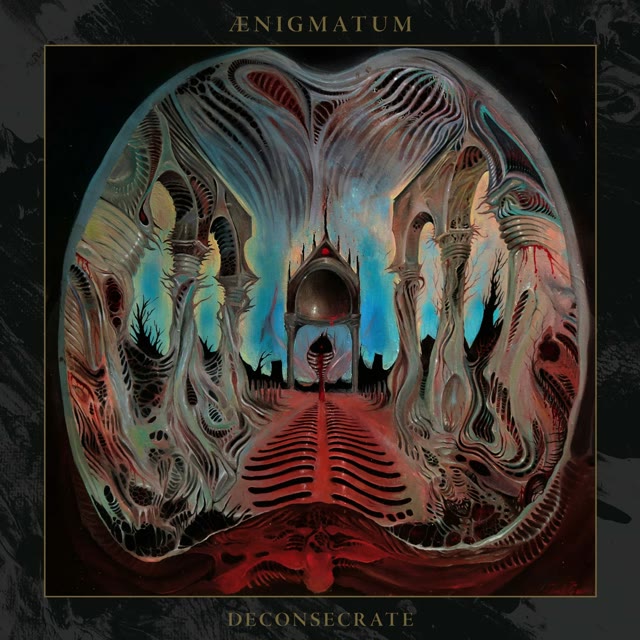}
\par\scriptsize Source video, 01:40\end{minipage}\par\smallskip
\textbf{Upstream reasoning steps.}\par 1. [audio 00:00-00:27] The audio is fast-paced death metal with heavy distorted guitars and blast beat drums, establishing an aggressive and chaotic mood.\par 2. [visual 00:00-02:30] The visual is a dark, surreal album cover with bone-like structures and a dark archway, establishing a dark and foreboding mood.\par 3. [audiovisual 00:00-02:30] Combining the aggressive audio and dark visual, the overall mood is dark and chaotic.
\par\smallskip\textit{Task label: source-assigned.}
\end{tcolorbox}
\begin{tcolorbox}[casecard,breakable,title={Complete synthetic thinking and final answer},title after break={Complete synthetic thinking and final answer (continued)}]
\textless{}think\textgreater{}
Got it, let's tackle the problem.
\par\smallskip
The user wants to determine the overall mood created by combining the audio and visual elements of the video. The crux is identifying how the aggressive musical style and the grotesque, surreal imagery on screen work together to establish a specific atmosphere.
\par\smallskip
First, let me scan the video for the segments relevant to the question:
\par\smallskip
- I need to examine the album artwork on screen and listen to the musical style of the track to characterize their combined aesthetic.
\par\smallskip
\textbf{- At 00:00-00:25, the video displays a static album cover featuring the band name "Ænigmatum" and the album title "Deconsecrate" written around a grotesque, fleshy, skull-like cavern with deep red and pale bone colors.}
\par\smallskip
\textbf{- At 00:26-01:15, the audio shifts from a dark guitar intro into frantic blast beats, heavy distorted riffs, and deep guttural vocals. If the overall mood were dramatic and epic rather than dark and chaotic, the music would need a grand, orchestral, or melodic structure instead of these dissonant, aggressive death metal elements.}
\par\smallskip
- At 01:16-03:00, the relentless tempo and chaotic technical death metal riffs continue alongside the nightmarish, biomechanical visual of a spine-like pathway leading to a dark altar.
\par\smallskip
- The combination of extreme, aggressive death metal and a grotesque, fleshy, hellish landscape creates a dark and chaotic atmosphere, which corresponds to option C.
\par\smallskip
Next, let me analyze the options:
\par\smallskip
- A. a melancholic and sorrowful atmosphere \textemdash{} The aggressive tempo and harsh vocals are not sad.
\par\smallskip
- B. a joyful and energetic atmosphere \textemdash{} The dark, grotesque imagery contradicts any joyful feeling.
\par\smallskip
- C. a dark and chaotic atmosphere \textemdash{} The dissonant death metal and nightmarish, fleshy artwork establish this exact feeling.
\par\smallskip
- D. a dramatic and epic atmosphere \textemdash{} The chaotic style is too aggressive for an epic tone.
\par\smallskip
- E. a tense and suspenseful atmosphere \textemdash{} The music is immediately explosive rather than building suspense.
\par\smallskip
- F. a serene and mystical atmosphere \textemdash{} The harsh vocals and gory visuals are not peaceful.
\par\smallskip
- G. a calm and peaceful atmosphere \textemdash{} The fast blast beats are highly intense and aggressive.
\par\smallskip
- H. a whimsical and fantastical atmosphere \textemdash{} The hellish, biomechanical landscape is macabre rather than whimsical.
\par\smallskip
\textbf{Finally, the correct option is C.
\textless{}/think\textgreater{}}
\par\smallskip
C
\end{tcolorbox}

\Needspace{12\baselineskip}
\subsection{A16 / Ego-Centric \& First-Person Understanding}
\label{case:sft15}
\begin{tcolorbox}[casecard,breakable,title={Question, evidence, and training answer}]
\textbf{Question.} When the sound of highway traffic passing underneath is heard for the first time, what gesture do I make with my left hand?
\par\smallskip
\textbf{Options.} A. Pointing to the right. B. Giving a thumbs up. C. Pointing to the left. D. Adjusting the bike computer.\par\smallskip\textbf{Training answer.} A: Pointing to the right.
\par\smallskip\noindent \begin{minipage}[t]{0.48\linewidth}\centering
\includegraphics[width=\linewidth,height=2.5cm,keepaspectratio]{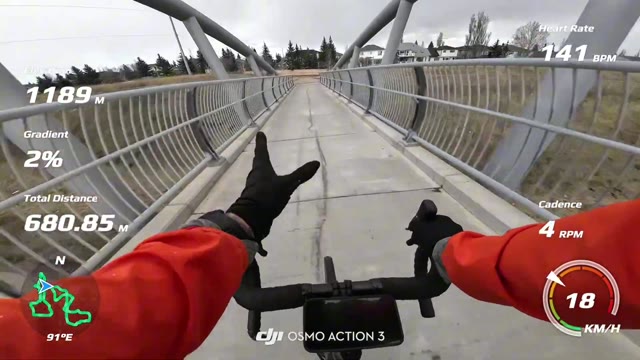}
\par\scriptsize Source video, 01:04\end{minipage}\hfill \begin{minipage}[t]{0.48\linewidth}\centering
\includegraphics[width=\linewidth,height=2.5cm,keepaspectratio]{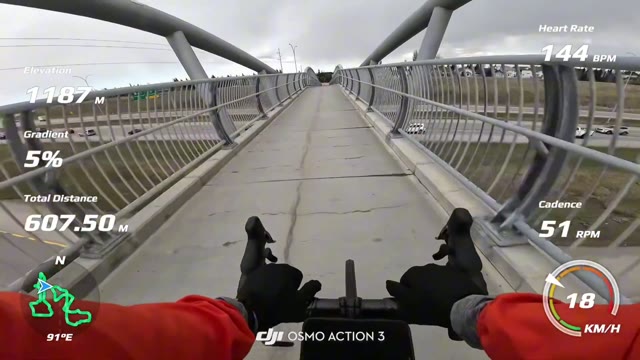}
\par\scriptsize Source video, 00:48\end{minipage}\par\smallskip
\textbf{Upstream reasoning steps.}\par 1. [audio 00:48-01:08] The distinct whooshing sound of cars passing underneath is heard as the cyclist rides over the first bridge.\par 2. [visual 01:03-01:05] The cyclist lifts their left hand off the handlebar and points to the right with their index finger.\par 3. [inference] The sound of highway traffic passing below is first heard when crossing the pedestrian bridge starting at 00:48; during this time, the cyclist raises their left hand and points to the right.
\par\smallskip\textit{Task label: source-assigned.}
\end{tcolorbox}
\begin{tcolorbox}[casecard,breakable,title={Complete synthetic thinking and final answer},title after break={Complete synthetic thinking and final answer (continued)}]
\textless{}think\textgreater{}
Got it, let's tackle the problem.
\par\smallskip
The user wants to identify the gesture made with the cyclist's left hand when the sound of highway traffic passing underneath is first heard. The crux is locating the first instance of audible highway traffic from below and observing the hand gesture that immediately follows.
\par\smallskip
First, let me scan the video for the segments relevant to the question:
\par\smallskip
- I need to find the first segment where the sound of highway traffic passing underneath is heard and identify the gesture made with the left hand.
- At 00:01, the cyclist rides on a paved path through a residential neighborhood.
- At 00:10, the cyclist raises their left hand briefly to adjust their glove.
- At 00:18, the cyclist passes a pedestrian on the path.
- At 00:48, the cyclist begins riding onto a pedestrian bridge over a highway.
- At 00:50, the sound of highway traffic passing underneath becomes clearly audible for the first time.
- At 01:03, the cyclist lifts their left hand off the handlebar.
- At 01:04, the cyclist points to the right with their left hand, extending the index and middle fingers.
- If the cyclist were giving a thumbs up at this moment, the thumb would be extended upward with the other fingers curled, but instead the index and middle fingers are extended to point toward the right side of the bridge.
- At 20:53, near the end of the ride, the cyclist gives a thumbs up with their left hand.
- The gesture made when the highway traffic is first heard is pointing to the right, which corresponds to option A.
\par\smallskip
Next, let me analyze the options:
\par\smallskip
\textbf{- A. Pointing to the right. \textemdash{} This matches the gesture at 01:04, pointing toward the highway.
- B. Giving a thumbs up. \textemdash{} observed only at 20:53.
- C. Pointing to the left. \textemdash{} hand points right instead.
- D. Adjusting the bike computer. \textemdash{} no adjustment occurs.}
\par\smallskip
\textbf{Finally, the correct option is A.
\textless{}/think\textgreater{}}
\par\smallskip
A
\end{tcolorbox}

\Needspace{12\baselineskip}
\subsection{B01 / Problem-solving Adaptation}
\label{case:sft16}
\begin{tcolorbox}[casecard,breakable,title={Question, evidence, and training answer}]
\textbf{Question.} The figure shows four projects with the quality setting selected in the Blend Collage app. According to the video's classification of the save quality options, which project is incorrectly configured?
\par\smallskip
\textbf{Options.} A. Photo Print B. Blog Header C. Email Newsletter D. Mobile Wallpaper\par\smallskip\textbf{Training answer.} B: Blog Header
\par\smallskip\noindent \begin{minipage}[t]{0.48\linewidth}\centering
\includegraphics[width=\linewidth,height=2.5cm,keepaspectratio]{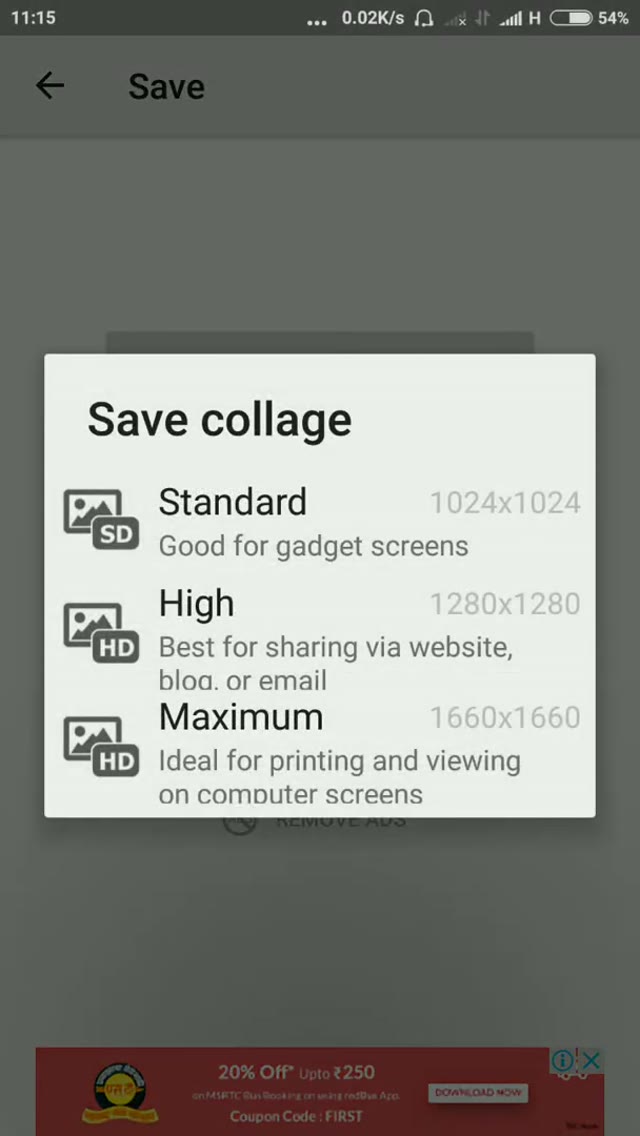}
\par\scriptsize Source video, 05:25\end{minipage}\hfill \begin{minipage}[t]{0.48\linewidth}\centering
\includegraphics[width=\linewidth,height=3.4cm,keepaspectratio]{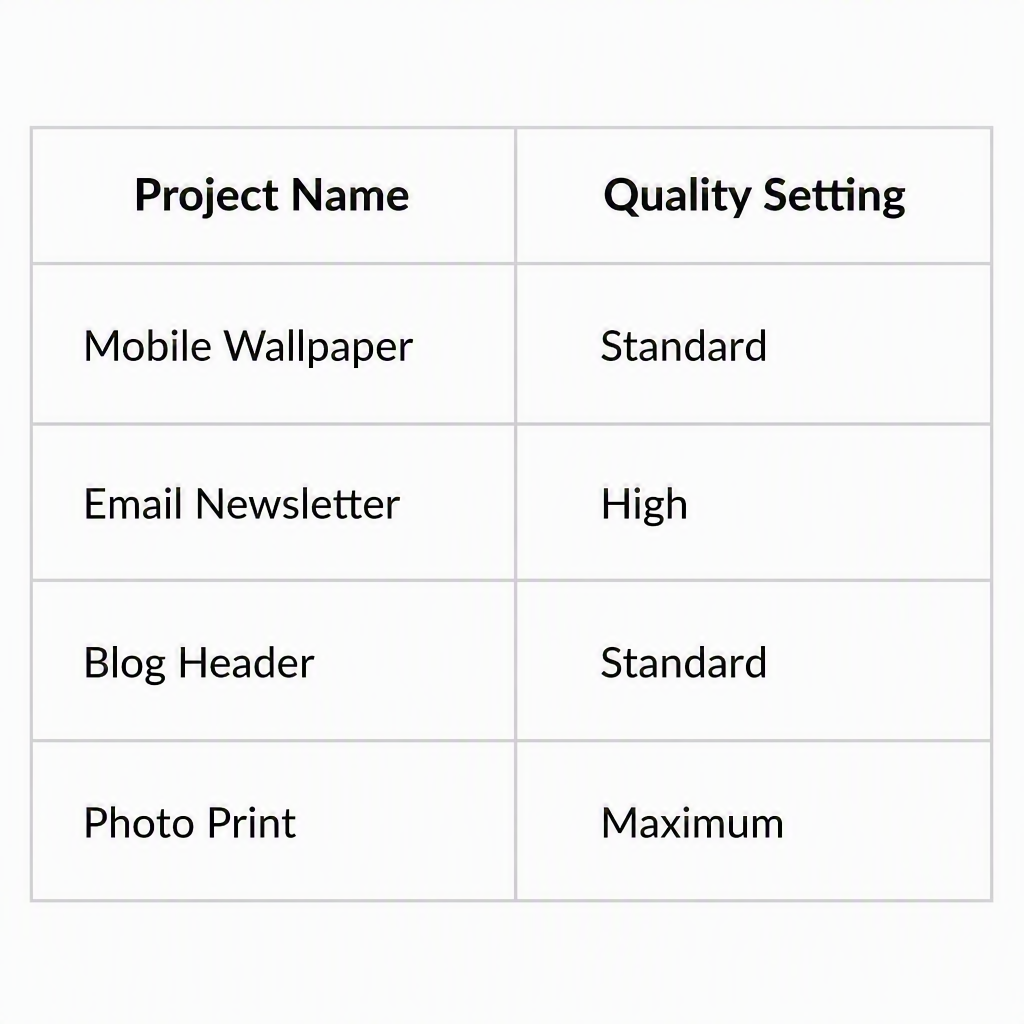}
\par\scriptsize New figure\end{minipage}\par\smallskip
\textbf{Upstream reasoning steps.}\par 1. [05:20-05:30] Recall the video's classification of the three save quality options: Standard is for phone screen, High is for email/website/blog, Maximum is for printing.\par 2. For each project, infer its intended use from the project name: 'Mobile Wallpaper' implies phone screen, 'Email Newsletter' implies email, 'Blog Header' implies blog, 'Photo Print' implies printing.\par 3. Check the quality setting against the classification: Mobile Wallpaper (Standard) \ensuremath{\rightarrow} correct; Email Newsletter (High) \ensuremath{\rightarrow} correct; Blog Header (Standard) \ensuremath{\rightarrow} should be High, so incorrect; Photo Print (Maximum) \ensuremath{\rightarrow} correct.\par 4. Identify the project with the mismatch: Blog Header.
\par\smallskip\textit{Task label: source-assigned.}
\end{tcolorbox}
\begin{tcolorbox}[casecard,breakable,title={Complete synthetic thinking and final answer},title after break={Complete synthetic thinking and final answer (continued)}]
\textless{}think\textgreater{}
Got it, let's tackle the problem.
\par\smallskip
The user wants to identify which project in the provided table is incorrectly configured based on the video's description of save quality options in the Blend Collage app. The crux is matching each project's quality setting in the table to the corresponding save quality option described in the video.
\par\smallskip
First, let me scan the video for the segments relevant to the question:
\par\smallskip
- I need to find the segment where the presenter explains the different save quality options in the app.
\par\smallskip
\textbf{- At 05:20, the presenter clicks the "SAVE" button, which brings up the "Save collage" dialog box.}
\par\smallskip
\textbf{- At 05:22-05:32, the dialog box shows "Standard" as "Good for gadget screens", "High" as "Best for sharing via website, blog, or email", and "Maximum" as "Ideal for printing and viewing on computer screens".}
\par\smallskip
- If the "High" option were meant for gadget screens, the dialog box would describe it as such, but it actually says "Best for sharing via website, blog, or email" while "Standard" is described as "Good for gadget screens" \textemdash{} so "Standard" is for gadget screens and "High" is for websites, blogs, or email.
\par\smallskip
Next, let me read the figure:
\par\smallskip
- The table lists four projects: "Mobile Wallpaper", "Email Newsletter", "Blog Header", and "Photo Print".
\par\smallskip
- For "Mobile Wallpaper", the quality setting is configured as "Standard".
\par\smallskip
- For "Email Newsletter", the quality setting is configured as "High".
\par\smallskip
- For "Blog Header", the quality setting is configured as "Standard".
\par\smallskip
- For "Photo Print", the quality setting is configured as "Maximum".
\par\smallskip
- According to the video, "Standard" is for gadget screens, "High" is for websites, blogs, or emails, and "Maximum" is for printing.
\par\smallskip
- Applying these rules, "Blog Header" should be configured as "High" because it is a blog project, but it is incorrectly configured as "Standard".
\par\smallskip
- Therefore, the project "Blog Header" is incorrectly configured \textemdash{} this is option B.
\par\smallskip
Next, let me analyze the options:
\par\smallskip
- A. Photo Print \textemdash{} correctly configured as Maximum.
\par\smallskip
- B. Blog Header \textemdash{} incorrectly configured as Standard, as the video states blogs require High quality.
\par\smallskip
- C. Email Newsletter \textemdash{} correctly configured as High.
\par\smallskip
- D. Mobile Wallpaper \textemdash{} correctly configured as Standard.
\par\smallskip
\textbf{Finally, the correct option is B.
\textless{}/think\textgreater{}}
\par\smallskip
B
\end{tcolorbox}

\Needspace{12\baselineskip}
\subsection{B02 / Case Study Analysis}
\label{case:sft17}
\begin{tcolorbox}[casecard,breakable,title={Question, evidence, and training answer}]
\textbf{Question.} Based on the lecture's classification of the original Bronco body styles, which labelled region in the figure corresponds to the model that the first Bronco brochure referred to as a 'sports utility'?
\par\smallskip
\textbf{Options.} A. Region A B. Region C C. Region D D. Region B\par\smallskip\textbf{Training answer.} D: Region B
\par\smallskip\noindent \begin{minipage}[t]{0.48\linewidth}\centering
\includegraphics[width=\linewidth,height=2.5cm,keepaspectratio]{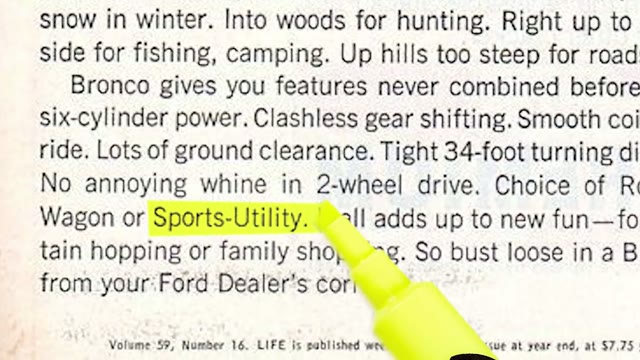}
\par\scriptsize Source video, 02:27\end{minipage}\hfill \begin{minipage}[t]{0.48\linewidth}\centering
\includegraphics[width=\linewidth,height=3.4cm,keepaspectratio]{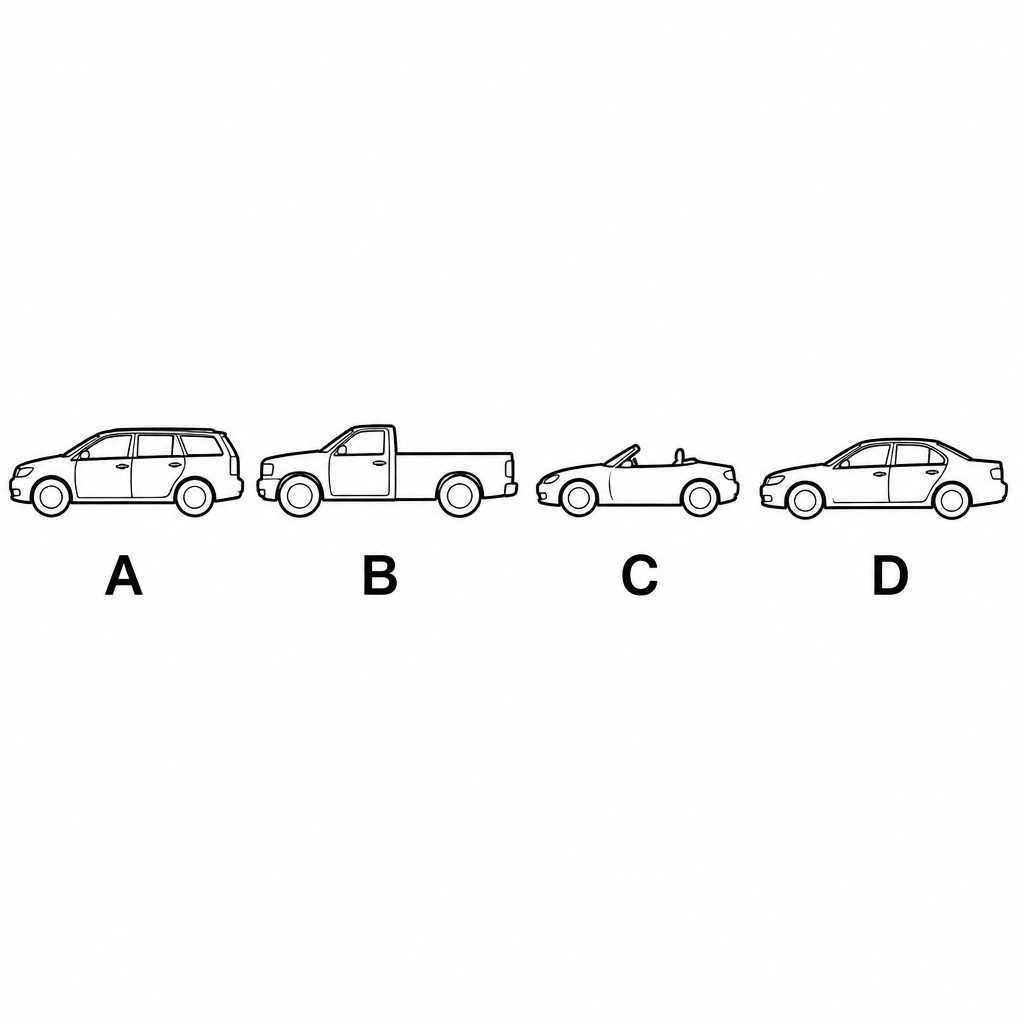}
\par\scriptsize New figure\end{minipage}\par\smallskip
\textbf{Upstream reasoning steps.}\par 1. [00:02:24-00:02:30] Recall the lecture's statement that the first Bronco brochure referred to the pickup as a 'sports utility'.\par 2. Identify the labelled body styles in the figure: Region A is a wagon, Region B is a pickup, Region C is a roadster, Region D is a sedan.\par 3. Apply the lecture's naming: the 'sports utility' corresponds to the pickup, which is Region B.\par 4. Conclude that the correct answer is Region B.
\par\smallskip\textit{Task label: source-assigned.}
\end{tcolorbox}
\begin{tcolorbox}[casecard,breakable,title={Complete synthetic thinking and final answer},title after break={Complete synthetic thinking and final answer (continued)}]
\textless{}think\textgreater{}
Got it, let's tackle the problem.
\par\smallskip
The user wants to identify which labeled vehicle silhouette in the provided diagram represents the body style that the original Ford Bronco brochure designated as a 'sports utility'. The crux is to determine which of the three original body styles mentioned in the video corresponds to this specific term and match it to its silhouette.
\par\smallskip
First, let me scan the video for the segments relevant to the question:
\par\smallskip
- I need to find the segment of the lecture that discusses the original body styles of the first-generation Ford Bronco and listen for the term "sports utility".
\par\smallskip
\textbf{- At 02:21-02:30, the speaker explains the three original models of the Bronco, stating, "Three models were originally offered: a wagon, a pickup, which they did call a sports utility in the first brochure, and an unpopular roadster."}
\par\smallskip
\textbf{- Let me double-check the exact phrasing at 02:24: the speaker says "a pickup, which they did call a sports utility in the first brochure", confirming that the pickup truck body style is the one designated as a "sports utility".}
\par\smallskip
- If the brochure had referred to the wagon or the roadster as the "sports utility", we would expect the speaker to associate the term with those shapes, but he explicitly links it to the pickup \textemdash{} so we must look for the pickup silhouette.
\par\smallskip
- The pickup truck body style is represented by the silhouette in Region B, which corresponds to option D.
\par\smallskip
Next, let me read the figure:
\par\smallskip
- Silhouette A represents a station wagon or SUV body style.
\par\smallskip
- Silhouette B represents a standard pickup truck body style with an open cargo bed.
\par\smallskip
- Silhouette C represents a two-door convertible or roadster body style.
\par\smallskip
- Silhouette D represents a standard four-door sedan body style.
\par\smallskip
- Matching the video's rule that the "pickup" was called the "sports utility", we find that silhouette B, which is the pickup truck, is the correct match.
\par\smallskip
- This means Region B is the correct region, which corresponds to option D.
\par\smallskip
Next, let me analyze the options:
\par\smallskip
- A. Region A \textemdash{} this represents the wagon model.
\par\smallskip
- B. Region C \textemdash{} this represents the roadster model, which was unpopular.
\par\smallskip
- C. Region D \textemdash{} this represents a sedan.
\par\smallskip
- D. Region B \textemdash{} this represents the pickup model, which the first brochure called a "sports utility".
\par\smallskip
\textbf{Finally, the correct option is D.
\textless{}/think\textgreater{}}
\par\smallskip
D
\end{tcolorbox}

\Needspace{12\baselineskip}
\subsection{B03 / Cross-Scenario Transfer}
\label{case:sft18}
\begin{tcolorbox}[casecard,breakable,title={Question, evidence, and training answer}]
\textbf{Question.} A user needs a drone that can achieve at least 26 minutes of actual flight time in real-world conditions. Based on the video's realistic flight time estimates, which drone should the user choose?
\par\smallskip
\textbf{Options.} A. Only Drone A B. Only Drone B C. Both drones D. Neither drone\par\smallskip\textbf{Training answer.} B: Only Drone B
\par\smallskip\noindent \begin{minipage}[t]{0.48\linewidth}\centering
\includegraphics[width=\linewidth,height=2.5cm,keepaspectratio]{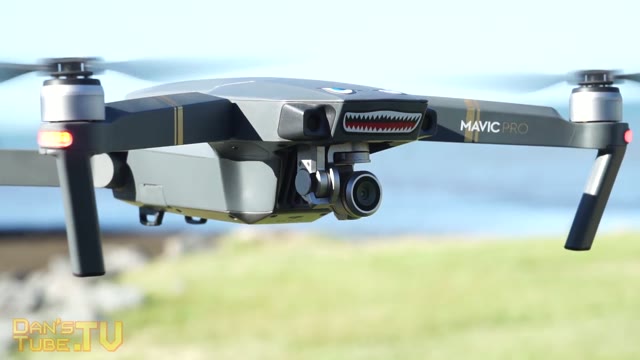}
\par\scriptsize Source video, 04:15\end{minipage}\hfill \begin{minipage}[t]{0.48\linewidth}\centering
\includegraphics[width=\linewidth,height=3.4cm,keepaspectratio]{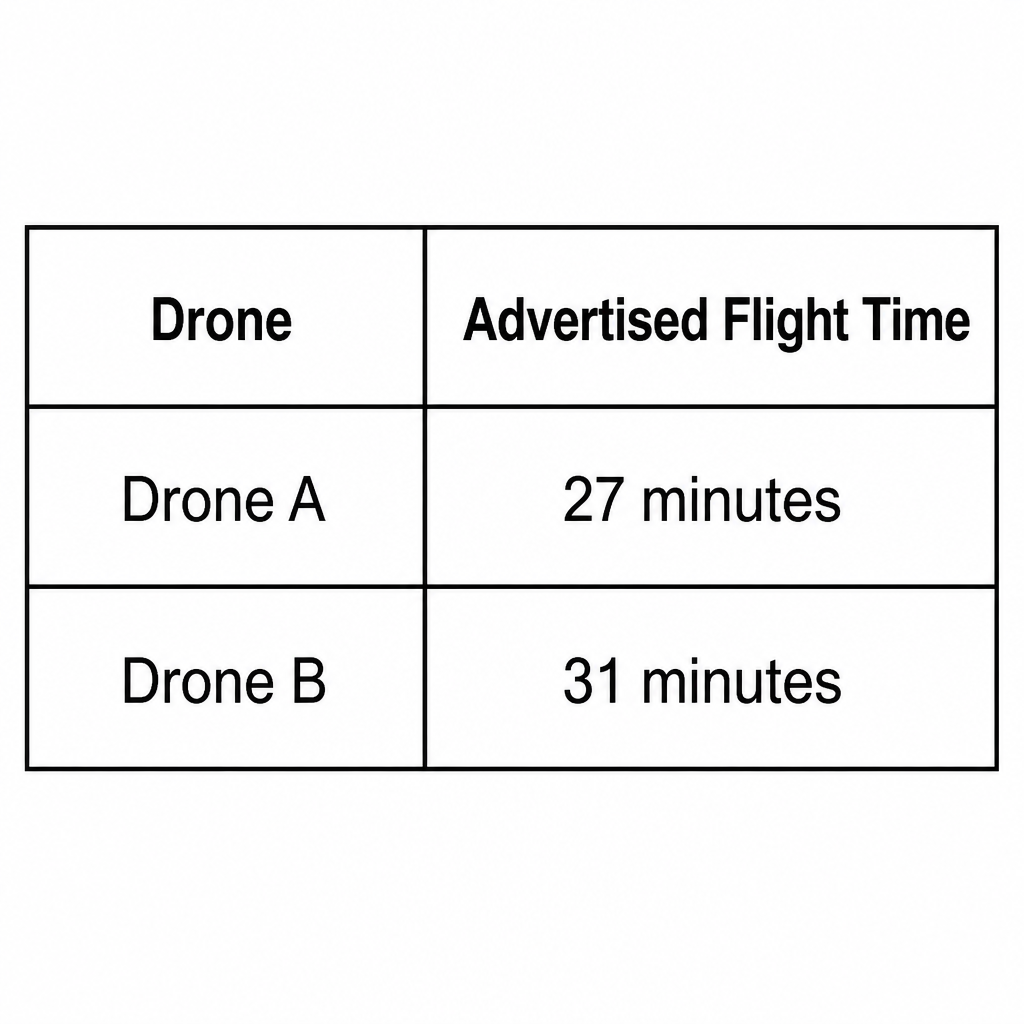}
\par\scriptsize New figure\end{minipage}\par\smallskip
\textbf{Upstream reasoning steps.}\par 1. [04:04-04:26] Recall the video's realistic flight time estimates: Mavic Pro \textasciitilde{}24-25 min, Mavic 2 Pro \textasciitilde{}28 min.\par 2. Match the drones in the figure to the models: Drone A (27 min advertised) corresponds to the Mavic Pro; Drone B (31 min advertised) corresponds to the Mavic 2 Pro.\par 3. Apply the realistic estimates: Drone A's realistic flight time is \textasciitilde{}24-25 min, which is below the 26 min requirement. Drone B's realistic flight time is \textasciitilde{}28 min, which meets the 26 min requirement.\par 4. Conclude that only Drone B meets the requirement.
\par\smallskip\textit{Task label: source-assigned.}
\end{tcolorbox}
\begin{tcolorbox}[casecard,breakable,title={Complete synthetic thinking and final answer},title after break={Complete synthetic thinking and final answer (continued)}]
\textless{}think\textgreater{}
Got it, let's tackle the problem.
\par\smallskip
The user wants to identify which drone can achieve at least 26 minutes of actual flight time in real-world conditions based on the realistic estimates given in the video. The crux is comparing the realistic flight times of Drone A and Drone B to the 26-minute threshold.
\par\smallskip
First, let me scan the video for the segments relevant to the question:
\par\smallskip
- I need to listen for the realistic flight time estimates of the drones and watch the flight time specifications on screen.
- At 04:03, the video shows the Mavic Pro flying as the speaker says, "The original Mavic Pro has a flight time of 27 minutes, which you will not get as much as 27 minutes, that's just unrealistic. You probably get about 24, 25 minutes maybe."
- At 04:17, the video shows the Mavic 2 Pro flying as the speaker says, "And the Mavic 2 Pro says you will get about 31 minutes. So you may get 28, 27 minutes with that."
- If Drone A were to meet the requirement, its realistic flight time would need to be at least 26 minutes, but the video estimates it at only 24 to 25 minutes, meaning it falls short.
\par\smallskip
Next, let me read the figure:
\par\smallskip
- The table shows "Drone A" has an "Advertised Flight Time" of "27 minutes".
- The table shows "Drone B" has an "Advertised Flight Time" of "31 minutes".
- Based on the video, Drone A corresponds to the Mavic Pro with an advertised flight time of 27 minutes, giving a realistic flight time of 24 to 25 minutes.
- Based on the video, Drone B corresponds to the Mavic 2 Pro with an advertised flight time of 31 minutes, giving a realistic flight time of 27 to 28 minutes.
- Comparing the realistic flight times to the user's requirement of at least 26 minutes: Drone A (24-25 minutes) \textless{} 26 minutes, so it does not qualify.
- Comparing Drone B's realistic flight time: Drone B (27-28 minutes) \textgreater{}= 26 minutes, so it qualifies.
- Only Drone B meets the requirement, which corresponds to option B.
\par\smallskip
Next, let me analyze the options:
\par\smallskip
- A. Only Drone A \textemdash{} Drone A's realistic flight time is under 26 minutes.
- B. Only Drone B \textemdash{} Drone B's realistic flight time is 27 to 28 minutes, which meets the requirement.
- C. Both drones \textemdash{} Drone A fails the requirement.
- D. Neither drone \textemdash{} Drone B successfully meets the requirement.
\par\smallskip
\textbf{Finally, the correct option is B.
\textless{}/think\textgreater{}}
\par\smallskip
B
\end{tcolorbox}

\Needspace{12\baselineskip}
\subsection{B04 / Multi-step Quantitative Reasoning}
\label{case:sft19}
\begin{tcolorbox}[casecard,breakable,title={Question, evidence, and training answer}]
\textbf{Question.} A retailer has the items shown in the figure. He also has 5 individual pillow covers from another source. He wants to create gift packs each containing 2 cushion covers and 1 pillow cover. How many complete gift packs can he make?
\par\smallskip
\textbf{Options.} A. 4 B. 7 C. 6 D. 10\par\smallskip\textbf{Training answer.} B: 7
\par\smallskip\noindent \begin{minipage}[t]{0.48\linewidth}\centering
\includegraphics[width=\linewidth,height=2.5cm,keepaspectratio]{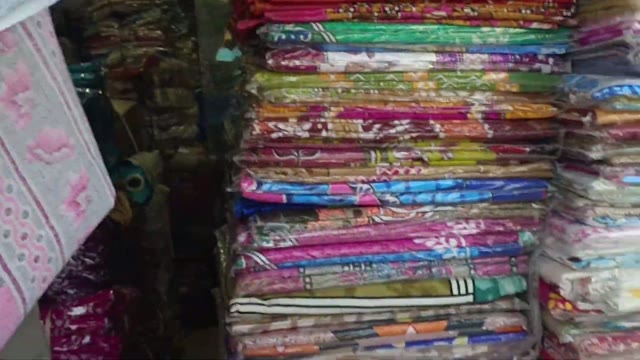}
\par\scriptsize Source video, 05:23\end{minipage}\hfill \begin{minipage}[t]{0.48\linewidth}\centering
\includegraphics[width=\linewidth,height=3.4cm,keepaspectratio]{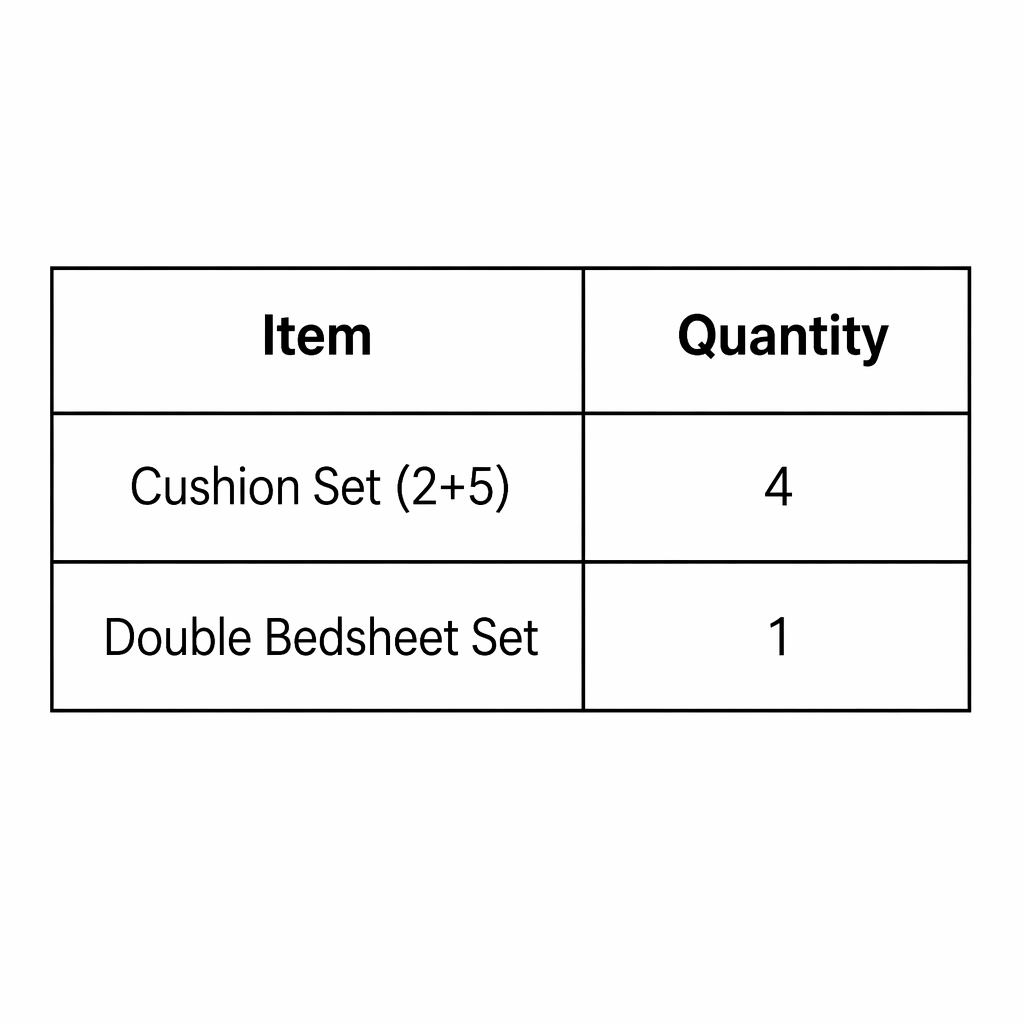}
\par\scriptsize New figure\end{minipage}\par\smallskip
\textbf{Upstream reasoning steps.}\par 1. [08:01-08:08] Recall from the video that a Cushion Set (2+5) contains 5 cushion covers.\par 2. Compute total cushion covers from the figure: 4 sets  x  5 = 20 cushion covers.\par 3. [05:20-05:27] Recall from the video that a Double Bedsheet Set contains 2 pillow covers.\par 4. Compute total pillow covers from the figure: 1 set  x  2 = 2 pillow covers.\par 5. Add the 5 individual pillow covers: total pillow covers = 2 + 5 = 7.\par 6. Each gift pack requires 2 cushion covers and 1 pillow cover. The number of packs is limited by the smaller of floor(20/2)=10 and floor(7/1)=7, so 7 packs.
\par\smallskip\textit{Task label: source-assigned.}
\end{tcolorbox}
\begin{tcolorbox}[casecard,breakable,title={Complete synthetic thinking and final answer},title after break={Complete synthetic thinking and final answer (continued)}]
\textless{}think\textgreater{}
Got it, let's tackle the problem.
\par\smallskip
The user wants to find the maximum number of complete gift packs that can be made using the items shown in the table and some additional pillow covers. The crux is determining the exact number of cushion covers and pillow covers contained in each set listed in the table by watching the video.
\par\smallskip
First, let me scan the video for the segments relevant to the question:
\par\smallskip
- I need to find the segments in the video where the shopkeeper explains the contents of a double bedsheet set and a cushion set (2+5).
\par\smallskip
\textbf{- At 05:25-05:28, the shopkeeper explains that a double bedsheet set comes with two pillow covers and one bedsheet, saying "Double 300 ka set aayega, do pillow cover ek bedsheet."}
\par\smallskip
- If a double bedsheet set only contained one pillow cover instead of two, the total number of pillow covers would be 1 + 5 = 6, which would limit the gift packs to 6, but the shopkeeper states it contains "do pillow cover".
\par\smallskip
\textbf{- At 07:53-08:08, the shopkeeper describes the cushion set (2+5), stating "5 cushion aate hain, 2 load aate hain" and specifying that it contains "5 cushion ke cover aate hain" and "2 load cover aate hain."}
\par\smallskip
Next, let me read the figure:
\par\smallskip
- The figure shows a table with two items: "Cushion Set (2+5)" with a quantity of 4, and "Double Bedsheet Set" with a quantity of 1.
\par\smallskip
- Wait, let me re-read the question: it states the retailer also has 5 individual pillow covers from another source, which must be added to the pillow covers obtained from the double bedsheet set.
\par\smallskip
- Total cushion covers = 4 * 5 = 20 cushion covers.
\par\smallskip
- Pillow covers from the bedsheet set = 1 * 2 = 2 pillow covers.
\par\smallskip
- Total pillow covers = 2 + 5 = 7 pillow covers.
\par\smallskip
- Packs possible from cushion covers = 20 / 2 = 10 packs.
\par\smallskip
- Packs possible from pillow covers = 7 / 1 = 7 packs.
\par\smallskip
- Complete packs = min(10, 7) = 7 packs, which corresponds to option B.
\par\smallskip
Next, let me analyze the options:
\par\smallskip
- A. 4 \textemdash{} we can make more than this number.
\par\smallskip
- B. 7 \textemdash{} the 7 available pillow covers limit the total complete gift packs to this amount, as we have 20 cushion covers.
\par\smallskip
- C. 6 \textemdash{} this would leave one pillow cover unused.
\par\smallskip
- D. 10 \textemdash{} we would need 10 pillow covers, but only have 7.
\par\smallskip
\textbf{Finally, the correct option is B.
\textless{}/think\textgreater{}}
\par\smallskip
B
\end{tcolorbox}

\Needspace{12\baselineskip}
\subsection{B05 / Comparative Reasoning}
\label{case:sft20}
\begin{tcolorbox}[casecard,breakable,title={Question, evidence, and training answer}]
\textbf{Question.} For the four project plans shown in the figure, which one would have taken less total time than Judi Hu's original project (which took about 9 hours)? Use the flower-making and patch-sewing rates that Judi experienced during her project.
\par\smallskip
\textbf{Options.} A. Plan C B. Plan A C. Plan B D. Plan D\par\smallskip\textbf{Training answer.} B: Plan A
\par\smallskip\noindent \begin{minipage}[t]{0.48\linewidth}\centering
\includegraphics[width=\linewidth,height=2.5cm,keepaspectratio]{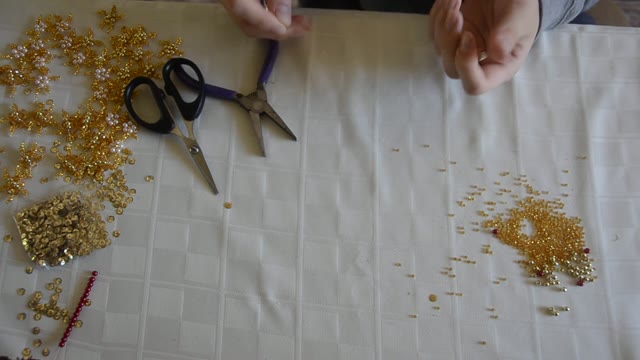}
\par\scriptsize Source video, 04:48\end{minipage}\hfill \begin{minipage}[t]{0.48\linewidth}\centering
\includegraphics[width=\linewidth,height=3.4cm,keepaspectratio]{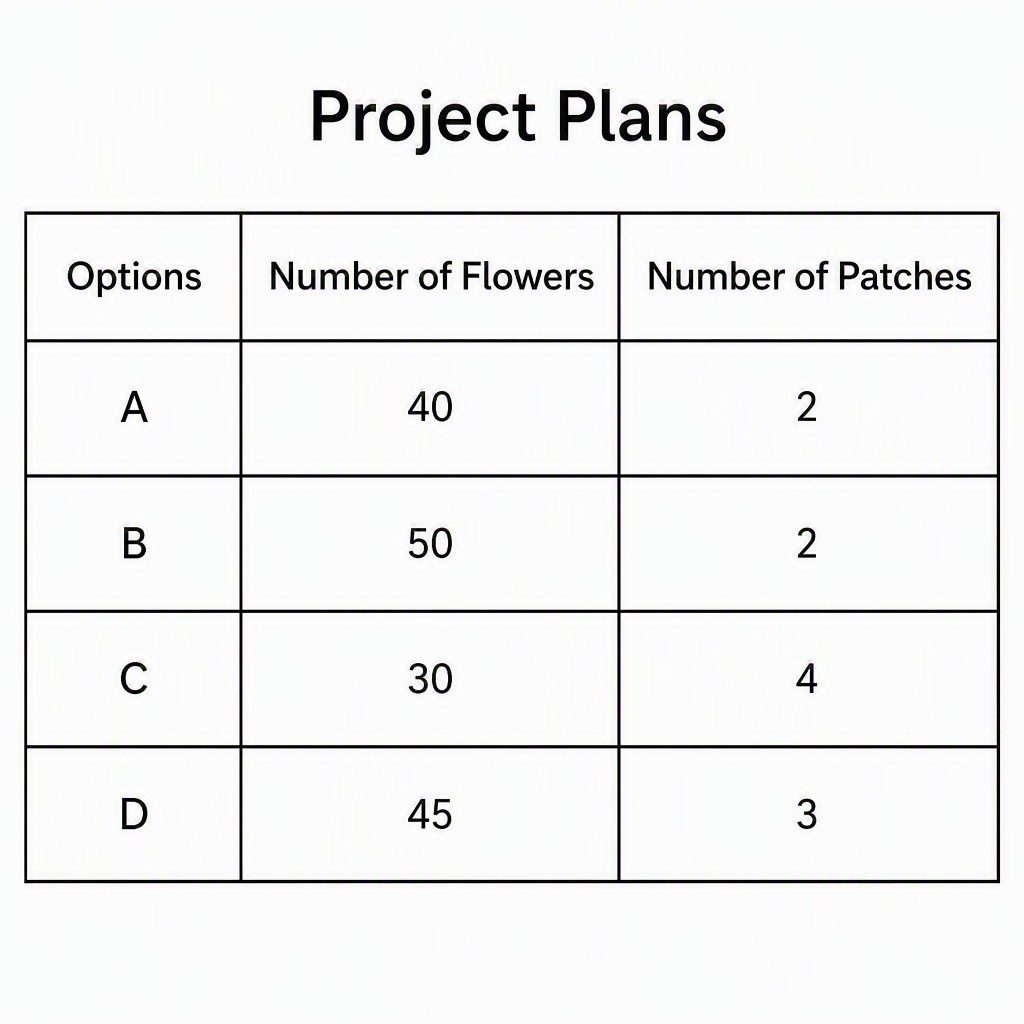}
\par\scriptsize New figure\end{minipage}\par\smallskip
\textbf{Upstream reasoning steps.}\par 1. [04:45-04:52] Recall that Judi stated each flower takes about 8 minutes to make.\par 2. [06:38-06:45] Recall that the on-screen text indicated each embroidered patch takes 1.5 hours (90 minutes) to sew.\par 3. Compute total time for Plan A: 40 flowers  x  8 min = 320 min = 5.33 h; 2 patches  x  90 min = 180 min = 3 h; sum = 500 min = 8.33 h (\textless{} 9 h).\par 4. Compute total time for Plan B: 50  x  8 = 400 min = 6.67 h; 2  x  90 = 180 min = 3 h; sum = 580 min = 9.67 h (\textgreater{} 9 h).\par 5. Compute total time for Plan C: 30  x  8 = 240 min = 4 h; 4  x  90 = 360 min = 6 h; sum = 600 min = 10 h (\textgreater{} 9 h).\par 6. Compute total time for Plan D: 45  x  8 = 360 min = 6 h; 3  x  90 = 270 min = 4.5 h; sum = 630 min = 10.5 h (\textgreater{} 9 h).\par 7. Only Plan A's total (8.33 h) is less than the original 9 h, so Plan A is correct.
\par\smallskip\textit{Task label: source-assigned.}
\end{tcolorbox}
\begin{tcolorbox}[casecard,breakable,title={Complete synthetic thinking and final answer},title after break={Complete synthetic thinking and final answer (continued)}]
\textless{}think\textgreater{}
Got it, let's tackle the problem.
\par\smallskip
The user wants to know which of the four project plans shown in the figure would have taken less total time than Judi Hu's original project. The crux is calculating the total time for each plan using the flower-making and patch-sewing rates established in the video.
\par\smallskip
First, let me scan the video for the segments relevant to the question:
\par\smallskip
- I need to find the segments in the video where Judi Hu shares the time it takes to make a single flower and to sew a single embroidered patch.
\par\smallskip
\textbf{- At 04:49-04:54, Judi says "I took about 8 minutes for one flower until I started getting better at it", establishing the base rate for flower-making.}
\par\smallskip
\textbf{- Wait, let me verify the two rates against their sources: the 8-minute flower rate comes from Judi's spoken statement at 04:49-04:54, while the on-screen text at 06:34-06:36 displays "1 AND 1/2 HOURS FOR 1 PATCH", which means each patch takes 1.5 hours, or 90 minutes.}
\par\smallskip
- If a patch took only 1 hour (60 minutes) instead of 1.5 hours (90 minutes) at 06:36, the original project's total time would be calculated as 45 * 8 = 360 minutes for flowers plus 2 * 60 = 120 minutes for patches, giving 480 minutes (8 hours) instead of the 9 hours mentioned in the prompt \textemdash{} so a patch must indeed take 90 minutes.
\par\smallskip
- At 06:30-06:33, the on-screen text shows the calculation for the original project's flowers: "15 X 3 = 45 FLOWERS" and "45 X 8 = 360 MINUTES", which equals 6 hours.
\par\smallskip
- At 06:34-06:36, the on-screen text shows "2 EMBROIDERED PATCHES" and "1 AND 1/2 HOURS FOR 1 PATCH", which means the patch-sewing time is 2 * 90 = 180 minutes, or 3 hours.
\par\smallskip
- Combining these rates, the original project's total time is 360 + 180 = 540 minutes, which is exactly 9 hours.
\par\smallskip
Next, let me read the figure:
\par\smallskip
- The figure shows a table titled "Project Plans" with columns "Options", "Number of Flowers", and "Number of Patches".
\par\smallskip
- Row "A" lists "40" flowers and "2" patches.
\par\smallskip
- Row "B" lists "50" flowers and "2" patches.
\par\smallskip
- Row "C" lists "30" flowers and "4" patches.
\par\smallskip
- Row "D" lists "45" flowers and "3" patches.
\par\smallskip
- For Plan A, the total time is (40 * 8) + (2 * 90) = 320 + 180 = 500 minutes.
\par\smallskip
- For Plan B, the total time is (50 * 8) + (2 * 90) = 400 + 180 = 580 minutes; for Plan C, it is (30 * 8) + (4 * 90) = 240 + 360 = 600 minutes; and for Plan D, it is (45 * 8) + (3 * 90) = 360 + 270 = 630 minutes.
\par\smallskip
- Comparing these to the original project's 540 minutes, only Plan A takes less total time (500 minutes), which corresponds to Option B.
\par\smallskip
Next, let me analyze the options:
\par\smallskip
- A. Plan C \textemdash{} this takes 600 minutes, exceeding the original.
\par\smallskip
- B. Plan A \textemdash{} this plan takes 500 minutes, which is less than the original project's 540 minutes.
\par\smallskip
- C. Plan B \textemdash{} this plan takes 580 minutes, which is more than the original project's 540 minutes.
\par\smallskip
- D. Plan D \textemdash{} this takes 630 minutes, exceeding the original.
\par\smallskip
\textbf{Finally, the correct option is B.
\textless{}/think\textgreater{}}
\par\smallskip
B
\end{tcolorbox}

\Needspace{12\baselineskip}
\subsection{B06 / Design / Optimization}
\label{case:sft21}
\begin{tcolorbox}[casecard,breakable,title={Question, evidence, and training answer}]
\textbf{Question.} For the three decorative picks shown in the figure, which one should be attached to the existing swag using the method demonstrated in the video (bending the stem back and tying it with existing branches)?
\par\smallskip
\textbf{Options.} A. Pick C: Pre-attached clip B. Pick B: Thin stem (easy to cut) C. Pick A: Thick stem (hard to cut)\par\smallskip\textbf{Training answer.} C: Pick A: Thick stem (hard to cut)
\par\smallskip\noindent \begin{minipage}[t]{0.48\linewidth}\centering
\includegraphics[width=\linewidth,height=2.5cm,keepaspectratio]{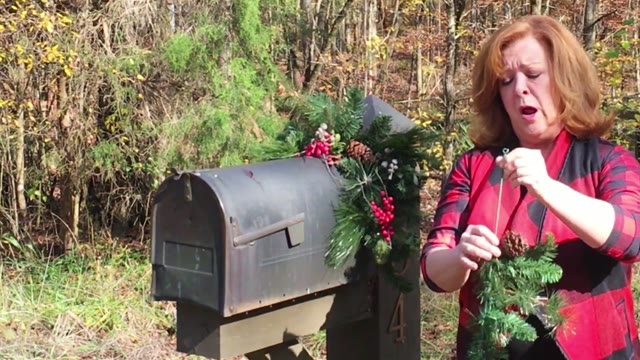}
\par\scriptsize Source video, 00:53\end{minipage}\hfill \begin{minipage}[t]{0.48\linewidth}\centering
\includegraphics[width=\linewidth,height=3.4cm,keepaspectratio]{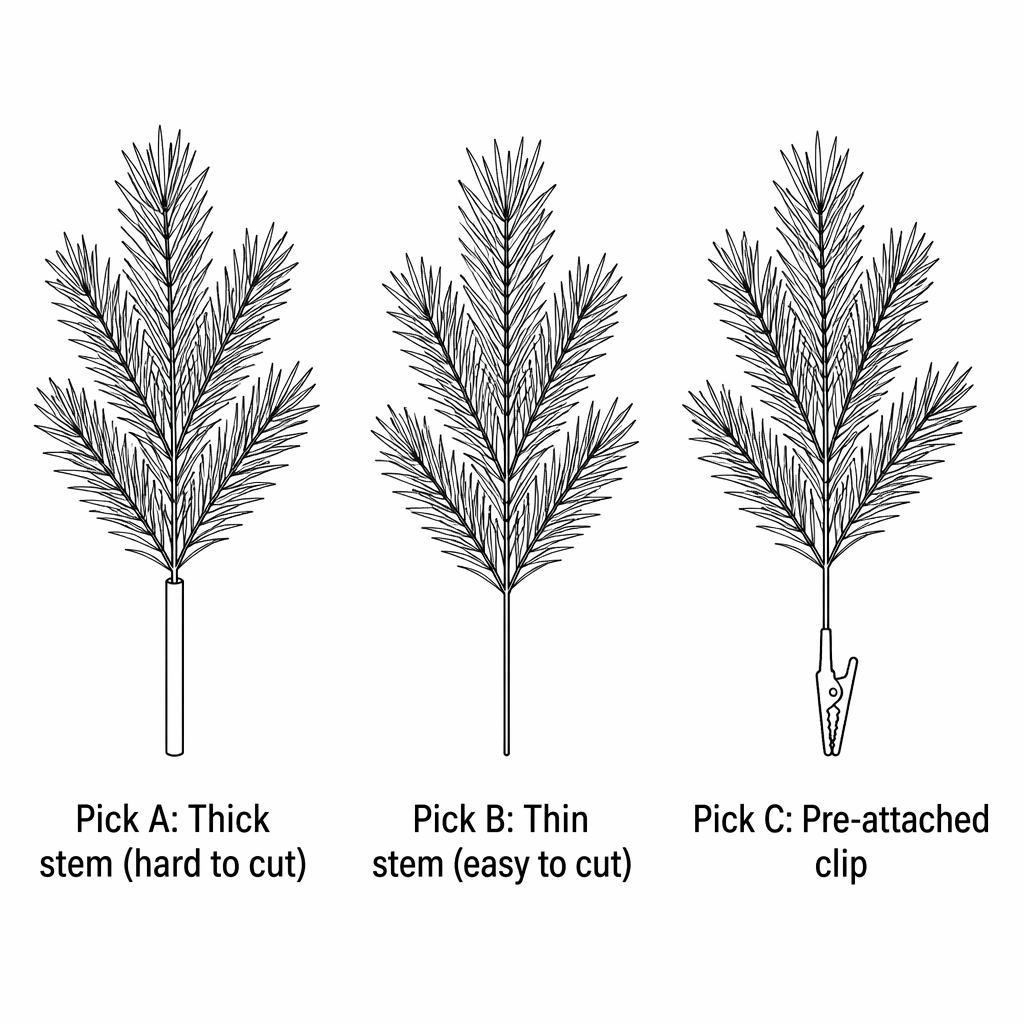}
\par\scriptsize New figure\end{minipage}\par\smallskip
\textbf{Upstream reasoning steps.}\par 1. [00:48-00:59] Recall the video's decision rule: When the stem is too thick to cut easily, use the bending method (bend the stem back, hide it, and tie with existing branches). For thinner stems, cutting is preferred.\par 2. From the figure, identify which pick has a thick stem that is hard to cut. The figure labels Pick A as 'Thick stem (hard to cut)'.\par 3. Apply the rule: Pick A is the one that should be attached using the bending-and-tying method.
\par\smallskip\textit{Task label: source-assigned.}
\end{tcolorbox}
\begin{tcolorbox}[casecard,breakable,title={Complete synthetic thinking and final answer},title after break={Complete synthetic thinking and final answer (continued)}]
\textless{}think\textgreater{}
Got it, let's tackle the problem.
\par\smallskip
The user wants to identify which of the three decorative picks shown in the figure should be attached to the swag using the bending and tying method shown in the video. The crux is matching the physical characteristics of the pick the woman describes and uses in the video to the descriptions of the picks in the diagram.
\par\smallskip
First, let me scan the video for the segments relevant to the question:
\par\smallskip
- I need to listen for the woman's description of the pick's stem and watch how she manipulates and attaches it to the swag.
- At 0:47-0:55, the woman holds up a pinecone pick and explains that while cutting the stem would be ideal, "this is so thick and hard to cut, you almost need bolt cutters to cut it."
- If the pick had a thin stem that was easy to cut, she would have cut it as she noted that would be "ideal" at 0:50, but instead she has to bend it because of its thickness.
- Wait, let me double check her exact words at 1:00-1:08: she says "So I'm just gonna bend it and attempt to hide it within my swag," and she proceeds to bend the thick green wire stem back on itself.
- At 1:09-1:25, she tucks the bent stem into the swag and wraps the existing pine branches around it to secure it in place.
- The video establishes that the pick used has a stem that is too thick and hard to cut, requiring it to be bent back and tied with existing branches.
\par\smallskip
Next, let me read the figure:
\par\smallskip
- The figure shows three decorative picks: "Pick A: Thick stem (hard to cut)", "Pick B: Thin stem (easy to cut)", and "Pick C: Pre-attached clip".
- Pick A features a thick, solid green stem labeled "Thick stem (hard to cut)".
- Pick B features a very thin wire stem labeled "Thin stem (easy to cut)".
- Pick C features a clip at the base of the stem labeled "Pre-attached clip".
- Applying the video's rule, the pick with the "so thick and hard to cut" stem that must be bent corresponds directly to "Pick A: Thick stem (hard to cut)", which is listed as option C.
\par\smallskip
Next, let me analyze the options:
\par\smallskip
- A. Pick C: Pre-attached clip \textemdash{} This pick has a clip.
- B. Pick B: Thin stem (easy to cut) \textemdash{} This pick has a thin stem, which she says is ideal to cut.
- C. Pick A: Thick stem (hard to cut) \textemdash{} This matches the thick stem that she bends and ties.
\par\smallskip
\textbf{Finally, the correct option is C.
\textless{}/think\textgreater{}}
\par\smallskip
C
\end{tcolorbox}

\Needspace{12\baselineskip}
\subsection{B07 / Error Critique}
\label{case:sft22}
\begin{tcolorbox}[casecard,breakable,title={Question, evidence, and training answer}]
\textbf{Question.} A reviewer on a tech forum claims: 'The Garmin Nuvi 2555 LMT has a glossy screen that is hard to read in bright sunlight. To fix this, you should apply a matte screen protector.' What is the specific flaw in this recommendation?
\par\smallskip
\textbf{Options.} A. A matte screen protector will interfere with the touchscreen's sensitivity. B. The screen is too small for a screen protector to be effective at reducing glare. C. The screen is already anti-glare, so a matte protector is unnecessary and may reduce clarity. D. The reviewer's recommendation is correct; there is no flaw.\par\smallskip\textbf{Training answer.} C: The screen is already anti-glare, so a matte protector is unnecessary and may reduce clarity.
\par\smallskip\noindent \begin{minipage}[t]{0.48\linewidth}\centering
\includegraphics[width=\linewidth,height=2.5cm,keepaspectratio]{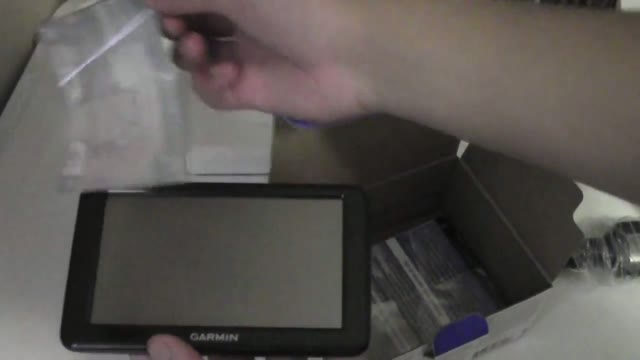}
\par\scriptsize Source video, 02:12\end{minipage}\par\smallskip
\textbf{Upstream reasoning steps.}\par 1. [02:08-02:16] Recall from the video that the Garmin Nuvi 2555 LMT screen is described as 'really meant for anti-glare' and compared to a Kindle screen, meaning it is matte and anti-glare by design.\par 2. The reviewer's claim that the screen is glossy is false based on the video-derived knowledge. Therefore, the recommendation to apply a matte screen protector to reduce glare is based on a false premise.\par 3. The specific flaw is that the screen is already anti-glare, so the proposed solution is unnecessary and could degrade the screen's performance.
\par\smallskip\textit{Task label: source-assigned.}
\end{tcolorbox}
\begin{tcolorbox}[casecard,breakable,title={Complete synthetic thinking and final answer},title after break={Complete synthetic thinking and final answer (continued)}]
\textless{}think\textgreater{}
Got it, let's tackle the problem.
\par\smallskip
The user wants to identify the specific flaw in a reviewer's recommendation to apply a matte screen protector to the Garmin Nuvi 2555 LMT. The crux is determining whether the device's screen is already designed to prevent glare.
\par\smallskip
First, let me scan the video for the segments relevant to the question:
\par\smallskip
- I need to listen for the narrator's description of the device's screen surface and watch the screen as the protective film is removed.
\par\smallskip
\textbf{- At 02:05, the narrator holds up the Garmin Nuvi 2555 LMT with the printed screen sticker still attached.}
\par\smallskip
\textbf{- At 02:09, the narrator peels off the printed screen sticker, revealing the actual screen underneath.}
\par\smallskip
- If the reviewer's recommendation is correct and there is no flaw, the screen would have to be glossy and highly reflective once the sticker is removed, but the surface is matte and non-reflective.
\par\smallskip
- At 02:16-02:21, as the narrator rubs his finger across the display, he says, "you're gonna see a really nice screen, it almost looks like a Kindle screen 'cause it's really meant for anti-glare, and it's really, really nice."
\par\smallskip
- Since the device already features a built-in anti-glare screen similar to a Kindle, applying a matte screen protector is unnecessary, which points directly to option C: The screen is already anti-glare, so a matte protector is unnecessary and may reduce clarity.
\par\smallskip
Next, let me analyze the options:
\par\smallskip
- A. A matte screen protector will interfere with the touchscreen's sensitivity. \textemdash{} The video does not mention touch sensitivity issues.
\par\smallskip
- B. The screen is too small for a screen protector to be effective at reducing glare. \textemdash{} Screen size does not dictate glare reduction.
\par\smallskip
- C. The screen is already anti-glare, so a matte protector is unnecessary and may reduce clarity. \textemdash{} The narrator explicitly states the screen is anti-glare, making an additional matte protector redundant.
\par\smallskip
- D. The reviewer's recommendation is correct; there is no flaw. \textemdash{} The screen is already anti-glare, so this claim is false.
\par\smallskip
\textbf{Finally, the correct option is C.
\textless{}/think\textgreater{}}
\par\smallskip
C
\end{tcolorbox}

\Needspace{12\baselineskip}
\subsection{B08 / Procedure / Planning}
\label{case:sft23}
\begin{tcolorbox}[casecard,breakable,title={Question, evidence, and training answer}]
\textbf{Question.} You have a Tashtego 2.0 toiletry roll that is stained after a camping trip. You have removed all toiletries. What is the correct next step according to the video's demonstrated cleaning procedure?
\par\smallskip
\textbf{Options.} A. Place the bag in a washing machine and run a normal cycle. B. Spot clean the stained area with a damp cloth. C. Hand wash the bag with mild soap and let it air dry. D. Apply a new coat of wax to the bag before washing.\par\smallskip\textbf{Training answer.} A: Place the bag in a washing machine and run a normal cycle.
\par\smallskip\noindent \begin{minipage}[t]{0.48\linewidth}\centering
\includegraphics[width=\linewidth,height=2.5cm,keepaspectratio]{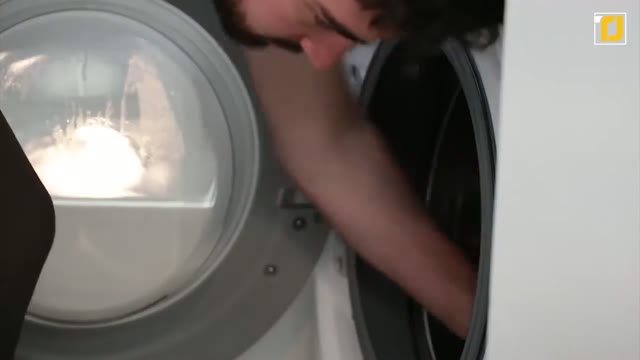}
\par\scriptsize Source video, 07:05\end{minipage}\par\smallskip
\textbf{Upstream reasoning steps.}\par 1. [07:00-07:11] Recall that the video visually demonstrates the Tashtego bag being tossed into a washing machine for cleaning.\par 2. Since the bag is stained and all toiletries have been removed, the correct next step is to machine wash it as shown in the video.
\par\smallskip\textit{Task label: source-assigned.}
\end{tcolorbox}
\begin{tcolorbox}[casecard,breakable,title={Complete synthetic thinking and final answer},title after break={Complete synthetic thinking and final answer (continued)}]
\textless{}think\textgreater{}
Got it, let's tackle the problem.
\par\smallskip
The user wants to know the correct next step to clean a stained Tashtego 2.0 toiletry roll based on the video. The crux is identifying the cleaning method demonstrated in the video for this specific product.
\par\smallskip
First, let me scan the video for the segments relevant to the question:
\par\smallskip
- I need to find the segment showing the "Tashtego 2.0" and watch for any cleaning demonstrations.
\par\smallskip
\textbf{- At 07:00, the video shows the toiletry roll being rinsed under running water in a sink while the narrator explains that "Last but not least, Tashtego is made of number 10 waxed cotton and coated with a water repellent wax so you can get it wet without soaking your toiletries".}
\par\smallskip
\textbf{- At 07:03, the video shows the toiletry roll being placed directly into a top-loading washing machine.}
\par\smallskip
- Wait, is the appliance shown at 07:03 a sink or a washing machine? Looking closely, it is a top-loading washing machine with an agitator and water inside, into which the bag is dropped.
\par\smallskip
- Let me re-read the narration to see if machine washing is mentioned: the narrator says "you can get it wet without soaking your toiletries", but the visual of dropping it into the washer is what demonstrates the actual cleaning step.
\par\smallskip
- If the demonstrated cleaning procedure were hand washing the bag with mild soap, the video would have to show the bag being scrubbed by hand and left to air dry, but instead it is placed directly into a washing machine at 07:03.
\par\smallskip
- The visual demonstration of dropping the bag into the washing machine at 07:03 confirms that placing the bag in a washing machine is the correct next step, which corresponds to option A: Place the bag in a washing machine and run a normal cycle.
\par\smallskip
Next, let me analyze the options:
\par\smallskip
- A. Place the bag in a washing machine and run a normal cycle. \textemdash{} This matches the visual demonstration of the bag being placed in a washing machine.
\par\smallskip
- B. Spot clean the stained area with a damp cloth. \textemdash{} This method is not shown in the video.
\par\smallskip
- C. Hand wash the bag with mild soap and let it air dry. \textemdash{} Although water is run over it, the bag is ultimately placed in a washing machine.
\par\smallskip
- D. Apply a new coat of wax to the bag before washing. \textemdash{} No waxing process is demonstrated prior to washing.
\par\smallskip
\textbf{Finally, the correct option is A.
\textless{}/think\textgreater{}}
\par\smallskip
A
\end{tcolorbox}

\Needspace{12\baselineskip}
\subsection{B09 / Cross-Domain Transfer}
\label{case:sft24}
\begin{tcolorbox}[casecard,breakable,title={Question, evidence, and training answer}]
\textbf{Question.} A firefighter is learning to spin a long weighted hose nozzle in a forward twirl, exactly as the bo staff forward twirl taught in the video. Which hand orientation should he use to initiate the spin?
\par\smallskip
\textbf{Options.} A. Lead with the pinky, palm down B. Lead with the thumb, palm up (thumb dives down) C. Lead with the thumb, palm down D. Lead with the pinky, palm up\par\smallskip\textbf{Training answer.} B: Lead with the thumb, palm up (thumb dives down)
\par\smallskip\noindent \begin{minipage}[t]{0.48\linewidth}\centering
\includegraphics[width=\linewidth,height=2.5cm,keepaspectratio]{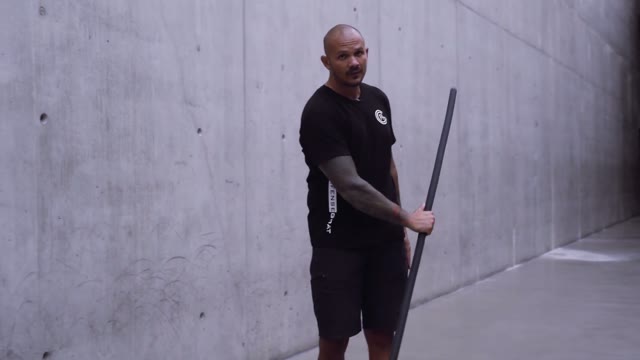}
\par\scriptsize Source video, 01:25\end{minipage}\par\smallskip
\textbf{Upstream reasoning steps.}\par 1. [01:15-01:35] Recall the video's instruction for the forward twirl: the instructor says 'I'm going to dive my thumb across my body and down' and then 'bring that palm up'. This establishes that the forward twirl is initiated by leading with the thumb (diving it down) and having the palm face up.\par 2. Apply this to the firefighter's scenario: he wants to perform a forward twirl with the hose nozzle, so he must use the same hand orientation -- lead with the thumb (dive it down) and keep the palm up.
\par\smallskip\textit{Task label: source-assigned.}
\end{tcolorbox}
\begin{tcolorbox}[casecard,breakable,title={Complete synthetic thinking and final answer},title after break={Complete synthetic thinking and final answer (continued)}]
\textless{}think\textgreater{}
Got it, let's tackle the problem.
\par\smallskip
The user wants to identify the correct hand orientation to initiate a forward twirl with a bo staff as demonstrated in the video. The crux is determining whether the spin begins with the thumb or pinky leading, and whether the palm should be facing up or down.
\par\smallskip
First, let me scan the video for the segments relevant to the question:
\par\smallskip
- I need to find the segment where the instructor demonstrates and explains the hand orientation and movement for the forward twirl.
\par\smallskip
\textbf{- At 01:13, the instructor introduces the "forward twirl" and prepares to demonstrate it.}
\par\smallskip
\textbf{- At 01:15, he places the center of the staff in his palm to begin the movement.}
\par\smallskip
- At 01:19, he grabs the staff and explains that he is going to "dive my thumb across my body and down."
\par\smallskip
- At 01:25, he continues the motion, explaining to "bring that palm up and come across the body and down" as he completes the rotation.
\par\smallskip
- At 01:53, he demonstrates the handoff to the left hand, stating "palm up, I grab it with palm up with my left hand" before diving to the other side.
\par\smallskip
- At 02:16, he repeats the motion on both sides, showing that the thumb leads the dive on each rotation.
\par\smallskip
- Let me re-listen at 02:27: the instructor says "palm goes up, palm goes up, dive" to emphasize the handoff and initiation sequence.
\par\smallskip
- If the forward twirl were initiated by leading with the pinky and keeping the palm down, the instructor would have to say to lead with the pinky when demonstrating the forward twirl at 01:19, but he explicitly says "dive my thumb across my body and down" and reserves the pinky-lead for the backward twirl.
\par\smallskip
- At 03:26, he introduces the "backwards" twirl, instructing to "think your pinky now" and "drive up in the air."
\par\smallskip
- At 03:35, he confirms the hand orientation for the backward twirl is "palm down now."
\par\smallskip
- The instructor's breakdown shows that the forward twirl is initiated by grabbing palm up and leading with the thumb diving down, which corresponds to option B.
\par\smallskip
Next, let me analyze the options:
\par\smallskip
- A. Lead with the pinky, palm down \textemdash{} this orientation is used for the backward twirl.
\par\smallskip
- B. Lead with the thumb, palm up (thumb dives down) \textemdash{} this matches the forward twirl initiation where the hand grabs palm up and the thumb dives down.
\par\smallskip
- C. Lead with the thumb, palm down \textemdash{} incorrect grab hand orientation.
\par\smallskip
- D. Lead with the pinky, palm up \textemdash{} incorrect on both counts.
\par\smallskip
\textbf{Finally, the correct option is B.
\textless{}/think\textgreater{}}
\par\smallskip
B
\end{tcolorbox}

\FloatBarrier
% ===== END inlined from sections/appendix_cases.tex =====

% ===== END inlined from sections/appendix.tex =====
\end{document}